\documentclass[%
paper=letter,%
pagesize,%
numbers=noendperiod,%
captions=nooneline,%
abstracton,%
DIV=14
]{scrartcl}

\usepackage[T1]{fontenc}%
\usepackage{lmodern}%
\usepackage[american]{babel}%
\usepackage{microtype}%

\usepackage[
hyperref,%
table%
]{xcolor}%

\usepackage[nouppercase]{scrpage2}%

\usepackage{eso-pic}%
\usepackage{rotating}%
\usepackage{amsmath}%
\usepackage{mathtools}%
\usepackage{amssymb}%
\usepackage{amsthm}%
\usepackage{etoolbox}%
\usepackage{bm}%
\usepackage{bbm}%
\usepackage{enumitem}%
\usepackage{graphicx}%
\usepackage{grffile}%
\usepackage{tikz}%
\usetikzlibrary{shapes}
\usepackage{wrapfig}%
\usepackage{tabularx}%
\usepackage{bbm}
\usepackage{siunitx}%
\usepackage{booktabs}%
\usepackage{multirow}%
\usepackage{vruler}%
\usepackage{fancyvrb}%
\usepackage{listings}%
\usepackage{csquotes}%
\usepackage[
backend=biber,%
natbib=true, %
style=authoryear,%
dashed=false,%
uniquename=init,%
hyperref=true,%
useprefix=true,%
maxcitenames=2,%
maxbibnames=6%
]{biblatex}%
\usepackage[
hypertexnames=false,%
setpagesize=false,%
pdfborder={0 0 0},%
pdfstartview=Fit,%
bookmarksopen=true,%
bookmarksnumbered=true%
]{hyperref}%
\xdefinecolor{lightgray}{RGB}{247, 247, 247}%
\xdefinecolor{semilightgray}{RGB}{240, 240, 240}%
\xdefinecolor{middlegray}{RGB}{127, 127, 127}%
\xdefinecolor{blue}{RGB}{58, 95, 205}%
\xdefinecolor{deepskyblue}{RGB}{0, 154, 205}%
\xdefinecolor{chocolate}{RGB}{205, 102, 29}%

\setkomafont{pageheadfoot}{\normalfont\normalcolor\sffamily}%
\setkomafont{pagenumber}{\normalfont\normalcolor\sffamily}%
\automark{section}%
\setkomafont{captionlabel}{\normalfont\normalcolor\sffamily\bfseries}%

\setlist{%
  align=left,%
  labelsep=*,%
  leftmargin=*,%
  topsep=1mm,%
  itemsep=0mm%
}
\newcommand*{\mysquare}{\rule[0.18em]{0.36em}{0.36em}}
\newcommand*{\mytriangle}{\raisebox{0.12em}{\resizebox{0.48em}{0.48em}{$\blacktriangleright$}}}
\newcommand*{\mybar}{\rule[0.32em]{0.62em}{0.08em}}
\newcommand*{\mydot}{\raisebox{0.14em}{\resizebox{0.44em}{!}{$\bullet$}}}
\setlist[itemize,1]{label={\mysquare\ }}%
\setlist[itemize,2]{label={\mytriangle\ }}%
\setlist[itemize,3]{label={\mybar\ }}%
\setlist[itemize,4]{label={\mydot\ }}%
\setlist[enumerate,1]{label=\arabic*)}%
\setlist[enumerate,2]{label=\arabic{enumi}.\arabic*)}%
\setlist[enumerate,3]{label=\arabic{enumi}.\arabic{enumii}.\arabic*)}%

\makeatletter
\newcommand\myisodate{\number\year-\ifcase\month\or 01\or 02\or 03\or 04\or 05\or 06\or 07\or 08\or 09\or 10\or 11\or 12\fi-\ifcase\day\or 01\or 02\or 03\or 04\or 05\or 06\or 07\or 08\or 09\or 10\or 11\or 12\or 13\or 14\or 15\or 16\or 17\or 18\or 19\or 20\or 21\or 22\or 23\or 24\or 25\or 26\or 27\or 28\or 29\or 30\or 31\fi}%
\makeatother
\newcommand*{\abstractnoindent}{}%
\let\abstractnoindent\abstract
\renewcommand*{\abstract}{\let\quotation\quote\let\endquotation\endquote
  \abstractnoindent}
\deffootnote[1em]{1em}{1em}{\textsuperscript{\thefootnotemark}}%
\lstdefinestyle{input}{
  backgroundcolor=\color{semilightgray},%
  commentstyle=\itshape\color{chocolate},%
  keywordstyle=\color{blue},%
  stringstyle=\color{blue},%
  numbers=left,%
  numbersep=4.8pt,%
  numberstyle=\color{darkgray!80}\tiny%
}
\lstdefinestyle{output}{
  backgroundcolor=\color{lightgray}%
}

\lstdefinestyle{Lstyle}{
  language=[LaTeX]TeX,%
  texcs={},%
  otherkeywords={}%
}

\lstnewenvironment{Linput}[1][]{%
  \lstset{style=input, style=Lstyle}
  #1%
}{\vspace{-0.25\baselineskip}}%

\lstnewenvironment{Loutput}[1][]{%
  \lstset{style=output, style=Lstyle}
  #1%
}{\vspace{-0.25\baselineskip}}%

\lstdefinestyle{Rstyle}{
  language=R,%
  keywords={if, else, repeat, while, function, for, in, next, break},%
  otherkeywords={}%
}

\lstnewenvironment{Rinput}[1][]{%
  \lstset{style=input, style=Rstyle}
  #1%
}{\vspace{-0.25\baselineskip}}%

\lstnewenvironment{Routput}[1][]{%
  \lstset{style=output, style=Rstyle}
  #1%
}{\vspace{-0.25\baselineskip}}%

\DeclareNameAlias{sortname}{last-first}%
\DefineBibliographyExtras{american}{\DeclareQuotePunctuation{}}%
\renewbibmacro*{volume+number+eid}{%
  \setunit*{\addcomma\space}%
  \printfield{volume}%
  \printfield{number}}
\DeclareFieldFormat*{number}{(#1)}
\DeclareFieldFormat*{title}{#1}%
\DeclareFieldFormat{doi}{%
  \ifhyperref
  {\href{http://dx.doi.org/#1}{\nolinkurl{doi:#1}}}%
  {\nolinkurl{doi:#1}}}%
\renewbibmacro*{in:}{}%
\DeclareFieldFormat{isbn}{ISBN #1}%
\DeclareFieldFormat{pages}{#1}%
\DeclareFieldFormat{url}{\url{#1}}%
\DeclareFieldFormat{urldate}{\mkbibparens{#1}}%
\renewcommand*{\cite}[2][]{\textcite[#1]{#2}}%

\newif\ifstarttheorem
\newtheoremstyle{mythmstyle}%
{0.5em}%
{0.5em}%
{}%
{}%
{\sffamily\bfseries\global\starttheoremtrue}%
{}%
{\newline}%
{\thmname{#1}\ \thmnumber{#2}\ \thmnote{(#3)}}%
\theoremstyle{mythmstyle}%
\newtheorem{definition}{Definition}[section]%
\newtheorem{proposition}[definition]{Proposition}

\newtheorem{theorem}[definition]{Theorem}
\newtheorem{corollary}[definition]{Corollary}
\newtheorem{remark}[definition]{Remark}

\newtheorem{algorithm}[definition]{Algorithm}

\makeatletter
\preto\itemize{%
  \if@inlabel
  \ifstarttheorem
  \mbox{}\par\nobreak\vskip\glueexpr-\parskip-\baselineskip+0.25em\relax\hrule\@height\z@
  \fi%
  \fi%
  \global\starttheoremfalse%
  \def\tempa{proof}%
  \ifx\tempa\mycurrenvir
  \ifstarttheorem
  \mbox{}\par\nobreak\vskip\glueexpr-\parskip-\baselineskip+0.25em\relax\hrule\@height\z@
  \fi%
  \fi%
  \global\starttheoremfalse%
}
\preto\enditemize{\global\starttheoremfalse}
\makeatother

\makeatletter
\preto\enumerate{%
  \if@inlabel
  \ifstarttheorem
  \mbox{}\par\nobreak\vskip\glueexpr-\parskip-\baselineskip+0.25em\relax\hrule\@height\z@
  \fi%
  \fi%
  \global\starttheoremfalse%
  \def\tempa{proof}%
  \ifx\tempa\mycurrenvir
  \ifstarttheorem
  \mbox{}\par\nobreak\vskip\glueexpr-\parskip-\baselineskip+0.25em\relax\hrule\@height\z@
  \fi%
  \fi%
  \global\starttheoremfalse%
}
\preto\endenumerate{\global\starttheoremfalse}
\makeatother

\newcommand{\ou}[3]{%
  \mathrel{%
    \vcenter{\offinterlineskip
      \ialign{##\cr$#1$\cr\noalign{\kern-#3}$#2$\cr}%
    }%
  }%
}
\newcommand*{\omu}[3]{\underset{#3}{\overset{#1}{#2}}}

\newcommand*{\varhat}[3]{{{\hat{#1}}_{#2}^{\text{#3}}}}

\newcommand*{\isim}{\omu{\text{\tiny{ind.}}}{\sim}{}}

\newcommand*{\IN}{\mathbbm{N}}
\newcommand*{\IZ}{\mathbbm{Z}}

\newcommand*{\IR}{\mathbbm{R}}

\newcommand*{\LN}{\operatorname{LN}}

\newcommand*{\U}{\operatorname{U}}

\newcommand*{\N}{\operatorname{N}}

\newcommand*{\ARMA}{\operatorname{ARMA}}
\newcommand*{\GARCH}{\operatorname{GARCH}}

\newcommand*{\I}{\mathbbm{1}}
\newcommand*{\rd}{\mathrm{d}}
\renewcommand*{\mod}{\operatorname{mod}}

\newcommand*{\supp}{\operatorname*{supp}}

\newcommand*{\D}{\operatorname{D}}

\newcommand*{\E}{\mathbbm{E}}

\newcommand*{\Var}{\operatorname{Var}}

\newcommand*{\ES}{\operatorname{ES}}
\newcommand*{\AC}{\operatorname{AC}}

\newcommand*{\MMD}{\operatorname{MMD}}

\newcommand*{\R}{\textsf{R}}
\newcommand*{\eps}{\varepsilon}
\newcommand*{\ntrn}{n_{\text{trn}}}
\newcommand*{\nbat}{n_{\text{bat}}}
\newcommand*{\nepo}{n_{\text{epo}}}
\newcommand*{\ngen}{n_{\text{gen}}}
\newcommand*{\nkrn}{n_{\text{krn}}}

\begin{document}
\thispagestyle{plain}
\begin{center}
  \sffamily
  {\bfseries\LARGE Quasi-random sampling for multivariate distributions via generative neural networks\par}
  \bigskip\smallskip
  {\Large Marius Hofert\footnote{Department of Statistics and Actuarial Science, University of
      Waterloo, 200 University Avenue West, Waterloo, ON, N2L
      3G1,
      \href{mailto:marius.hofert@uwaterloo.ca}{\nolinkurl{marius.hofert@uwaterloo.ca}}. The
      author acknowledges support from NSERC (Grant RGPIN-5010-2015).},
    Avinash Prasad\footnote{Department of Statistics and Actuarial Science, University of
      Waterloo, 200 University Avenue West, Waterloo, ON, N2L
      3G1,
      \href{mailto:a2prasad@uwaterloo.ca}{\nolinkurl{a2prasad@uwaterloo.ca}}. The author acknowledges support from NSERC (PGS D Scholarship).},
    Mu Zhu\footnote{Department of Statistics and Actuarial Science, University of
      Waterloo, 200 University Avenue West, Waterloo, ON, N2L
      3G1,
      \href{mailto:mu.zhu@uwaterloo.ca}{\nolinkurl{mu.zhu@uwaterloo.ca}}. The
      author acknowledges support from NSERC (RGPIN-2016-03876).}
    \par\bigskip
    \myisodate\par}
\end{center}
\par\smallskip
\begin{abstract}
  Generative moment matching networks (GMMNs) are introduced for generating
  quasi-random samples from multivariate models with any underlying copula in
  order to compute estimates under variance reduction. So far, quasi-random
  sampling for multivariate distributions required a careful design, exploiting
  specific properties (such as conditional distributions) of the implied
  parametric copula or the underlying quasi-Monte Carlo (QMC) point set, and
  was only tractable for a small number of models. Utilizing GMMNs allows one
  to construct quasi-random samples for a much larger variety of multivariate
  distributions without such restrictions, including empirical ones from real
    data with dependence structures not well captured by parametric copulas.
  Once trained on pseudo-random samples from a parametric model
    or on real data, these neural networks only require a multivariate standard
  uniform randomized QMC point set as input and are thus fast in estimating
  expectations of interest under dependence with variance reduction. Numerical
  examples are considered to demonstrate the approach, including applications
  inspired by risk management practice. All results are reproducible with the
  demos \texttt{GMMN\_QMC\_paper}, \texttt{GMMN\_QMC\_data} and
  \texttt{GMMN\_QMC\_timings} as part of the \R\ package \texttt{gnn}.
\end{abstract}
\minisec{Keywords}
Maximum mean discrepancy,
generative moment matching networks,
quasi-random numbers,
copulas,
sums of dependent random variables,
expected shortfall.
\minisec{MSC2010}
62H99, 65C60, 60E05, 00A72, 65C10. %

\section{Introduction}
Let $\bm{X}=(X_1,\dots,X_d)$ be a $d$-dimensional random vector with
distribution function $F_{\bm{X}}$ and continuous margins
$F_{X_1},\dots,F_{X_d}$. It is not a trivial task in general to generate
quasi-random samples $\bm{X}_1,...,\bm{X}_n$ from $F_{\bm{X}}$, i.e., samples
that mimic realizations from $F_{\bm{X}}$ but preserve low-discrepancy in the
sense of being locally more homogeneous with fewer ``gaps'' or ``clusters''
\parencite[Section~4.2]{cambou2017}.

By Sklar's Theorem, we always have the decomposition
\begin{align}
  F_{\bm{X}}(\bm{x})=C(F_{X_1}(x_1),\dots,F_{X_d}(x_d)),\quad \bm{x}=(x_1,\dots,x_d)\in\IR^d,\label{eq:sklar}
\end{align}
where $C:[0,1]^d \rightarrow [0,1]$ is the unique underlying copula
\parencite{nelsen2006,joe2014}. Since, in distribution,
$\bm{X}=F_{\bm{X}}^{-1}(\bm{U})$ for $\bm{U}\sim C$ and
$F_{\bm{X}}^{-1}(\bm{u})=(F_{X_1}^{-1}(u_1),\dots,F_{X_d}^{-1}(u_d))$, we shall
mostly focus on the problem of generating quasi-random samples
$\bm{U}_1,...,\bm{U}_n$ from $C$ rather than $\bm{X}_1,...,\bm{X}_n$ from
$F_{\bm{X}}$, as the latter are easily obtained from the former.

\subsection{Existing difficulty}
For the independence copula, $C(\bm{u})=u_1\cdot\ldots\cdot u_d$, quasi-random
samples can be obtained simply by using randomized quasi-Monte Carlo (RQMC)
point sets such as randomized Sobol' or generalized Halton sequences
\parencite[see, e.g.,][Chapter~5]{lemieux2009}.

Recently, \cite{cambou2017} demonstrated that for a limited number of copulas
$C$ (normal, $t$ or Clayton copulas), one can obtain quasi-random samples by
transforming RQMC point sets with the inverse Rosenblatt transform of $C$
\parencite{rosenblatt1952}; this method is known as the \emph{conditional
  distribution method (CDM)} --- see, e.g., \cite{embrechtslindskogmcneil2003}
or \cite[p.~45]{hofert2010c}. \cite{cambou2017} also showed that transformations
to quasi-random copula samples may exist for copulas
with a sufficiently simple stochastic representation. For most copulas, the
latter is not the case and the CDM is numerically intractable. In other words,
there exists no universal and numerically tractable transformation from RQMC
point sets to quasi-random samples from copulas. For the majority of copula
models, including grouped normal variance mixture copulas, Archimax copulas,
nested Archimedean copulas or extreme-value copulas, we simply do not know how
to generate quasi-random samples from them.

\subsection{Our contribution}
The main contribution of this paper is to introduce a new approach for
quasi-random sampling from $F_{\bm{X}}$ with {\em any} underlying copula $C$,
using generative neural networks. Even when we do {\em not} know the
distribution $F_{\bm{X}}$, our approach can still provide quasi-random samples
from the corresponding empirical distribution $\hat{F}_{\bm{X}}$ as long as we
have a dataset from $F_{\bm{X}}$. This is especially useful when the dependence
structure in the data cannot be adequately captured by a readily available
parametric copula; see Section~\ref{sec:realdata} where we present a real-data
example to show how useful our approach can be in this case where no adequate
copula model is known in the first place.

Specifically, let $f_{\bm{\theta}}$ denote a neural network (NN) parameterized
by $\bm{\theta}$.  We train $f_{\bm{\theta}}$ so that, given a $p$-dimensional
input $\bm{Z} \sim F_{\bm{Z}}$ with independent components $Z_1,\dots,Z_p$ from
known distributions $F_{Z_1},\dots,F_{Z_p}$, the trained NN can generate
$d$-dimensional output from the desired distribution,
$f_{\hat{\bm{\theta}}}(\bm{Z}) \sim F_{\bm{X}}$, where $\hat{\bm{\theta}}$
denotes the parameter vector of the trained NN. We can thus turn a uniform RQMC
point set, $\{\tilde{\bm{v}}_1,\dots,\tilde{\bm{v}}_n\}$, into a quasi-random
sample from $F_{\bm{X}}$ by letting
\begin{align}
  \bm{Y}_i=f_{\hat{\bm{\theta}}}\circ F_{\bm{Z}}^{-1}(\tilde{\bm{v}}_i),\quad i=1,\dots,n,\label{def:Yi}
\end{align}
where $F_{\bm{Z}}^{-1}(\bm{u}) = (F_{Z_1}^{-1}(u_1),\dots,F_{Z_p}^{-1}(u_p))$.

\subsection{Assessment}
The theoretical properties of quasi-randomness (or low-discrepancy) under
dependence are hard to assess; see \cite{cambou2017} and
Appendix~\ref{sec:appendix}. In low-dimensional cases (see
Section~\ref{sec:GMMN:visual}), we use data visualization tools to assess the
quality of the generated quasi-random samples, such as contour plots (or level
curves) showing that the empirical copula of our GMMN quasi-random samples is closer
to the true target copula than that of GMMN pseudo-random
samples. In higher-dimensional cases (see Section~\ref{sec:GMMN:accuracy}), we
use a Cram\'{e}r-von-Mises goodness-of-fit statistic to make the same point.

Since the main application of quasi-random sampling is to obtain low-variance
Monte-Carlo estimates of
\begin{align}
  \mu=\E(\Psi(\bm{X}))=\E\bigl(\Psi(F_{\bm{X}}^{-1}(\bm{U}))\bigr)\quad\text{for}\quad \bm{X} \sim F_{\bm{X}},\ \bm{U} \sim C\label{eq:prblm:X}
\end{align}
for an integrable $\Psi:\IR^d\rightarrow \IR$, we also assess our method in such a specific context.
The \emph{Monte Carlo estimator} approximates this expectation by
\begin{align}
  \varhat{\mu}{n}{MC}=\frac{1}{n}\sum_{i=1}^n\Psi(F_{\bm{X}}^{-1}(\bm{U}_i)),  \label{eq:prblm:approx}
\end{align}
where $\bm{U}_1,\dots,\bm{U}_n\isim C$. Using NN-generated quasi-random samples $\bm{Y}_1,\dots,\bm{Y}_n$ from \eqref{def:Yi}, we can approximate $\E\bigl(\Psi(F_{\bm{X}}^{-1}(\bm{U}))\bigr)$ by
\begin{align}
  \varhat{\mu}{n}{NN}=\frac{1}{n}\sum_{i=1}^n\Psi(\bm{Y}_i)=
  \frac{1}{n}\sum_{i=1}^n\Psi(f_{\hat{\bm{\theta}}}\circ F_{\bm{Z}}^{-1}(\tilde{\bm{v}}_i)).\label{eq:rqmc}
\end{align}
Theoretically (Section~\ref{sec:GMMN:sampling} and Appendix~\ref{sec:appendix}), we establish various guarantees that the estimation error of \eqref{eq:prblm:X} by \eqref{eq:rqmc} will be small as long as both $f_{\hat{\bm{\theta}}}$ and $\Psi$ are sufficiently smooth; we also establish the corresponding convergence rates. Empirically (Section~\ref{sec:conv:analysis}), we verify that \eqref{eq:rqmc} indeed has lower variance and converges faster than \eqref{eq:prblm:approx}.

Although being the main focus in this paper, let us stress that estimating expectations
such as \eqref{eq:prblm:X} is not the only application of quasi-random
sampling. For example, quasi-random sampling can also be used for estimating
quantiles for the distribution of a sum of dependent random variables.

All results presented in this paper (and more) are reproducible with
the demos \texttt{GMMN\_QMC\_paper}, \texttt{GMMN\_QMC\_data} and
\texttt{GMMN\_QMC\_timings} as part of the \R\ package \texttt{gnn}.

\section{Quasi-random GMMN samples}\label{sec:quasi:GMMN}
\subsection{Generative moment matching networks}
In this paper, we work with the \emph{multi-layer perceptron (MLP)}, which is
regarded as the quintessential \emph{neural network (NN)}.
Let $L$ be the number of (hidden) layers in the NN and, for each
$l=0,\dots,L+1$, let $d_l$ %
be the dimension of layer $l$, that is, the number of neurons in layer $l$. In
this notation, layer $l=0$ refers to the \emph{input layer} which consists of
the \emph{input} $\bm{z} \in \IR^p$ for $d_0=p$, and layer $l=L+1$ refers to the
\emph{output layer} which consists of the \emph{output} $\bm{y}\in \IR^d$ for
$d_{L+1}=d$. Layers $l=1,\dots,L+1$ can be described in terms of the output
$\bm{a}_{l-1} \in \IR^{d_{l-1}}$ of layer $l-1$ via
\begin{align*}
  \bm{a}_0 &= \bm{z}\in\IR^{d_0},\\
  \bm{a}_l &= f_l(\bm{a}_{l-1})=\phi_l(W_{l}\bm{a}_{l-1} +\bm{b}_l)\in\IR^{d_l},\quad l=1,\dots,L+1,\\
  \bm{y} &= \bm{a}_{L+1}\in\IR^{d_{L+1}},
\end{align*}
with \emph{weight matrices} $W_{l} \in \IR^{d_l \times d_{l-1}}$, \emph{bias
  vectors} $\bm{b}_{l} \in \IR^{d_{l}}$ and \emph{activation functions} $\phi_l$;
note that for vector inputs the activation function $\phi_l$ is understood to be
applied componentwise. %

The NN $f_{\bm{\theta}}: \IR^p \rightarrow \IR^d$ can then be written as the composition
\begin{align*}
  f_{\bm{\theta}}=f_{L+1} \circ f_L\circ \dots \circ f_2 \circ f_1,
\end{align*}
with its (flattened) parameter vector given by
$\bm{\theta}=(W_{1},\dots,W_{L+1},\bm{b}_{1},\dots,\bm{b}_{L+1})$. To fit
$\bm{\theta}$, we use the backpropagation algorithm (a stochastic gradient
descent) based on a \emph{cost function} $E$.
Conceptually, $E$ computes a distance between the \emph{target output}
$\bm{x}\in\IR^{d}$ and the \emph{actual output}
$\bm{y}=\bm{y}(\bm{z})\in\IR^{d}$ predicted by the NN; what is actually
computed is a sample version of $E$ based on a subsample (the so-called \emph{mini-batch}),
see Section~\ref{sec:cost:training}.

The expressive power of NNs is primarily characterized by the \emph{universal
  approximation theorem}; see \cite[Chapter~6]{goodfellow2016}. In particular,
given suitable activation functions, a single hidden layer NN with a finite
number of neurons can approximate any continuous function on a compact subset of
the multidimensional Euclidean space; see \cite[Chapter~4]{nielsen2015} for a
visual account of the validity of the universal approximation theorem.
\cite{cybenko1989} first proposed such universal approximation results for the
sigmoid activation function $\phi_l(x)=1/(1+\mathrm{e}^{-x})$ %
and \cite[Theorem~1]{hornik1991} then generalized the results to include arbitrary bounded
and non-constant activation functions. In recent years, the \emph{rectified
  linear unit (ReLU)} $\phi_l(x)=\max\{0,x\}$ has become the most popular
activation function for efficiently training NNs. This unbounded activation
function does not satisfy the assumptions of the universal approximation theorem
in \cite{hornik1991}. However, there have since been numerous theoretical
investigations into the expressive power of NNs with ReLU activation
functions; see, for example, \cite{pascanu2013}, \cite{montufar2014} or
\cite{arora2016}. In particular, for certain conditions on the number of layers
and neurons in the NN, \cite{arora2016} provide a similar universal
approximation theorem for NNs with ReLU activation functions.

\cite{li2015} and \cite{dziugaite2015} simultaneously introduced a type of generative neural network known as the \emph{generative moment matching network (GMMN)} or the Maximum Mean Discrepancy (MMD) net.
 GMMNs are NNs $f_{\bm{\theta}}$ of
the above form which utilize a (kernel) maximum mean discrepancy statistic as the
cost function (see later).
Conceptually, they can be thought of as parametric maps of a given sample
$\bm{Z}=(Z_1,\dots,Z_p)$ from an \emph{input distribution} $F_{\bm{Z}}$ to a
sample $\bm{X}=(X_1,\dots,X_d)$ from the \emph{target distribution}
$F_{\bm{X}}$.  As is standard in the literature, we assume independence among
the components of $\bm{Z}=(Z_1,\dots,Z_p)$. %
Typical choices for the distribution of the $Z_j$'s are $\U(0,1)$ or $\N(0,1)$.
The objective is then to generate samples from the target
distribution %
via the trained GMMN $f_{\hat{\bm{\theta}}}$. The MMD nets introduced in \cite{dziugaite2015} are
almost identical to GMMNs but with a slight difference in the training
procedure; additionally, \cite{dziugaite2015} provided a theoretical framework
for analyzing optimization algorithms with (kernel) MMD cost
functions.

\subsection{Cost function and training of GMMNs}\label{sec:cost:training}
To learn $f_{\bm{\theta}}$ (or, statistically speaking, to estimate
the parameter vector $\bm{\theta}$) we assume that we have $\ntrn$ training data points
$\bm{X}_1,\dots,\bm{X}_{\ntrn}$ from $\bm{X}$, either in the form of a pseudo-random sample from $F_{\bm{X}}$ or as real data.
Based on a sample $\bm{Z}_1,\dots,\bm{Z}_{\ntrn}$ from the
input distribution, the GMMN generates the output sample
$\bm{Y}_1,\dots,\bm{Y}_{\ntrn}$, where $\bm{Y}_i=f_{\bm{\theta}}(\bm{Z}_i)$,
$i=1,\dots,\ntrn$. Stacking $\bm{X}_1,\dots,\bm{X}_{\ntrn}$ into an $\ntrn\times d$ matrix $X$ and likewise $\bm{Y}_1,\dots,\bm{Y}_{\ntrn}$ into $Y$, we are thus interested in whether the two samples $X$ and $Y$ come from the
same distribution.

To this end, GMMNs utilize as cost function $E$
the \emph{maximum mean discrepancy (MMD)} statistic from the kernel two-sample
test introduced by \cite{gretton2007}, %
whose sample version is given by
\begin{align}
  \MMD(X,Y)=\sqrt{\frac{1}{\ntrn^2} \sum_{i_1=1}^{\ntrn}\sum_{i_2=1}^{\ntrn}(K(\bm{X}_{i_1},\bm{X}_{i_2})- 2K(\bm{X}_{i_1},\bm{Y}_{i_2}) + K(\bm{Y}_{i_1},\bm{Y}_{i_2}))},\label{def:MMD}
\end{align}
where $K(\cdot,\cdot): \IR^d \times \IR^d \rightarrow \IR$ denotes a kernel
(similarity) function.
If $K(\cdot,\cdot)$ is a so-called universal kernel function (e.g., Gaussian or
Laplace), then it can be shown \parencite{gretton2007,gretton2012} that the $\MMD$
converges in probability to $0$ for $\ntrn\to\infty$ if and only if
$\bm{Y}=\bm{X}$ in distribution. This makes the $\MMD$ with a universal
  kernel function an intuitive choice as cost function for training
  $f_{\bm{\theta}}$ to learn a random number generator from a multivariate
  distribution $F_{\bm{X}}$; our choice of kernel $K$ is addressed in
  Section~\ref{sec:details}.  For the Gaussian kernel in particular, expanding
  the exponential, one can see that minimizing the $\MMD$ can be interpreted as
  matching all the moments of the distributions of $\bm{X}$ and $\bm{Y}$ which
  gives an interpretation of the $\MMD$ in this case.  Note that, in comparison
  to a quadratic cost function, all pairs of observations enter the $\MMD$,
  which turns out to be a crucial property for learning a random number generator
  of $F_{\bm{X}}$.

  Computing $\MMD(X,Y)$ in~\eqref{def:MMD} requires one to evaluate the
  kernel for all $\binom{\ntrn}{2}$ pairs of observations, which is
  memory-prohibitive for even moderately large $\ntrn$. As suggested
  by~\cite{li2015}, we thus adopt a mini-batch optimization procedure.
  Instead of directly optimizing the $\MMD$ for the entire training dataset, we
partition the data into \emph{batches} of size $\nbat$ and use the batches
sequentially to update the parameters $\bm{\theta}$ of the GMMN with the Adam
optimizer of \cite{kingma2014b}. Rather than following the gradient at each iterative step, the
Adam optimizer essentially uses a ``memory-sticking
gradient'' --- a weighted combination of the current gradient and past gradients
from earlier iterations. After all the training data are exhausted, i.e., roughly
after $(\ntrn/\nbat)$-many batches or gradient steps, one \emph{epoch} of the
training of the GMMN is completed. The overall training procedure is considered
completed after $\nepo$ epochs. The training of the GMMN can thus be summarized
as follows:
\begin{algorithm}[Training GMMNs]\label{algorithm:GMMN:train}
  \begin{enumerate}
  \item Fix the number $\nepo$ of
    epochs %
    and the batch size $1\le\nbat\le\ntrn$ per epoch, where $\nbat$ is assumed
    to divide $\ntrn$. Initialize the epoch counter $k=0$ and the GMMN's
    parameter vector $\bm{\theta}$; we follow
    \cite{glorot2010} and initialize the components of $\bm{\theta}$
    as $W_l\sim \U(-\sqrt{6/(d_l+d_{l-1})},\sqrt{6/(d_l+d_{l-1})})^{d_l \times
      d_{l-1}}$ and $\bm{b}_{l}=\bm{0}$ for $l=1,\dots,L+1$.
  \item For epoch $k=1,\dots,\nepo$, do:
    \begin{enumerate}
    \item Randomly partition the input distribution sample $\bm{Z}_1,\dots,\bm{Z}_{\ntrn}$ and training sample $\bm{X}_1,\dots,\bm{X}_{\ntrn}$ into corresponding $\ntrn/\nbat$ non-overlapping batches $\bm{Z}_1^{(b)},\dots,$ $\bm{Z}_{\nbat}^{(b)}$ and $\bm{X}_1^{(b)},\dots,$ $\bm{X}_{\nbat}^{(b)}$, $b=1,\dots,\ntrn/\nbat$, of size $\nbat$ each.
    \item For batch $b=1,\dots,\ntrn/\nbat$, do:
      \begin{enumerate}
      \item Compute the GMMN output $\bm{Y}_i^{(b)}=f_{\bm{\theta}}(\bm{Z}_i^{(b)})$, $i=1,\dots,\nbat$.
      \item Compute the gradient
        $\frac{\partial}{\partial\bm{\theta}}\MMD(X^{(b)},Y^{(b)})$ from the
        samples $X^{(b)}$ (stacking $\bm{X}_1^{(b)},\dots,\bm{X}_{\nbat}^{(b)}$)
        and $Y^{(b)}$ (stacking $\bm{Y}_1^{(b)},\dots,\bm{Y}_{\nbat}^{(b)}$) via
        automatic differentiation.
      \item Take a gradient step to update $\bm{\theta}$ with the Adam optimizer popularized by
        \cite[Algorithm~1]{kingma2014b}.
      \end{enumerate}
    \end{enumerate}
  \item Return $\hat{\bm{\theta}}=\bm{\theta}$; the fitted GMMN is then $f_{\hat{\bm{\theta}}}$.
  \end{enumerate}
\end{algorithm}

\subsection{Pseudo- and quasi-random sampling by GMMN}\label{sec:GMMN:sampling}
The following algorithm describes how to obtain a pseudo-random sample of $\bm{Y}$ via the trained GMMN
$f_{\hat{\bm{\theta}}}$ from a pseudo-random sample $\bm{Z}\sim F_{\bm{Z}}$.
\begin{algorithm}[Pseudo-random sampling by GMMN]\label{algorithm:GMMN:prng}
  \begin{enumerate}
  \item Fix the number $\ngen$ of samples to generate from $\bm{Y}$.
  \item Draw $\bm{Z}_i\isim F_{\bm{Z}}$,
    $i=1,\dots,\ngen$, for example, via $\bm{Z}_i=F_{\bm{Z}}^{-1}(\bm{U}'_i)$,
    $i=1,\dots,\ngen$, where $\bm{U}'_1,\dots,\bm{U}'_{\ngen}\isim\U(0,1)^p$.
  \item Return $\bm{Y}_i=f_{\hat{\bm{\theta}}}(\bm{Z}_i)$, $i=1,\dots,\ngen$; to obtain a sample from $C$, return
    the pseudo-observations of $\bm{Y}_1,\dots,\bm{Y}_{\ngen}$ \parencite{genestghoudirivest1995}.
  \end{enumerate}
\end{algorithm}
To obtain quasi-random samples from $F_{\bm{X}}$ with underlying copula $C$, we replace
$\bm{U}'_1,\dots,\bm{U}'_{\ngen}\isim\U(0,1)^p$ in
Algorithm~\ref{algorithm:GMMN:prng} by an RQMC point set to obtain the following
algorithm; the randomization is done to obtain unbiased QMC estimators and
estimates of their variances.
\begin{algorithm}[Quasi-random sampling by GMMN]\label{algorithm:GMMN:qrng}
  \begin{enumerate}
  \item Fix the number $\ngen$ of samples to generate from $\bm{Y}$.
  \item Compute an RQMC point set
    $\tilde{P}_{\ngen}=\{\tilde{\bm{v}}_1,\dots,\tilde{\bm{v}}_{\ngen}\}$ (for
    example, a randomized Sobol' or a generalized Halton sequence) and
    $\bm{Z}_i=F^{-1}_{\bm{Z}}(\tilde{\bm{v}}_i)$, $i=1,\dots,\ngen$.
  \item Return $\bm{Y}_i=f_{\hat{\bm{\theta}}}(\bm{Z}_i)$, $i=1,\dots,\ngen$;
    to obtain a sample from $C$, return the pseudo-observations of $\bm{Y}_1,\dots,\bm{Y}_{\ngen}$.
  \end{enumerate}
\end{algorithm}

Note that $\tilde{P}_{\ngen}$ mimics $\U(0,1)^p$, and not $C$.  As mentioned in
the introduction, \cite{cambou2017} presented transformations to convert
$\tilde{P}_{\ngen}$ to samples which mimic samples from $C$ but locally provide
a more homogeneous coverage. Unfortunately, these transformations are only
available for a few specific cases of $C$ and their numerical
evaluation in a fast and robust way is even more challenging.

To avoid these problems, we suggest to utilize the GMMN $f_{\hat{\bm{\theta}}}$
trained as in Algorithm~\ref{algorithm:GMMN:train}. Besides a straightforward evaluation,
this allows us to generate quasi-random samples from $F_{\bm{X}}$ with any
underlying copula $C$, by training a GMMN on pseudo-random samples generated
from $F_{\bm{X}}$. Alternatively, quasi-random samples which follow the same
empirical distribution as any given dataset can be obtained by
training a GMMN on the given dataset itself.  An additional advantage is that GMMNs
provide a sufficiently smooth map from the RQMC point set to the target
distribution which helps preserve the low-discrepancy of the point
set upon transformation and hence guarantees the improved performance of RQMC
estimators compared to the MC estimator (see Section~\ref{sec:conv:analysis} and Appendix~\ref{sec:appendix}).

With the mapping $F_{\bm{Z}}^{-1}(\bm{u}) = (F_{Z_1}^{-1}(u_1),\dots,F_{Z_p}^{-1}(u_p))$ to the
input distribution and the trained GMMN $f_{\hat{\bm{\theta}}}$ at hand, define a transform
\begin{align*}
  q(\bm{u}) = f_{\hat{\bm{\theta}}}\circ F_{\bm{Z}}^{-1}(\bm{u}),\quad \bm{u}\in (0,1)^p.
\end{align*}
Based on the RQMC point set
$\tilde{P}_{\ngen}=\{\tilde{\bm{v}}_1,\dots,\tilde{\bm{v}}_{\ngen}\}$ of size $\ngen$,
we can then obtain quasi-random samples by
\begin{align*}
  \bm{Y}_i=q(\tilde{\bm{v}}_i),\quad i=1,\dots,\ngen,
\end{align*}
(compare with~\eqref{def:Yi}) and define a
\emph{GMMN RQMC estimator} of \eqref{eq:prblm:X} by
\begin{align}
  \varhat{\mu}{\ngen}{NN} = \frac{1}{\ngen} \sum_{i=1}^{\ngen} \Psi(\bm{Y}_i) = \frac{1}{\ngen} \sum_{i=1}^{\ngen} \Psi(q(\tilde{\bm{v}}_i))=\frac{1}{\ngen} \sum_{i=1}^{\ngen} \Psi\bigl(f_{\hat{\bm{\theta}}}(F_{\bm{Z}}^{-1}(\tilde{\bm{v}}_i))\bigr).\label{eq:GMMN:C:RQMC}
\end{align}

We thus have the approximations
\begin{align}
  \E(\Psi(\bm{X}))\approx \E(\Psi(\bm{Y}))\approx \varhat{\mu}{\ngen}{NN}.\label{eq:approxis}
\end{align}
The error in the first approximation is small if the GMMN is trained well and
the error in the second approximation is small if the unbiased estimator
$\varhat{\mu}{\ngen}{NN}$ has a small variance. The primary
\emph{bottleneck} in this setup is the error in the first approximation which is
determined by the size $\ntrn$ of the training dataset and, in particular, by
the batch size $\nbat$ which is the major factor determining training efficiency
of the GMMN we found in all our numerical studies. Given a sufficiently large $\nbat$ and, by
extension, $\ntrn$, the GMMN is trained well, which renders the first
approximation error in~\eqref{eq:approxis} negligible.  However, in practice the
batch size $\nbat$ is constrained by the quadratically increasing memory demands
to compute the MMD cost function of the GMMN. For a theoretical result regarding
this approximation error, see \cite{dziugaite2015} where a bound on the error
between optimizing a sample version and a population version of $\MMD(X,Y)$ was
investigated. Finally, let us note that the task of GMMN training and generation
are separate steps which ensures that, once trained, generating quasi-random
GMMN samples is comparably fast; see Appendix~\ref{sec:timings}.

The error in the second approximation in~\eqref{eq:approxis} is small if the
composite function $\Psi\circ q$ is sufficiently smooth. The transform $q$ is
sufficiently smooth for GMMNs $f_{\hat{\bm{\theta}}}$ constructed using standard
activation functions and commonly used input distributions; see the discussion
following Corollary~\ref{corollary:rqmc:scr}. Given a sufficiently smooth
$\Psi$, we can establish a rate of convergence $O(\ngen^{-3}(\log \ngen)^{p-1})$ for the variance (and $O(\ngen^{-3/2}(\log \ngen)^{(p-1)/2})$ for the
approximation error) of the GMMN RQMC estimator $\varhat{\mu}{\ngen}{NN}$
constructed using scrambling as randomization; see
Appendix~\ref{appendix:rqmc:scr:analysis}.  With a stronger assumption on the
behavior of the composite function $\Psi\circ q$, we can show that the
Koksma--Hlawka bound on the error between the (non-randomized) GMMN QMC
estimator $\frac{1}{\ngen} \sum_{i=1}^{\ngen} \Psi(q(\bm{v}_i))$ and
$\E(\Psi(\bm{Y}))$ is satisfied which in turn implies a rate of convergence $O(\ngen^{-1}(\log \ngen)^p)$ for the
(non-randomized) GMMN QMC estimator; see
Appendix~\ref{appendix:qmc:analysis}. If the Koksma--Hlawka bound holds, we can
also establish a rate of convergence $O(\ngen^{-2}(\log \ngen)^{2p})$ for the variance of GMMN RQMC estimators
constructed using the digital shift method as randomization technique; see
Appendix~\ref{appendix:rqmc:shift:analysis}.
\section{GMMN pseudo- and quasi-random samples for copula models}\label{sec:GMMN:copula}
In this section we assess the quality of pseudo-random samples and quasi-random
samples generated from GMMNs. In both cases we train GMMNs on pseudo-random
samples $\bm{U}_1,\dots,\bm{U}_{\ntrn}\sim C$ from the respective copula $C$. We
start by addressing key implementation details and hyperparameters of
Algorithm~\ref{algorithm:GMMN:train} that we used in all examples thereafter. By
utilizing this algorithm to train $f_{\bm{\theta}}$ for a wide variety of copula
families, we then investigate the quality of the samples $\bm{Y}_1,\dots,\bm{Y}_{\ngen}$,
once generated by Algorithm~\ref{algorithm:GMMN:prng} and once by
Algorithm~\ref{algorithm:GMMN:qrng}.

\subsection{GMMN architecture, choice of kernel and training setup}\label{sec:details}
We find a single hidden layer architecture ($L=1$) to be sufficient for all the examples we considered. This is because, in this paper, we largely consider the cases of $d \in \{2,...,10\}$. Learning an entire distribution nonparametrically for $d>10$ would most likely require $L>1$, but it would also require a much larger sample size $\ntrn$ and become much more challenging computationally for GMMNs --- recall from Section~\ref{sec:cost:training} that the cost function requires $\binom{\ntrn}{2}$ evaluations.
After experimentation, we fix $d_1=300$, $\phi_1$ to be ReLU (it offers computational
efficiency via non-expensive and non-vanishing gradients) and $\phi_2$ to be
sigmoid (to obtain outputs in $[0,1]^d$).

To avoid the need of fine-tuning the bandwidth parameter, we follow \cite{li2015} and use a mixture of Gaussian kernels with different
bandwidth parameters as our kernel function for the $\MMD$ statistic
in~\eqref{def:MMD}; specifically,
\begin{align}
  K(\bm{x},\bm{y}) =\sum_{i=1}^{\nkrn} K(\bm{x},\bm{y};\sigma_i),\label{eq:kernel}
\end{align}
where $\nkrn$ denotes the number of mixture components and
$K(\bm{x},\bm{y};\sigma) = \exp(-\lVert \bm{x}-\bm{y}\rVert_2^2/(2\sigma^2))$ is
the Gaussian kernel with bandwidth parameter $\sigma>0$. After experimentation,
we fix $\nkrn=6$ and choose
$(\sigma_1,\dots,\sigma_6)=(0.001,0.01,0.15,0.25,0.50,0.75)$; note that
copula samples are in $[0,1]^d$.

Unless otherwise specified, we use the following setup across all examples.  We
use $\ntrn=60\,000$ training data points and find this to be sufficiently large
to obtain reliable $f_{\hat{\bm{\theta}}}$. As dimension of the input
distribution $F_{\bm{Z}}$, we choose $p=d$, that is, the GMMN $f_{\bm{\theta}}$
is set to be a $d$-to-$d$ transformation.
For $F_{\bm{Z}}$, we choose
$\bm{Z}\sim \N(\bm{0},I_d)$, where $I_d$ denotes the identity matrix in
$\IR^{d\times d}$, so $\bm{Z}$ consists of independent standard normal random
variables; this choice worked better than $\U(0,1)^d$ in practice despite the fact that
$\N(\bm{0},I_d)$ does not satisfy the assumptions of
Proposition~\ref{prop:bound:ass}. We choose a batch size of $\nbat=5000$ in Algorithm~\ref{algorithm:GMMN:train}; this decision is motivated from a practical
trade-off that a small $\nbat$ will lead to poor estimates of the population
MMD cost function but a large $\nbat$ will incur quadratically growing
memory requirements due to \eqref{def:MMD}. As the number of epochs we choose
$\nepo=300$ which is generally sufficient in our experiments to obtain
accurate results. The tuning parameters of the Adam optimizer is set to the
default values reported in \cite{kingma2014b}.

All results in this section, Section~\ref{sec:conv:analysis} and
Appendix~\ref{appendix:rqmc:shift:analysis} are reproducible with the demo
\texttt{GMMN\_QMC\_paper} of the \R\ package \texttt{gnn}. Our implementation
utilizes the \R\ packages \texttt{keras} and \texttt{tensorflow} which serve as
\R\ interfaces to the corresponding namesake Python libraries.  Furthermore, all
GMMN training is carried out on one NVIDIA Tesla P100 GPU. To generate the RQMC
point set in Algorithm~\ref{algorithm:GMMN:qrng}, we use scrambled nets
\parencite{owen1995}; see also Appendix~\ref{appendix:rqmc}. Specifically, we use
the implementation \texttt{sobol(, randomize = "Owen")} from the \R\ package
\texttt{qrng}. Finally, our choice of \R\ as programming language for this work
was motivated by the fact that contributed packages providing functionality for
copula modeling and quasi-random number generation --- two of the three
major fields of research (besides machine learning) this work touches upon --- exist in \R.

\subsection{Visual assessments of GMMN samples}\label{sec:GMMN:visual}
In this section we primarily focus on the bivariate case but include an example involving a trivariate copula; for higher-dimensional copulas, see Sections~\ref{sec:GMMN:accuracy} and \ref{sec:conv:analysis}. For all
one-parameter copulas considered, the single parameter will be chosen such that
Kendall's tau, denoted by $\tau$, is equal to 0.25 (weak dependence), 0.50
(moderate dependence) or 0.75 (strong dependence); clearly, this only applies to
copula families where there is a one-to-one mapping between the copula parameter
and $\tau$.

\subsubsection{$t$, Archimedean copulas and their associated mixtures}\label{sec:t:AC:mix}
First, we consider Student $t$ copulas, Archimedean copulas, and their mixtures.

Student $t$ copulas are prominent members of the elliptical class of copulas and
are given by $C(\bm{u})= t_{\nu,P}(t^{-1}_{\nu}(u_1),\dots,t^{-1}_{\nu}(u_d))$,
$\bm{u}\in[0,1]^d$, where $t_{\nu,P}$ denotes the distribution function of the
$d$-dimensional $t$ distribution with $\nu$ degrees of freedom, location vector
$\bm{0}$ and correlation matrix $P$, and $t_{\nu}^{-1}$ denotes
the quantile function of the univariate $t$ distribution with $\nu$ degrees of
freedom. For all $t$ copulas considered in this work, we fix $\nu=4$. Student $t$ copulas have explicit inverse Rosenblatt transforms,
so one can utilize the CDM for generating quasi-random samples from them~\parencite{cambou2017}.

Archimedean copulas are copulas of the form
\begin{align*}
C(\bm{u})= \psi(\psi^{-1}(u_1),\dots,\psi^{-1}(u_d)),\quad\bm{u}\in[0,1]^d,
\end{align*}
for an Archimedean generator $\psi$ which is a continuous, decreasing
function $\psi:[0,\infty]\rightarrow[0,1]$ that satisfies $\psi(0)=1$,
$\psi(\infty)=\lim_{t\rightarrow\infty}\psi(t)=0$ and that is strictly
decreasing on $[0, \inf{t: \psi(t)=0}]$. Examples of Archimedean generators include
$\psi_C(t)=(1+t)^{-1/\theta}$ (for $\theta>0$) and
$\psi_G(t)=\exp(-t^{1/\theta})$ (for $\theta\ge 1$), generating Clayton and
Gumbel copulas, respectively.
While the inverse Rosenblatt transform and thus the CDM is available
analytically for Clayton copulas, this is not the case for Gumbel;
in Appendix~\ref{sec:timings} we used numerical root finding to include
the latter case for the purpose of timings only.

We additionally consider equally-weighted two-component mixture copulas in which one component is a $90$-degree-rotated $t_4$ copula with $\tau=0.50$ and the other component is either a Clayton copula ($\tau=0.50$) or a Gumbel copula ($\tau=0.50$). The two mixture copula models are referred to as Clayton-$t(90)$ and Gumbel-$t(90)$ copulas, respectively.

The top rows of Figures~\ref{fig:t:example}--\ref{fig:gumbel:example} display contour plots of true $t$, Clayton and Gumbel copulas respectively, with $\tau=0.25$ (left), $0.50$ (middle) and $0.75$ (right) along with contours of empirical copulas
based on GMMN pseudo-random and GMMN quasi-random samples corresponding to each true copula $C$. The top row of Figure~\ref{fig:mixtures:example} displays similar plots for Clayton-$t(90)$ (left) and Gumbel-$t(90)$ (right) copulas. In each plot, across all figures described above, we observe that the contour of the empirical copula based on GMMN pseudo-random samples is visually fairly similar to the contour of $C$, thus indicating that the 11 GMMNs have been trained sufficiently well. We also see that the contours of the empirical copulas based on GMMN quasi-random samples better approximate the contours of $C$ than the contours of the empirical copulas based on the corresponding pseudo-random samples. This observation indicates that, at least visually, the 11 GMMN transforms (corresponding to each $C$) have preserved the low-discrepancy of the input RQMC point sets.

The bottom rows of Figures~\ref{fig:t:example}--\ref{fig:mixtures:example} display Rosenblatt transformed GMMN quasi-random samples, corresponding to each of the 11 true copulas $C$ under consideration. The Rosenblatt transform for a bivariate copula $C$ maps $(U_1,U_2)\sim C$ to $(R_1,R_2)=(U_1,C_{2|1}(U_2\,|\,U_1))$, where $C_{2|1}(u_2\,|\,u_1)$ denotes the conditional distribution function of $U_2$ given $U_1=u_1$ under $C$. We exploit the fact that $(R_1,R_2)\sim\U(0,1)^2$ if and only if $(U_1,U_2)\sim C$. Moreover, Rosenblatt transforming the GMMN quasi-random samples should yield a more homogeneous coverage of $[0,1]^2$. From each of the scatter plots in Figures~\ref{fig:t:example}--\ref{fig:mixtures:example}, we observe no significant departure from $\U(0,1)^2$, thus indicating that the GMMNs have learned sufficient approximations to the corresponding true copulas $C$. Furthermore, the lack of gaps or clusters in the scatter plots provides some visual confirmation of the low-discrepancy of the Rosenblatt-transformed GMMN quasi-random samples.

\begin{figure}[htbp]
  \centering
  \includegraphics[width=0.32\textwidth]{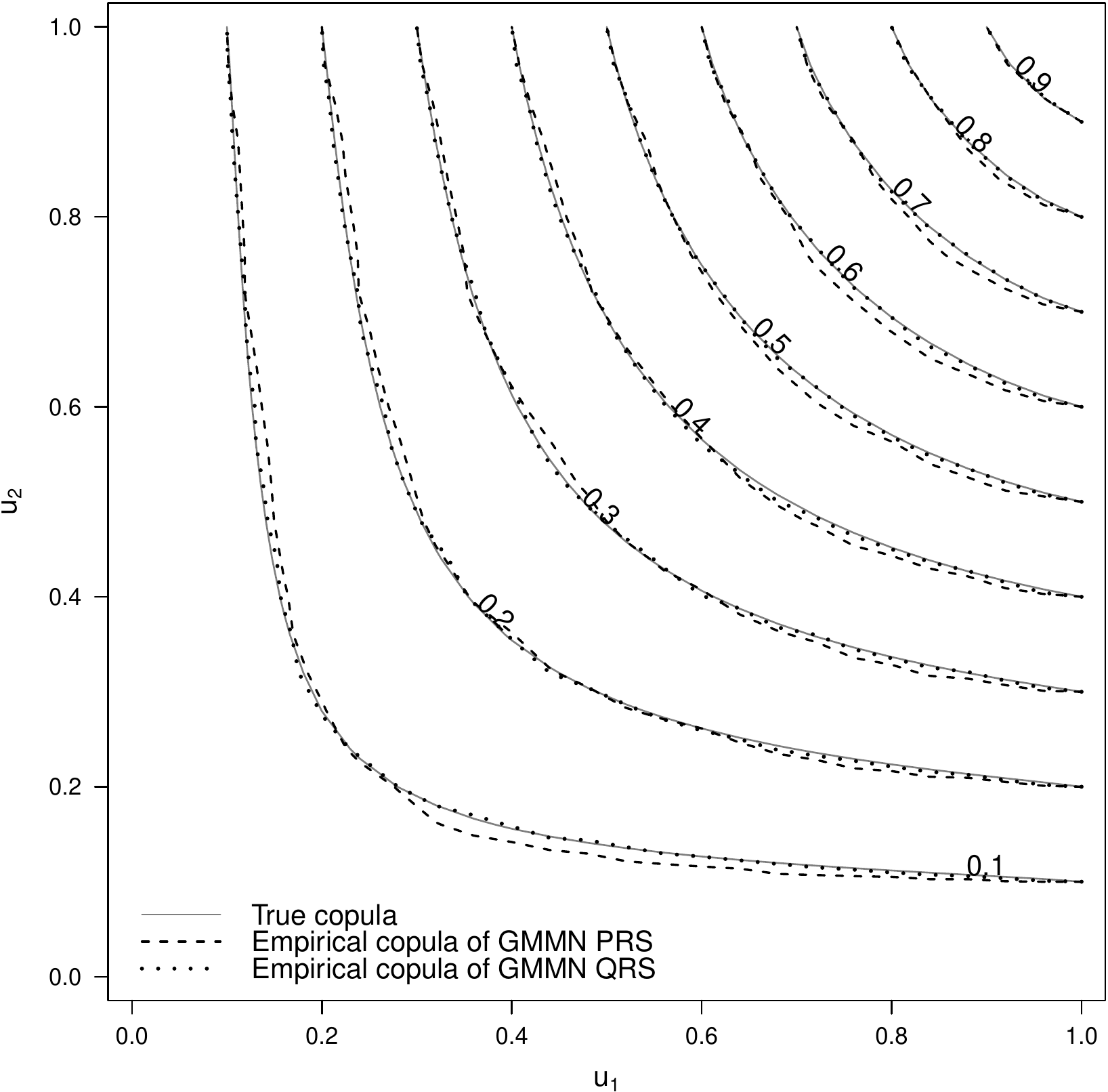}\hfill
  \includegraphics[width=0.32\textwidth]{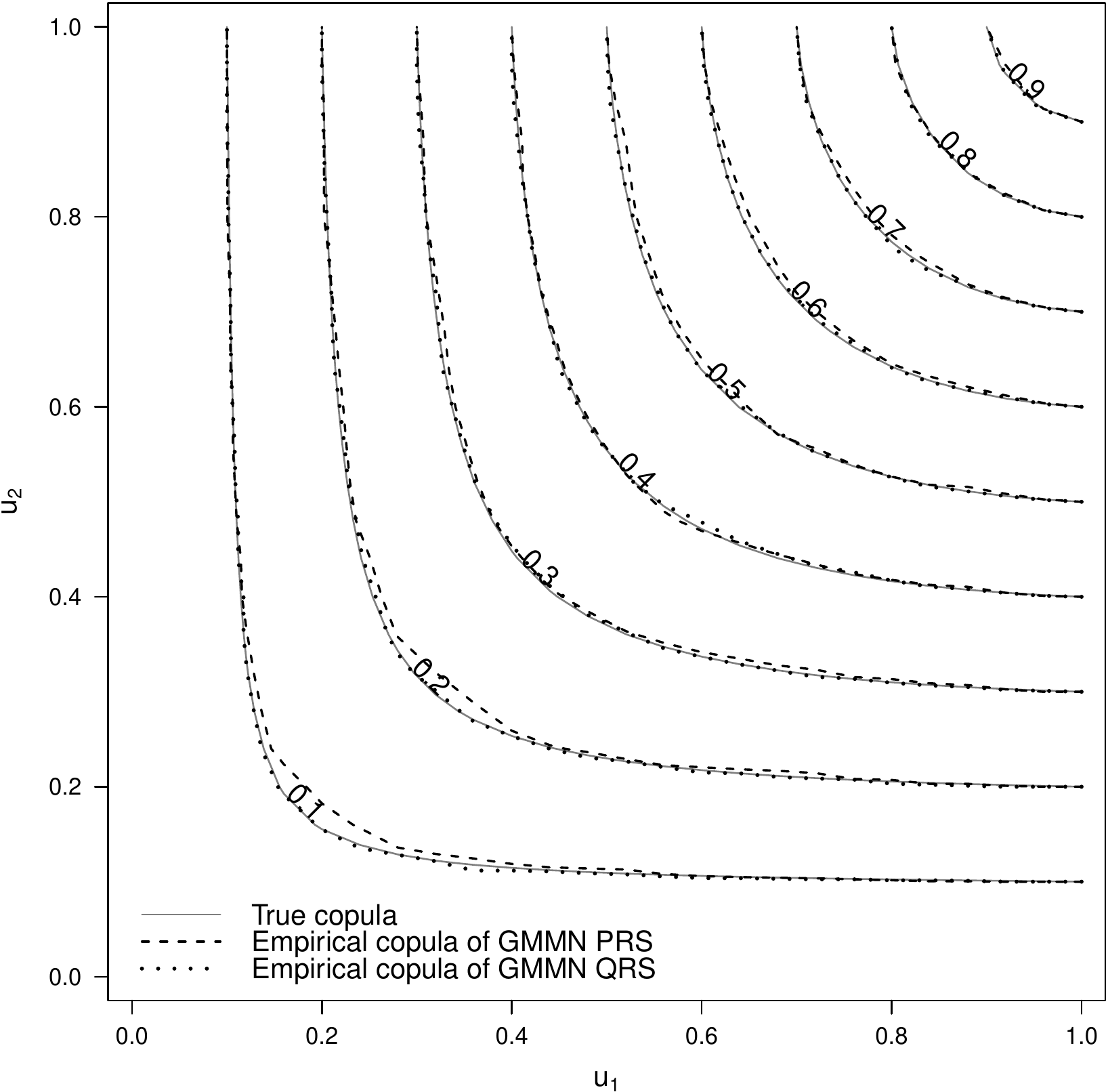}\hfill
  \includegraphics[width=0.32\textwidth]{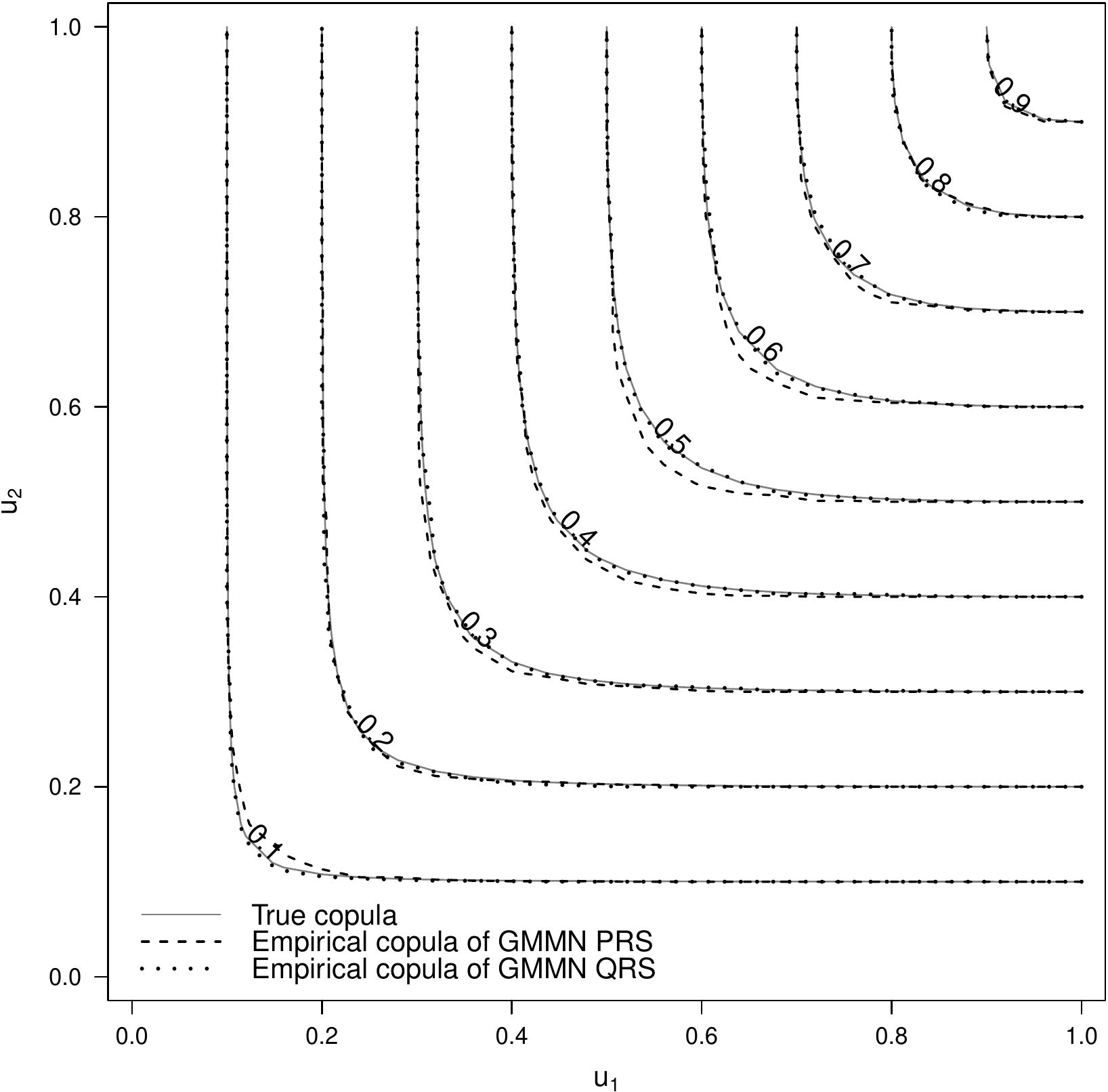}
  \includegraphics[width=0.32\textwidth]{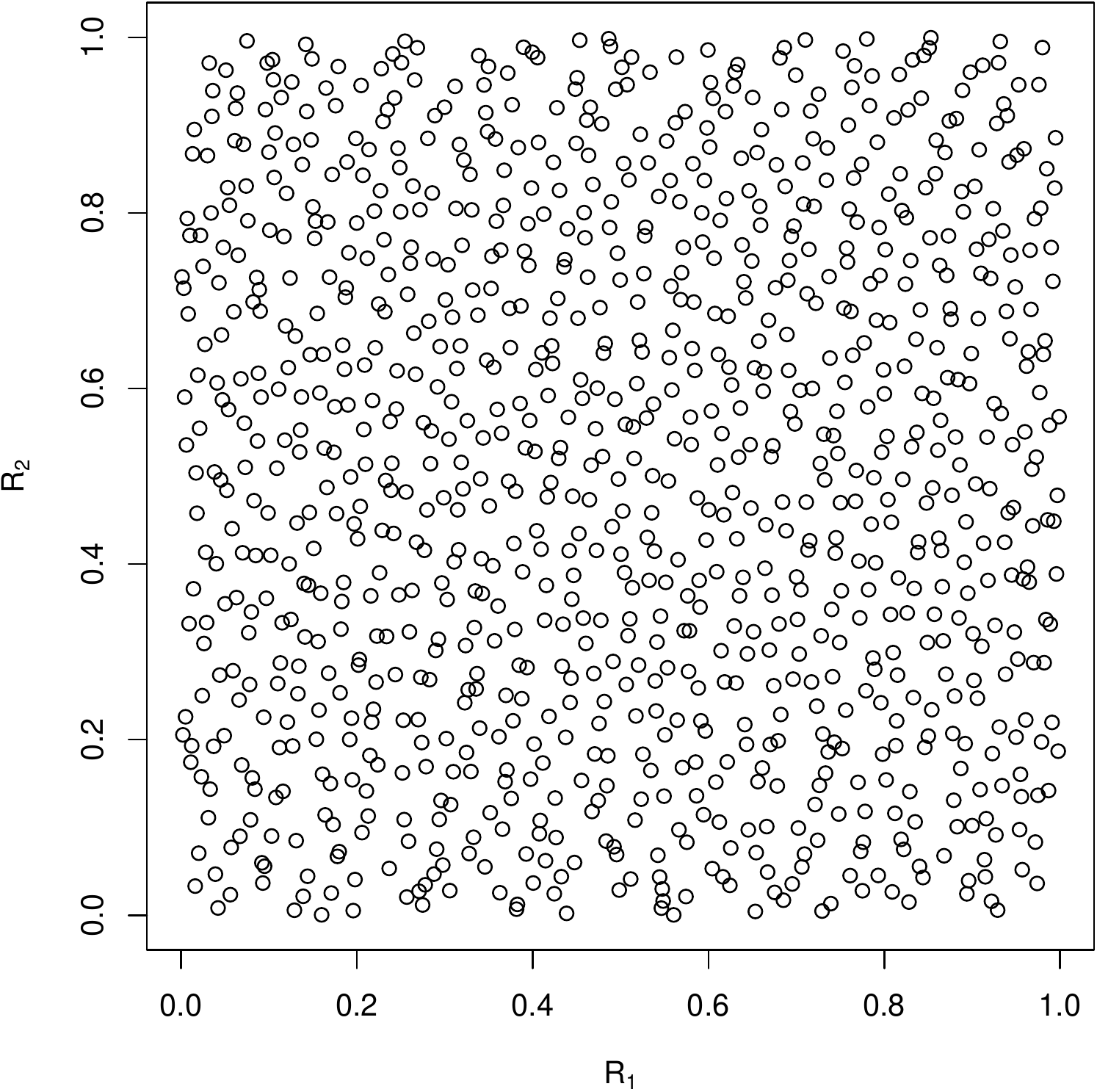}\hfill
  \includegraphics[width=0.32\textwidth]{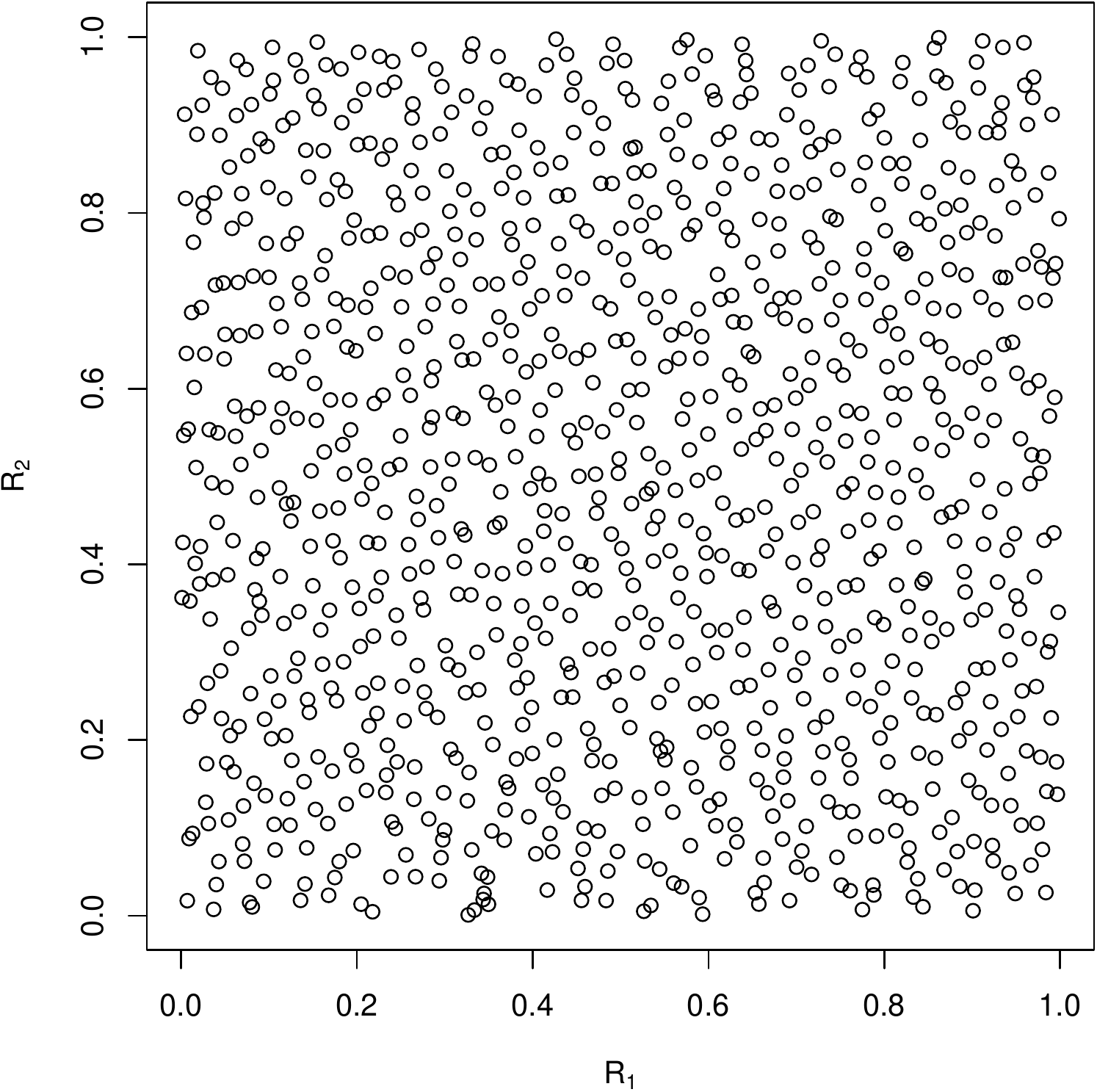}\hfill
  \includegraphics[width=0.32\textwidth]{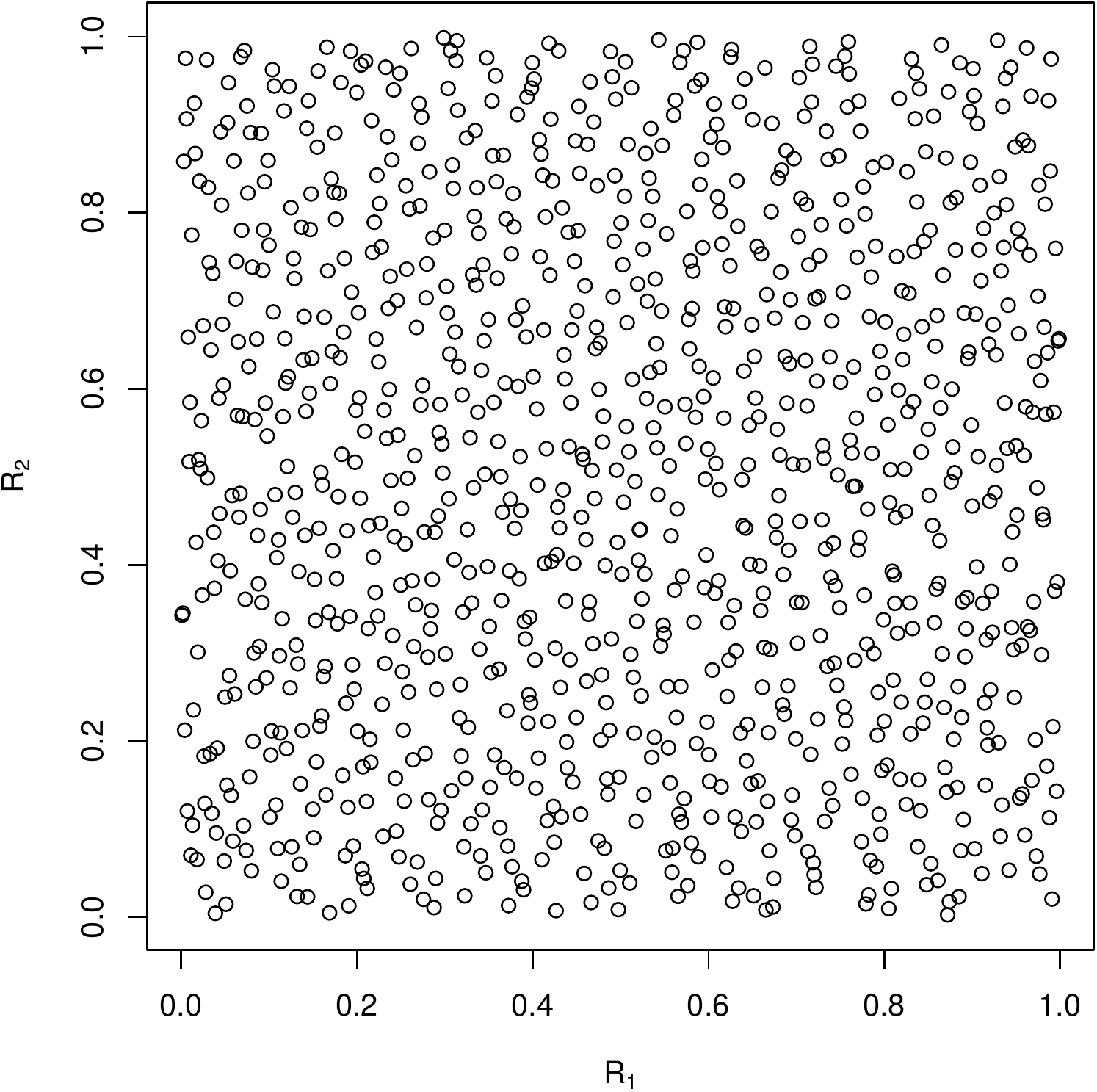}
  \caption{Top row contains contour plots of true $t_4$ copulas with $\tau=0.25$ (left), $0.50$ (middle) and $0.75$ (right) along with the corresponding contour plots of empirical copulas based on GMMN pseudo-random and GMMN quasi-random samples (respectively, GMMN PRS and GMMN QRS), both of size $\ngen=1000$. Bottom row contains Rosenblatt-transformed GMMN QRS corresponding to the same three $t_4$ copulas.
  }\label{fig:t:example}
\end{figure}

\begin{figure}[htbp]
	\centering
	\includegraphics[width=0.32\textwidth]{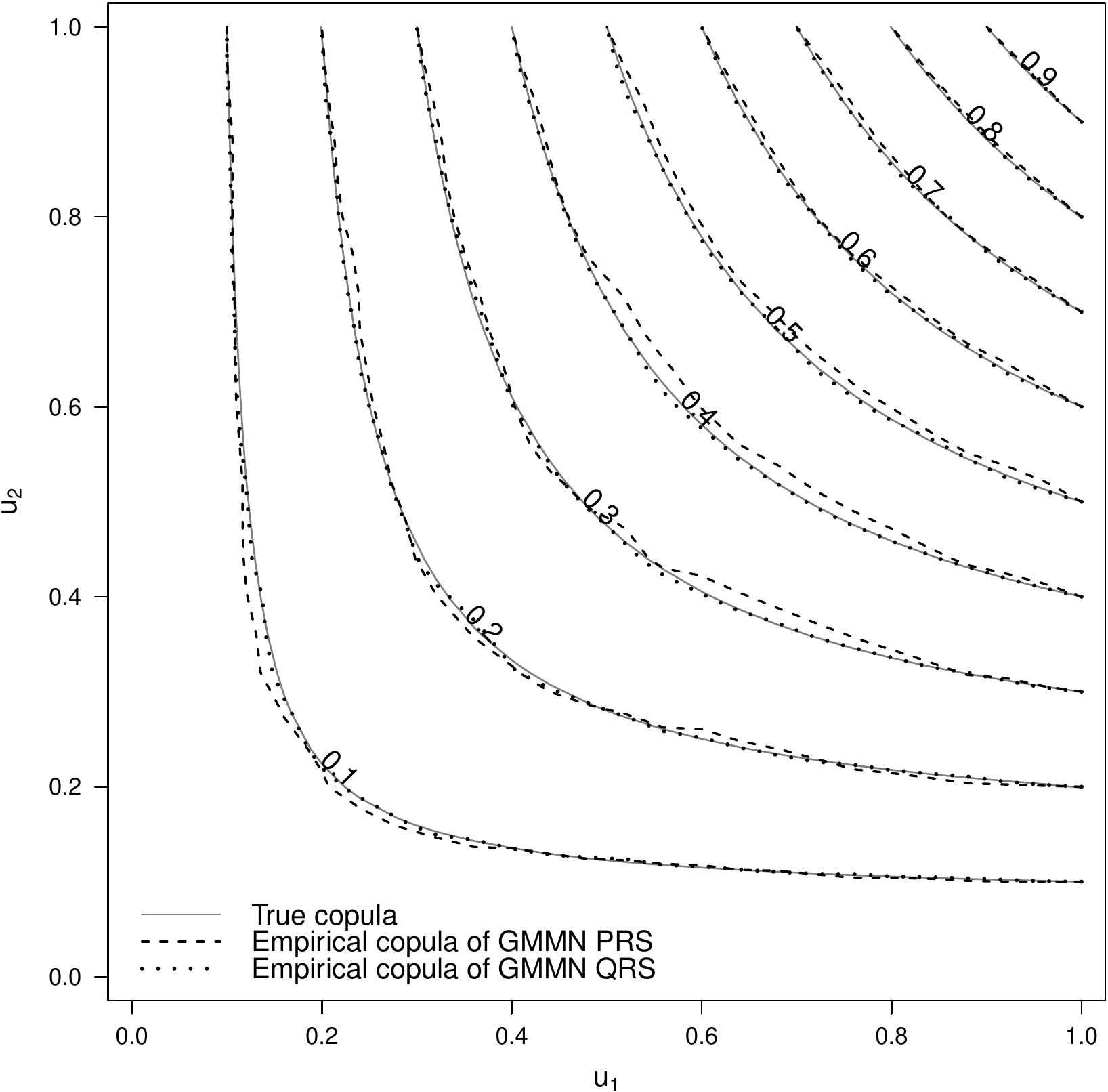}\hfill
	\includegraphics[width=0.32\textwidth]{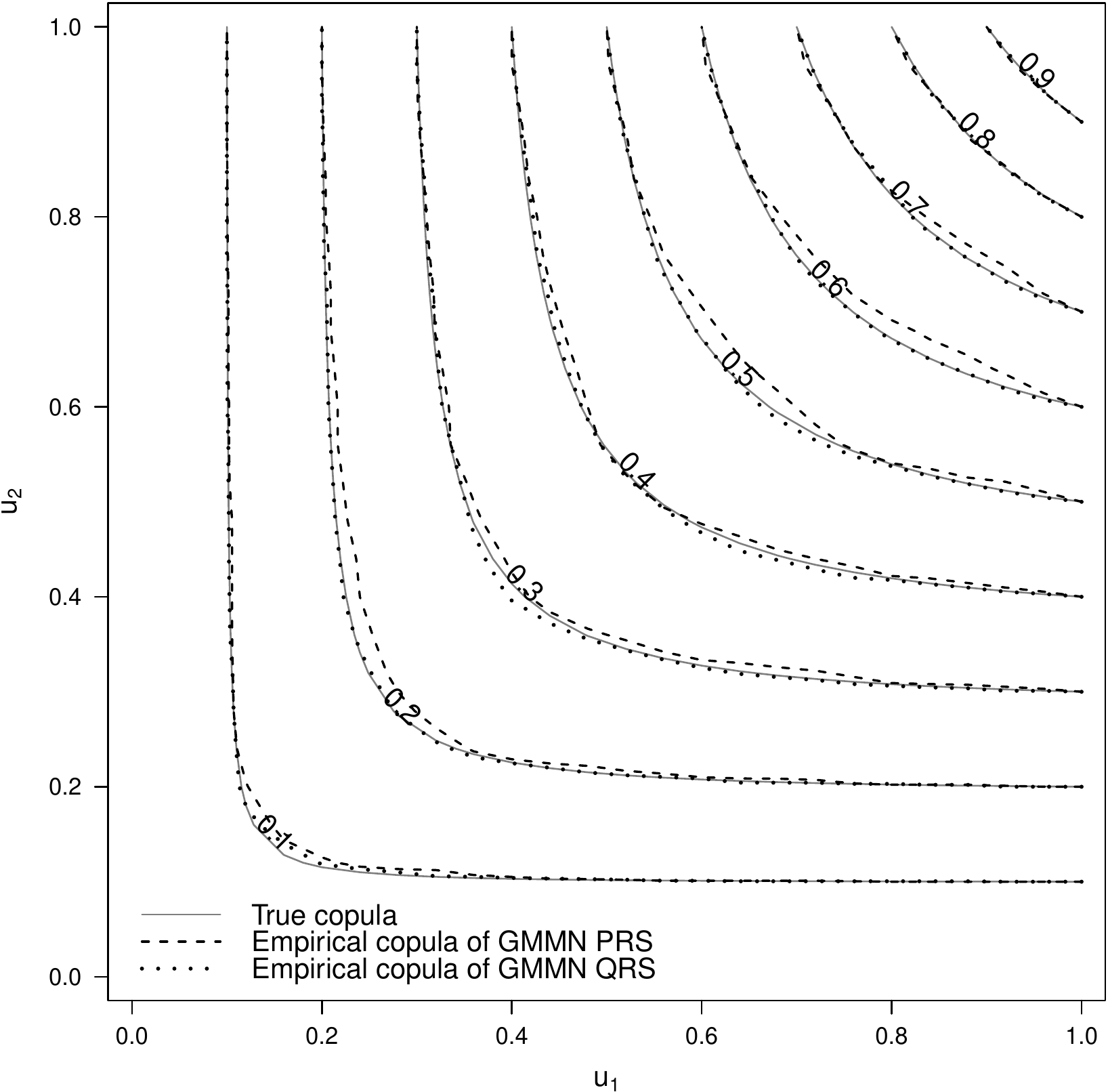}\hfill
	\includegraphics[width=0.32\textwidth]{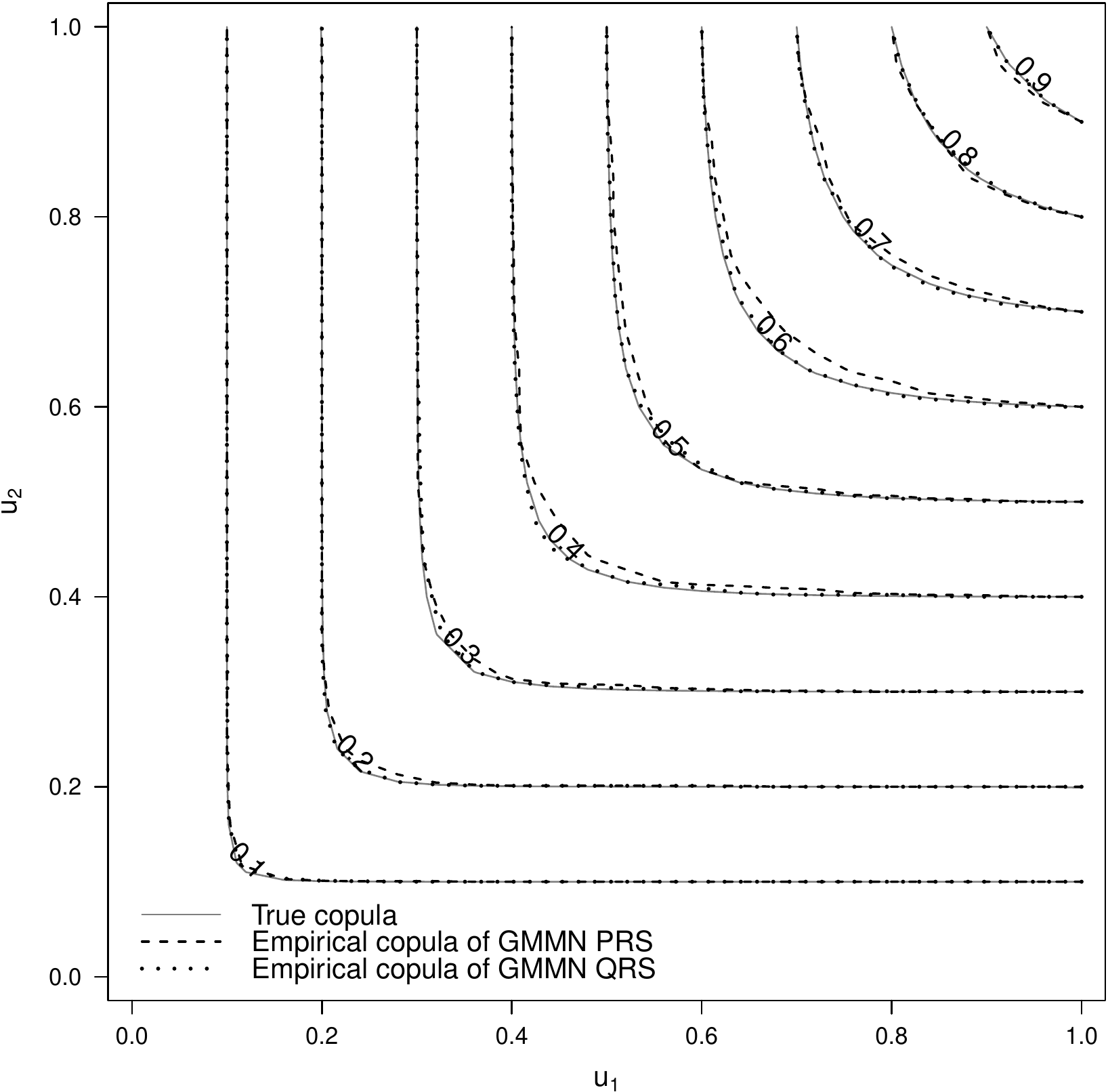}
	\includegraphics[width=0.32\textwidth]{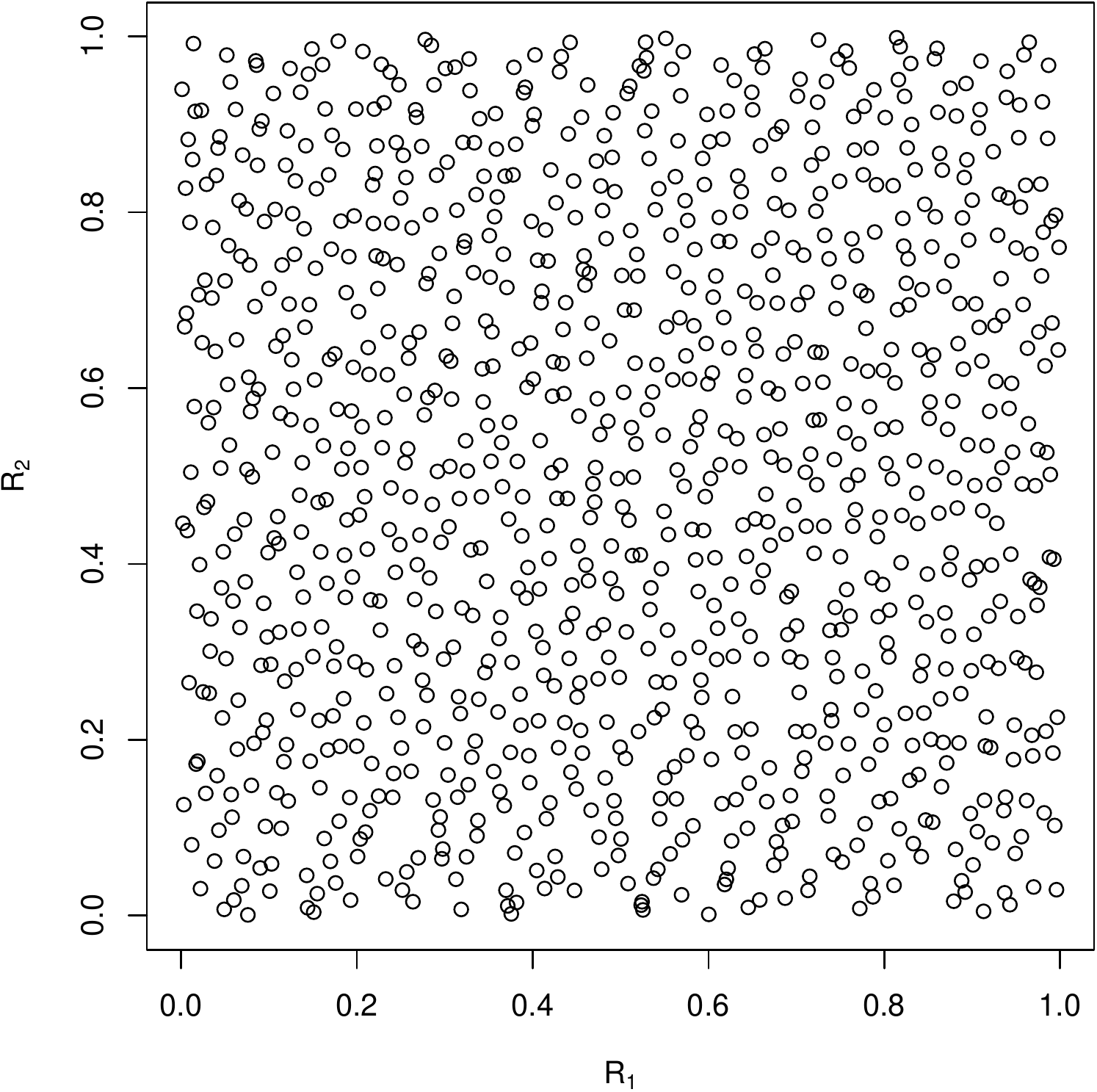}\hfill
        \includegraphics[width=0.32\textwidth]{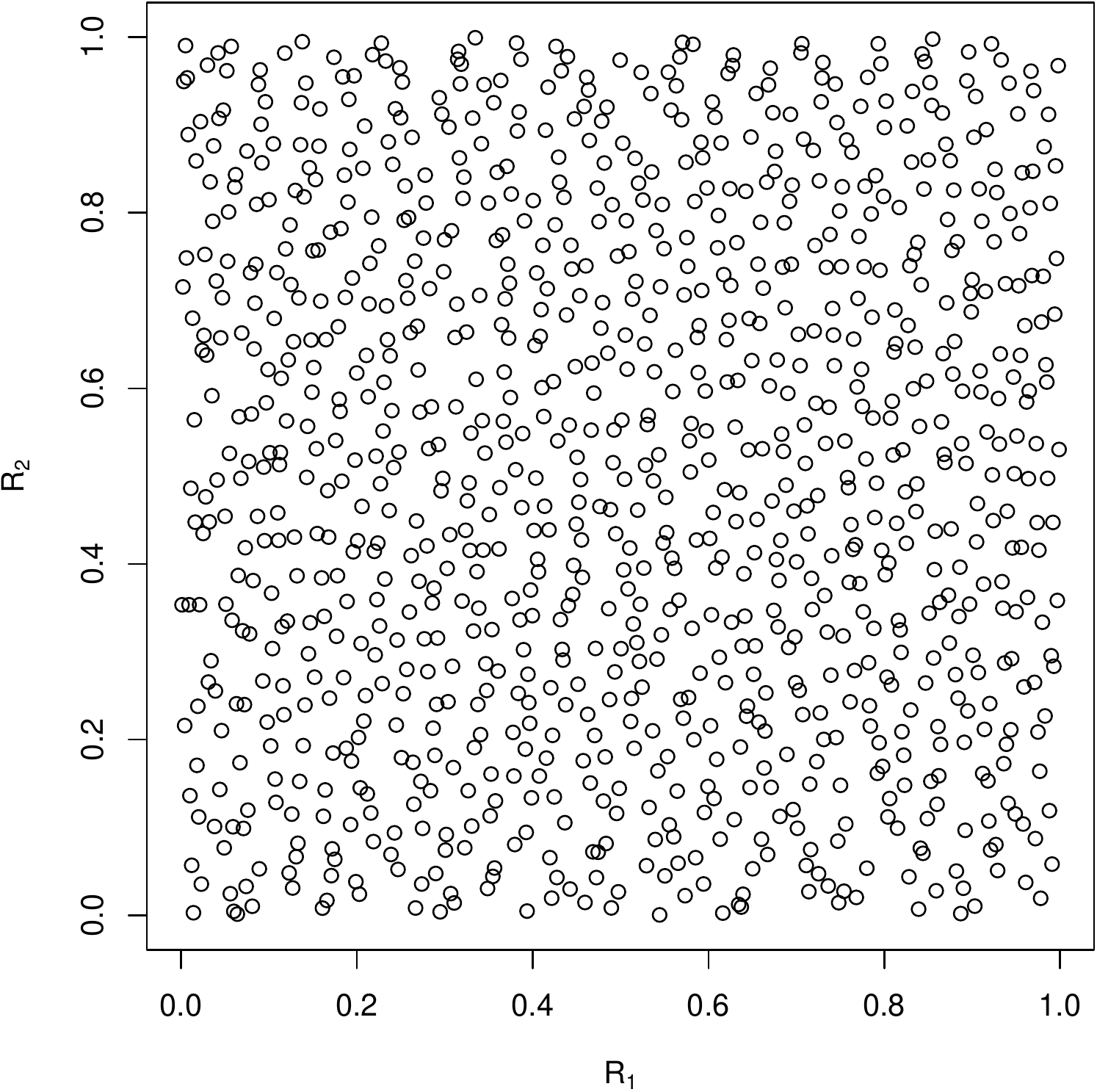}\hfill
        \includegraphics[width=0.32\textwidth]{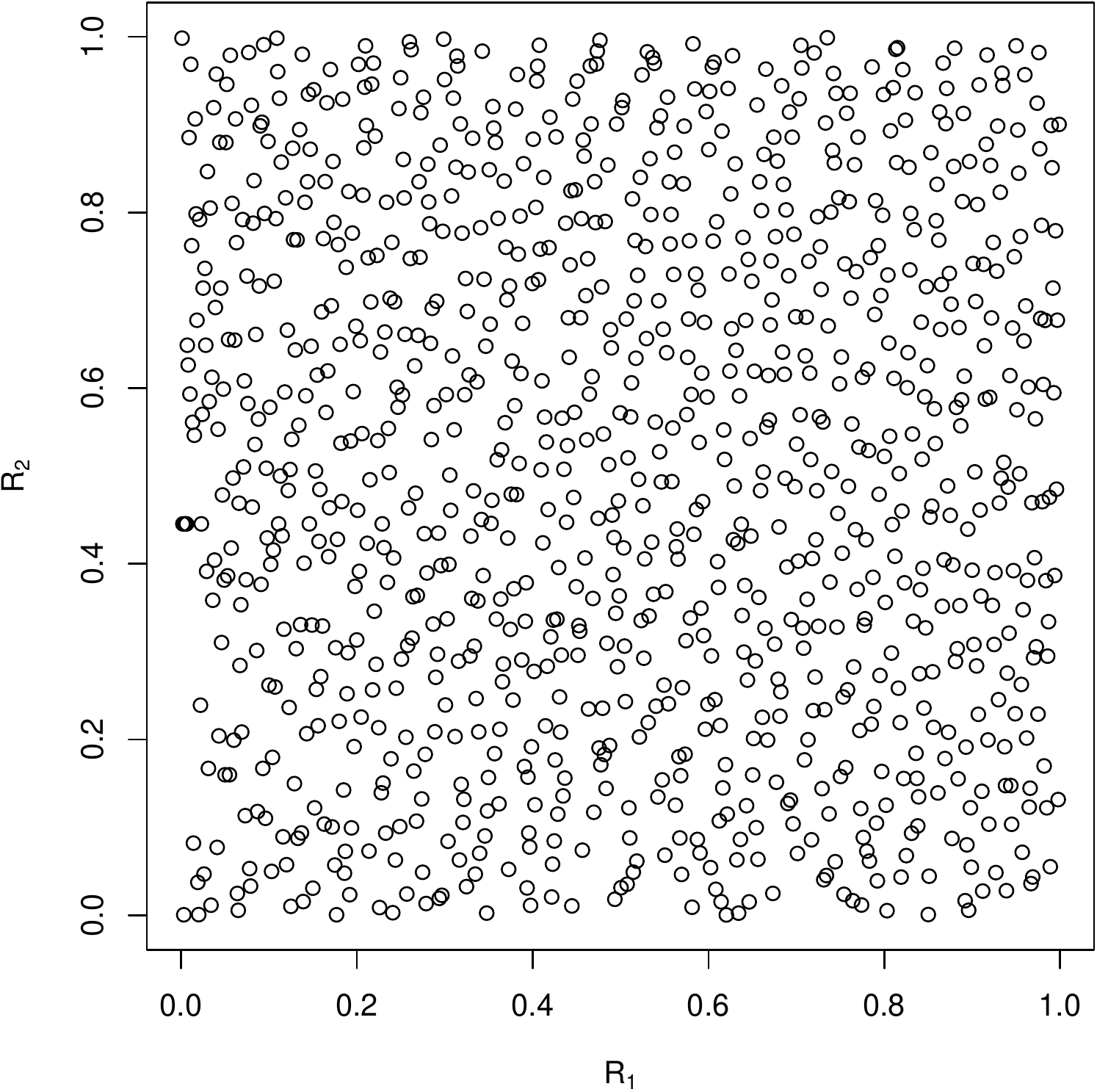}
	\caption{Top row contains contour plots of true Clayton copulas with $\tau=0.25$ (left), $0.50$ (middle) and $0.75$ (right) along with the corresponding contour plots of empirical copulas based on GMMN PRS and GMMN QRS, both of size $\ngen=1000$. Bottom row contains Rosenblatt-transformed GMMN QRS corresponding to the same three Clayton copulas.
	}\label{fig:clayton:example}
\end{figure}

\begin{figure}[htbp]
	\centering
	\includegraphics[width=0.32\textwidth]{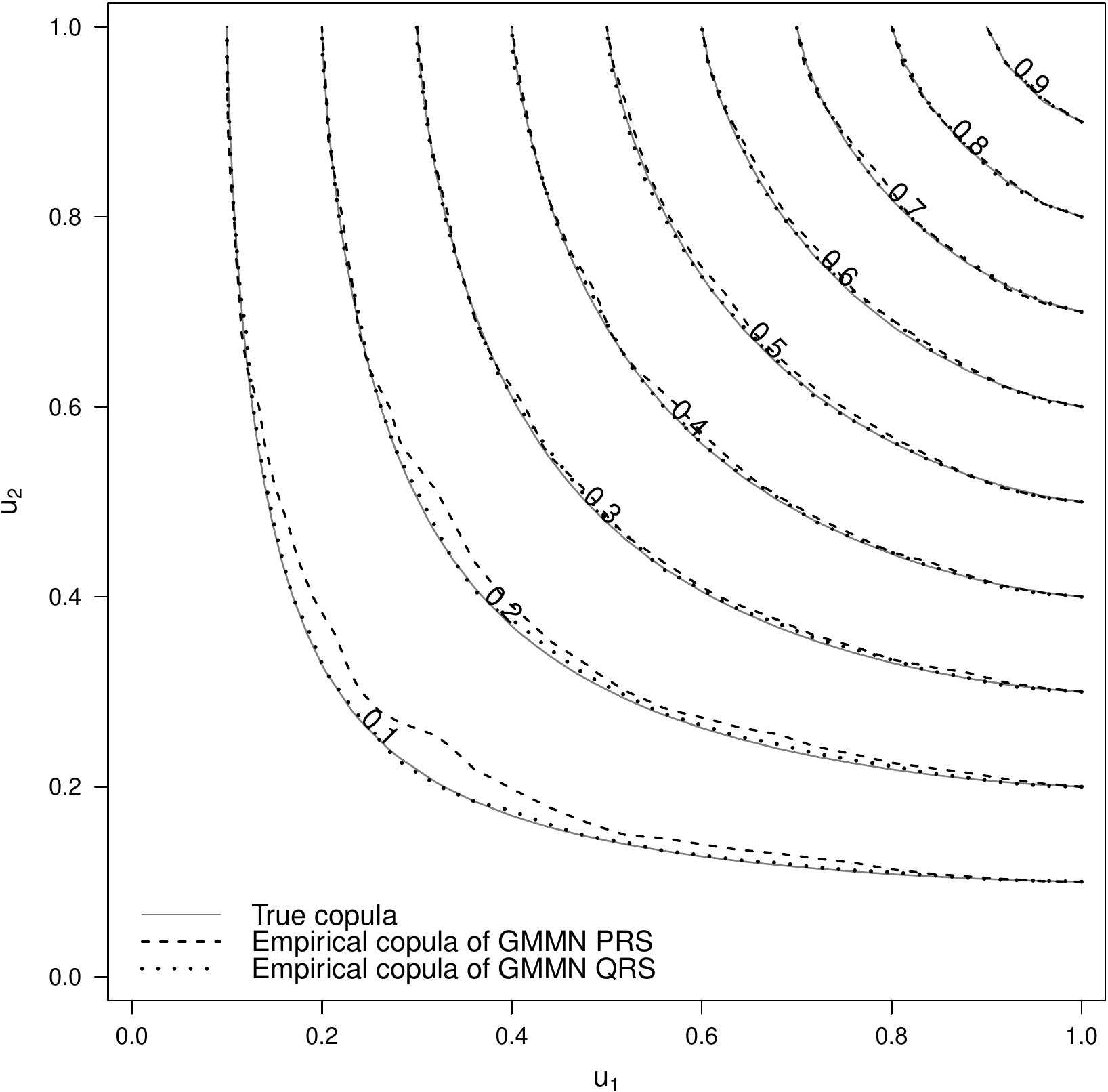}\hfill
	\includegraphics[width=0.32\textwidth]{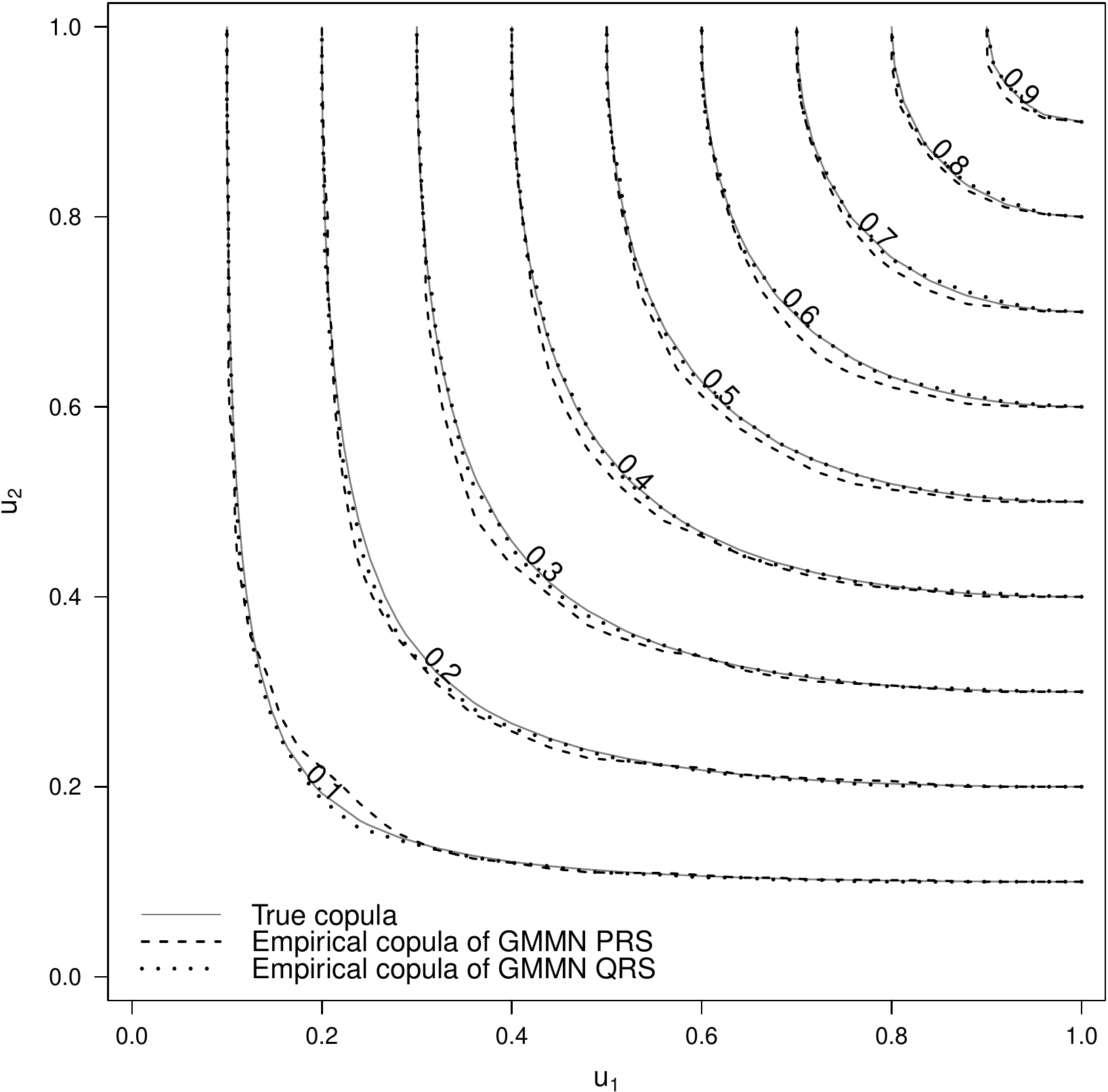}\hfill
	\includegraphics[width=0.32\textwidth]{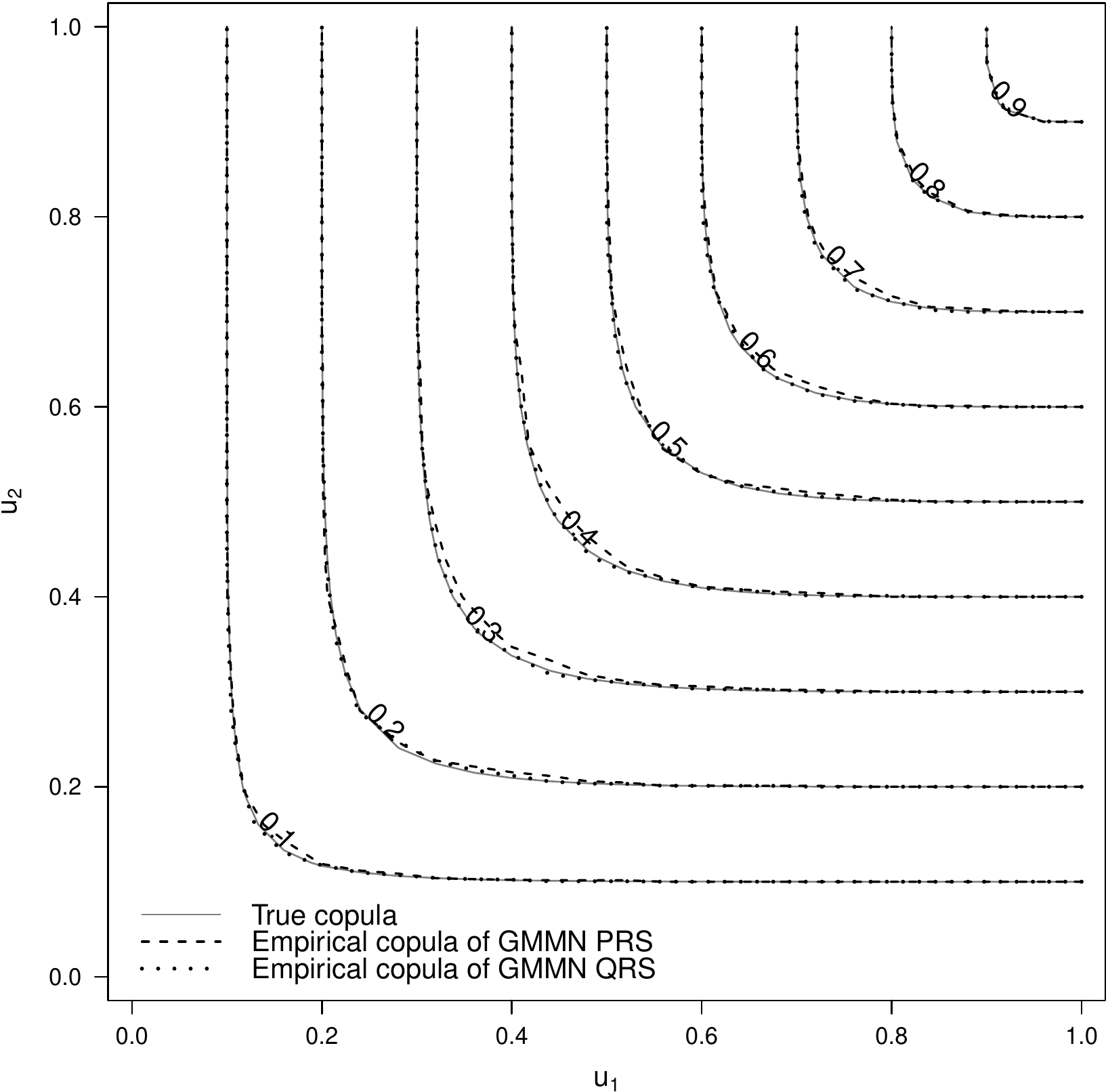}
	\includegraphics[width=0.32\textwidth]{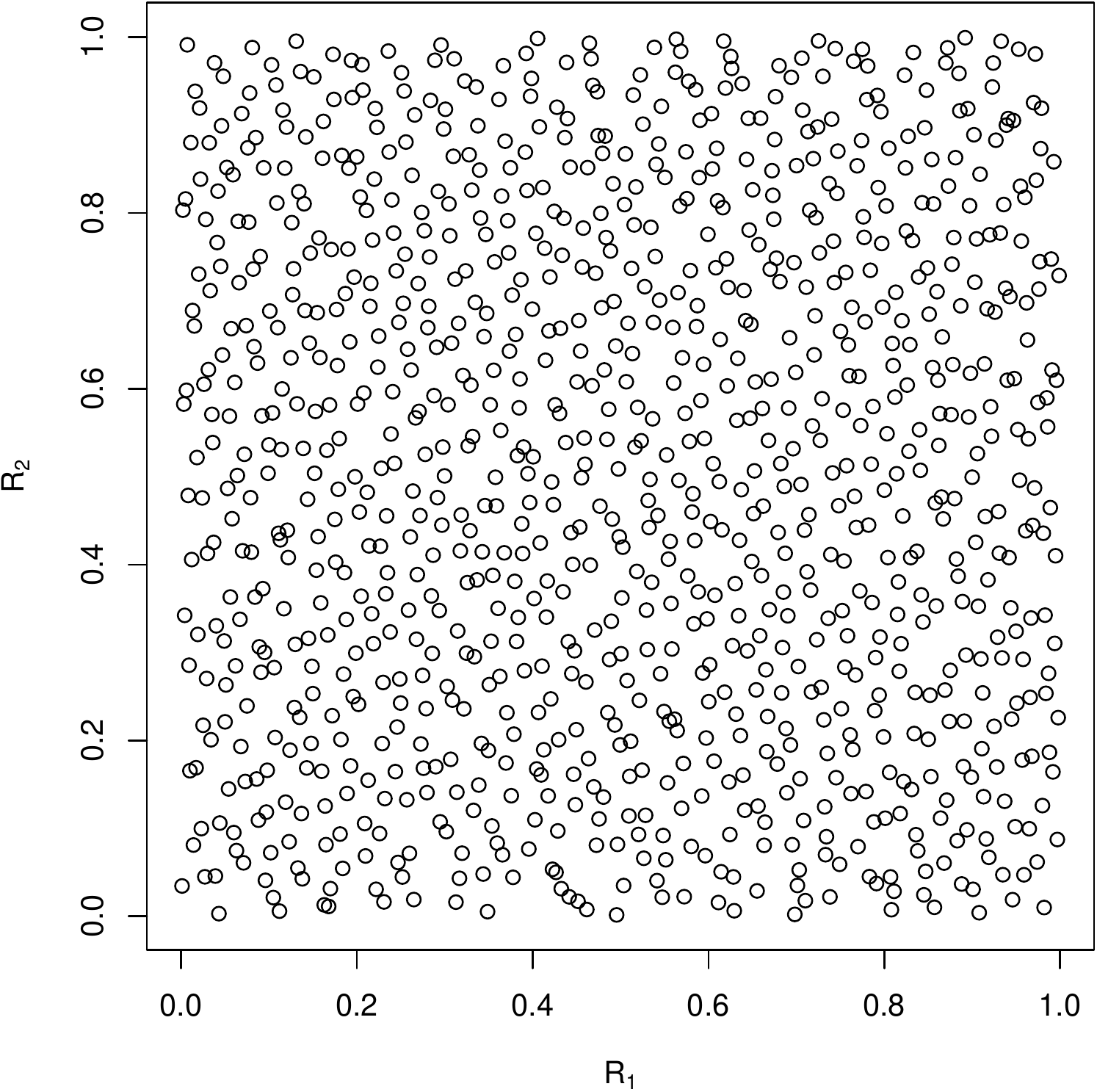}\hfill
        \includegraphics[width=0.32\textwidth]{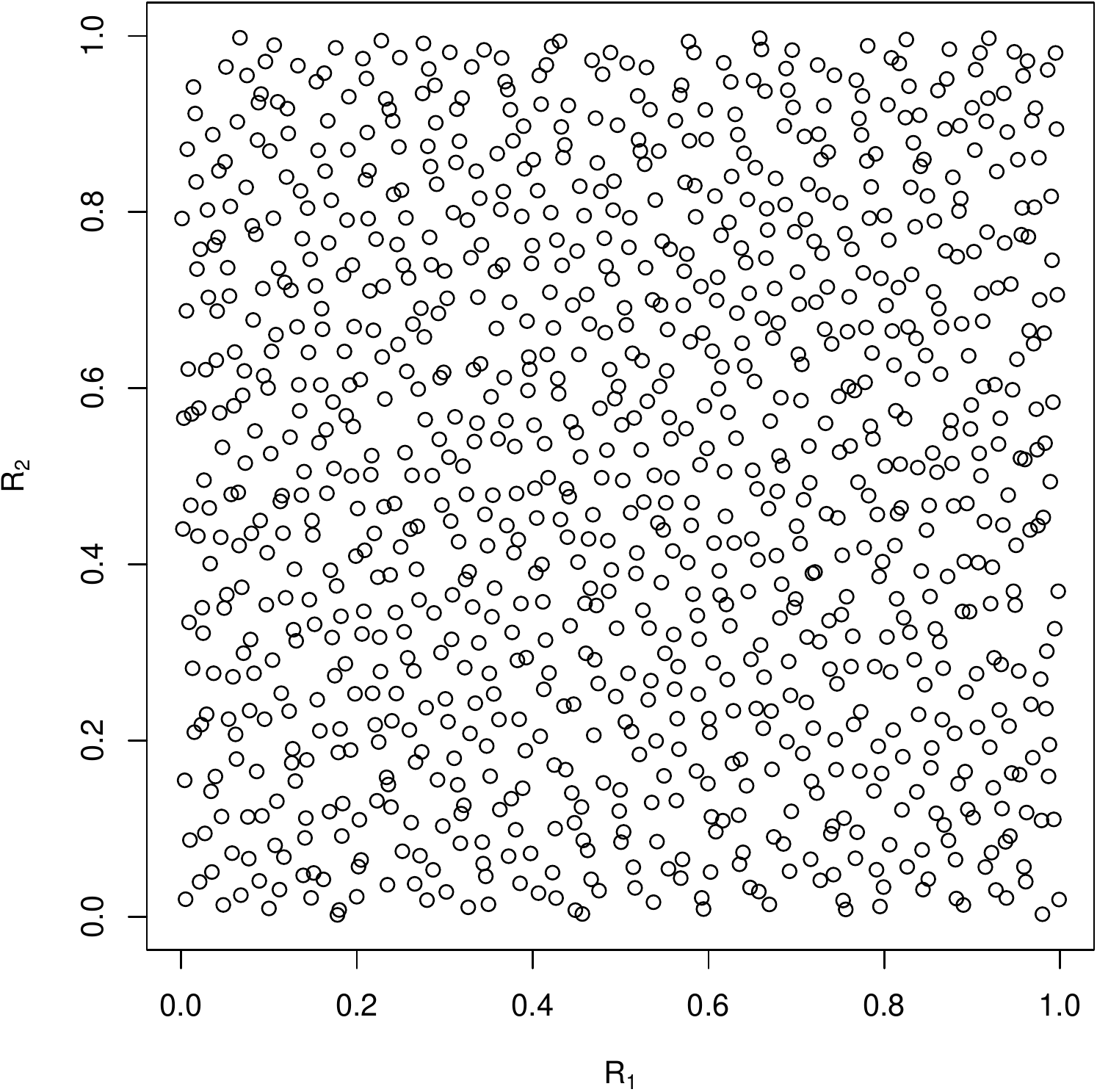}\hfill
        \includegraphics[width=0.32\textwidth]{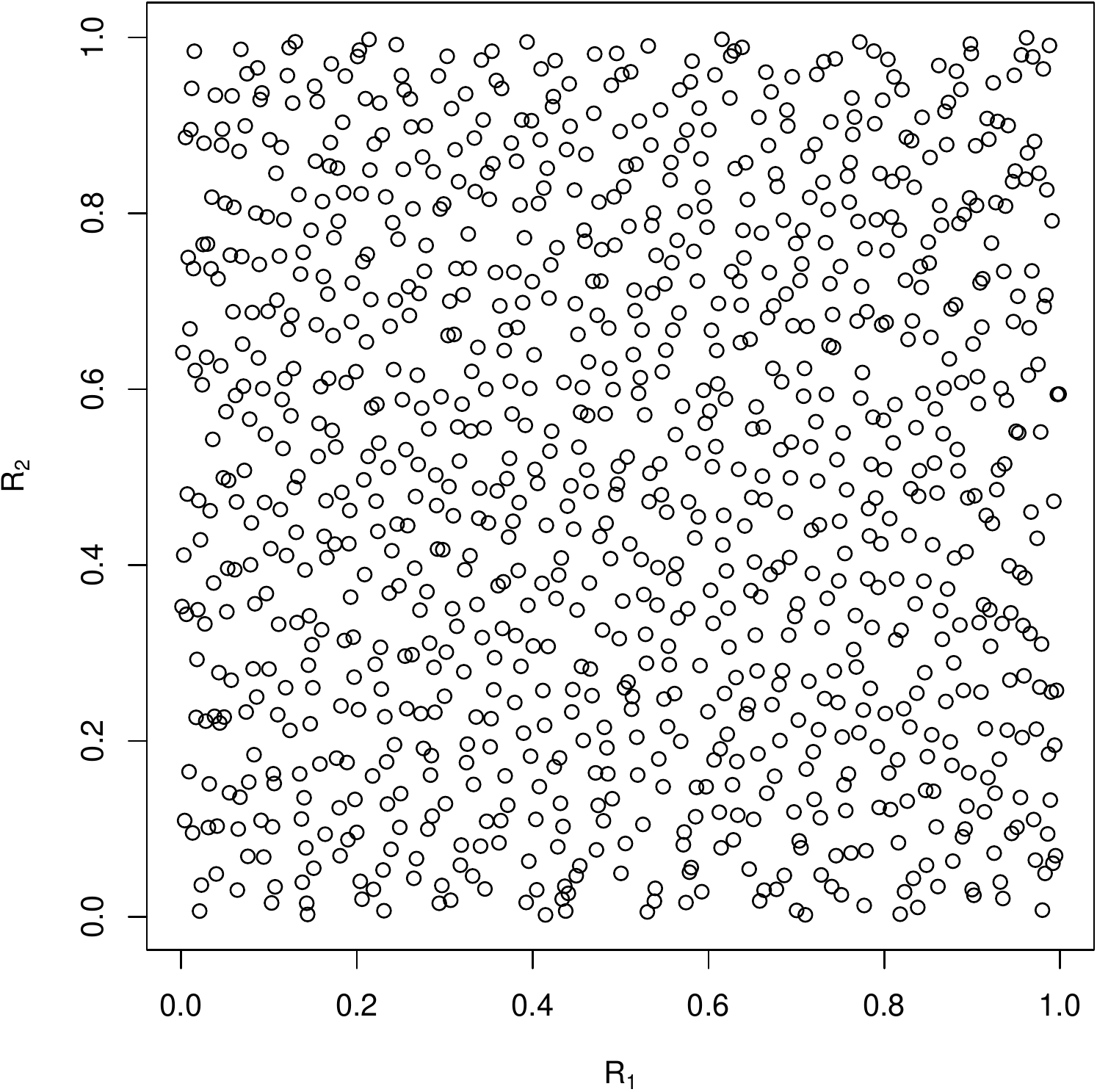}
	\caption{Top row contains contour plots of true Gumbel copulas with $\tau=0.25$ (left), $0.50$ (middle) and $0.75$ (right) along with the corresponding contour plots of empirical copulas based on GMMN PRS and GMMN QRS, both of size $\ngen=1000$. Bottom row contains Rosenblatt-transformed GMMN QRS corresponding to the same three Gumbel copulas.
	}\label{fig:gumbel:example}
\end{figure}

\begin{figure}[htbp]
  \centering
  \includegraphics[width=0.32\textwidth]{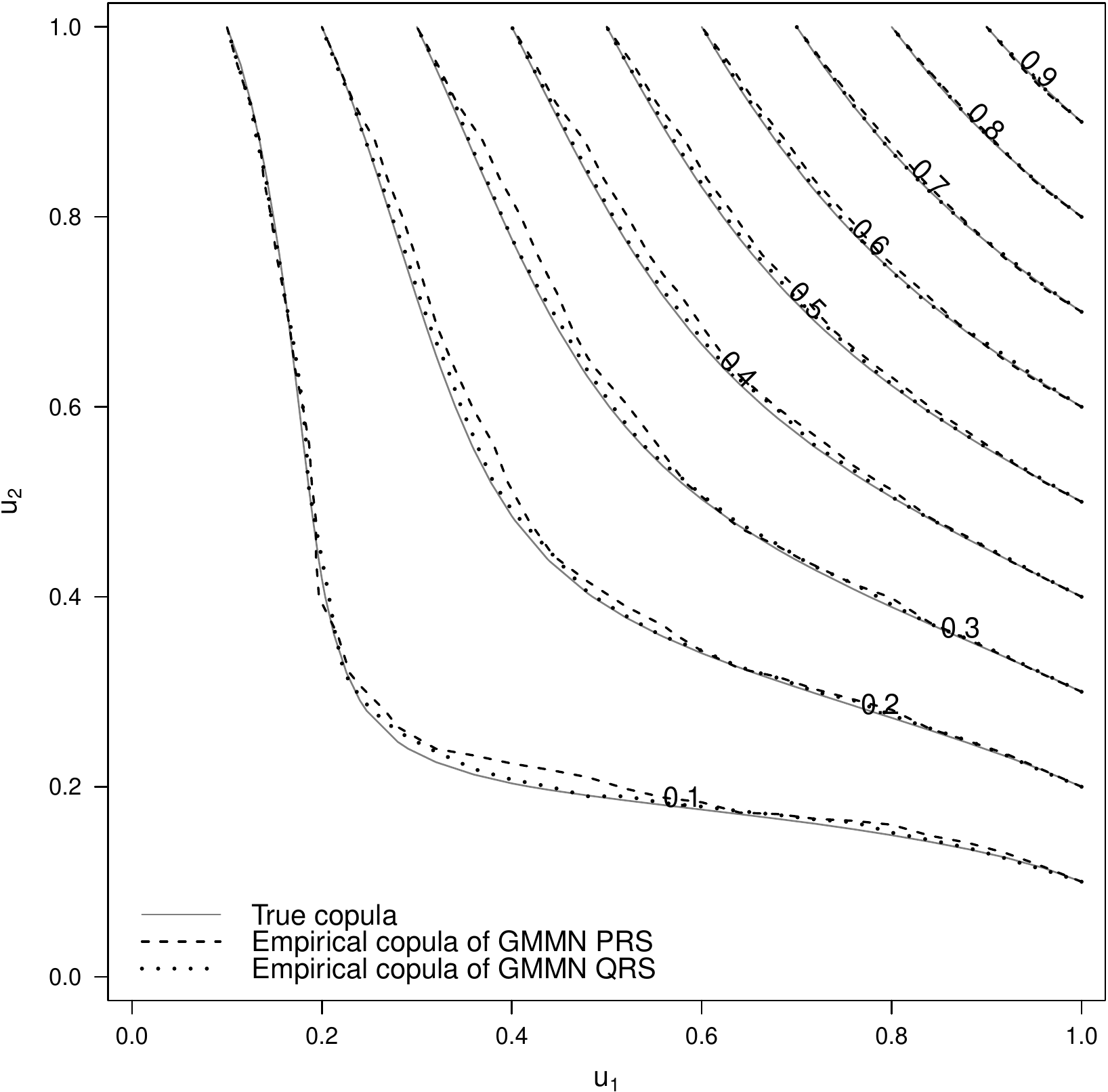}\qquad
  \includegraphics[width=0.32\textwidth]{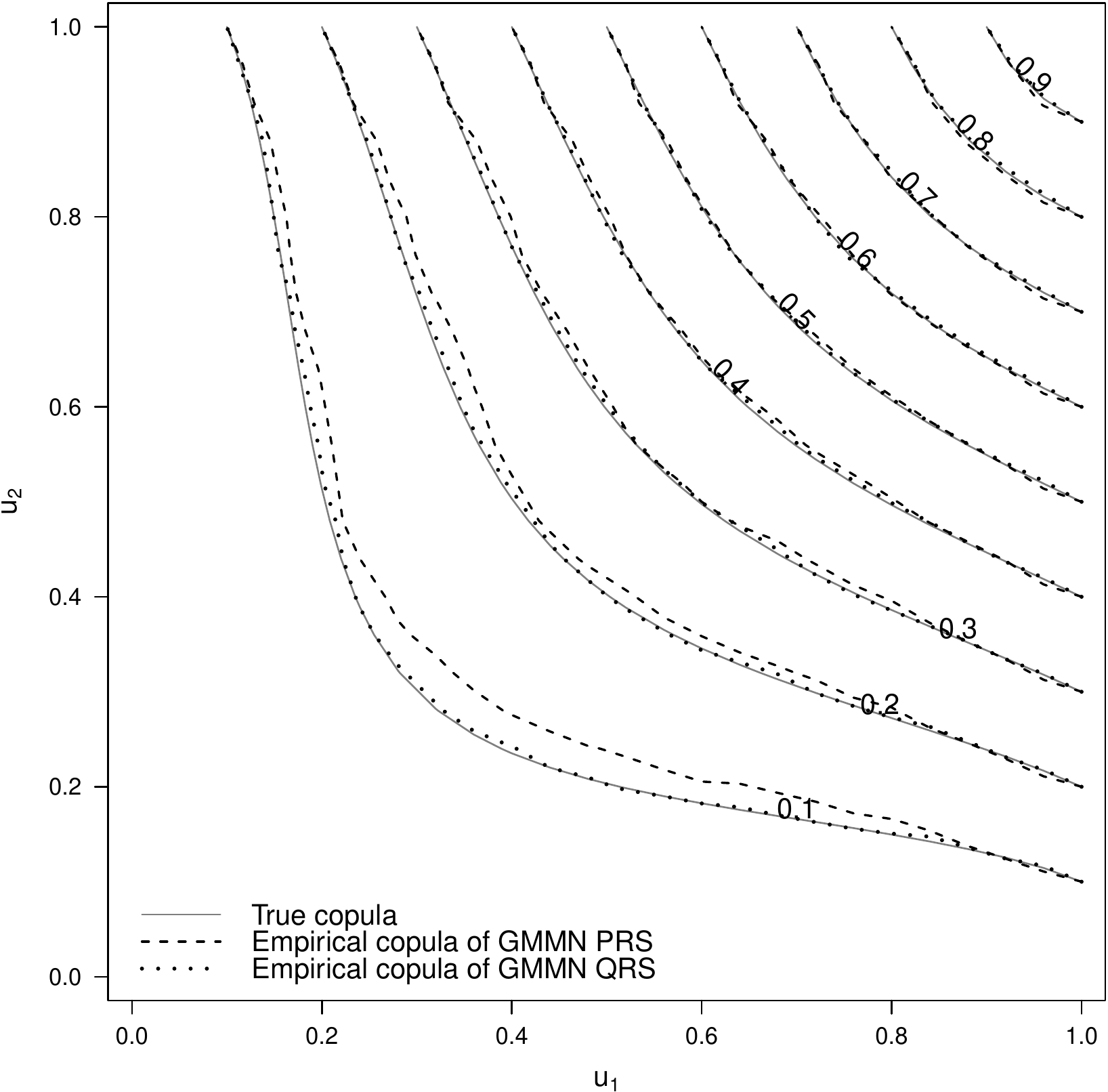}\\
  \includegraphics[width=0.32\textwidth]{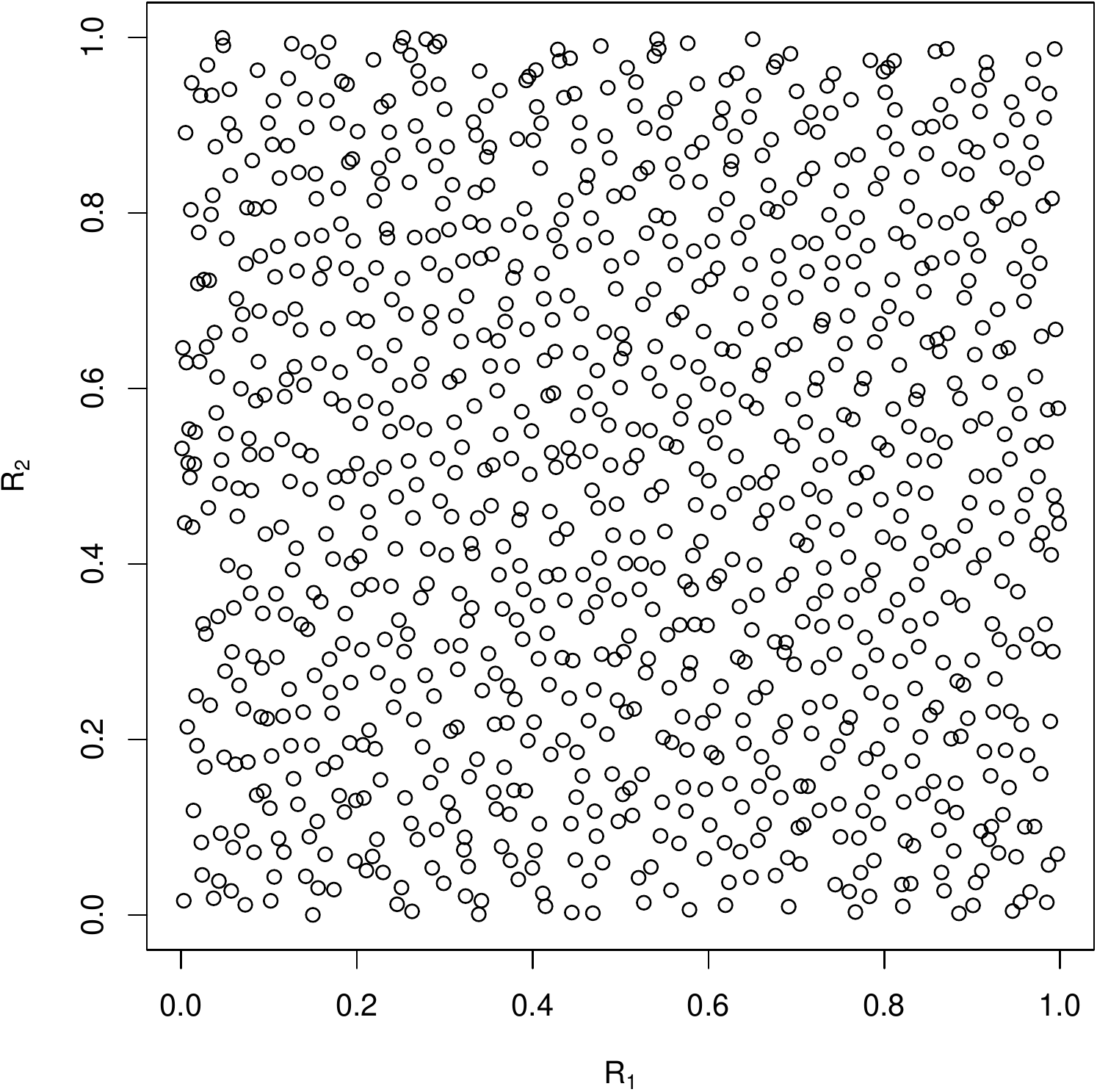}\qquad
  \includegraphics[width=0.32\textwidth]{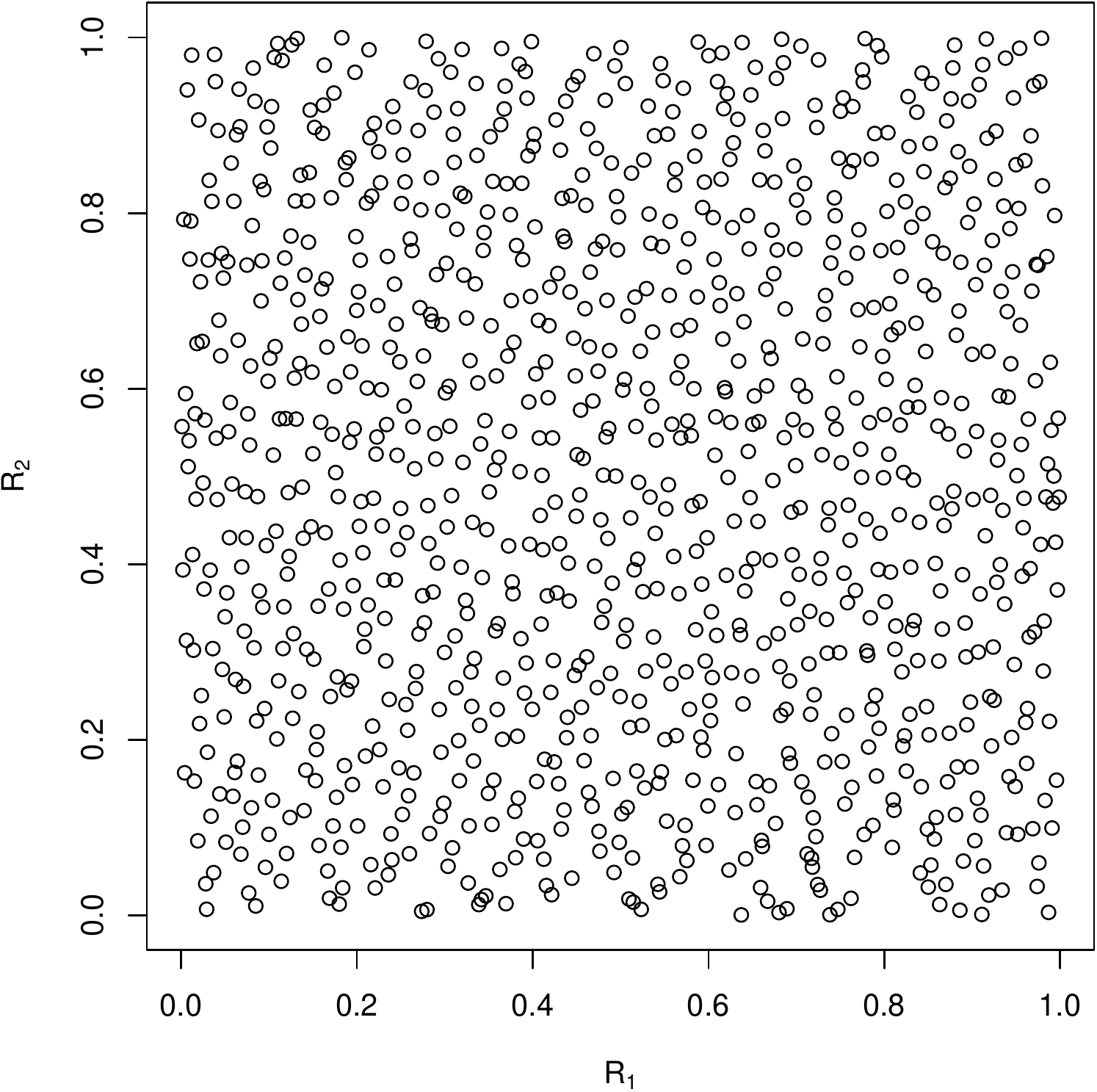}%
  \caption{Top row contains contour plots of true Clayton-$t(90)$ (left) and Gumbel-$t(90)$ (right) mixture copulas along with the corresponding contour plots of empirical copulas based on GMMN PRS and GMMN QRS, both of size $\ngen=1000$. Bottom row contains Rosenblatt-transformed GMMN QRS corresponding to the same two mixture copulas.
  	 }\label{fig:mixtures:example}
\end{figure}

\subsubsection{Nested Archimedean, Marshall--Olkin and mixture copulas}
Next, we consider more complex copulas such as nested Archimedean copulas and
Marshall--Olkin copulas. We also re-consider the two mixture copulas introduced
in the previous section along with an additional mixture copula. To better
showcase the complexity of these dependence structures, we use scatter plots
instead of contour plots to display copula and GMMN-generated samples. We omit
the plots containing the Rosenblatt transformed samples since they are harder to
obtain for the copulas we investigate in this section.

Nested Archimedean copulas (NACs) are Archimedean copulas with arguments
possibly replaced by other NACs; see \cite{mcneil2008} or \cite{hofert2012b}. In
particular, this class of copulas allows us to construct asymmetric extensions of
Archimedean copulas. Important to note here is that NACs are copulas for which
there is no known (tractable) CDM. To demonstrate the ability of GMMNs to capture such
dependence structures, we consider the simplest three-dimensional copula for
visualization and investigate higher-dimensional NACs in
Sections~\ref{sec:GMMN:accuracy} and \ref{sec:conv:analysis}. The
three-dimensional NAC we consider here is
\begin{align}
  C(\bm{u})=C_0(C_1(u_1,u_2),u_3), \quad \bm{u}\in [0,1]^3,\label{eq:NAC}
\end{align}
where $C_0$ is a Clayton copula with Kendall's tau $\tau_{0}=0.25$ and $C_1$ is a Clayton
copula with Kendall's tau $\tau_{1}=0.50$. In Sections~\ref{sec:GMMN:accuracy} and
\ref{sec:conv:analysis}, we will present examples of five- and ten-dimensional
NACs.

Bivariate Marshall--Olkin copulas are of the form
\begin{align*}
C(u_1,u_2)=\min\{u_1^{1-\alpha_1}u_2,u_1u_2^{1-\alpha_2}\},\quad u_1,u_2\in[0,1],
\end{align*}
where $\alpha_1, \alpha_2 \in [0,1]$. A notable feature of Marshall--Olkin
copulas is that they have both an absolutely continuous component and a singular
component. In particular, the singular component is determined by all points
which satisfy $u_1^{\alpha_1}=u_2^{\alpha_2}$. Accurately capturing this
singular component may present a different challenge for GMMNs, which is why we
included this copula despite the fact that there also exists
a CDM for this copula; see \cite{cambou2017}. As an example for visual
assessment, we consider a Marshall--Olkin copula with $\alpha_1=0.75$ and
$\alpha_2=0.60$.

We also consider three mixture models, all of which are equally weighted two-component mixture copulas with one component being a 90-degree-rotated $t_4$ copula with $\tau=0.50$. The first two models are the Clayton-$t(90)$ and Gumbel-$t(90)$ mixture copulas as previously introduced. The second component in the third model is a Marshall--Olkin copula with parameters $\alpha_1=0.75$ and $\alpha_2=0.60$. We refer to this third model as the MO-$t(90)$ copula.

Figures~\ref{fig:NAC:claytonexample}--\ref{fig:mixtures:example2} display pseudo-random samples (left column) from a $(2,1)$-nested Clayton copula as in~\eqref{eq:NAC}, a MO copula, and the three mixture copulas, respectively, along with GMMN pseudo-random samples (middle column) and GMMN quasi-random samples (right column) corresponding to each copula $C$. The similarity between the GMMN pseudo-random samples in the middle column and the pseudo-random samples in the left column indicate that the copulas $C$ were learned sufficiently well by their corresponding GMMNs. Note that in the case of the nested Clayton copula example, we can only comment on how well the bivariate margins of the copula $C$ were learned. From the right columns, we can mainly observe that the GMMN quasi-random samples contain less gaps and clusters when compared with the corresponding pseudo-random and GMMN pseudo-random samples. The fact that GMMNs were capable of learning the main features, including the singular components, of the MO copula and the MO-$t(90)$ mixture copula is particularly noteworthy given how challenging it seems to be to learn a Lebesgue null set from a finite amount of samples.

\begin{figure}[htbp]
  \centering
  \includegraphics[width=0.32\textwidth]{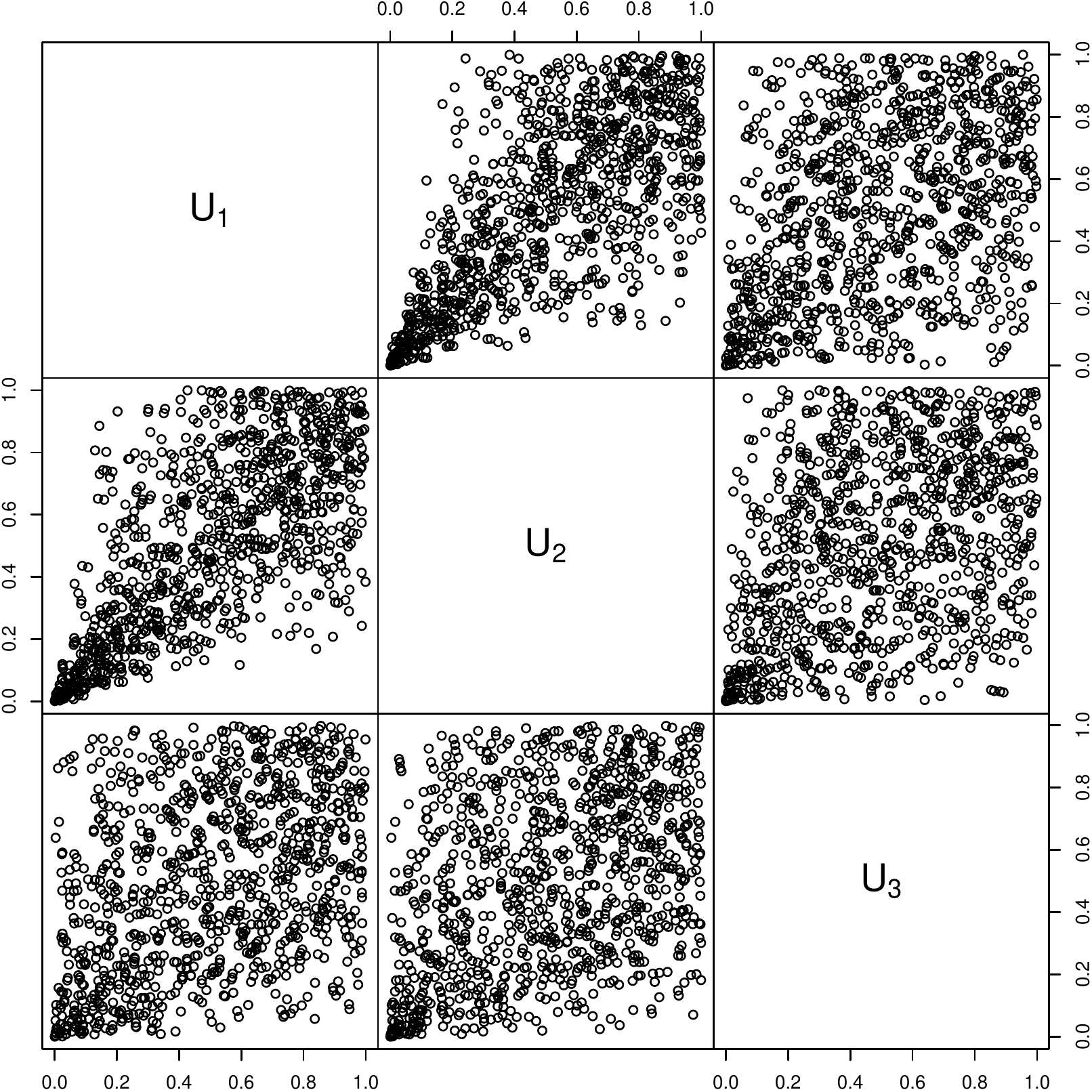}\hfill
  \includegraphics[width=0.32\textwidth]{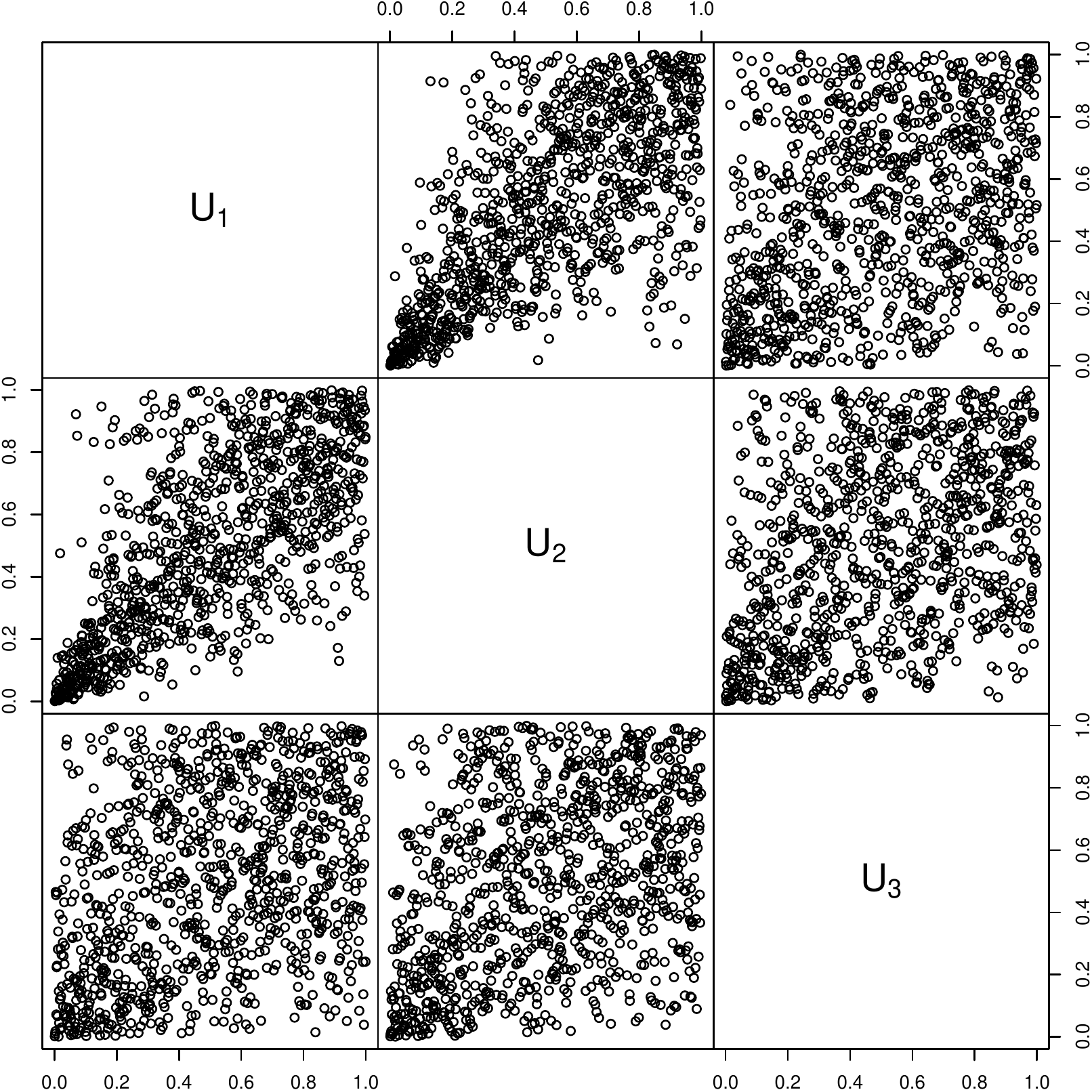}\hfill
  \includegraphics[width=0.32\textwidth]{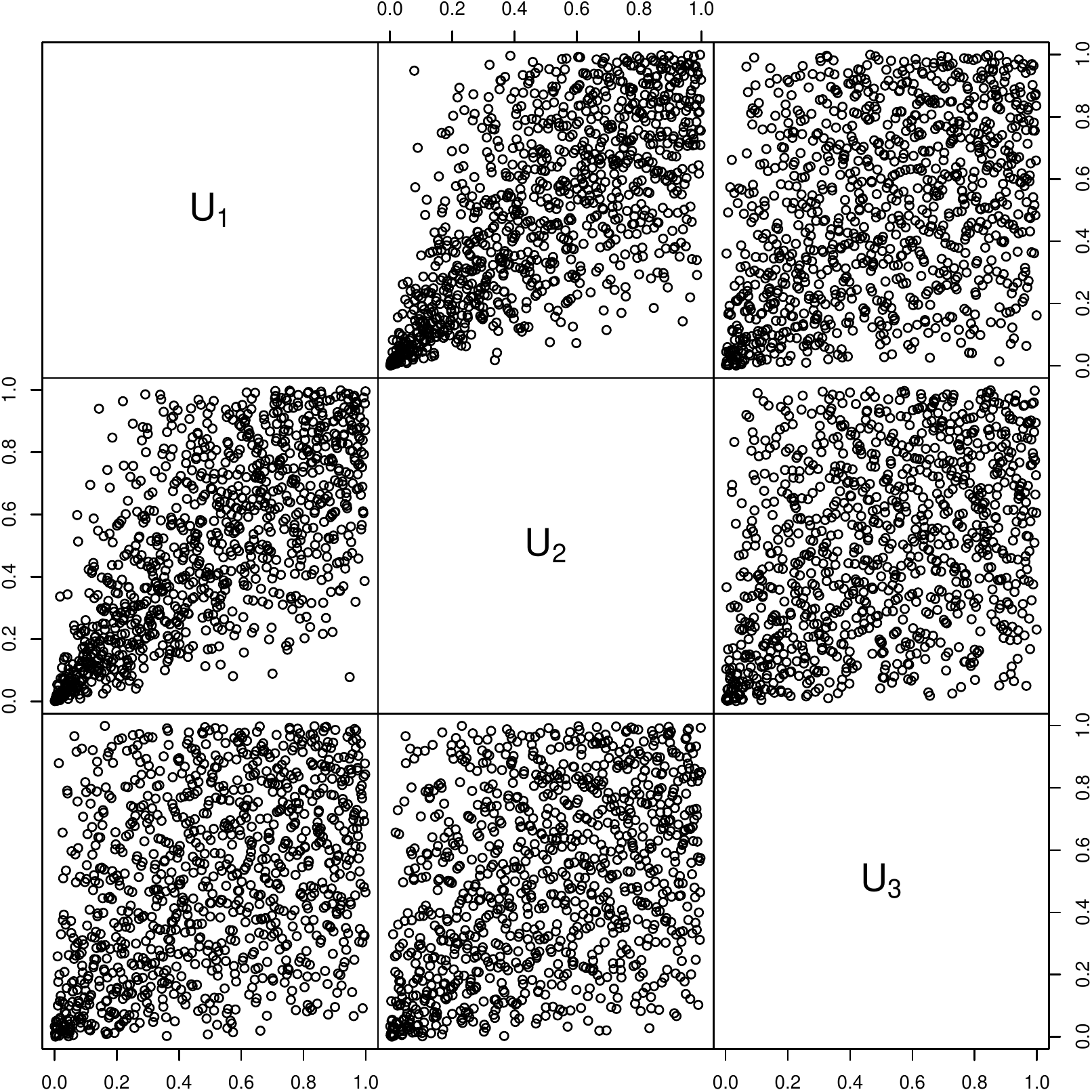}
  \caption{Pseudo-random samples (PRS; left), GMMN pseudo-random samples (GMMN PRS; middle) and
    GMMN quasi-random samples (GMMN QRS; right), all of size $\ngen=1000$, from a (2,1)-nested
    Clayton copula as in~\eqref{eq:NAC} with $\tau_{0}=0.25$ and $\tau_{1}=0.50$.}\label{fig:NAC:claytonexample}
\end{figure}

\begin{figure}[htbp]
  \centering
  \includegraphics[width=0.32\textwidth]{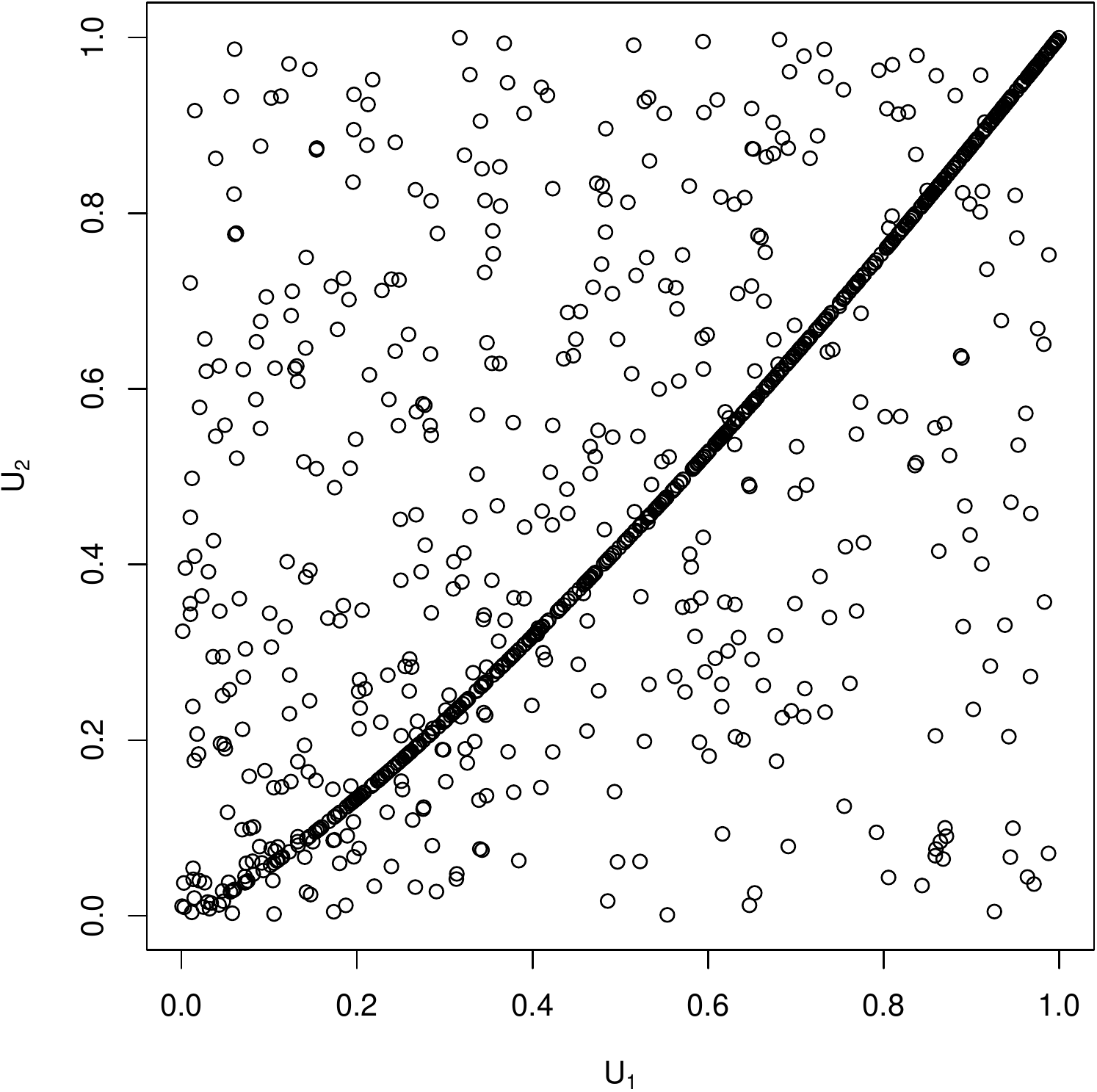}\hfill
  \includegraphics[width=0.32\textwidth]{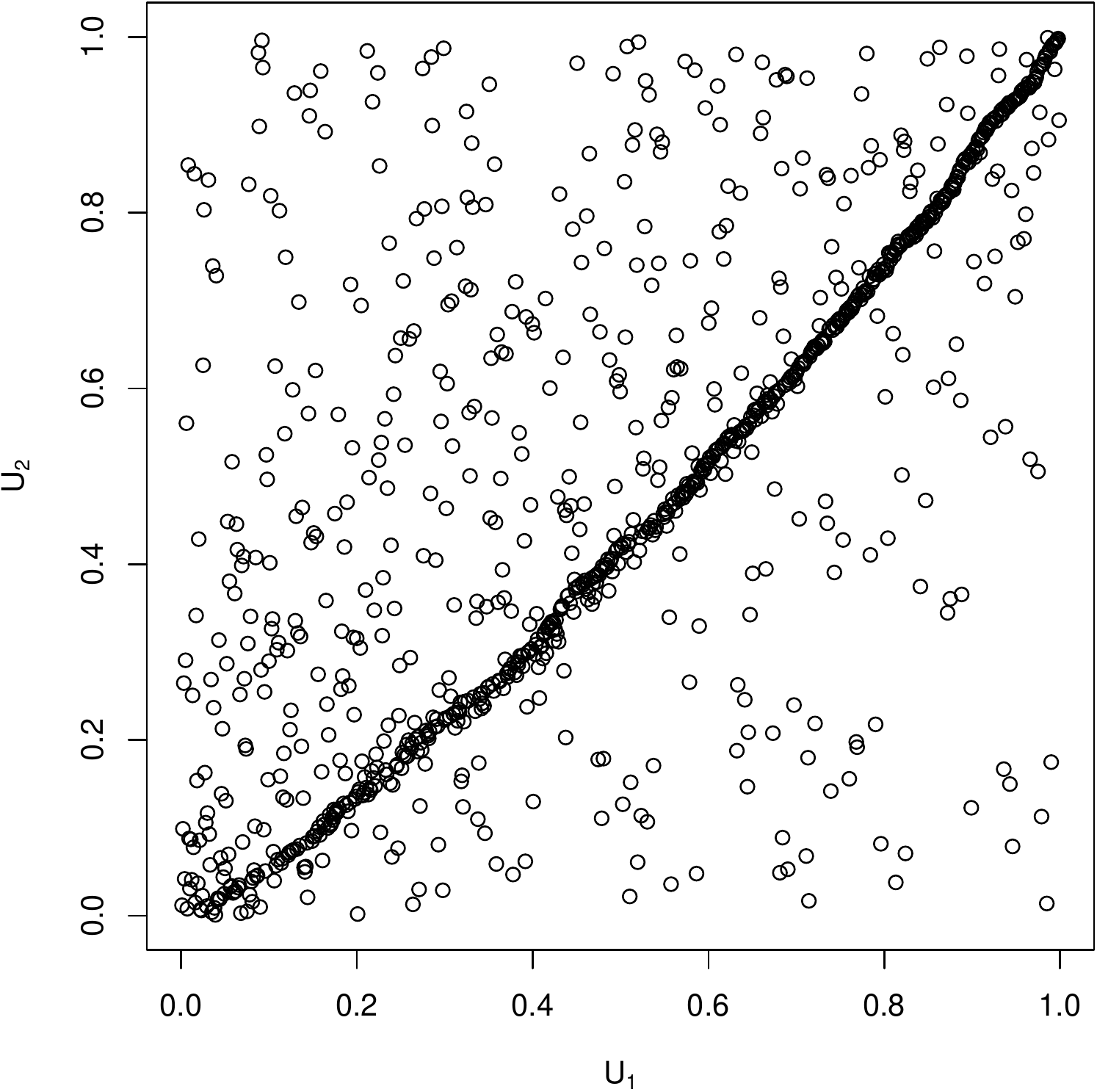}\hfill
  \includegraphics[width=0.32\textwidth]{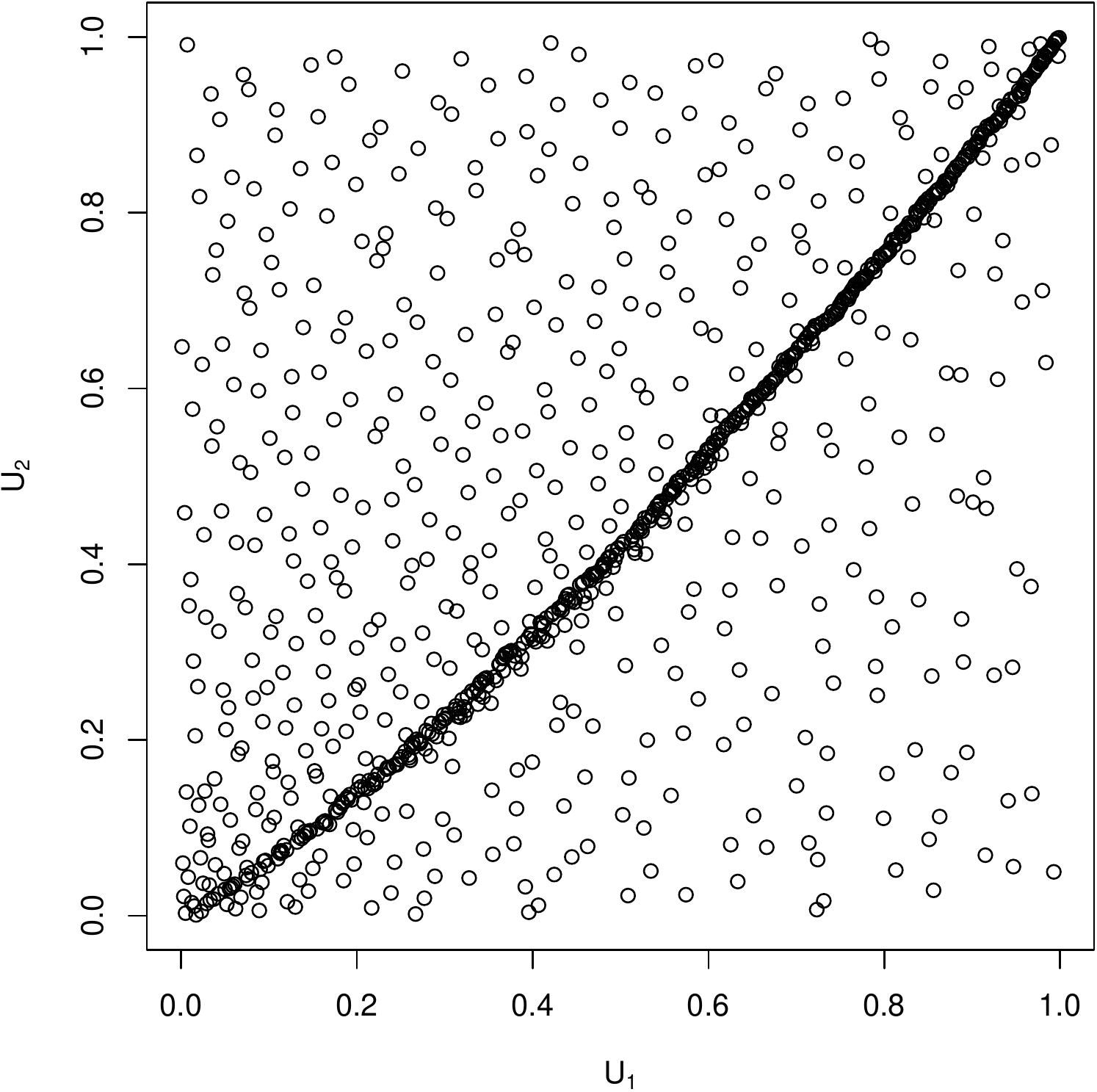}
  \caption{PRS (left), GMMN PRS (middle) and
    GMMN QRS (right), all of size $\ngen=1000$, from a Marshall--Olkin copula with $\alpha_{1}=0.75$
    and $\alpha_{2}=0.60$ (Kendall's tau equals $0.5$).}\label{fig:MOexample}
\end{figure}

\begin{figure}[htbp]
  \centering
  \includegraphics[width=0.32\textwidth]{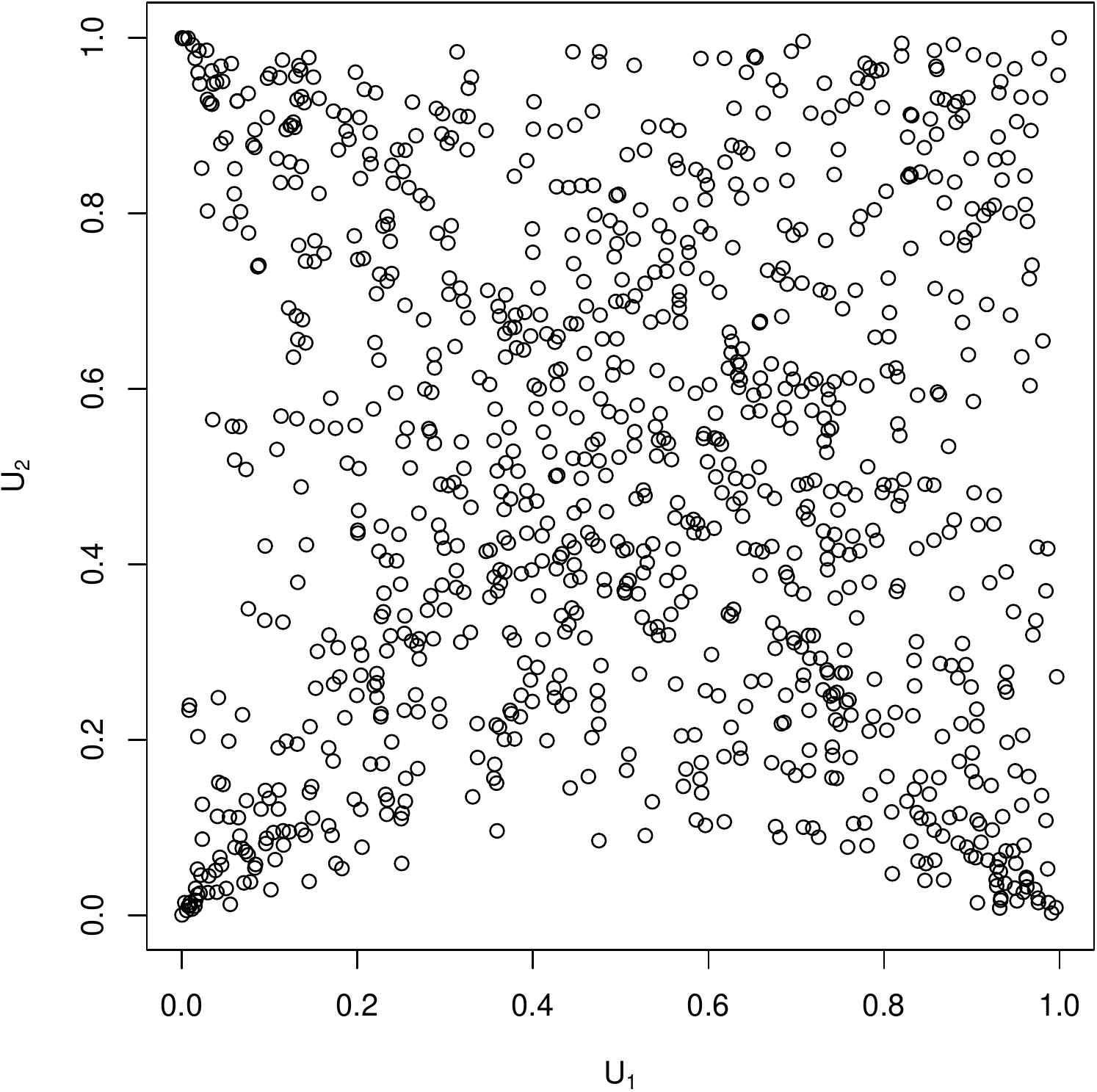}\hfill
  \includegraphics[width=0.32\textwidth]{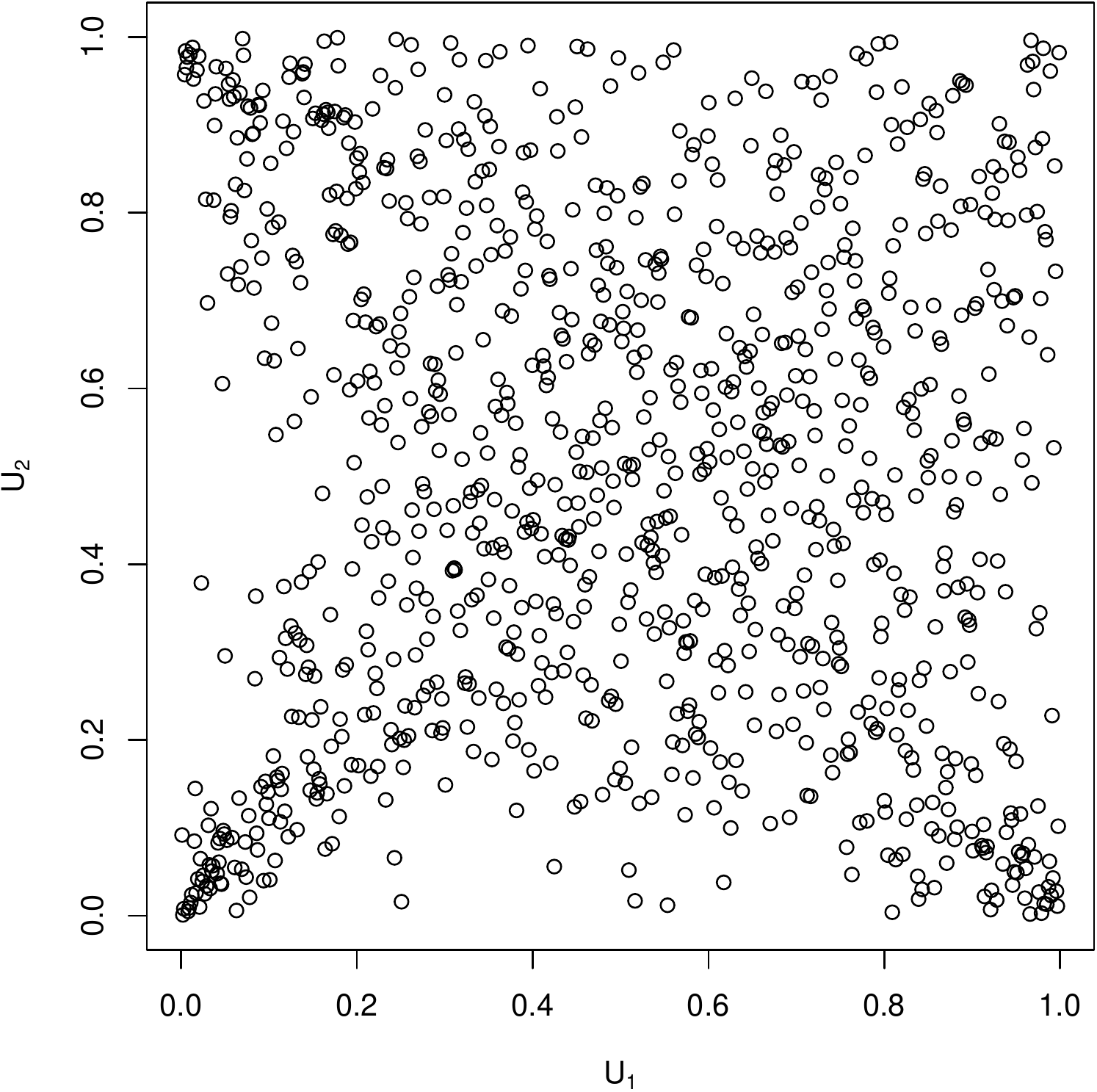}\hfill
  \includegraphics[width=0.32\textwidth]{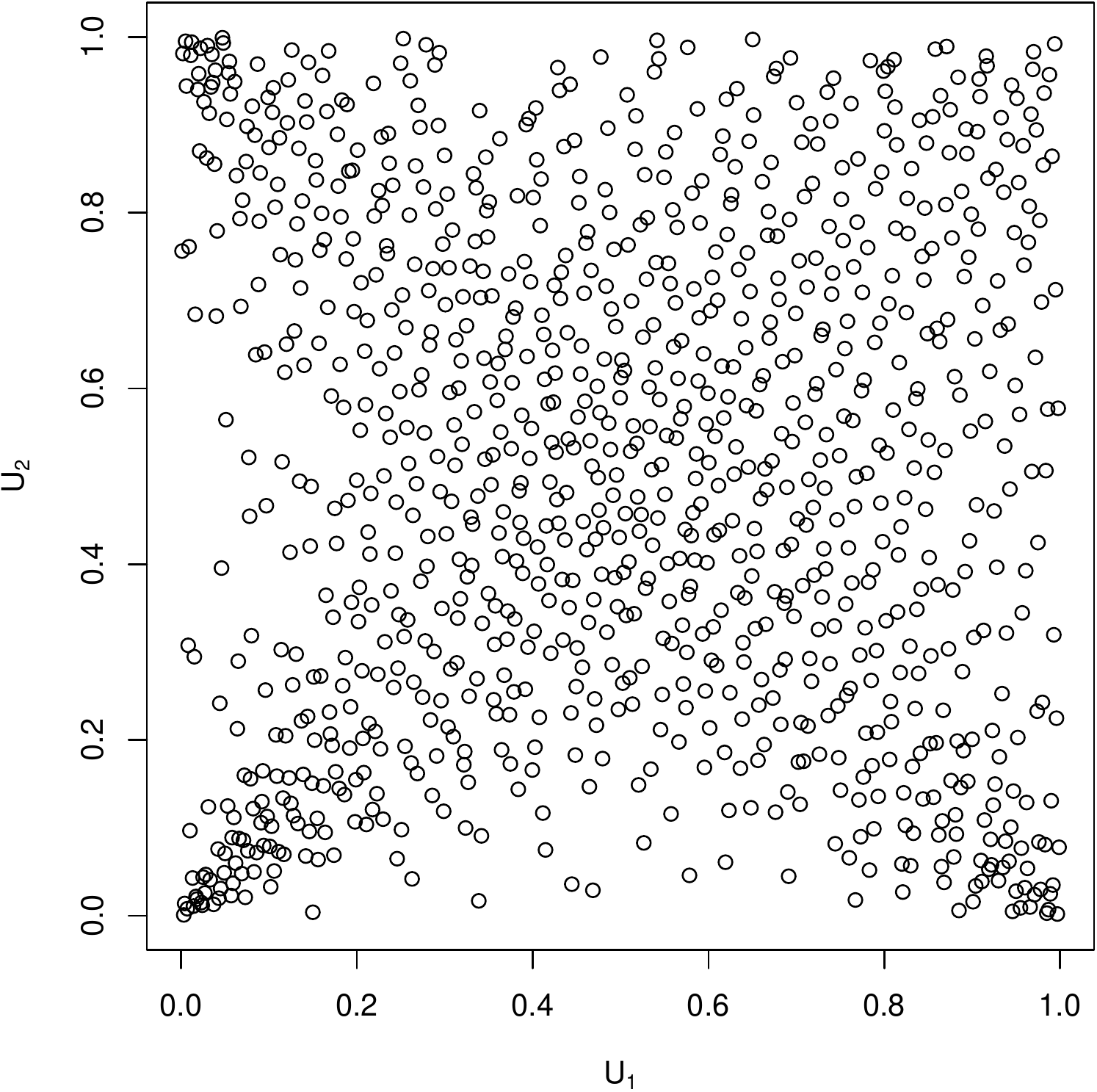}
  \includegraphics[width=0.32\textwidth]{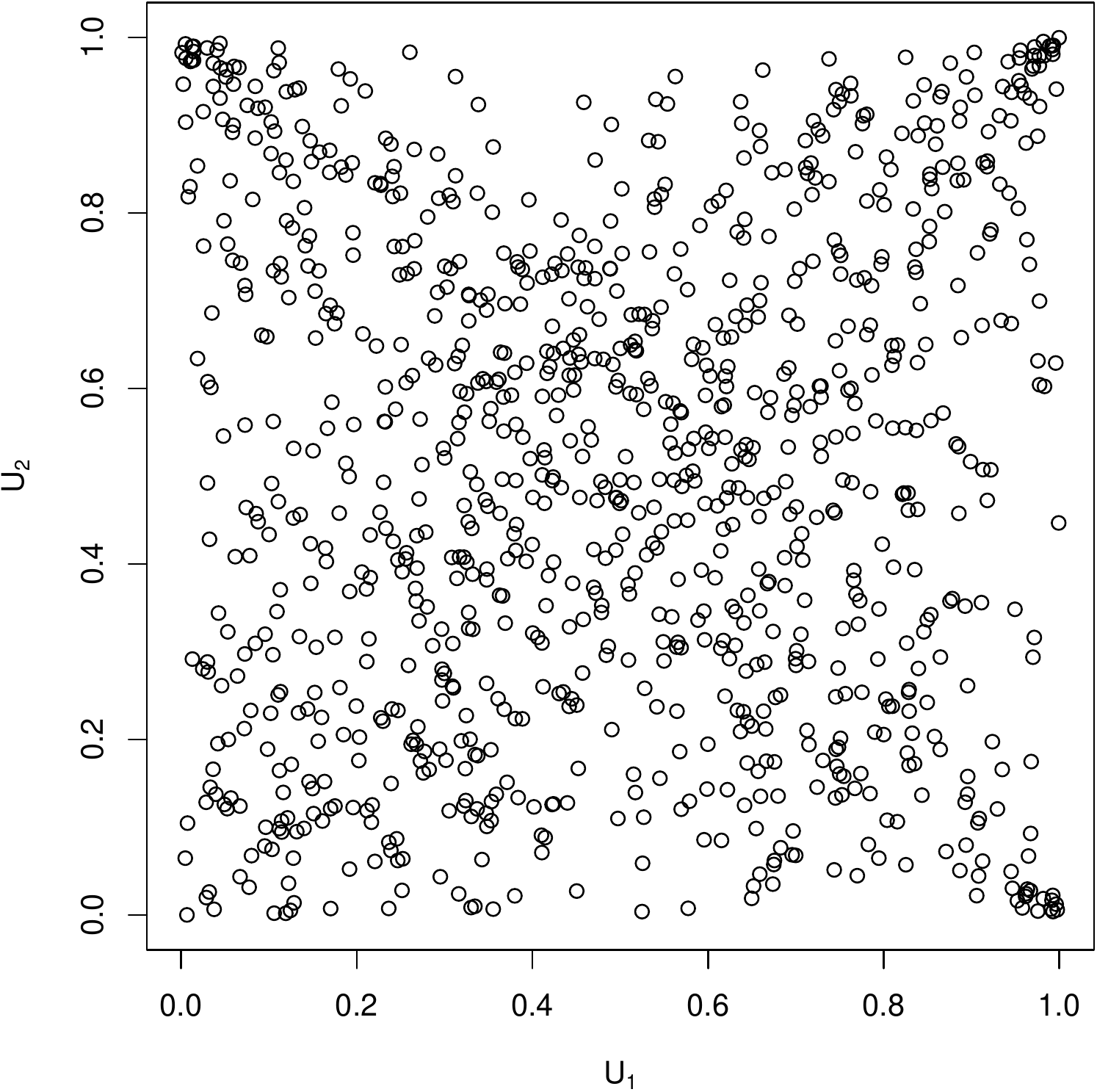}\hfill
  \includegraphics[width=0.32\textwidth]{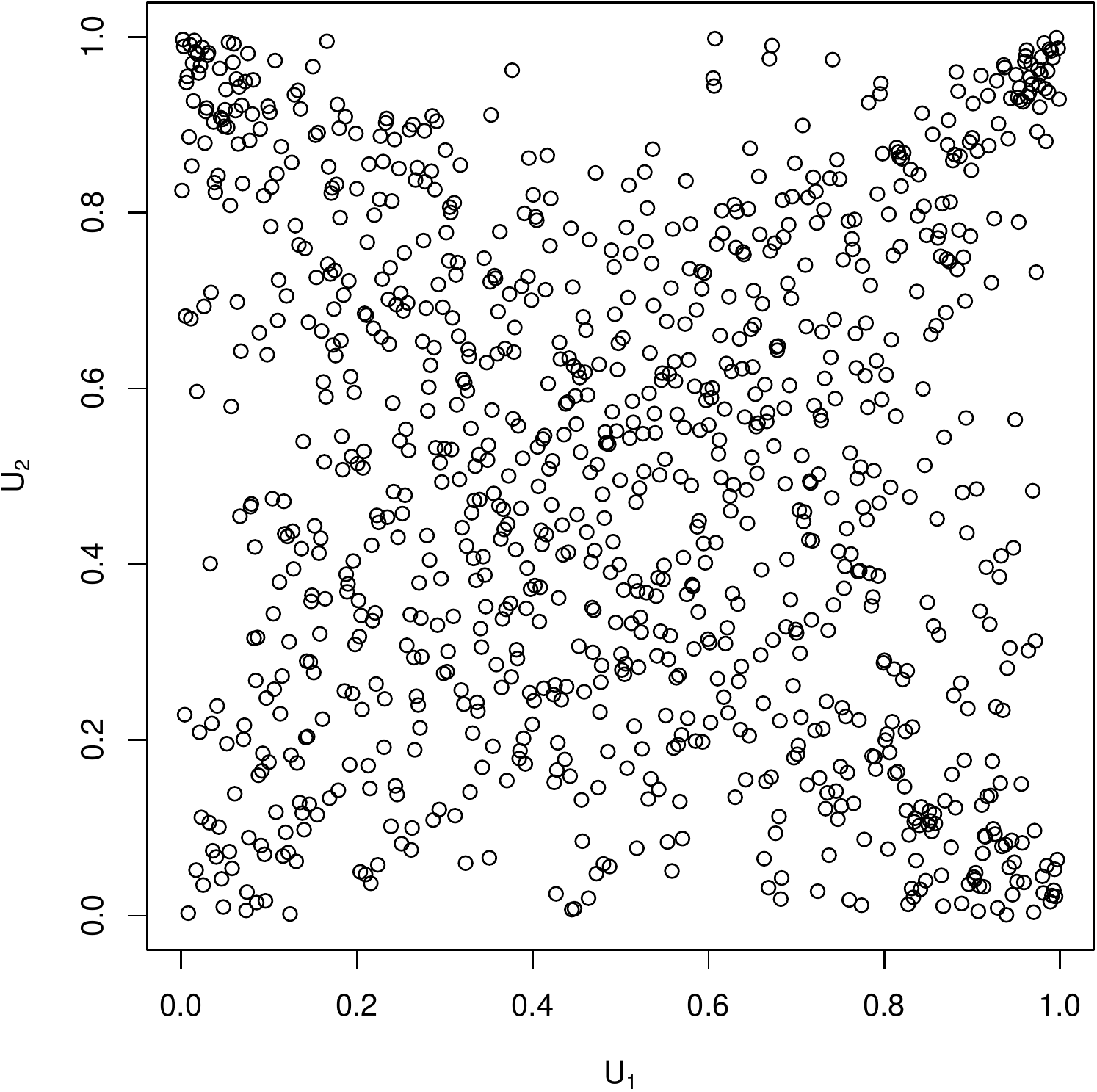}\hfill
  \includegraphics[width=0.32\textwidth]{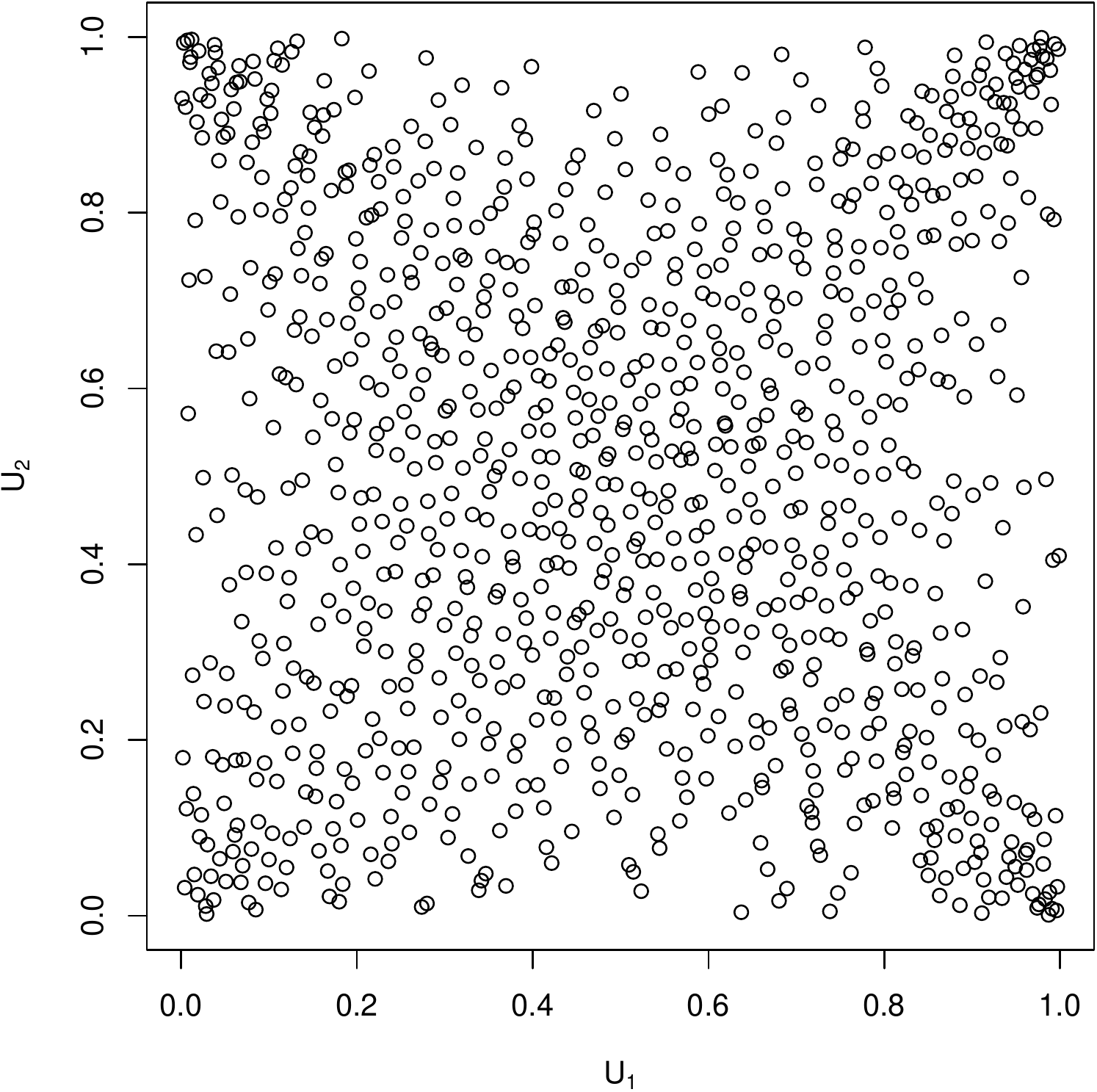}
  \includegraphics[width=0.32\textwidth]{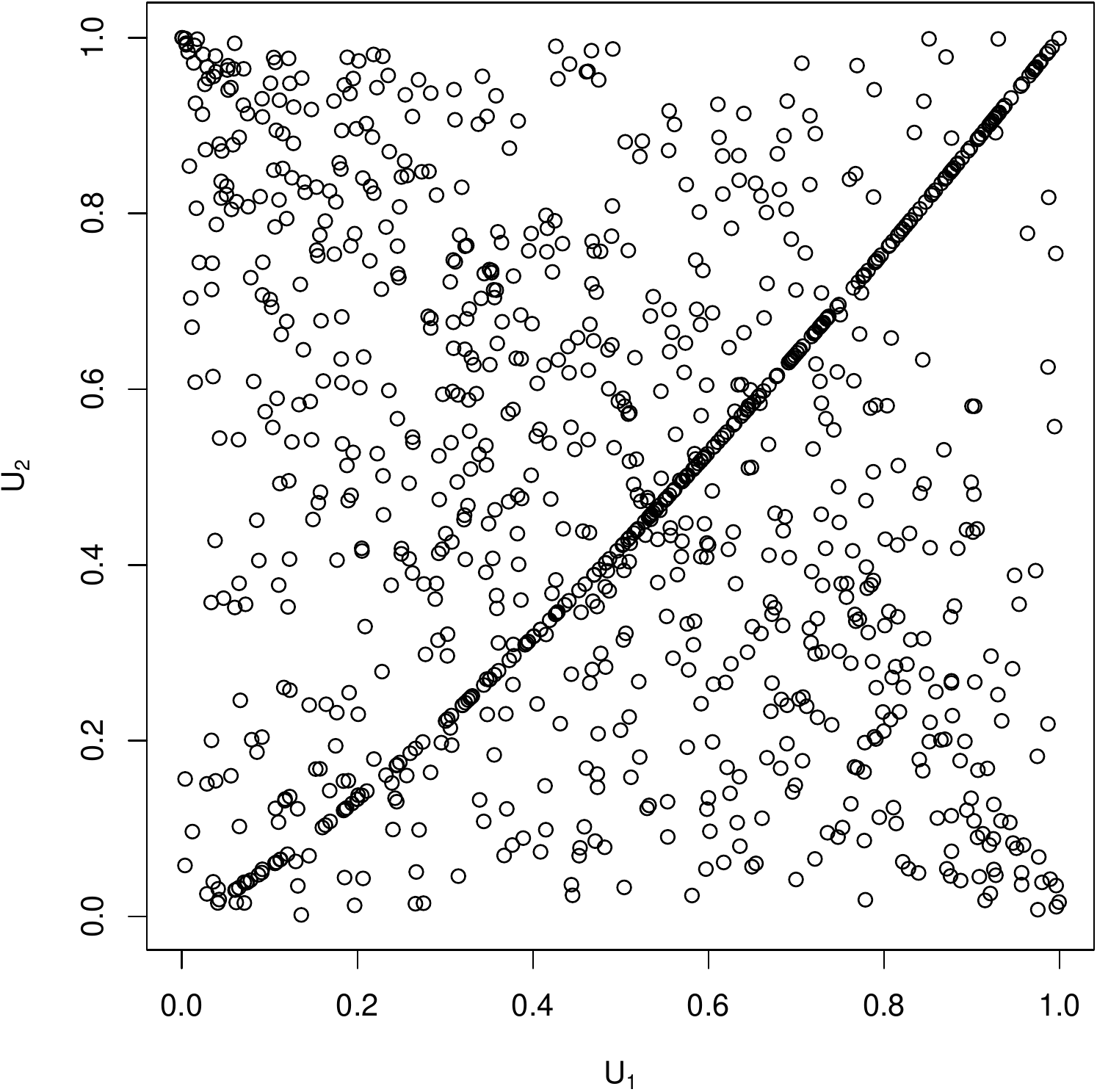}\hfill
  \includegraphics[width=0.32\textwidth]{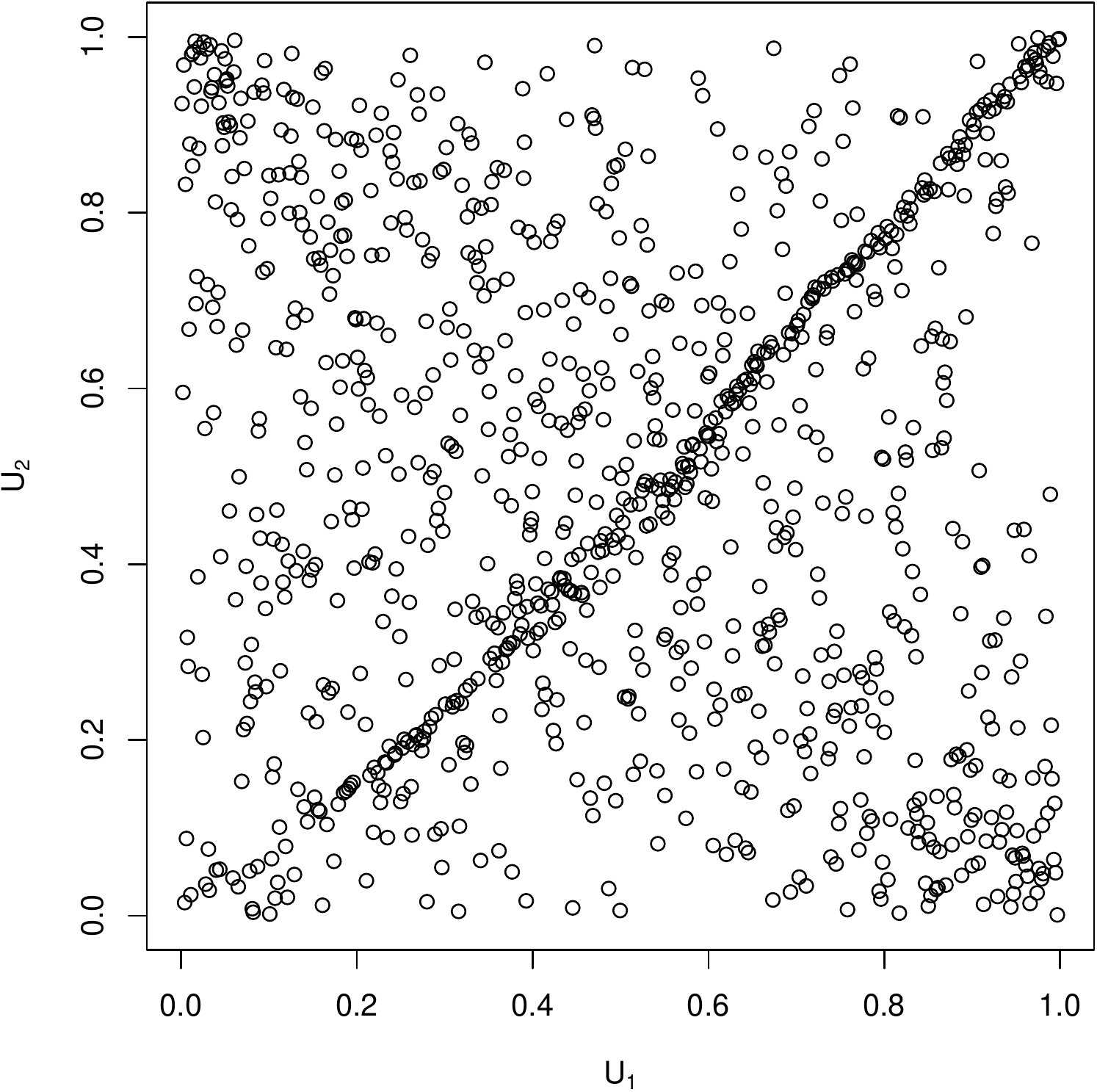}\hfill
  \includegraphics[width=0.32\textwidth]{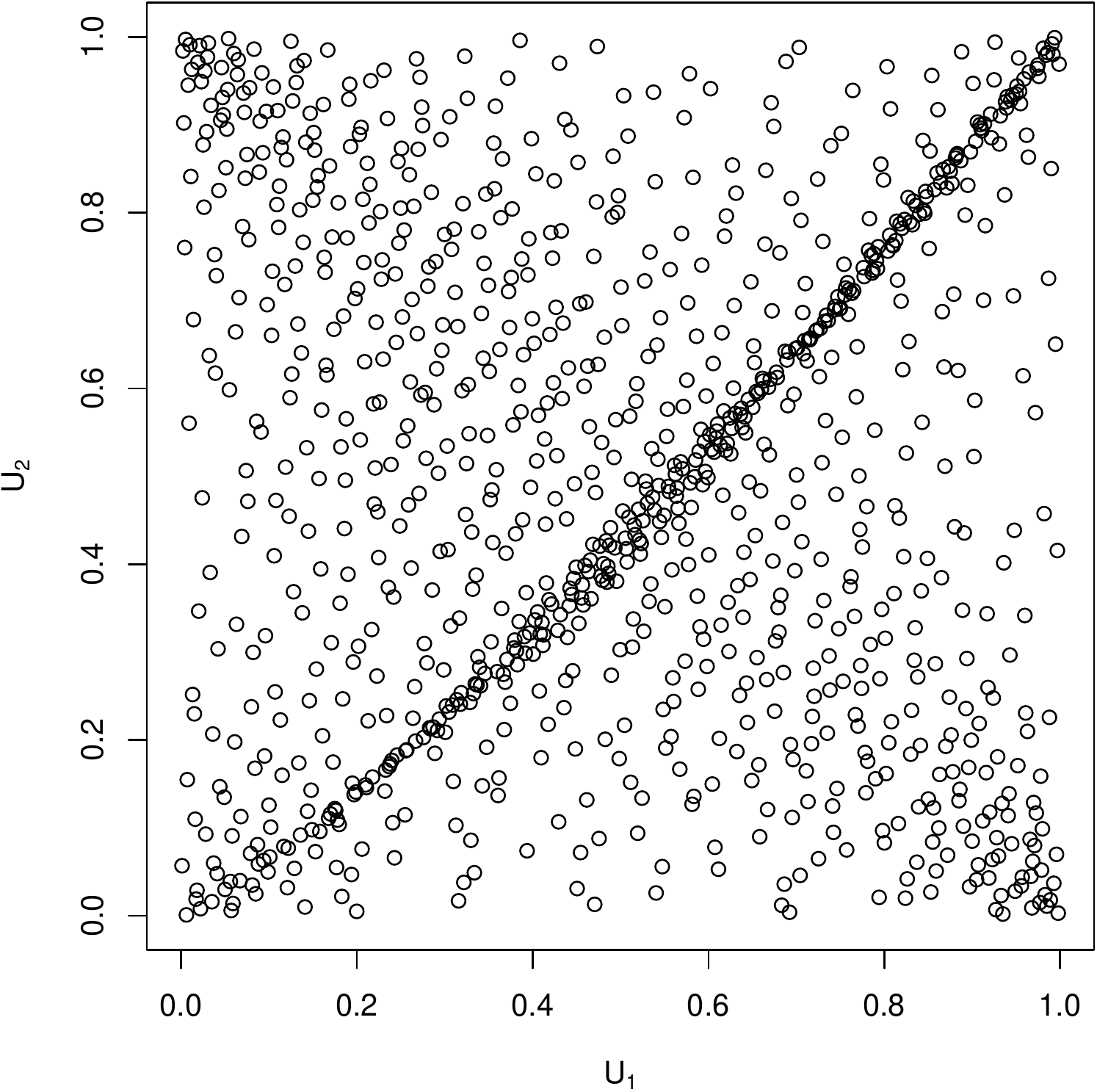}
  \caption{PRS (left column), GMMN PRS (middle column) and GMMN QRS (right
    column), all of size $\ngen=1000$, from a Clayton-$t(90)$ (top row),
    Gumbel-$t(90)$ (middle row) and a MO-$t(90)$ mixture (bottom row) copula.}
  \label{fig:mixtures:example2}
\end{figure}

\subsection{Assessment of GMMN samples by the Cram\'{e}r-von Mises statistic}\label{sec:GMMN:accuracy}
After a purely visual inspection of the generated samples, we now assess the
quality of GMMN pseudo-random and GMMN quasi-random samples more formally with the help of a
goodness-of-fit statistic. Since bivariate copulas have been investigated in
detail in the previous section, we focus on higher-dimensional copulas in this
section.

Specifically, we use the Cram\'er--von Mises statistic \parencite{genest2009},
\begin{align*}
  S_{\ngen}= \int_{[0,1]^{d}} \ngen(C_{\ngen}(\bm{u})-C(\bm{u}))^2\,\rd C_{\ngen}(\bm{u}),
\end{align*}
where the empirical copula
\begin{align}
  C_{\ngen}(\bm{u})=\frac{1}{\ngen} \sum_{i=1}^{\ngen} \I\{\hat{U}_{i1}\leq u_1,\dots,\hat{U}_{id}\leq u_d\},\quad\bm{u}\in[0,1]^d, \label{eq:emp:copula}
\end{align}
is the empirical distribution function of the pseudo-observations. For $\ngen =1000$ and each copula
$C$, we compute $B=100$ realizations of $S_{\ngen}$ three times ---
once for the case where $\hat{\bm{U}}_i$, $i=1,\dots,\ngen$, are pseudo-observations of the true underlying copula (as benchmark), once for GMMN pseudo-random samples and once for GMMN quasi-random samples. %
We then use box plots to depict the distribution of $S_{\ngen}$ in each
case. Figure~\ref{fig:cvm1} displays these box plots for $t_4$ (top row),
Clayton (middle row) and Gumbel (bottom row) copulas of dimensions $d=5$ (left
column), $d=10$ (right column) and $\tau=0.50$. Similarly, Figure~\ref{fig:cvm2}
displays such box plots for $d$-dimensional nested Clayton (left column) and
nested Gumbel (right column) copulas for $d=3$ (top row), $d=5$ (middle row) and
$d=10$ (bottom row). The three-dimensional NACs have a structure as given by
\eqref{eq:NAC} with $\tau_0=0.25$ and $\tau_1=0.50$; the five-dimensional NACs
have structure $C_0(C_1(u_1,u_2),C_2(u_3,u_4,u_5))$ with corresponding
$\tau_0=0.25$, $\tau_1=0.50$ and $\tau_2=0.75$; and the ten-dimensional NACs
have structure $C_0(C_1(u_1,\dots,u_5),C_2(u_6,\dots,u_{10}))$ with
corresponding $\tau_0=0.25$, $\tau_1=0.50$ and $\tau_2=0.75$.

We can observe from both figures that the distributions of $S_{\ngen}$ for pseudo-random samples
from $C$ and from the GMMN are similar, with slightly higher
$S_{\ngen}$ values for the GMMN pseudo-random samples, especially for $d=10$.
Additionally, we can observe that the distribution of $S_{\ngen}$ based on the
GMMN quasi-random samples is closer to zero than that of the GMMN pseudo-random
samples. This provides some evidence that the low-discrepancy of input RQMC
points sets has been preserved under the respective (trained) GMMN transform.

We also see that $S_{\ngen}$ values based on the GMMN quasi-random samples are clearly lower
than $S_{\ngen}$ values based on the copula pseudo-random samples, with the exception of
(some) copulas for $d=10$ where the distributions of $S_{\ngen}$ are more
similar.

\begin{figure}[htbp]
  \centering
  \includegraphics[width=0.41\textwidth]{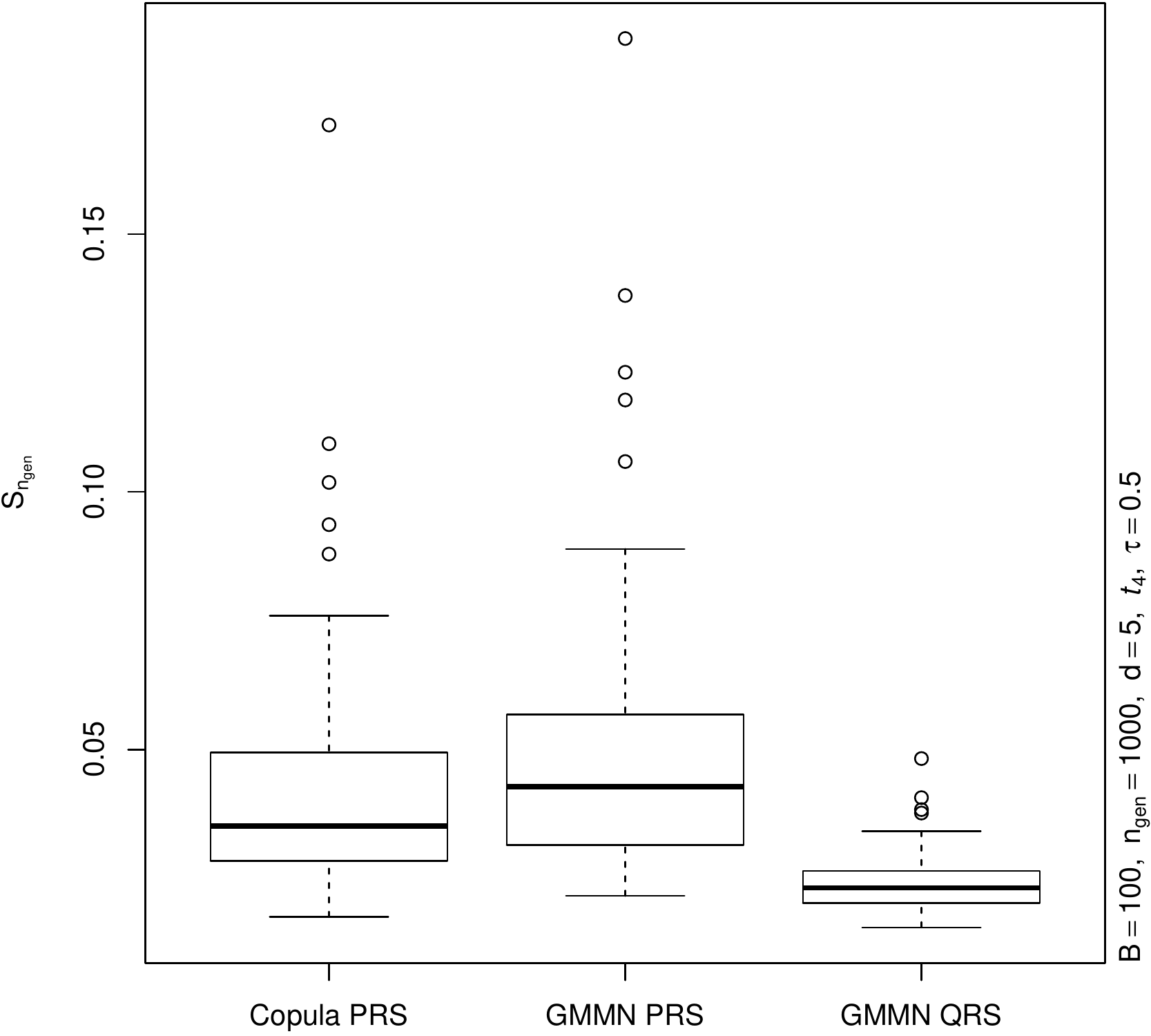}\hspace{5mm}
  \includegraphics[width=0.41\textwidth]{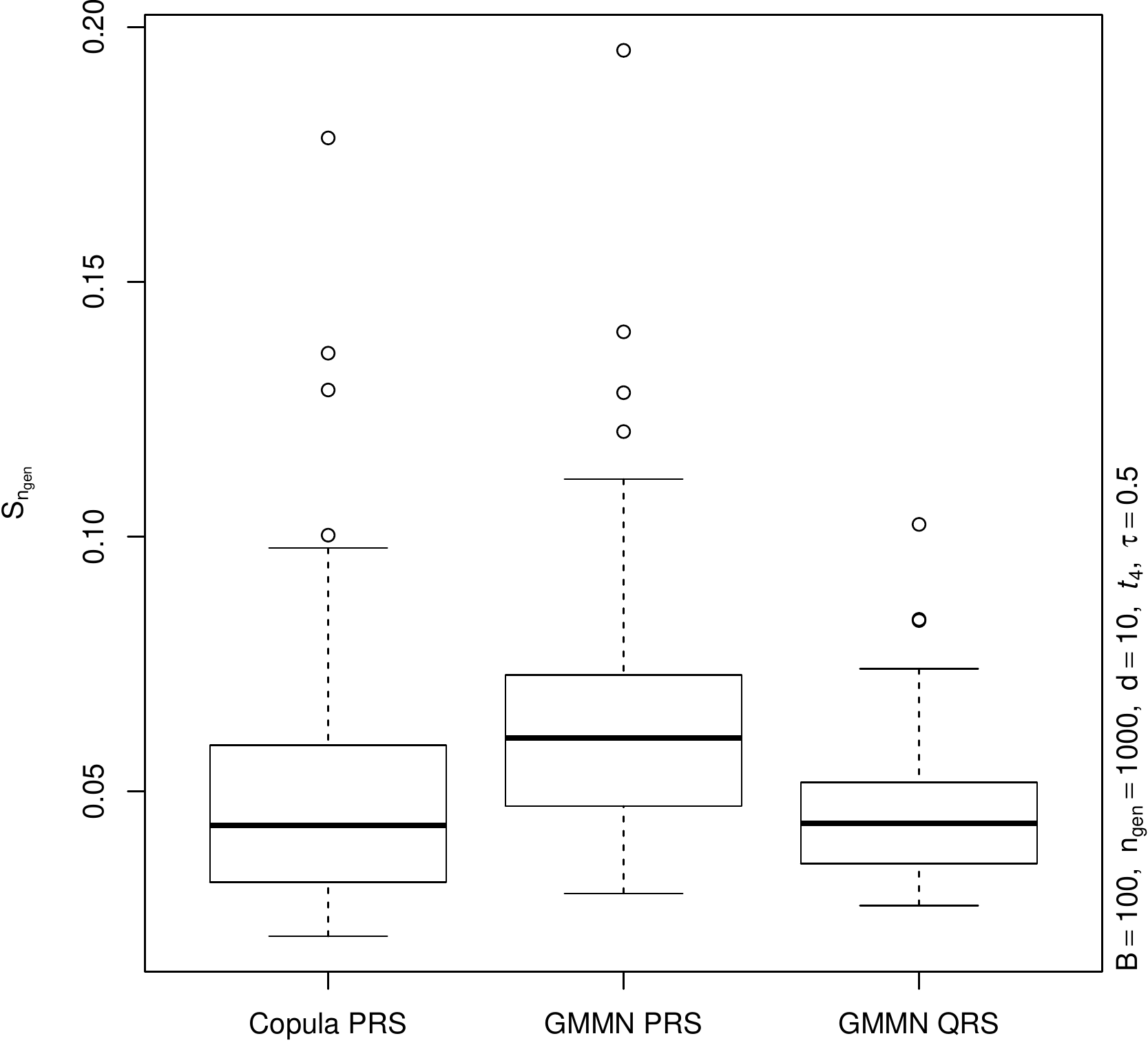}\\[2mm]
  \includegraphics[width=0.41\textwidth]{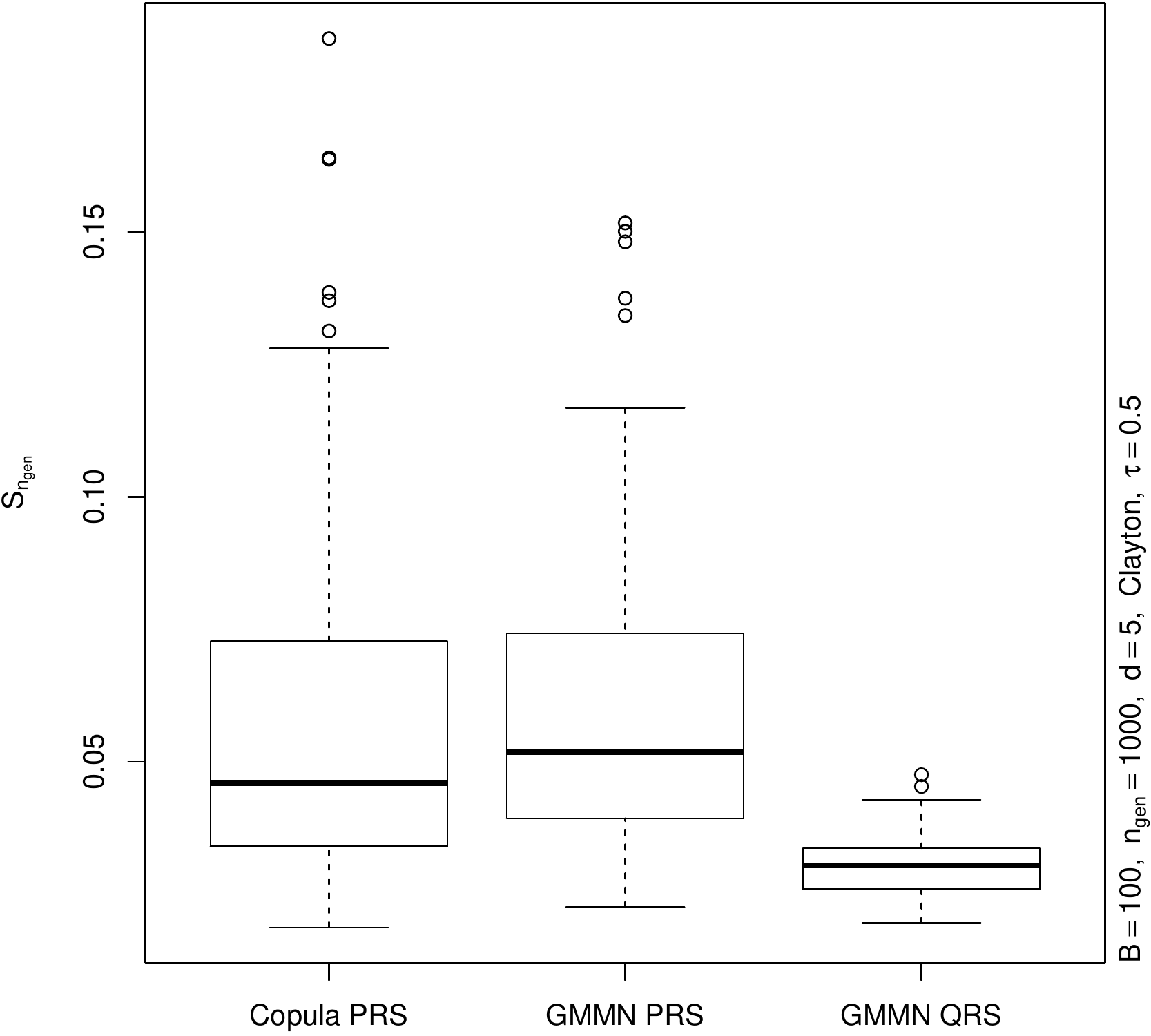}\hspace{5mm}
  \includegraphics[width=0.41\textwidth]{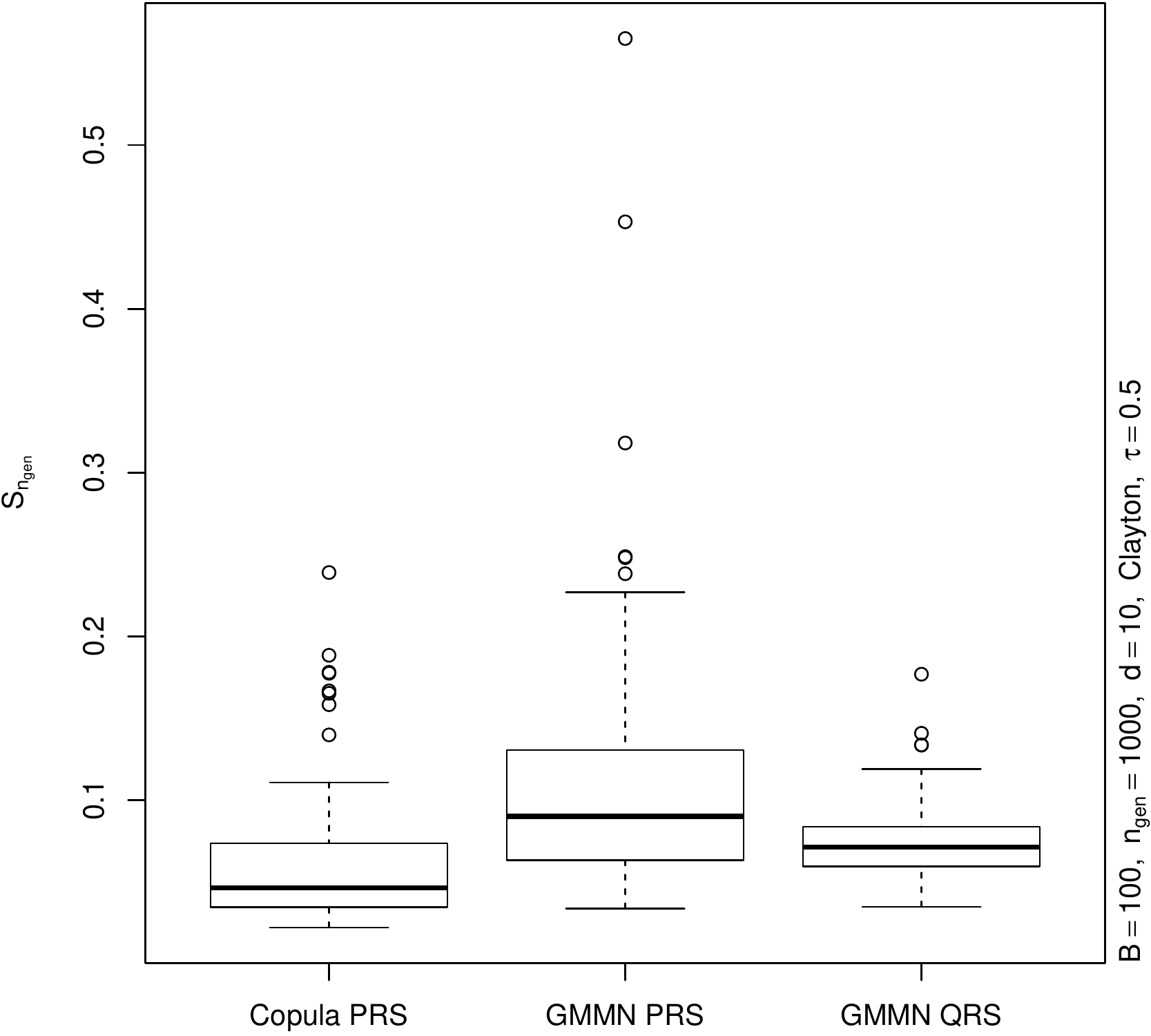}\\[2mm]
  \includegraphics[width=0.41\textwidth]{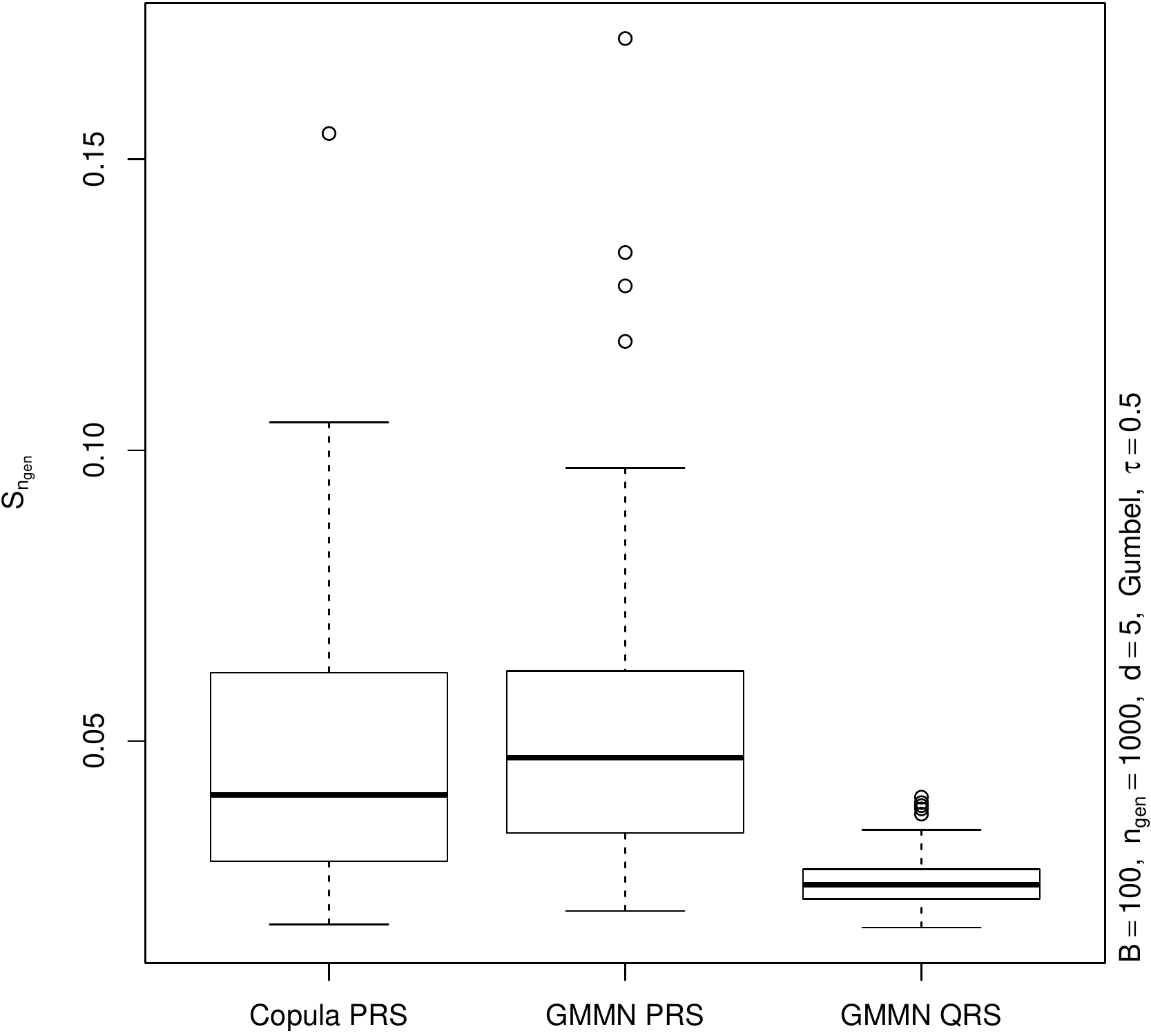}\hspace{5mm}
  \includegraphics[width=0.41\textwidth]{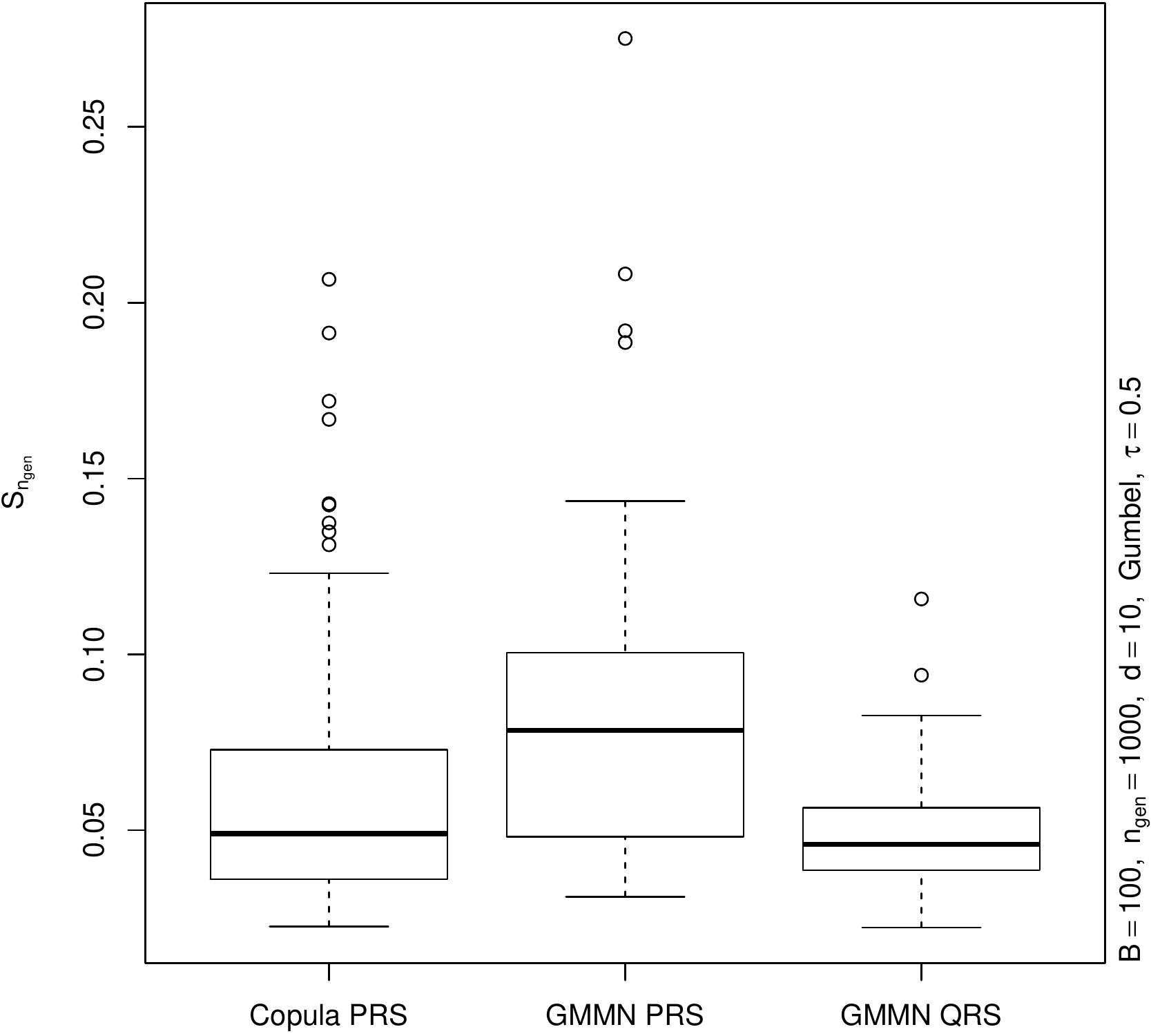}
  \caption{Box plots based on $B=100$ realization of $S_{\ngen}$ computed from
    (i) a pseudo-random sample (PRS) of $C$ (denoted by Copula PRS), (ii) a GMMN pseudo-random sample
    (denoted by GMMN PRS) and (iii) a GMMN quasi-random sample (denoted by GMMN
    QRS) --- all of size $\ngen=1000$ --- for a $t_4$ (top row), Clayton (middle
    row) and Gumbel (bottom row) copula with $\tau=0.5$ as well as $d=5$ (left
    column) and $d=10$ (right column).}\label{fig:cvm1}
\end{figure}

\begin{figure}[htbp]
  \centering
  \includegraphics[width=0.41\textwidth]{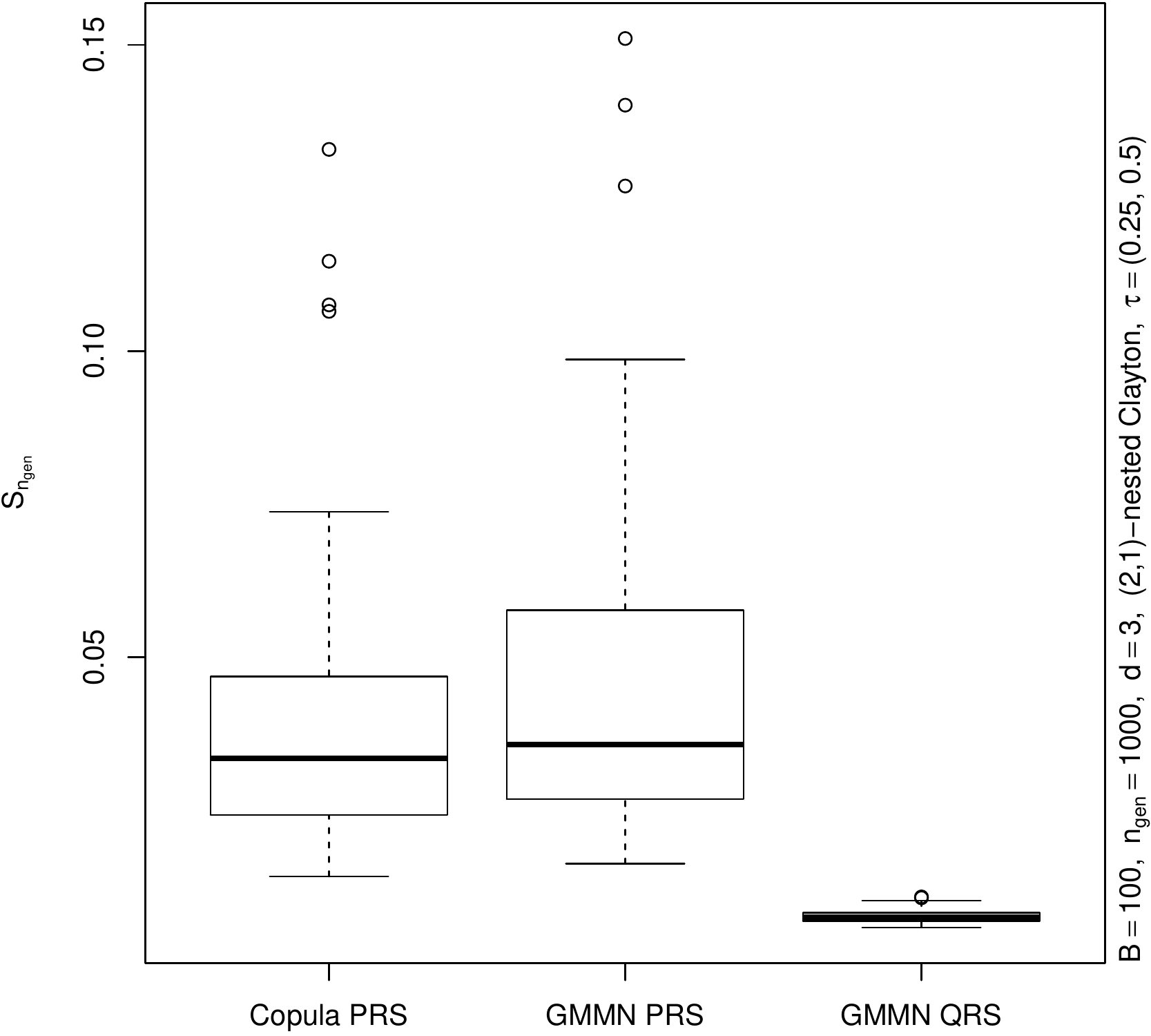}\hspace{5mm}
  \includegraphics[width=0.41\textwidth]{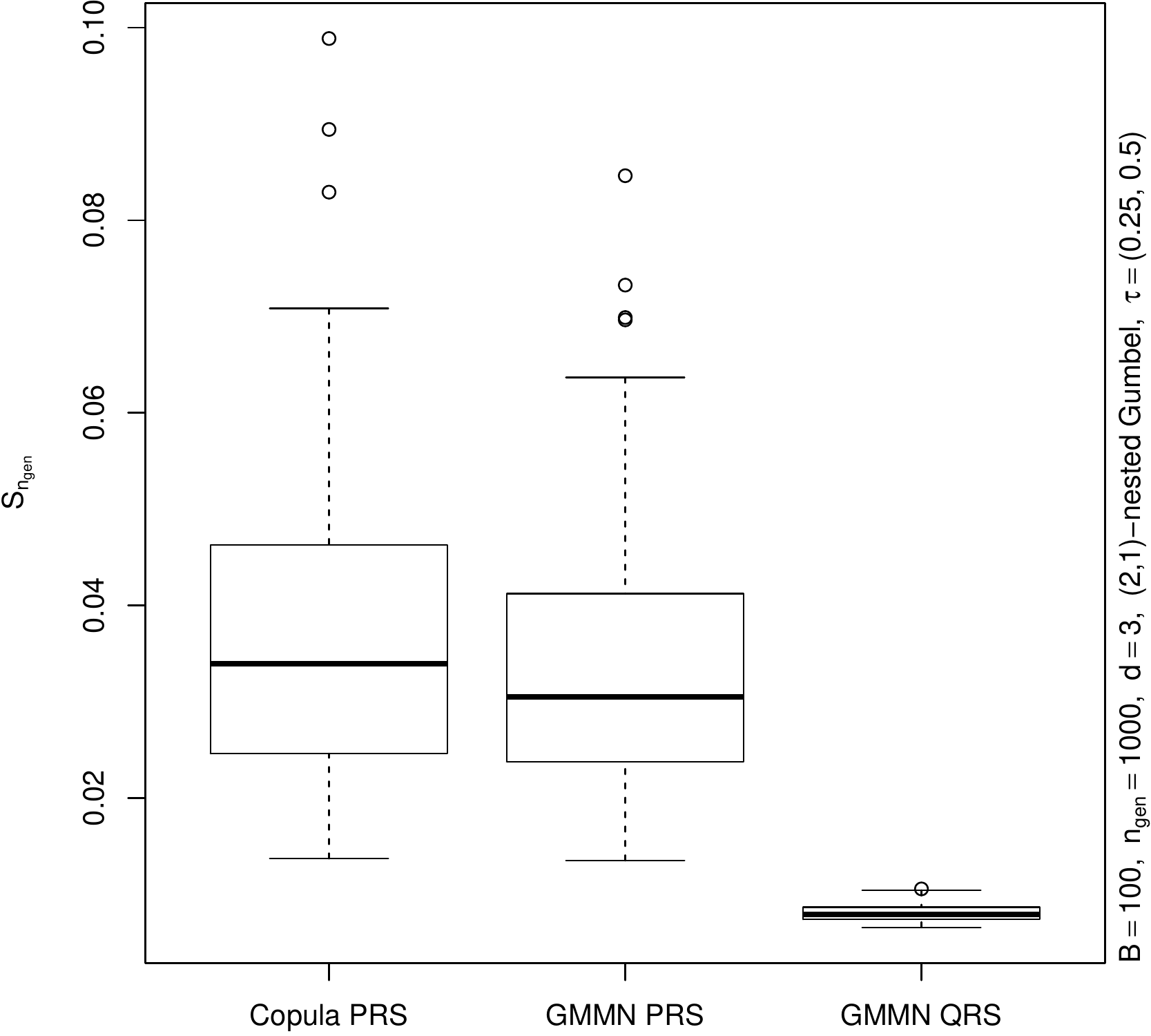}\\[2mm]
  \includegraphics[width=0.41\textwidth]{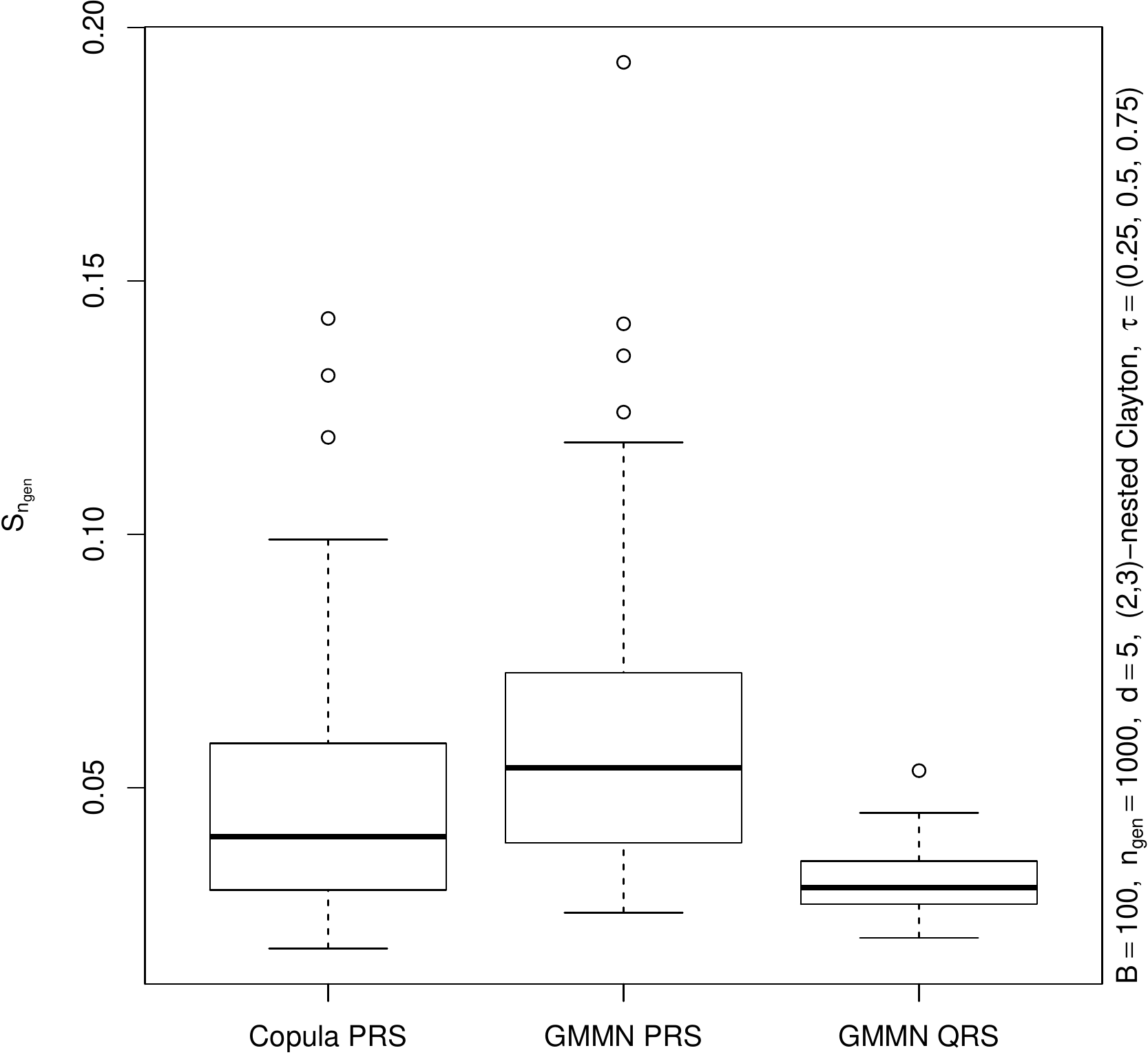}\hspace{5mm}
  \includegraphics[width=0.41\textwidth]{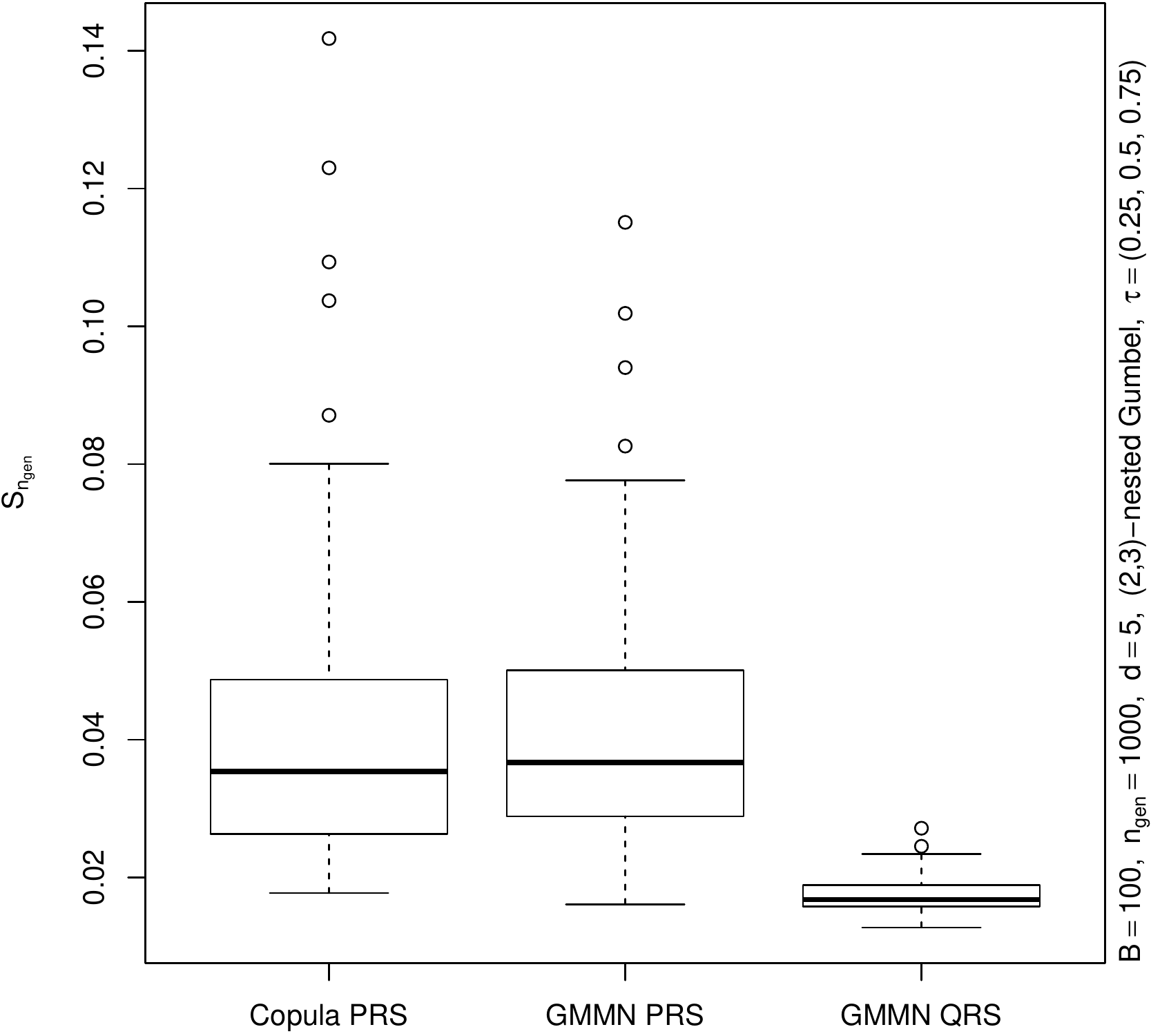}\\[2mm]
  \includegraphics[width=0.41\textwidth]{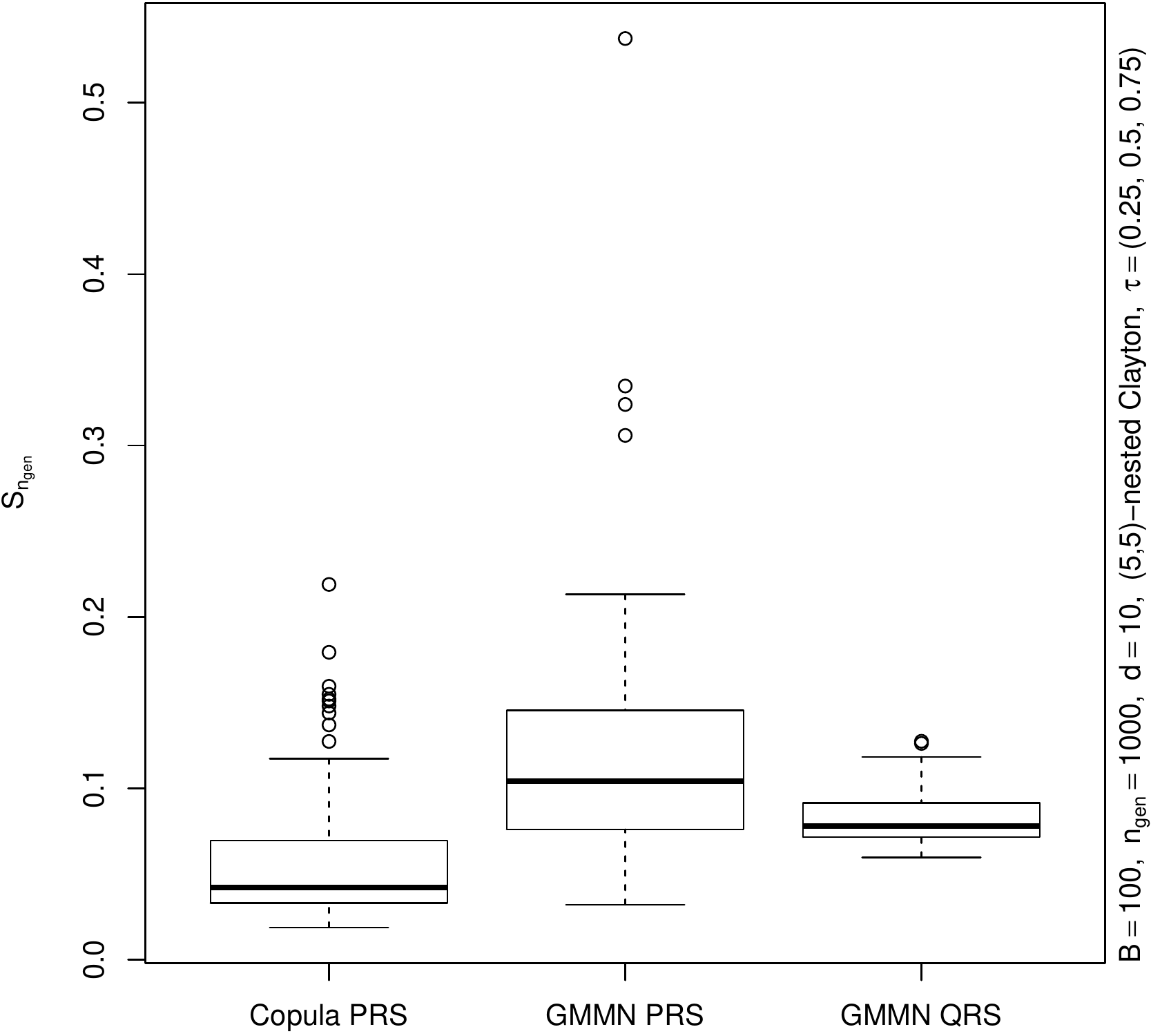}\hspace{5mm}
  \includegraphics[width=0.41\textwidth]{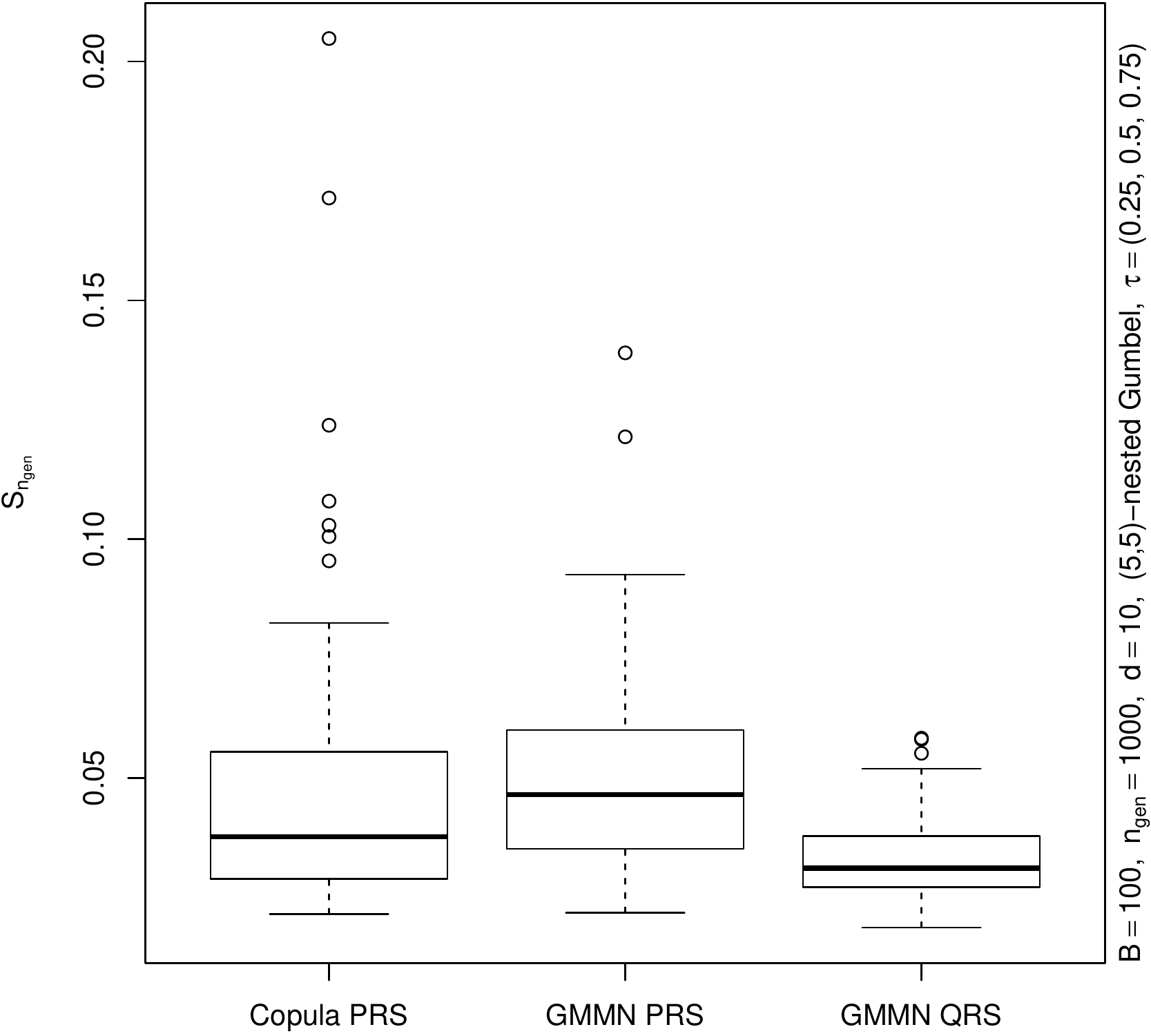}
  \caption{As Figure~\ref{fig:cvm1} but for nested Clayton (left column) and
    nested Gumbel (right column) copulas and for $d=3$ (top row), $d=5$ (middle
    row) and $d=10$ (bottom row).}\label{fig:cvm2}
\end{figure}

\section{Convergence analysis of the RQMC estimator}\label{sec:conv:analysis}
In this section we numerically investigate the variance-reduction properties of
the GMMN RQMC estimator $\varhat{\mu}{\ngen}{NN}$ in~\eqref{eq:GMMN:C:RQMC} for
two functions $\Psi$ and transforms
$q=f_{\hat{\bm{\theta}}}\circ\Phi^{-1}$ corresponding to different copulas
$C$. We compare $\varhat{\mu}{\ngen}{NN}$ with estimators based on
  standard copula pseudo-random and, where available, copula quasi-random
  samples. For the latter, we follow \cite{cambou2017} but note that
  quasi-random sampling procedures are only available for some of the copulas we
  consider here; for others, the procedures are either too slow (e.g., for
  Gumbel copulas; see Appendix~\ref{sec:timings}) or not known at all (e.g.,
  for nested Clayton or Gumbel copulas).

We consider two different types of functions $\Psi$.  The first is a \emph{test
  function} primarily used in the QMC literature to test the performance of
$\varhat{\mu}{\ngen}{NN}$ in terms of its ability to preserve the
low-discrepancy of $\tilde{P}_{\ngen}$.
The second function $\Psi$ is motivated from a practical application in
risk management. For both functions,
standard deviation estimates will be
computed to compare convergence rates,
based on $B=25$
randomized point sets $\tilde{P}_{\ngen}$ for each of
$\ngen \in \{2^{10},2^{10.5},\dots,2^{18}\}$ to help roughly gauge the convergence
rate for all estimators. Furthermore, regression coefficients $\alpha$ (obtained by
regressing the logarithm of the standard deviation on the logarithm of $\ngen$)
are computed and displayed to allow for an easy comparison of the corresponding
convergence rates $O(\ngen^{-\alpha})$ with the theoretical convergence rate $O(\ngen^{-0.5})$ of
the Monte Carlo estimator's standard deviation. %
For RQMC estimators one can expect $\alpha$ to be larger than $0.5$, but with an upper bound of $1.5-\eps$, where $\eps$ increases with dimension $d$; see Theorem~\ref{thm:rqmc} for further details.

\subsection{A test function}\label{sec:testfun}
The test function we consider is the \emph{Sobol' g} function
\parencite{faurelemieux2009} based on the Rosenblatt transform and is given
by
\begin{align*}
  \Psi_1(\bm{U})= \prod_{j=1}^{d} \frac{|4R_j-2|+j}{1+j},
\end{align*}
where $R_1=U_1$ and, for $j=2,\dots,d$ and if $\bm{U}\sim C$,
$R_{j}=C_{j|1,\dots,j-1}(U_j\,|\,U_{j-1},\dots,U_1)$ denotes the conditional
distribution function of $U_j$ given $U_1,\dots,U_{j-1}$.

\begin{figure}[htbp]
  \centering
  \includegraphics[width=0.32\textwidth]{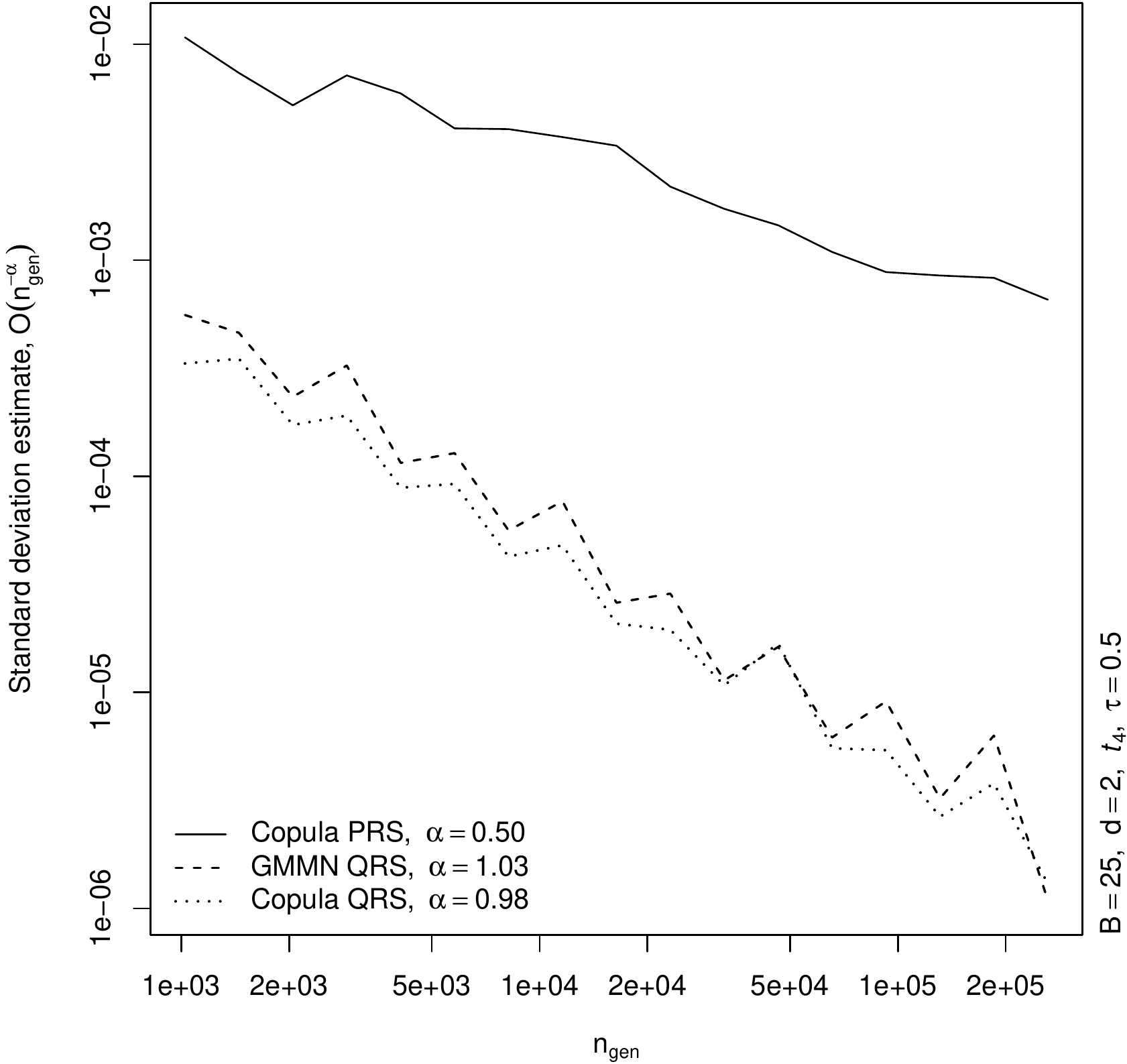}\hfill
  \includegraphics[width=0.32\textwidth]{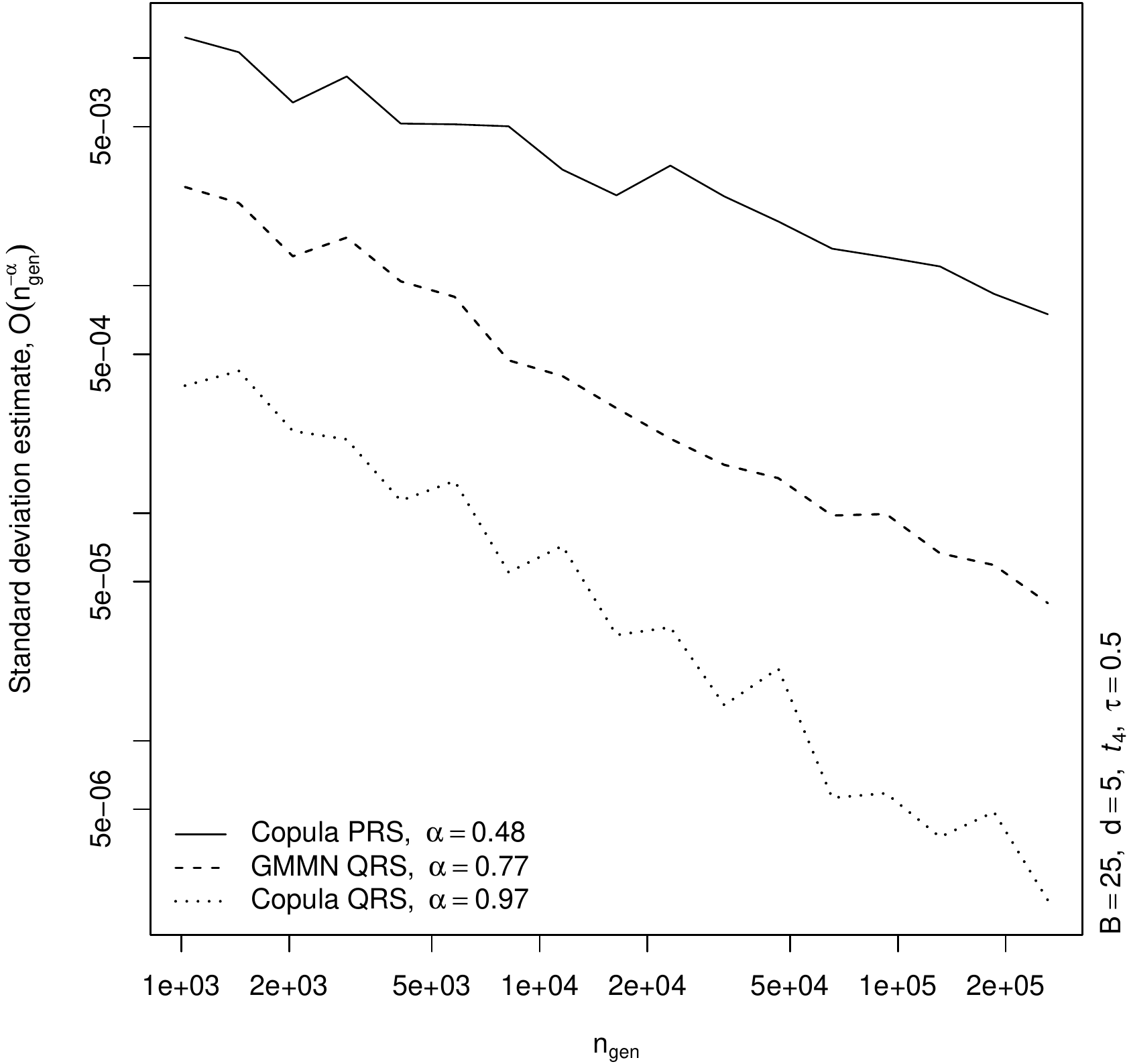}\hfill
  \includegraphics[width=0.32\textwidth]{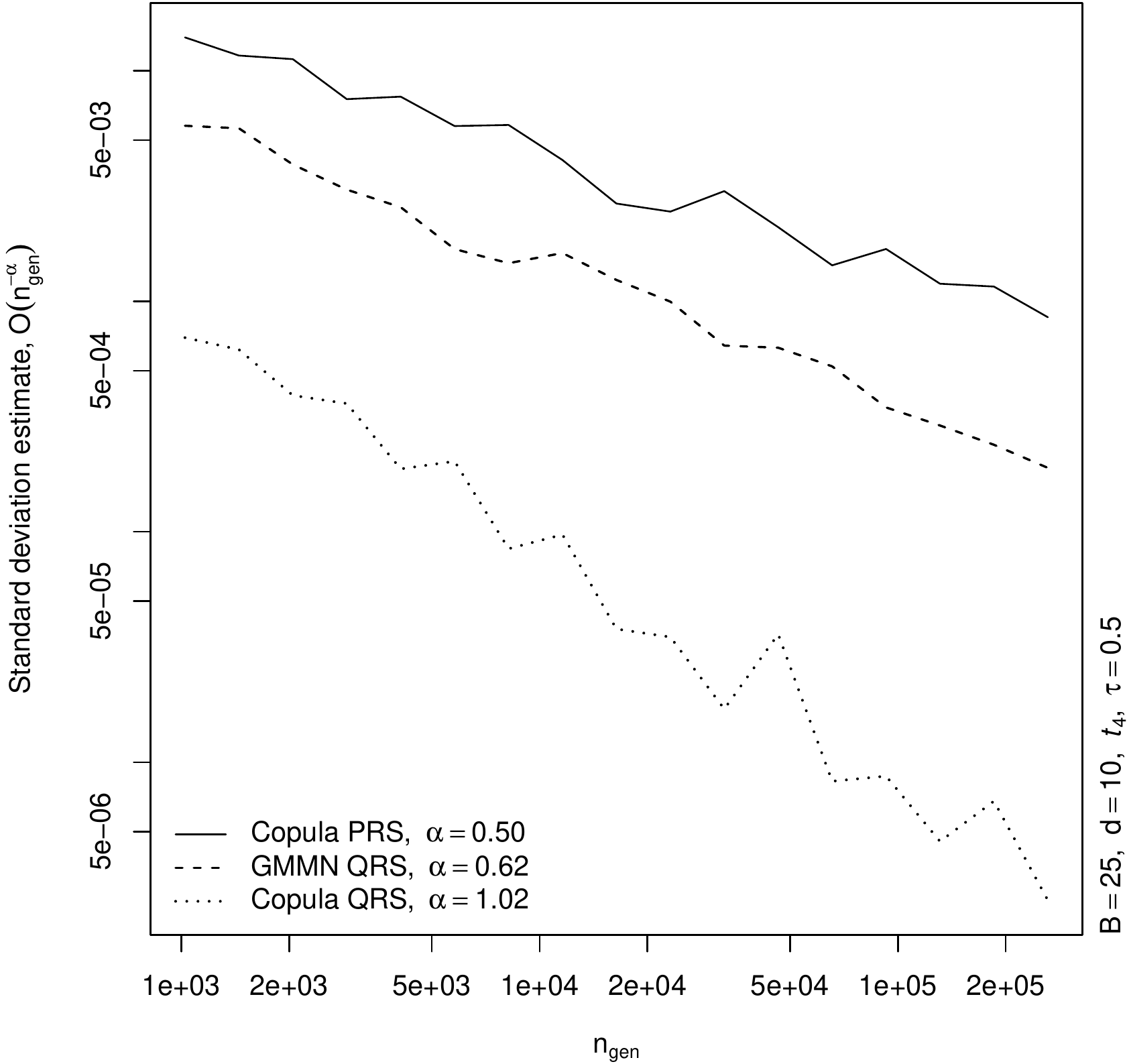}
  \includegraphics[width=0.32\textwidth]{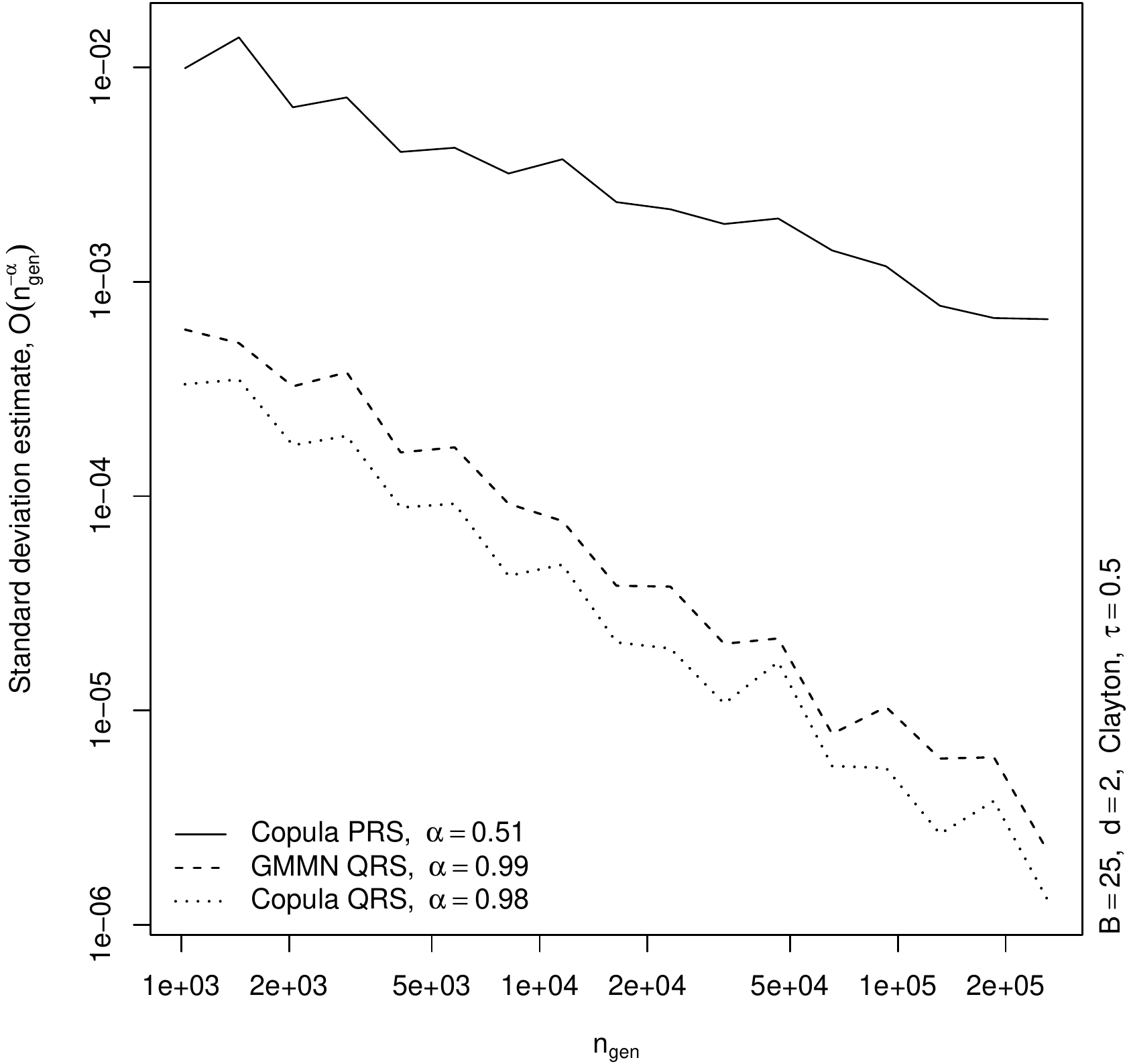}\hfill
  \includegraphics[width=0.32\textwidth]{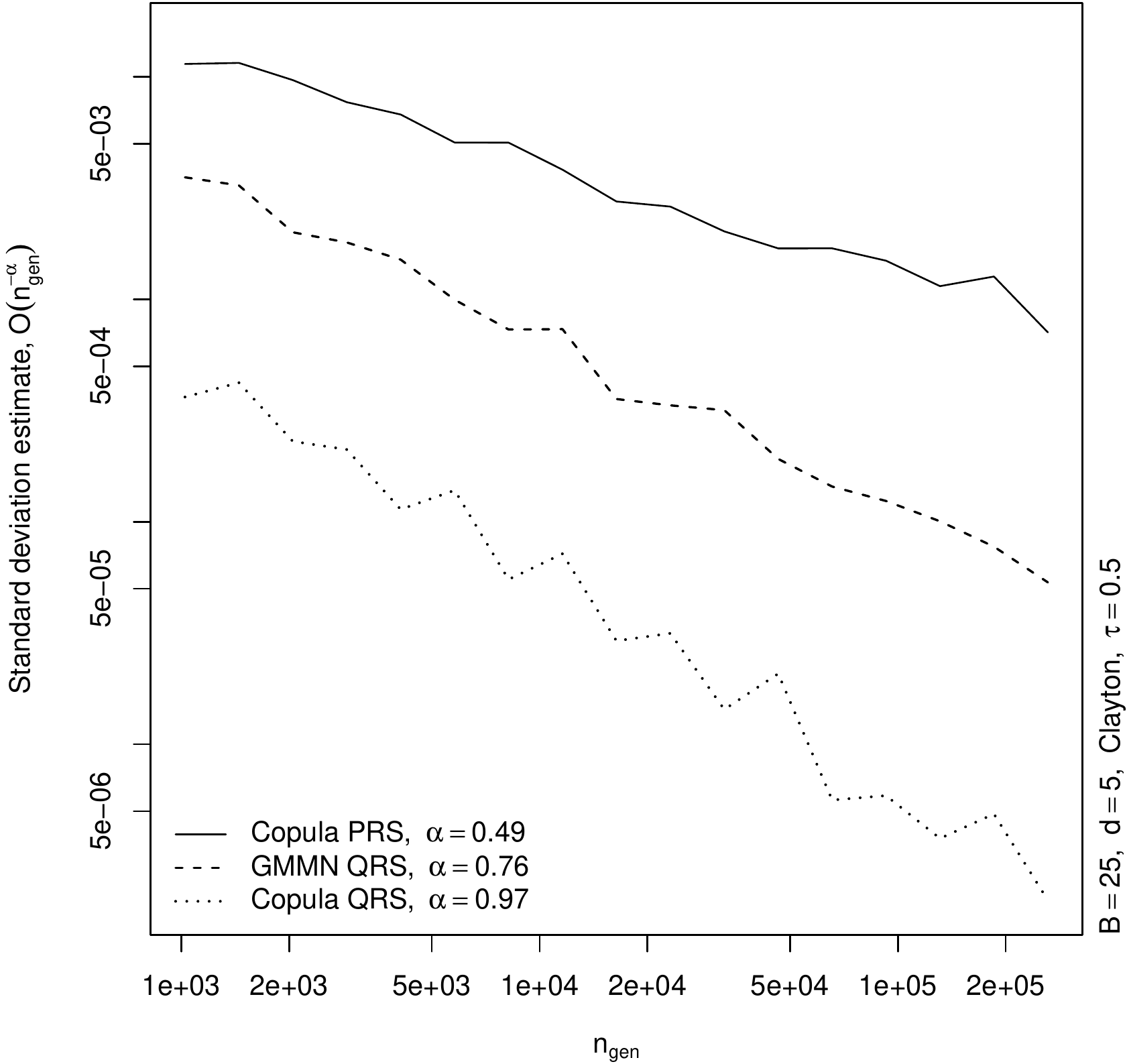}\hfill
  \includegraphics[width=0.32\textwidth]{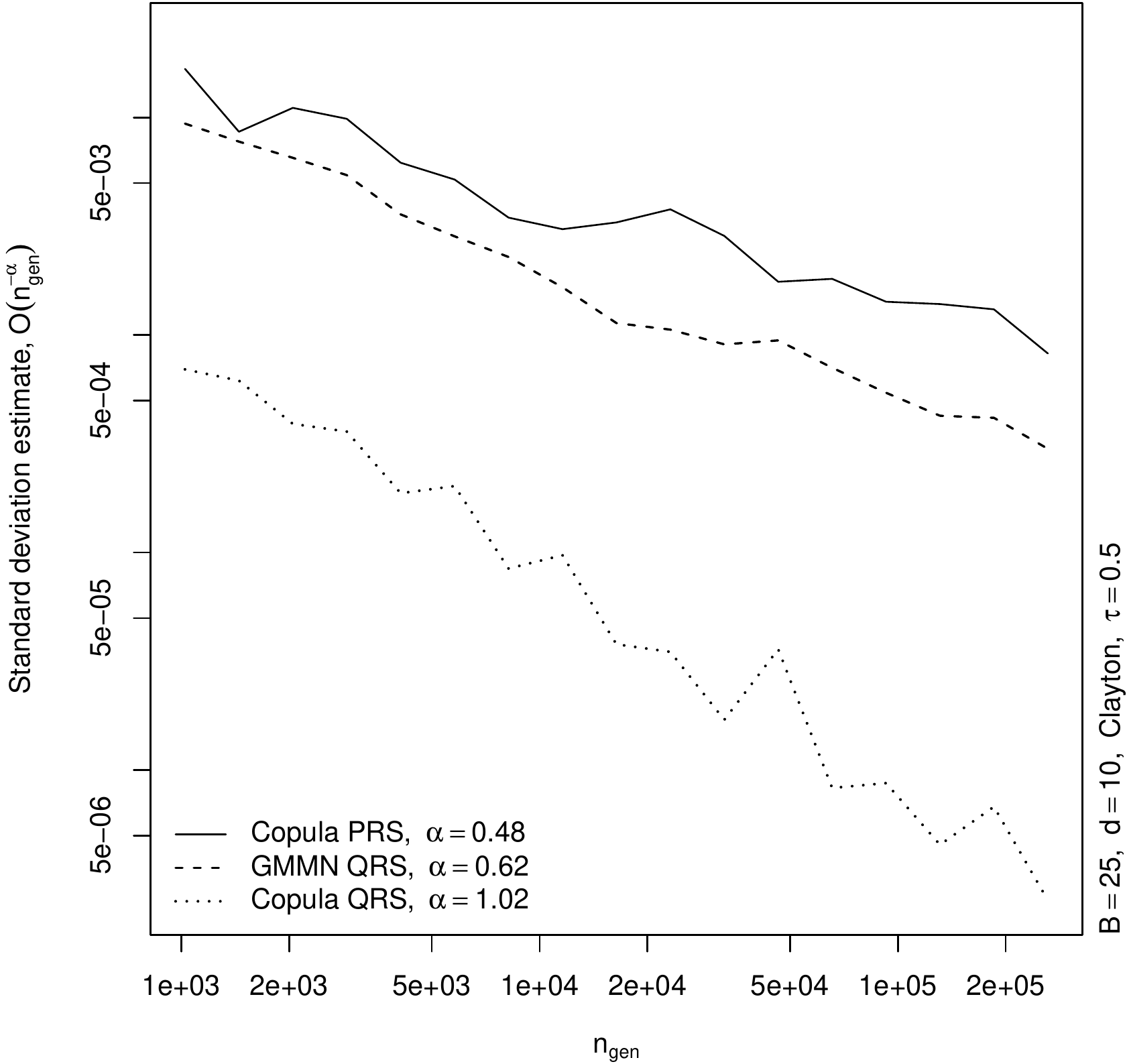}
  \includegraphics[width=0.32\textwidth]{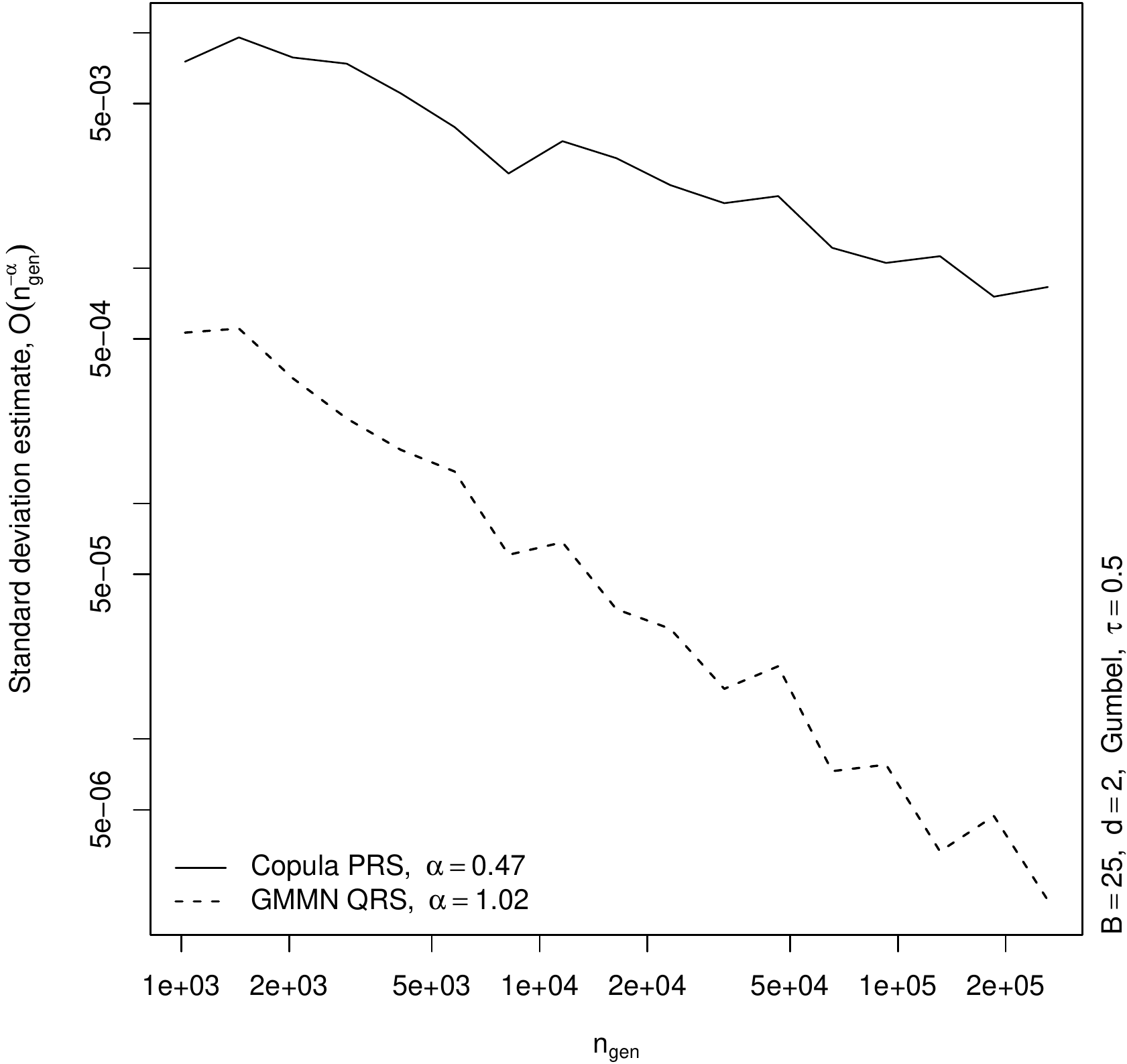}\hfill
  \includegraphics[width=0.32\textwidth]{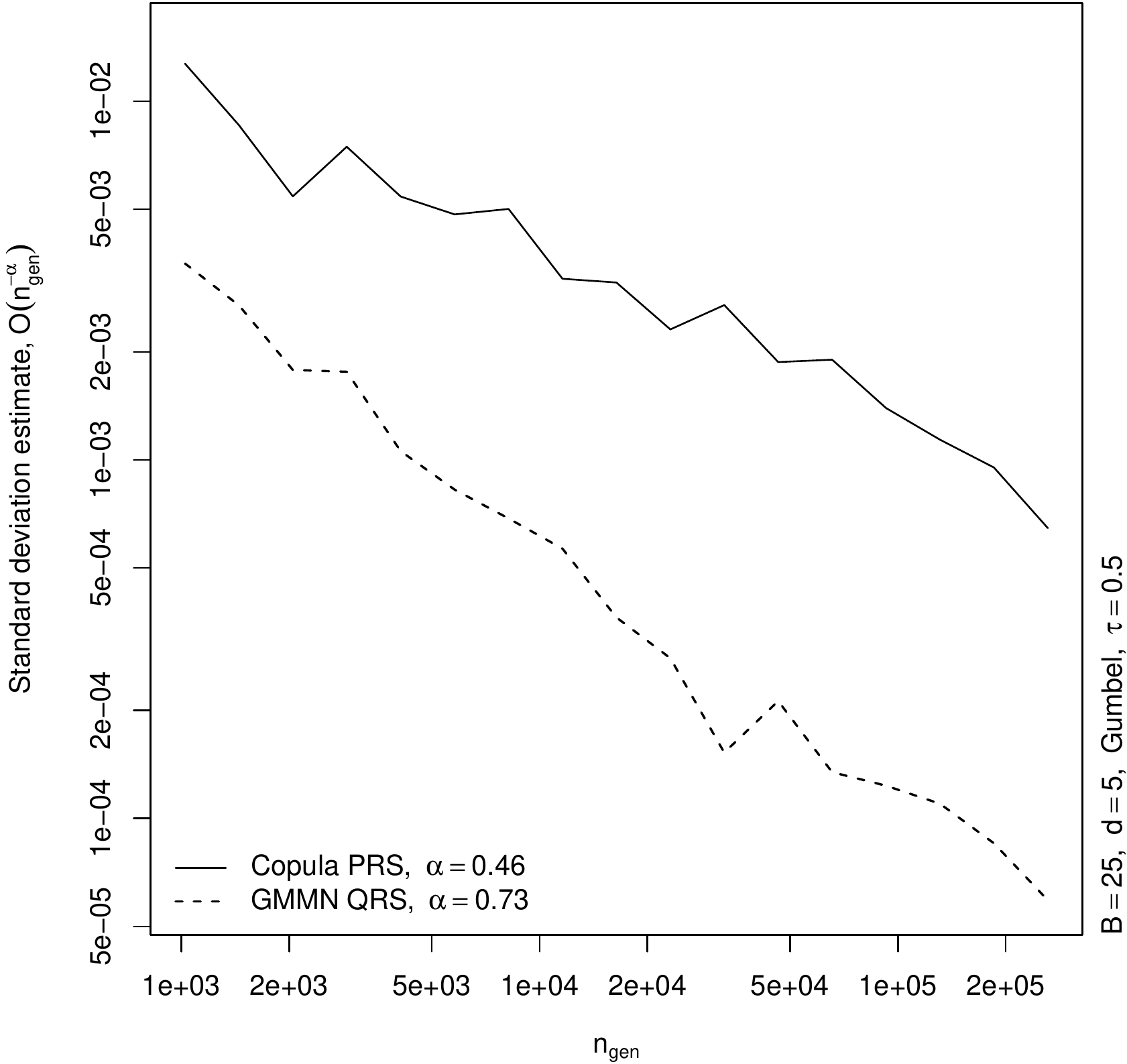}\hfill
  \includegraphics[width=0.32\textwidth]{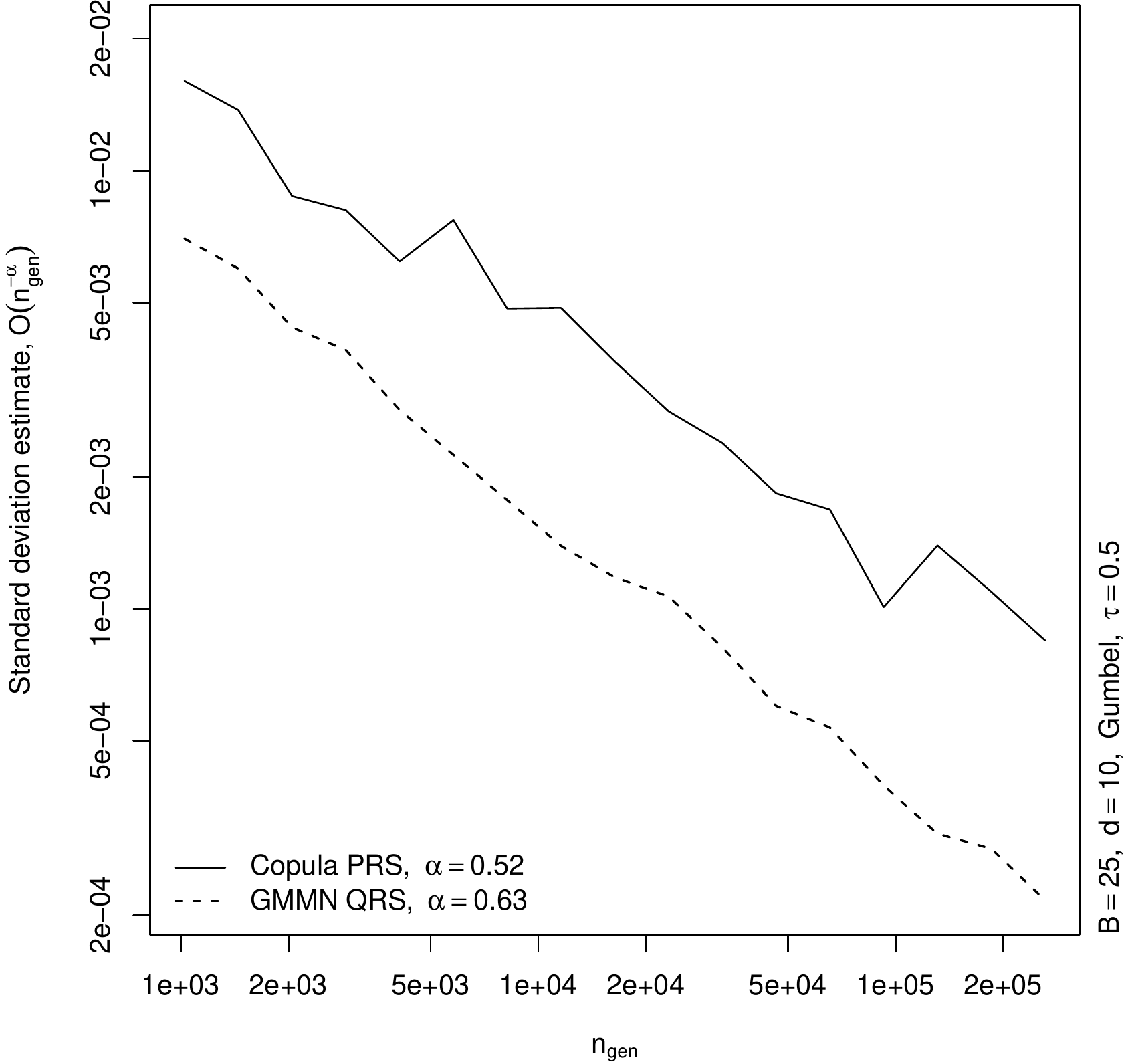}
  \caption{Standard deviation estimates based on $B=25$ replications
    for estimating $\E(\Psi_1(\bm{U}))$, expectation of the Sobol' g function,
    via MC based on a pseudo-random sample (PRS), via the copula RQMC estimator
    (whenever available; rows 1--2 only) and via the GMMN RQMC estimator.  Note
    that in rows 1--3, $d\in\{2,5,10\}$.}\label{fig:testfn2:scramble}
\end{figure}

Figure~\ref{fig:testfn2:scramble} shows plots of standard deviation
estimates for estimating $\E(\Psi_1(\bm{U}))$ for $t_4$ copulas (top row),
Clayton (middle row) and Gumbel (bottom row) copulas in dimensions $d=2$ (left
column), $d=5$ (middle column) and $d=10$ (right column). For the $t_4$ and
Clayton copulas we numerically compare the efficiency of the GMMN RQMC estimator
(with legend label ``GMMN QRS'') with the copula RQMC estimator based on the CDM
method (with legend label ``Copula QRS'') and the MC estimator (with legend
label ``Copula PRS''). For the Gumbel copula, however, the CDM approach
(``Copula QRS'') is computationally not feasible; see
  Section~\ref{sec:t:AC:mix} and Appendix~\ref{sec:timings}.  The legend of
each plot also provides the regression coefficient $\alpha$ which indicates the
convergence rate of each estimator.

From Figure~\ref{fig:testfn2:scramble}, we observe that the GMMN RQMC estimator
clearly outperforms the MC estimator. Naturally, so does the copula RQMC
estimator for the copulas for which it is available. On the one hand, the rate
of convergence of the GMMN RQMC estimator reduces with increasing copula
dimensions; see also the decreasing regression coefficients $\alpha$
when moving from the two- to the ten-dimensional case. On the other hand, the
GMMN RQMC estimator still outperforms the MC estimator.

\subsection{An example from risk management practice}\label{sec:ex:QRM}

Consider modeling the dependence of $d$ risk-factor changes (for example,
logarithmic returns) of a portfolio; see \cite[Chapters~2, 6 and
7]{mcneilfreyembrechts2015}. We now demonstrate the efficiency of our GMMN RQMC
estimator by considering the expected shortfall of the aggregate loss, a popular
risk measure in quantitative risk management practice.

Specifically, if $\bm{X}=(X_1,\dots,X_d)$ denotes a random vector of risk-factor
changes with $\N(0,1)$ margins, the aggregate loss is $S=\sum_{j=1}^{d}
X_j$. The \emph{expected shortfall} $\ES_{0.99}$ at level $0.99$ of $S$ is given by
\begin{align*}
  \ES_{0.99}(S)=\frac{1}{1-0.99} \int_{0.99}^{1}F^{-1}_{S}(u)\,\rd u=\E\bigl(S\,|\,S>F_S^{-1}(0.99)\bigr)=\E(\Psi_2(\bm{X})),
\end{align*}
where $F^{-1}_{S}$ denotes the quantile function of $S$. As done previously,
various copulas will be used to model the dependence between the components of
$\bm{X}$.

\begin{figure}[htbp]
  \centering
  \includegraphics[width=0.31\textwidth]{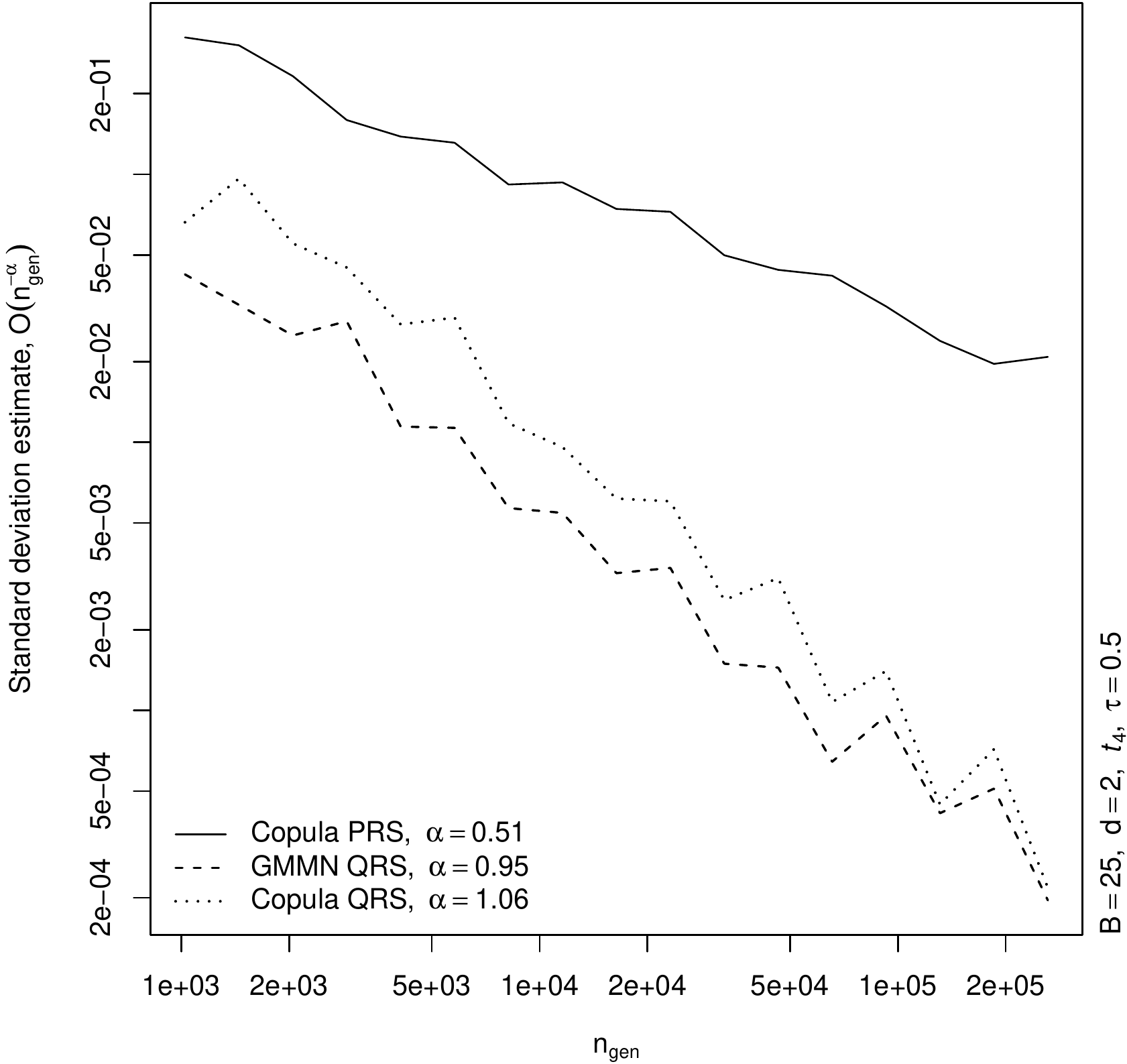}\hfill
  \includegraphics[width=0.31\textwidth]{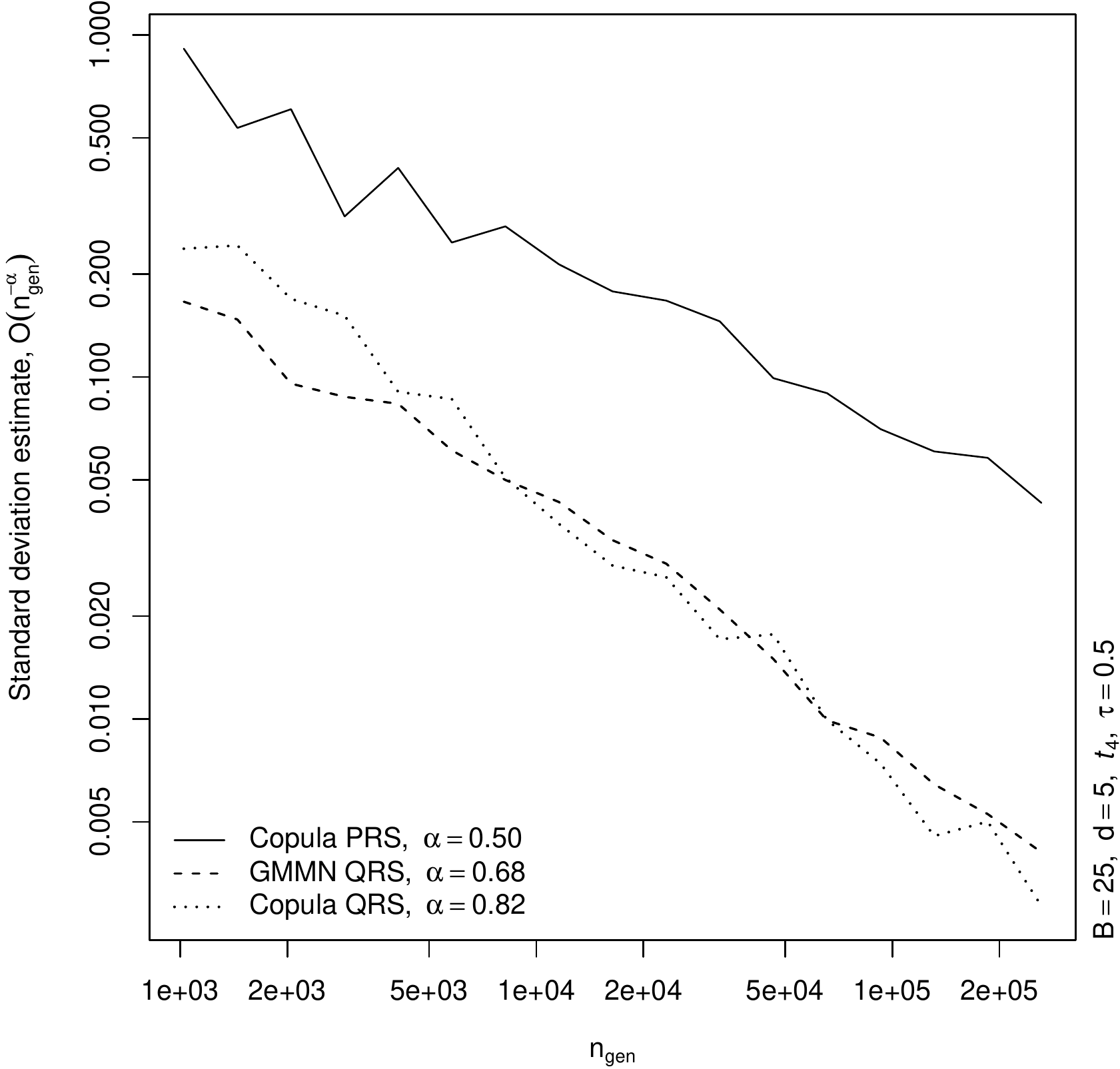}\hfill
  \includegraphics[width=0.31\textwidth]{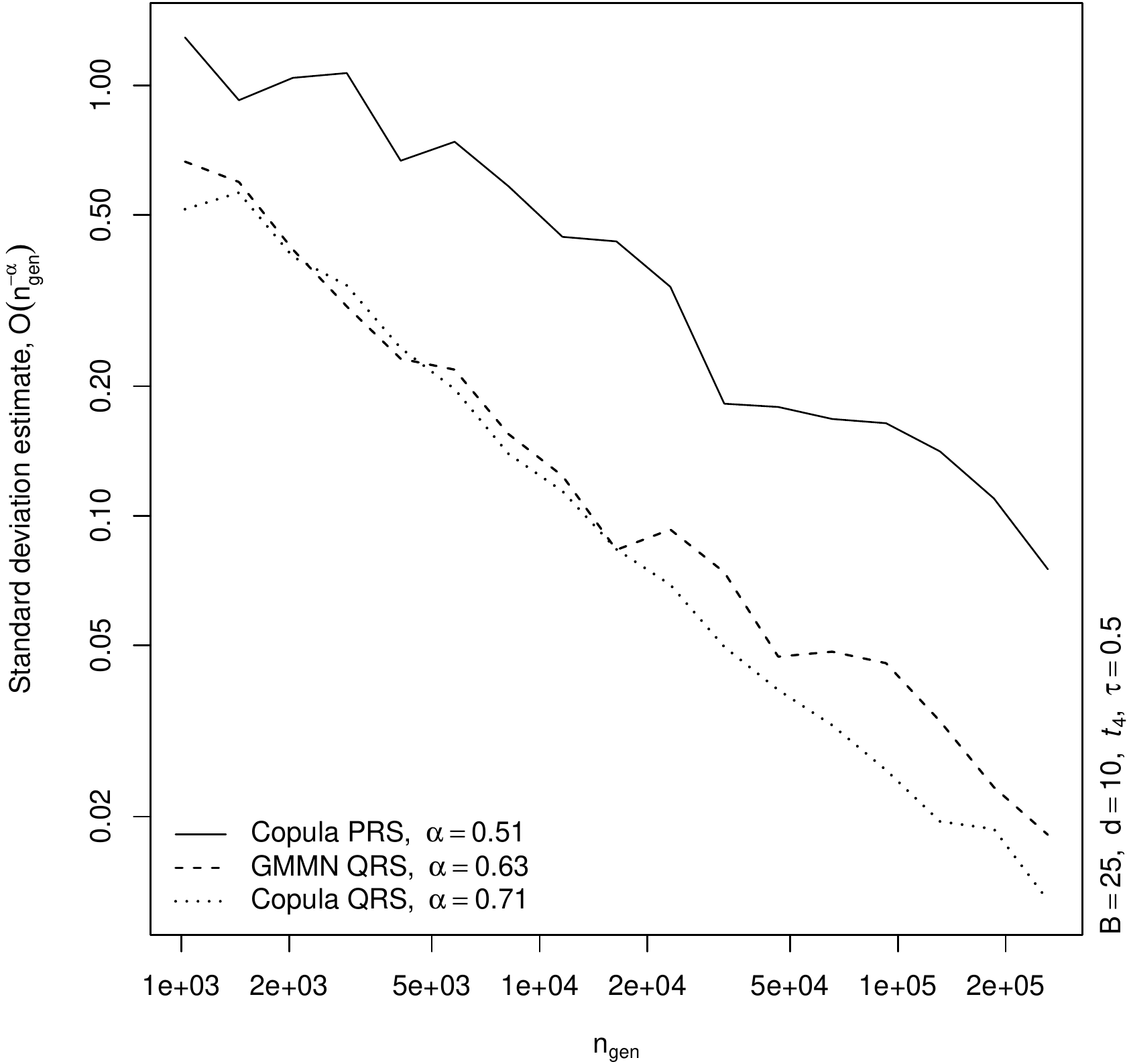}
  \includegraphics[width=0.31\textwidth]{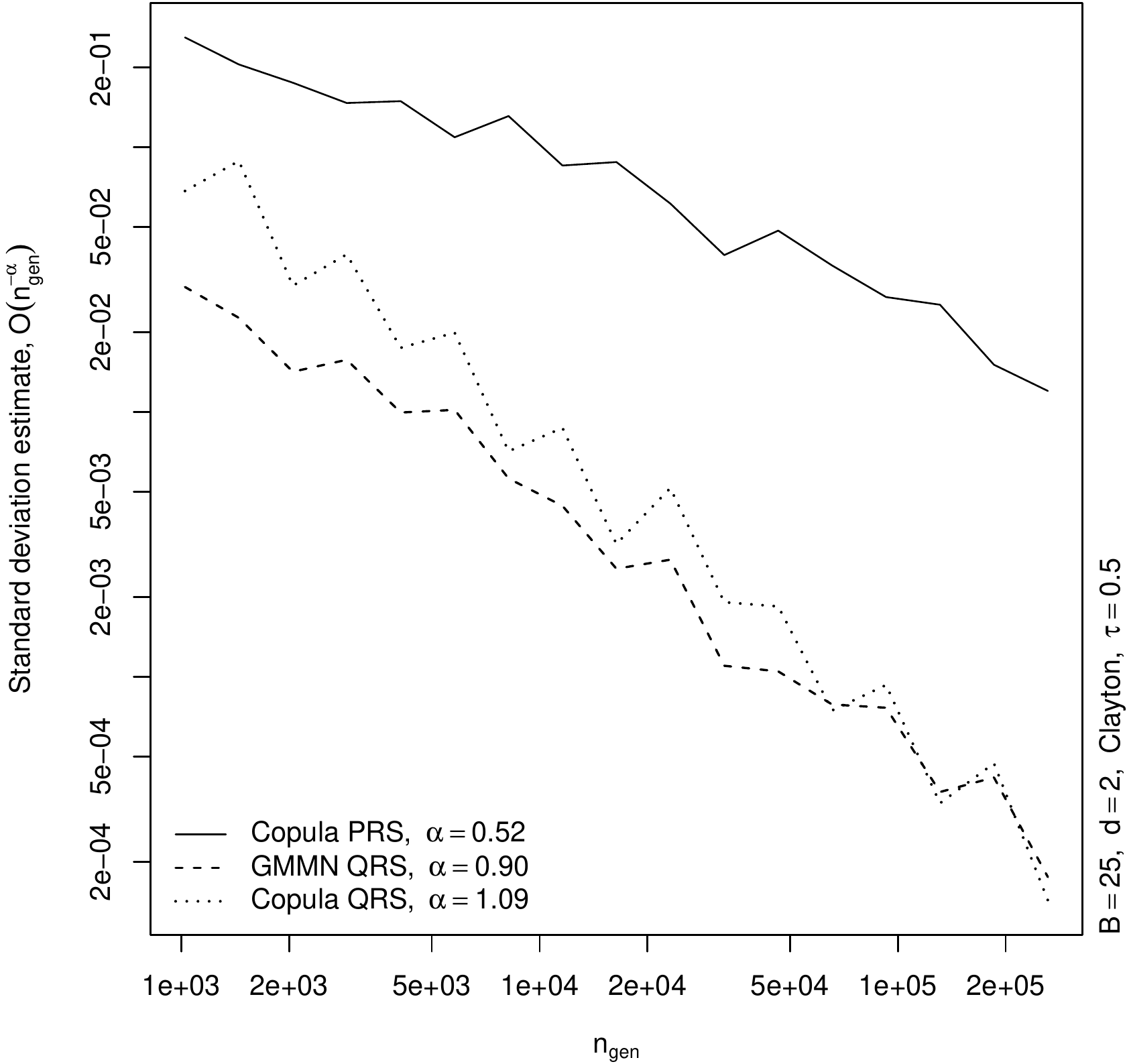}\hfill
  \includegraphics[width=0.31\textwidth]{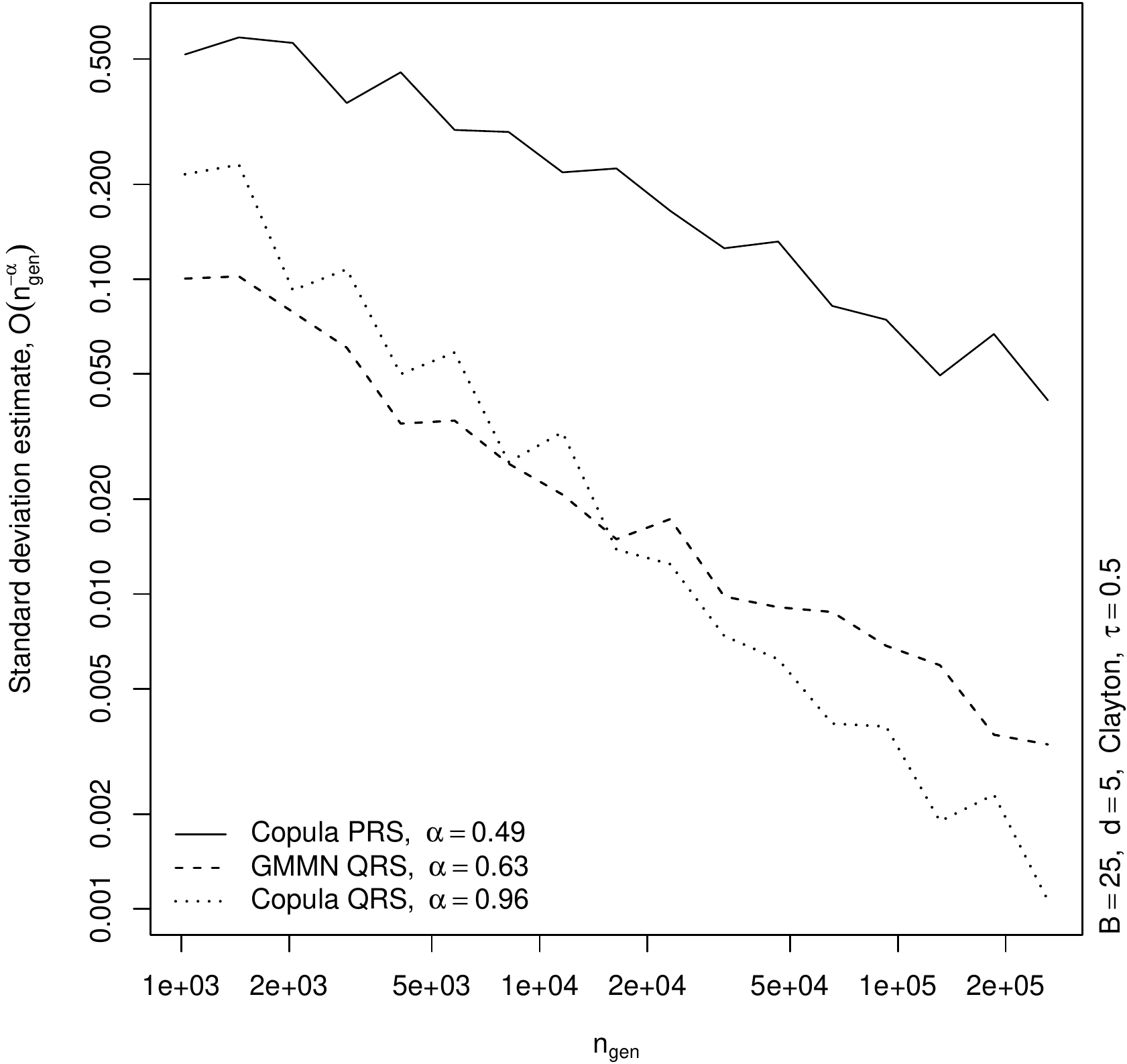}\hfill
  \includegraphics[width=0.31\textwidth]{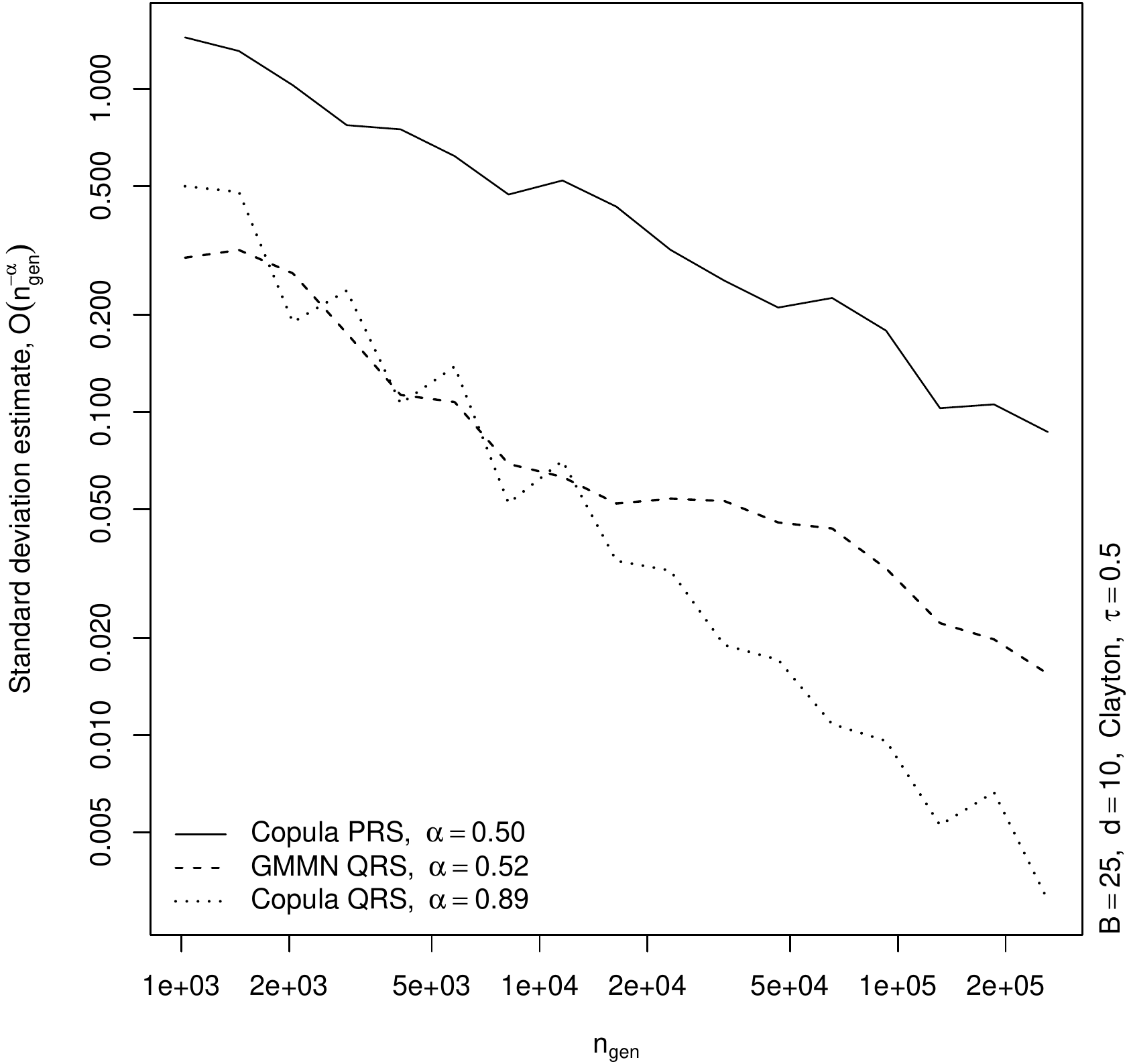}
  \includegraphics[width=0.31\textwidth]{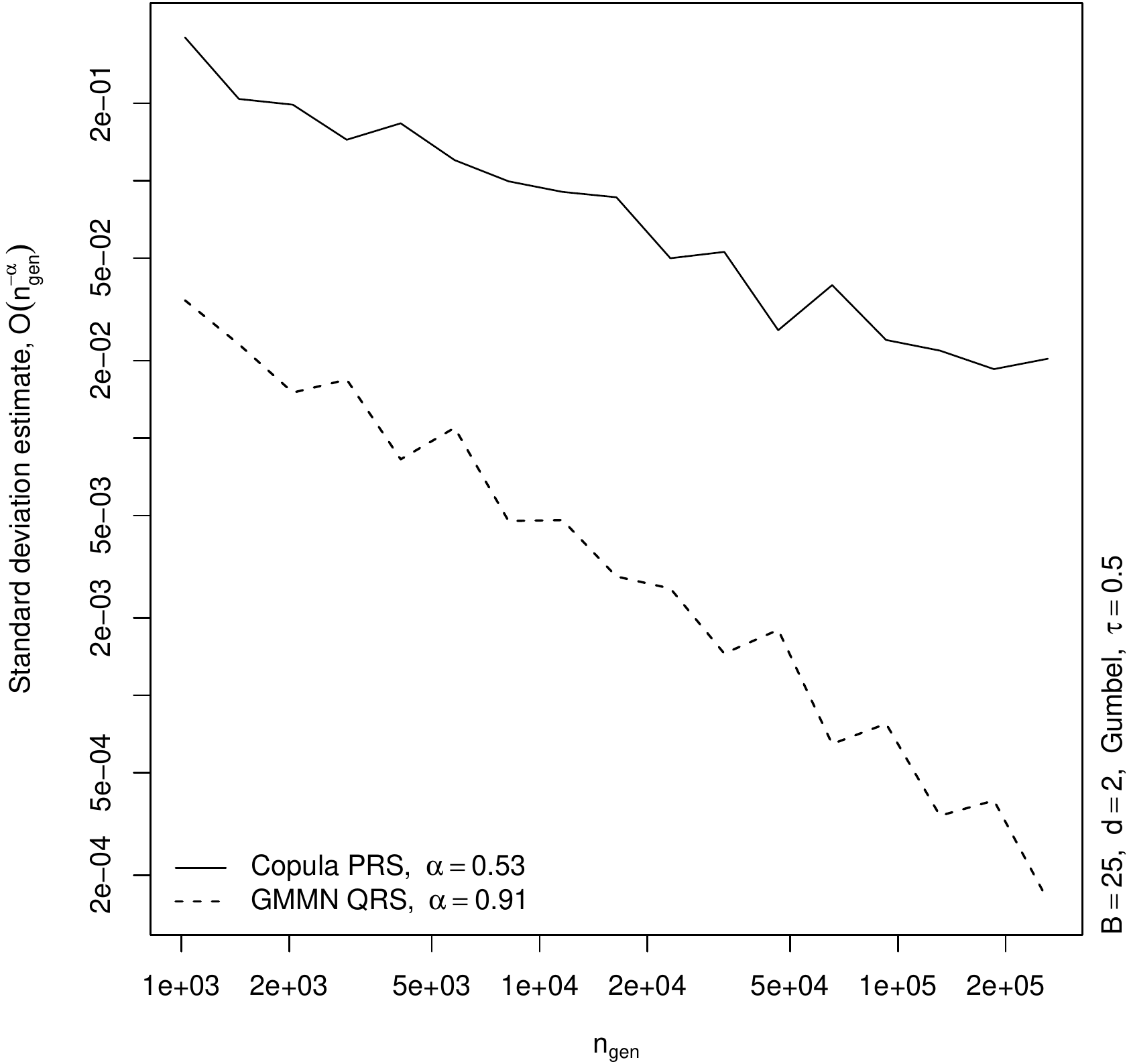}\hfill
  \includegraphics[width=0.31\textwidth]{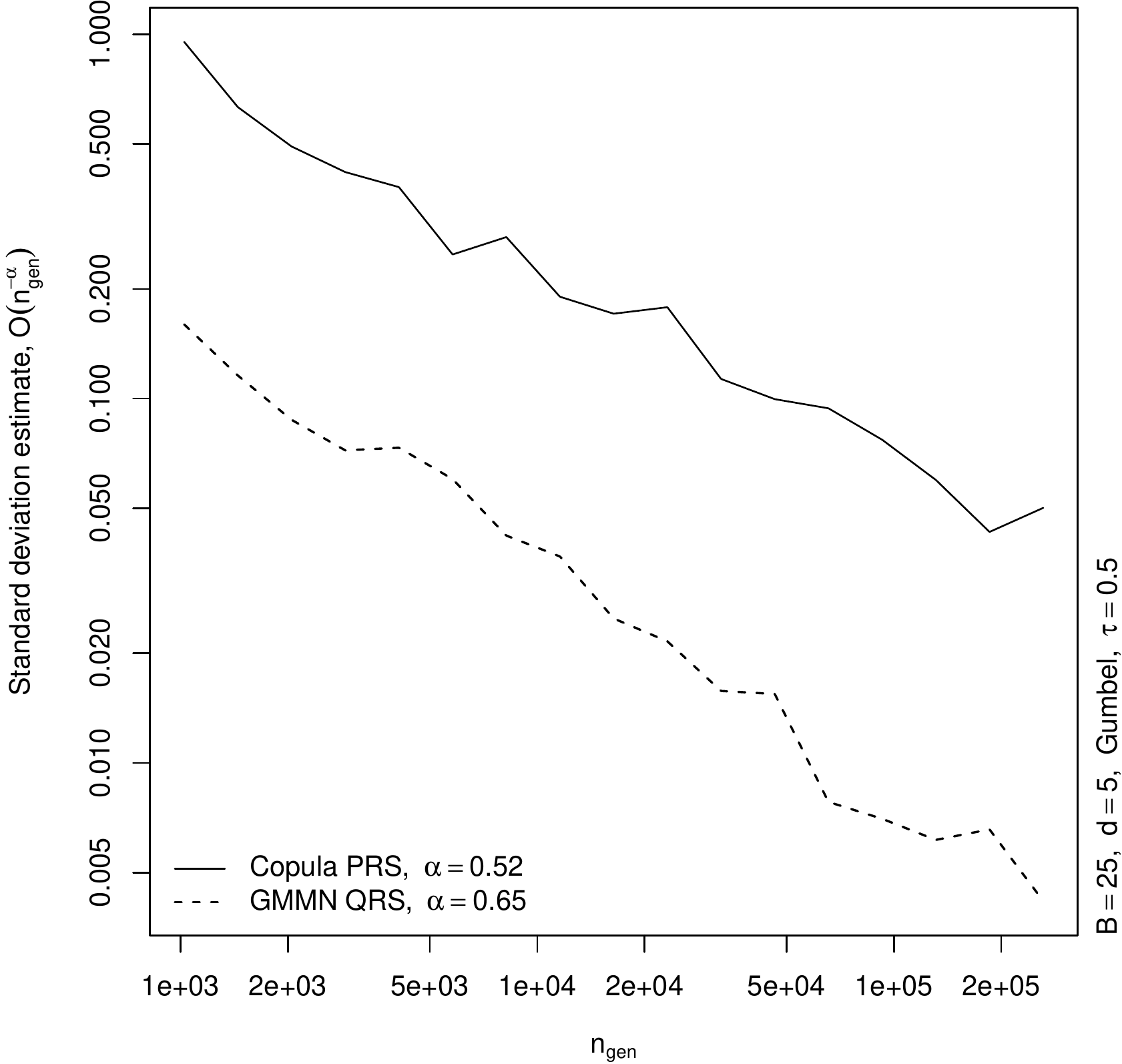}\hfill
  \includegraphics[width=0.31\textwidth]{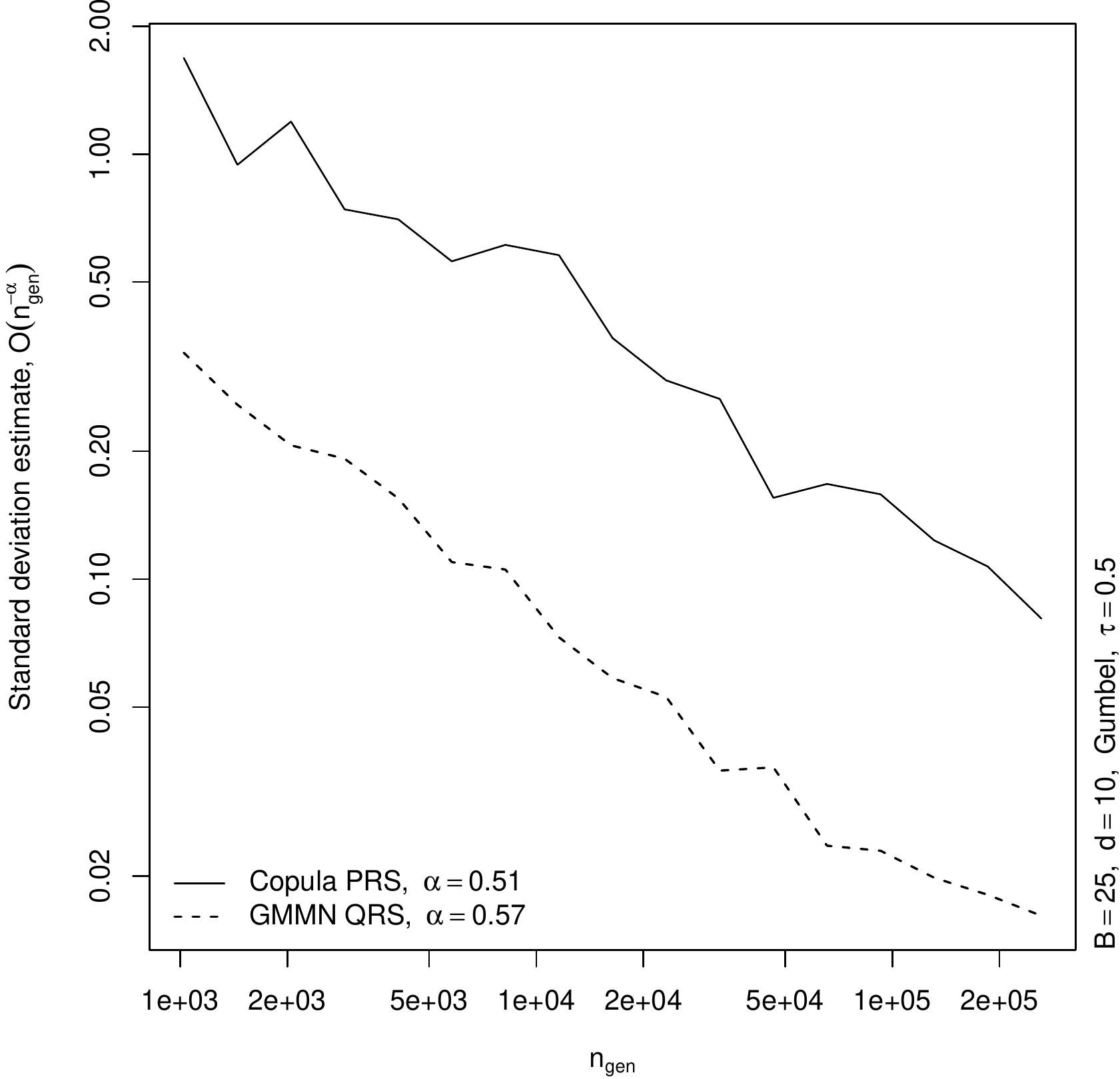}
  \includegraphics[width=0.31\textwidth]{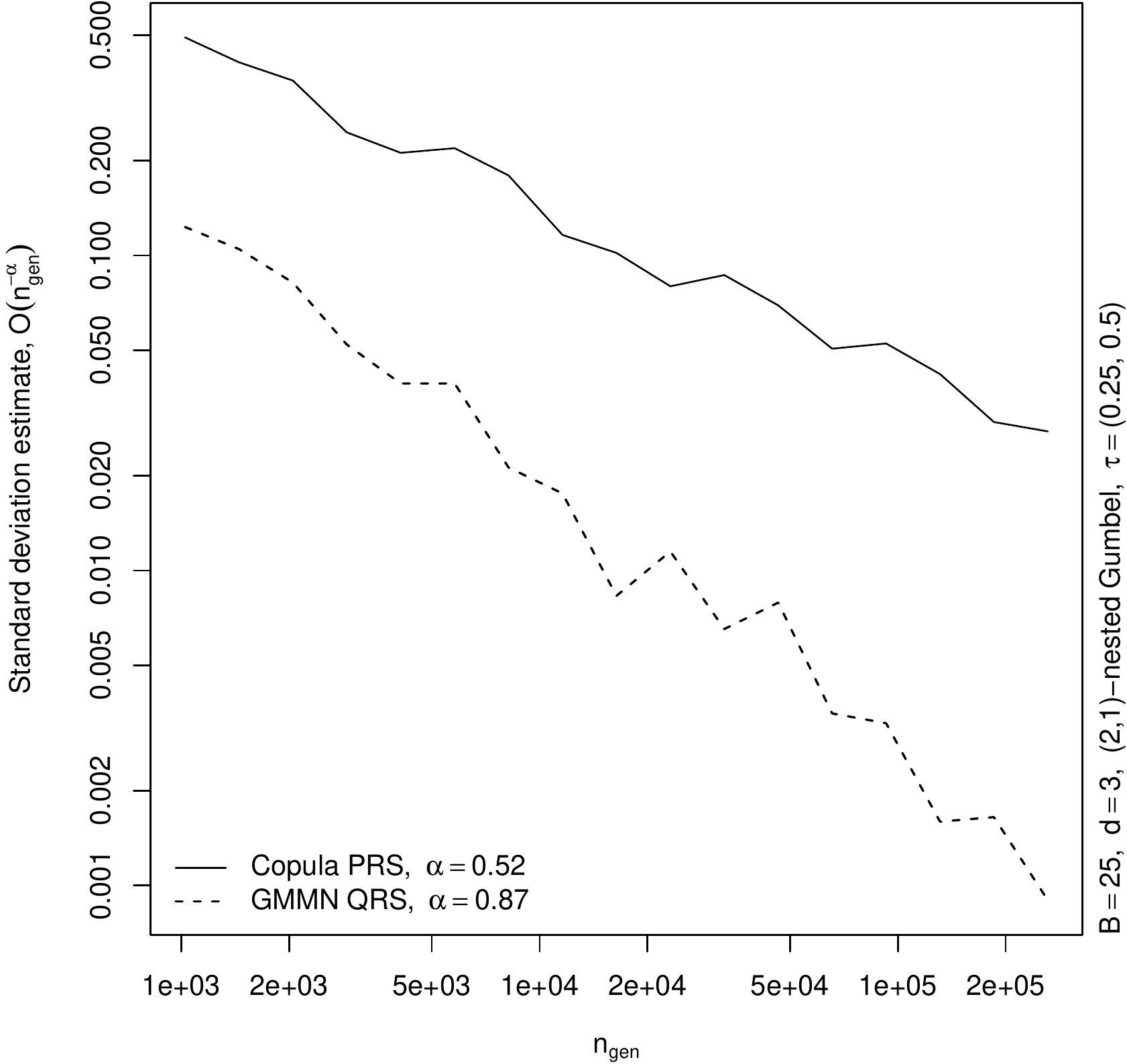}\hfill
  \includegraphics[width=0.31\textwidth]{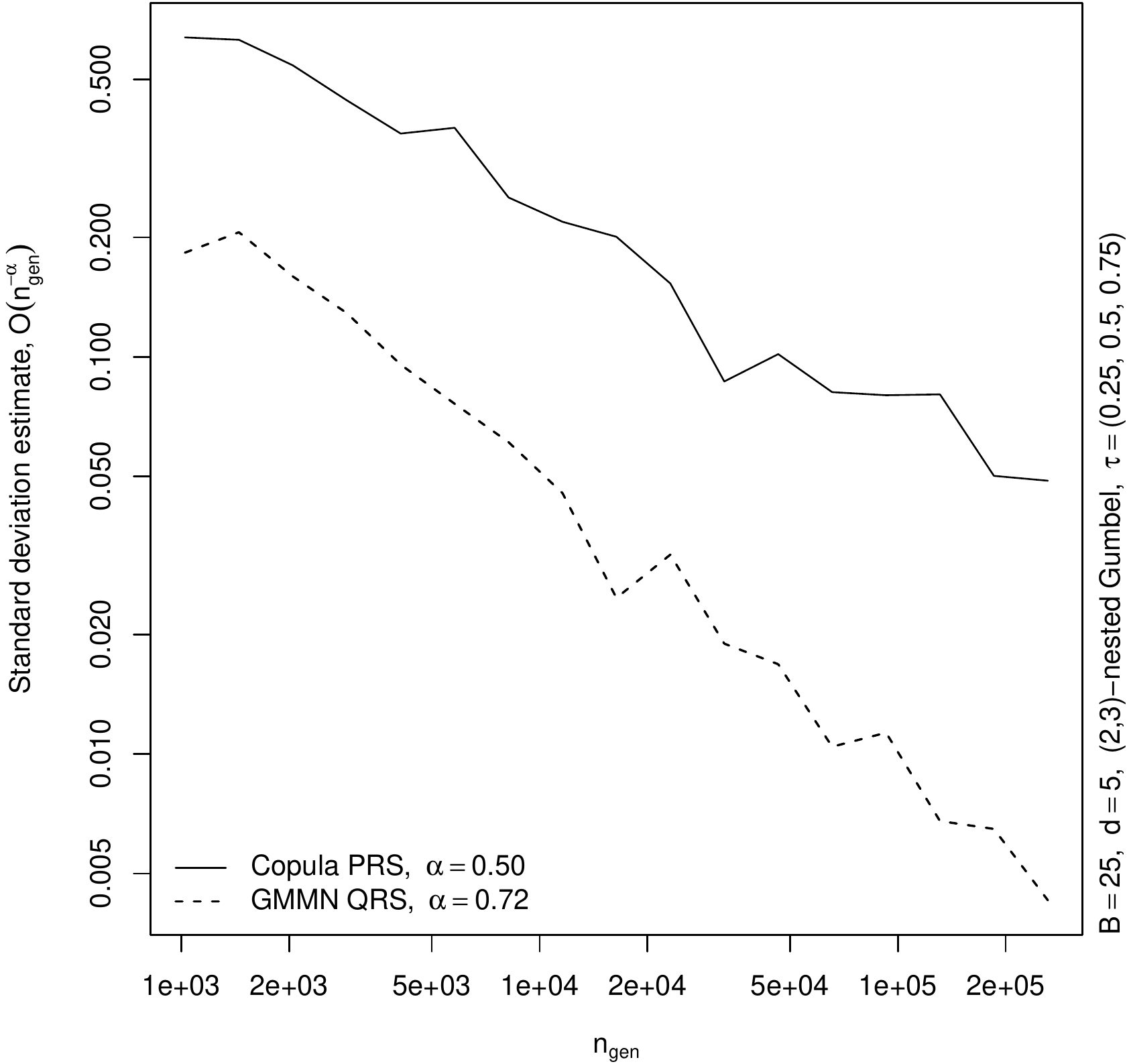}\hfill
  \includegraphics[width=0.31\textwidth]{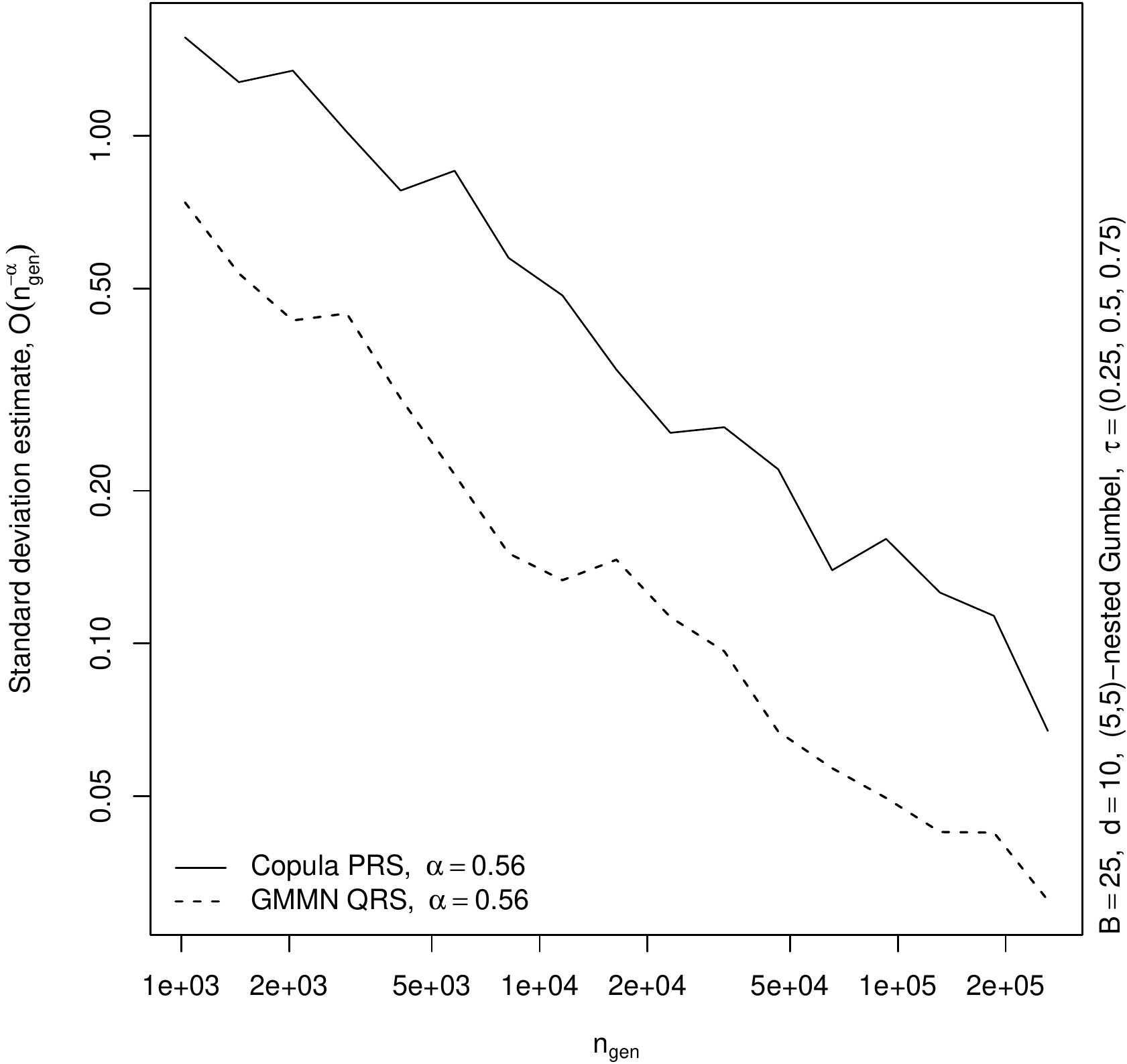}
  \caption{Standard deviation estimates based on $B=25$ replications for
    estimating $\E(\Psi_2(\bm{X}))$, the expected shortfall $\ES_{0.99}(S)$, via MC based on a PRS, via the copula RQMC
    (whenever available; rows 1--2 only) and via the GMMN RQMC estimator.  Note
    that in rows 1--3, $d\in\{2,5,10\}$, whereas in row 4,
    $d\in\{3,5,10\}$.}\label{fig:aggregateES:scramble}
\end{figure}

Figure~\ref{fig:aggregateES:scramble} shows plots of standard deviation
estimates for estimating $\E(\Psi_2(\bm{X}))$. The first three rows contain
results for the same copula models as considered in Section~\ref{sec:testfun}. The
fourth row contains results for nested Gumbel copula models with dimension
$d = 3$ (left column), $d = 5$ (middle column) and $d = 10$ (right column). The
specific hierarchical structures and parameterization have been described
earlier in Section \ref{sec:GMMN:accuracy};
note that there is no quasi-random sampling procedure known for these
  copulas.  We can observe from the plots that the GMMN RQMC
estimator outperforms the MC estimator. Similar as before, we see a decrease in
the convergence rate of the GMMN RQMC estimator as the copula dimension
increases, although it still outperforms the MC estimator.

\section{A real-data example}\label{sec:realdata}
In this section, we present examples based on real-data to show how our
  method can be useful in practice. To this end, we consider applications from
  finance and risk management. Such applications often involve the modeling of
dependent multivariate return data in order to estimate various quantities
$\mu$ of interest. In this context, utilizing GMMNs for dependence
modeling can yield two key advantages. Firstly, GMMNs are highly flexible and
hence can model dependence structures not adequately captured by prominent
parametric copula models; see, e.g., \cite{hofertoldford2018} for the latter
point.  Secondly, as demonstrated in Sections~\ref{sec:GMMN:copula}
and~\ref{sec:conv:analysis}, one can readily generate GMMN quasi-random samples
to achieve variance reduction when estimating $\mu$; this is especially
advantageous as oftentimes oversimplified parametric models are chosen just so
that this can be achieved.  In this section, we model asset portfolios
consisting of S\&P 500 constituents to showcase these advantages. All results
are reproducible with the demo \texttt{GMMN\_QMC\_data} of the \R\ package
\texttt{gnn}.

\subsection{Portfolios of S\&P 500 constituents}
We consider 10 constituent time series from the S\&P 500 in the time period from
1995-01-01 to 2015-12-31. The selected constituents include three stocks from
the information technology sector --- Intel Corp.\ (INTC), Oracle Corp.\ (ORCL)
and International Business Machines Corp.\ (IBM); three stocks from the
financial sector --- Capital One Financial Corp.\ (COF), JPMorgan Chase \& Co.\
(JPM) and American International Group Inc (AIG); and four stocks from the
industrial sector --- 3M Company (MMM), Boeing Company (BA), General Electric
(GE) and Caterpillar Inc.\ (CAT). We also investigate sub-portfolios of stocks
with dimensions $d=5$ (consisting of INTC, ORCL, IBM, COF and AIG) and $d=3$
(consisting of INTC, IBM and AIG). The data are obtained from the \R\ package
\texttt{qrmdata}.

To account for marginal temporal dependencies, we follow the copula--GARCH
approach \parencite{jondeaurockinger2006,patton2006} and model each marginal time
series of log-returns by a $\ARMA(1,1)$--$\GARCH(1,1)$
model with standardized $t$ innovation distributions (\emph{deGARCHing}).
We then extract the marginal standardized residuals (i.e., the realizations of
the standardized $t$ innovations) and compute, for each of the three portfolios,
their pseudo-observations for the purpose of modeling the cross-sectional
dependence among the corresponding portfolio's log-return series.

\subsection{Assessing the fit of dependence models}
As models for the pseudo-observations of each of the three portfolios we use
prominent parametric copulas (Gumbel, Clayton, exchangeable normal,
unstructured normal, exchangeable $t$ and unstructured $t$) and GMMNs of the
same architecture and with the same training setup as detailed in
Section~\ref{sec:details}. The rather small number of training data points
($\ntrn=5287$) allows us to use $\nbat=\ntrn$ here and hence directly train with the entire dataset.
All parametric copulas are fitted using the maximum
pseudo-likelihood method; see \cite[Section~4.1.2]{hofertkojadinovicmaechleryan2018}.

To evaluate the fit of a dependence model, we use a Cram\'{e}r-von-Mises type
test statistic introduced by \cite{remillardscaillet2009} to assess
the equality of two empirical copulas. This statistic is defined as
\begin{align*}
  S_{\ntrn,\ngen} = \int_{[0,1]^d} \bigg(\sqrt{\frac{1}{\ngen}+\frac{1}{\ntrn}}\bigg)^{-1}\bigg(C_{\ngen}(\bm{u})-C_{\ntrn}(\bm{u})\bigg)^2\,\rd \bm{u},
\end{align*}
where $C_{\ngen}(\bm{u})$ and $C_{\ntrn}(\bm{u})$ are the empirical copulas, defined according to \eqref{eq:emp:copula}, of the $\ngen$ samples generated from the fitted
dependence model and the $\ntrn$ pseudo-observations used to fit the dependence
model, respectively. For how $S_{\ntrn,\ngen}$ is evaluated, see
\cite[Section~2]{remillardscaillet2009}.

For each of the three portfolios and each of the seven dependence models
considered, we compute $B$ realizations of $S_{\ntrn,\ngen}$ based on
$\ngen=10\,000$ pseudo-random samples generated from the fitted
dependence model and the $\ntrn=5287$ pseudo-observations of each
portfolio considered. Figure~\ref{fig:gof:data} displays box plots depicting the distribution of
$S_{\ntrn,\ngen}$ for each portfolio and dependence
model. Across all three portfolios, we can observe that the distribution of
$S_{\ntrn,\ngen}$ based on GMMN models is concentrated closer to zero than those
of each fitted parametric copula. In fact, the difference in distributions of
$S_{\ntrn,\ngen}$ realizations between GMMN models and the best parametric
copula model (a $t$-copula with unstructured correlation matrix) is most noticeable
for $d=10$, where an adequate fit becomes more challenging for these
parametric copulas. For each of the three portfolios, a GMMN provides the
best fit. Hence, we use these fitted GMMNs to model the
  underlying dependence structure for the three portfolios in each of three
  applications considered next.
\begin{figure}[htbp]
  \centering
  \includegraphics[width=0.32\textwidth]{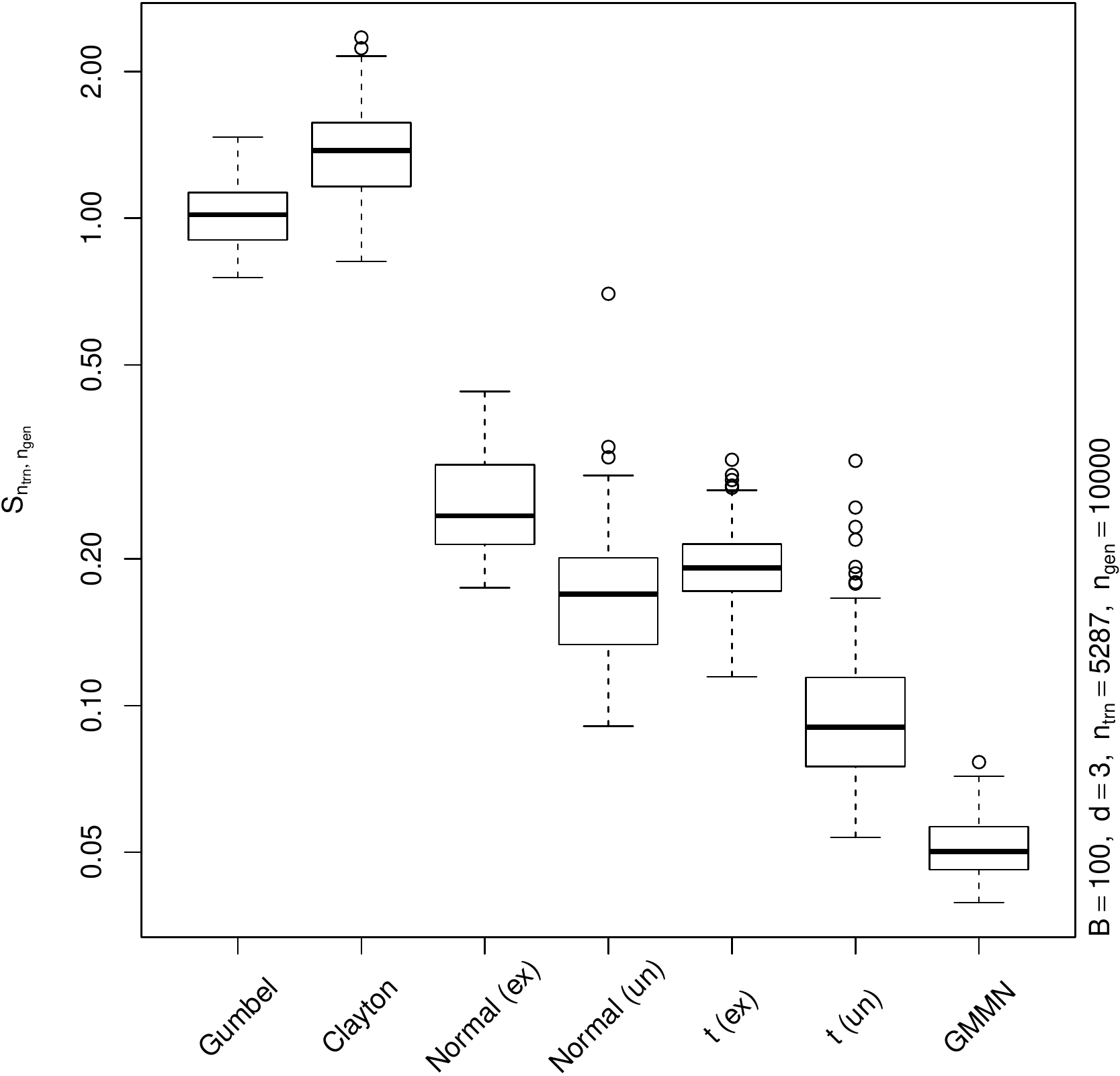}\hfill
  \includegraphics[width=0.32\textwidth]{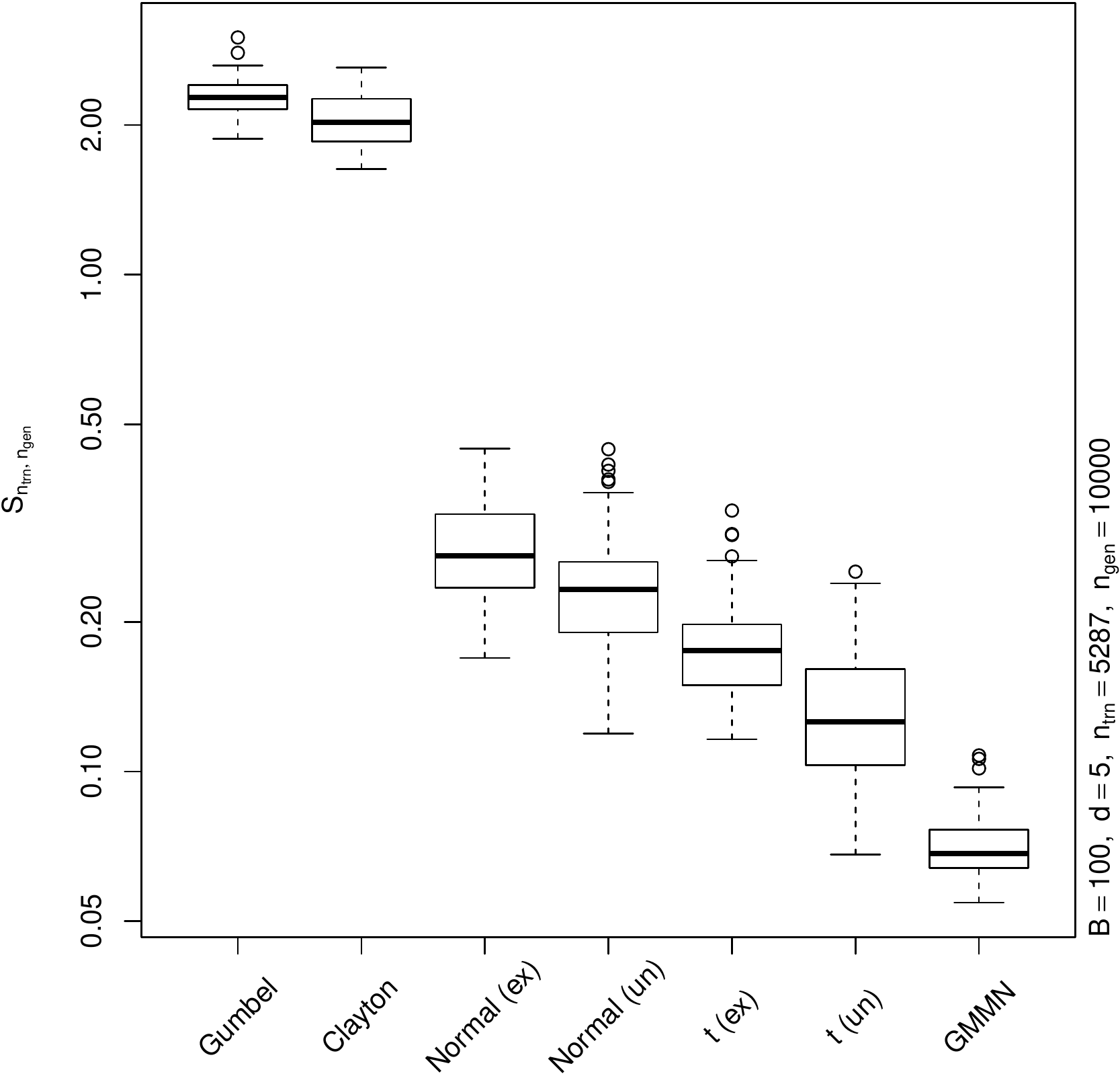}\hfill
  \includegraphics[width=0.32\textwidth]{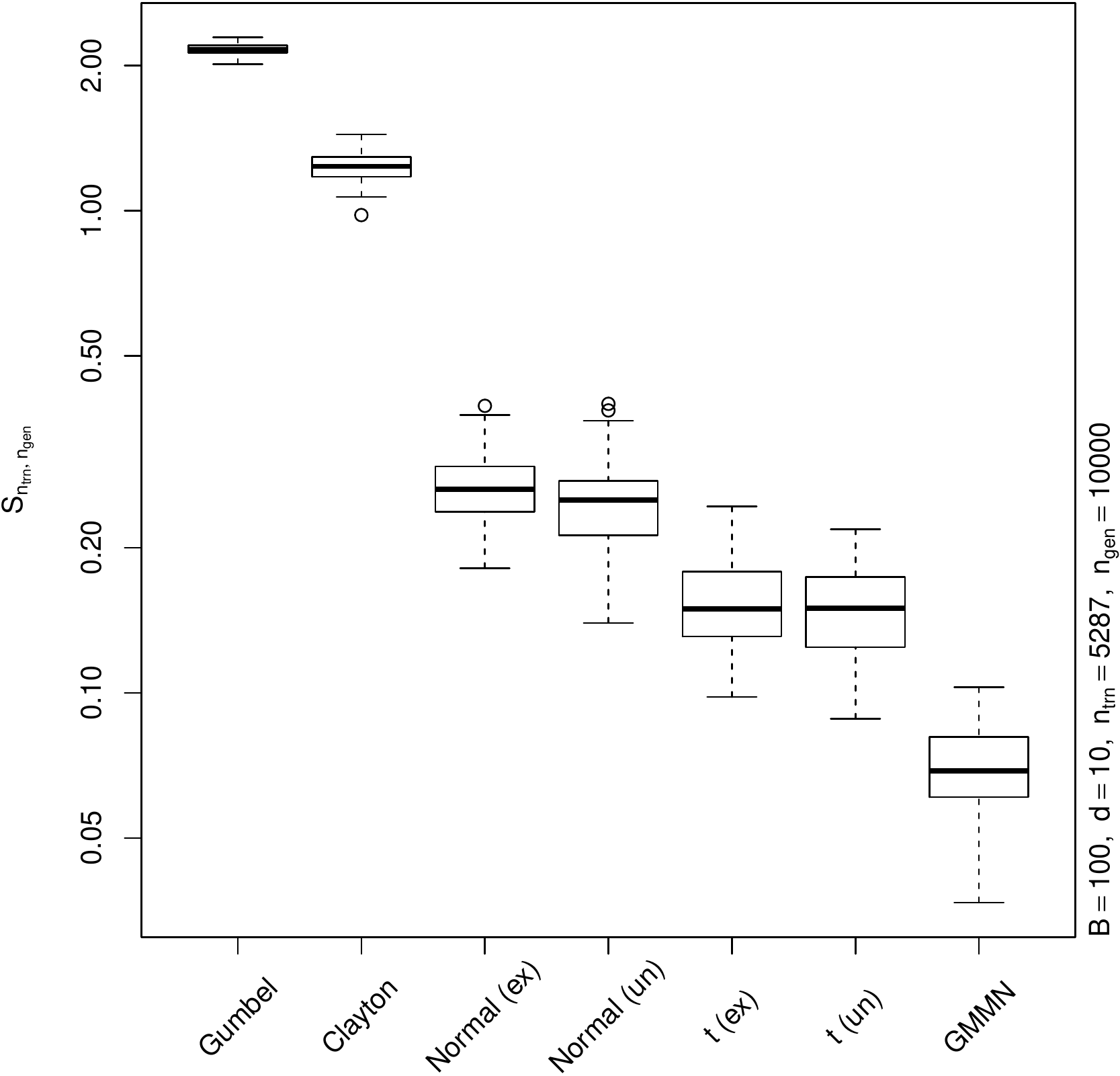}
  \caption{Box plots based on $B=100$ realizations of $S_{\ntrn,\ngen}$ computed
    for portfolios of dimensions $d=3$ (left), $d=5$ (middle) and $d=10$
    (right) and for each fitted dependence model using a pseudo-random sample of
    size $\ngen=10\,000$ from each corresponding fitted model.}\label{fig:gof:data}
\end{figure}

\subsection{Assessing the variance reduction effect}\label{sec:data:vr}
In three applications we study the variance reduction effect of our GMMN RQMC
estimator $\hat{\mu}^{\text{NN}}_{\ngen}$ computed from quasi-random samples in comparison to the GMMN MC estimator
$\hat{\mu}^{\text{NN},\text{MC}}_{\ngen}$ computed from pseudo-random samples.

Our first application concerns the estimation of the expected shortfall
$\mu=\ES_{0.99}(S)$ for $S=\sum_{j=1}^dX_j$ as in Section~\ref{sec:ex:QRM},
where the margins of $\bm{X}=(X_1,\dots,X_d)$ are now the fitted standardized $t$ distributions
as obtained by deGARCHing and the dependence structure is the previously fitted
GMMN. This is a classical task in risk management practice according to the
Basel guidelines. As a second application we consider a capital allocation
problem which concerns estimating how to allocate an amount of risk capital
(e.g., computed as $\ES_{0.99}(S)$) to each of $d$ business lines. Without loss
of generality, we consider one business line, the first, and estimate the
\emph{expected shortfall contribution}
$\mu=\AC_{1,0.99}=\E(X_1\,|\,S>F_S^{-1}(0.99))$ according to the Euler
principle; see \cite[Section~8.5]{mcneilfreyembrechts2015}. Our third
application comes from finance and concerns the estimation of the expected
payoff $\mu=\E(\exp(-r(T-t))\max\{(\sum_{j=1}^d S_{T,j})-K,0\})$ of a European
basked call option, where $r$ denotes the continuously compounded annual
risk-free interest rate, $t$ denotes the current time point, $T$ the maturity in years
and $K$ the strike price. We assume a Black--Scholes framework for the marginal stock prices
$(S_{T,1},\dots,S_{T,d})$ at maturity $T$, so
$S_{T,j}\sim\LN(\log(S_{t,j})+(r-\sigma_j^2/2)(T-t),\ \sigma_j^2(T-t))$, where
$S_{t,j}$ denotes the last available stock price of the $j$th constituent (i.e.,
the close price on 2015-12-31) and $\sigma_j$ denotes the volatility of the
$j$th constituent (estimated by the standard deviation of its log-returns over
the time period from 2014-01-01 to 2015-12-31). The dependence structure of
$(S_{T,1},\dots,S_{T,d})$ is modeled by the previously fitted GMMN.
Furthermore, we chose $t=0$ to be the last available point in the
data period considered (i.e., 2015-12-31), $T=1$ and $r=0.01$. The strike prices $K$
were chosen about 0.5\% higher than the average stock price of all stocks
in the respective portfolio at $t=0$.

For each of the three portfolios and for each of the three expectations $\mu$
considered, we compute $B=200$ realizations of the GMMN MC estimator
$\hat{\mu}^{\text{NN},\text{MC}}_{\ngen}$ and the GMMN RQMC estimator
$\hat{\mu}^{\text{NN}}_{\ngen}$, using $\ngen=10^5$ samples for both
estimators. Figure~\ref{fig:VRF:data} displays box plots of these realizations
of $\hat{\mu}^{\text{NN},\text{MC}}_{\ngen}$ (with x-axis label ``GMMN PRS'')
and of $\hat{\mu}^{\text{NN}}_{\ngen}$ (with x-axis label ``GMMN QRS'') for
$\ES_{0.99}$ (left column), $\AC_{1,0.99}$ (middle column), the expected payoff
of the basket call option (right column) and for the portfolio in dimension
$d=3$ (top row), $d=5$ (middle row) and $d=10$ (bottom row). Additionally, to
quantify the variance reduction effect of $\hat{\mu}^{\text{NN}}_{\ngen}$ over
$\hat{\mu}^{\text{NN},\text{MC}}_{\ngen}$, we report in the secondary y-axis of
each box plot the estimated variance reduction factor (VRF) --- namely, the sample variance of $\hat{\mu}^{\text{NN},\text{MC}}_{\ngen}$ over the sample variance of
$\hat{\mu}^{\text{NN}}_{\ngen}$ --- and the corresponding improvement in percentages.

From Figure~\ref{fig:VRF:data}, we observe that $\hat{\mu}^{\text{NN}}_{\ngen}$
is able to reduce the variance in all considered applications and across all
dimensions. While variance reduction is diminished in higher dimensions ($d=10$),
the GMMN RQMC estimator is still immensely useful in estimating
expectations $\mu$ for three reasons. Firstly, as demonstrated in the previous
section, GMMNs best fit the underlying dependence structure of the
data. Secondly, unlike many parametric copulas, we can generate quasi-random samples
independent of the type of dependence structure observed in the data. Finally,
we can generate GMMN quasi-random samples at no additional cost over
GMMN pseudo-random samples; see also Appendix~\ref{sec:timings}.

\begin{figure}[htbp]
  \centering
  \includegraphics[width=0.32\textwidth]{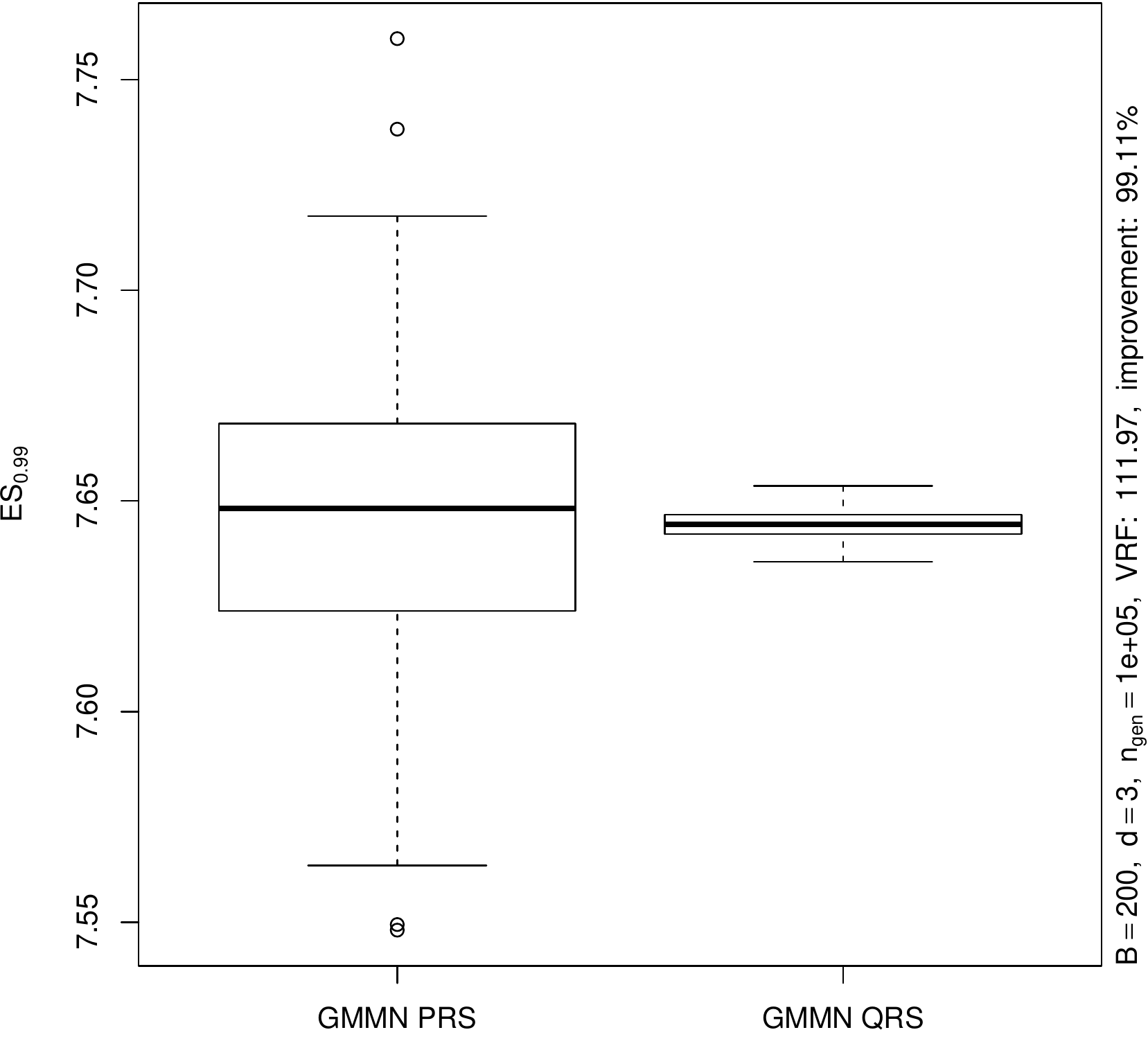}\hfill
  \includegraphics[width=0.32\textwidth]{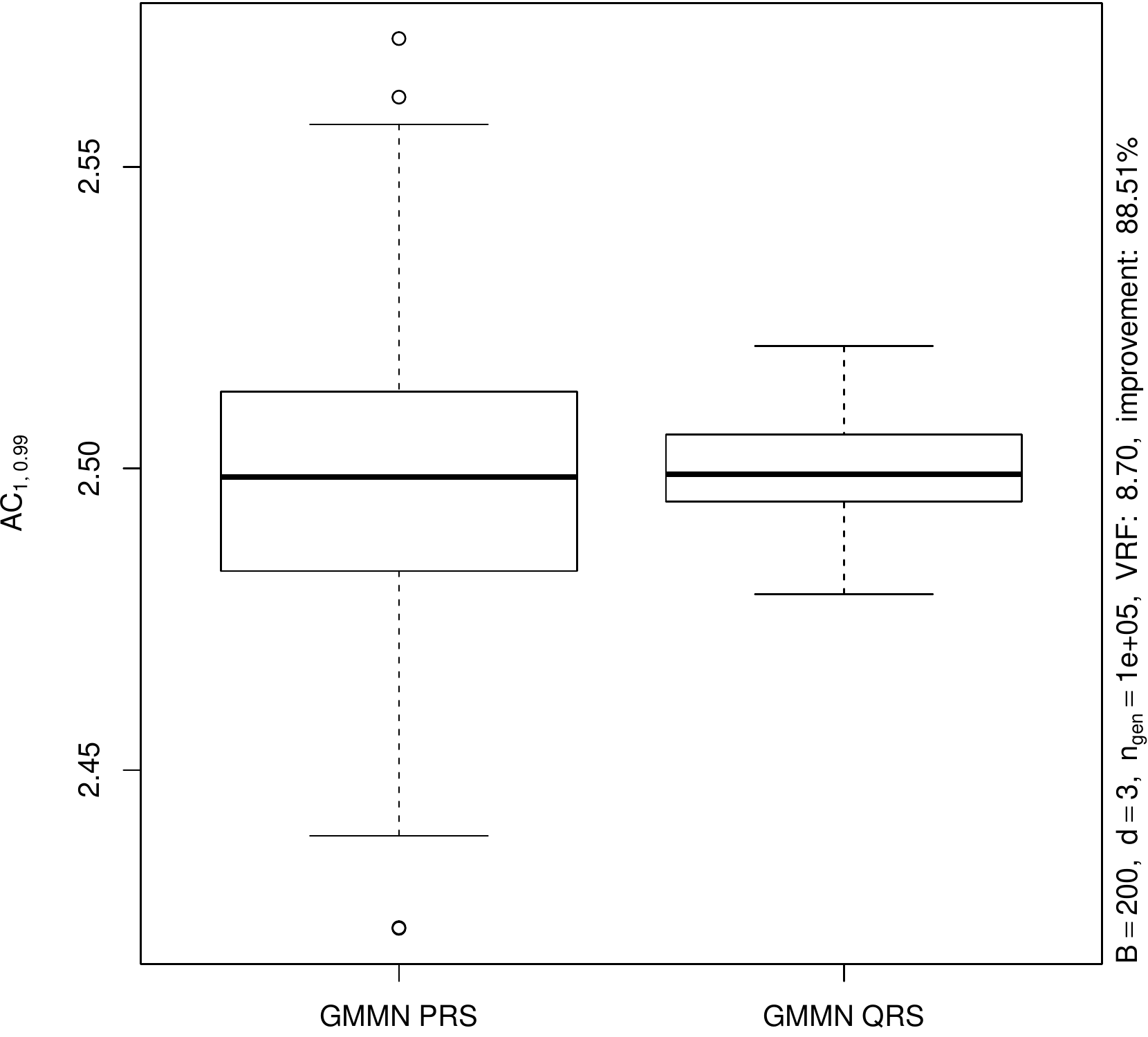}\hfill
  \includegraphics[width=0.32\textwidth]{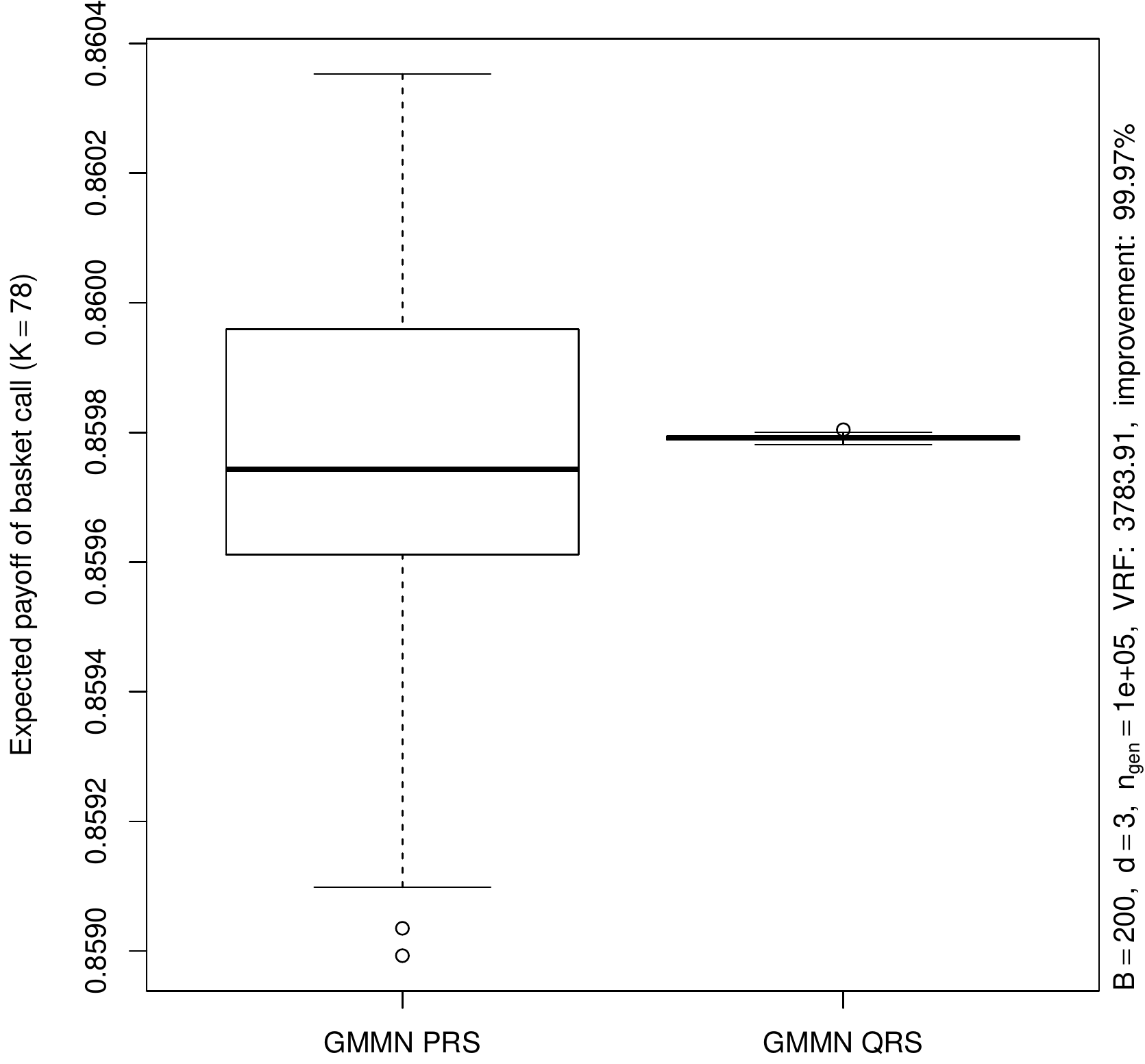}
  \includegraphics[width=0.32\textwidth]{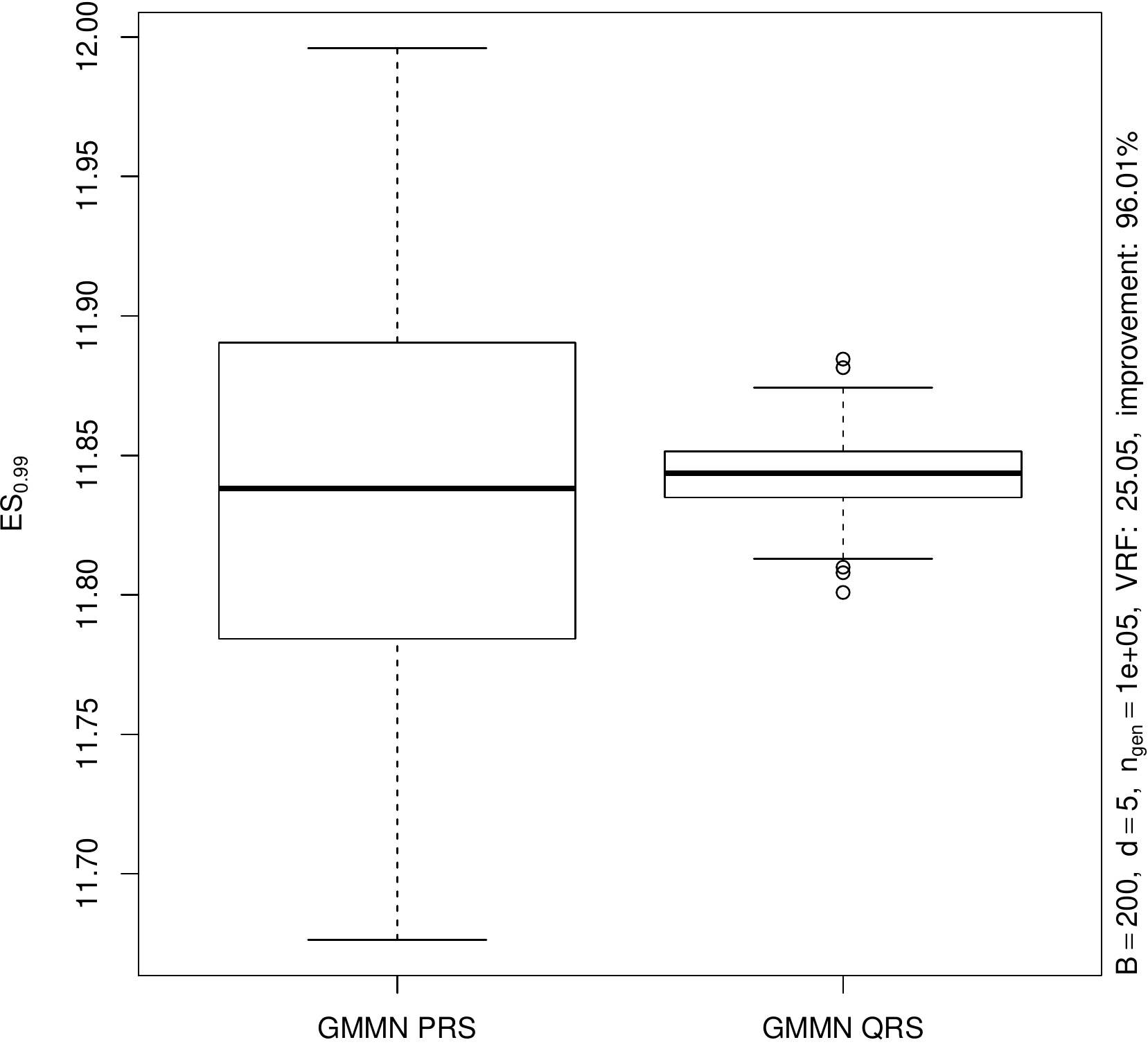}\hfill
  \includegraphics[width=0.32\textwidth]{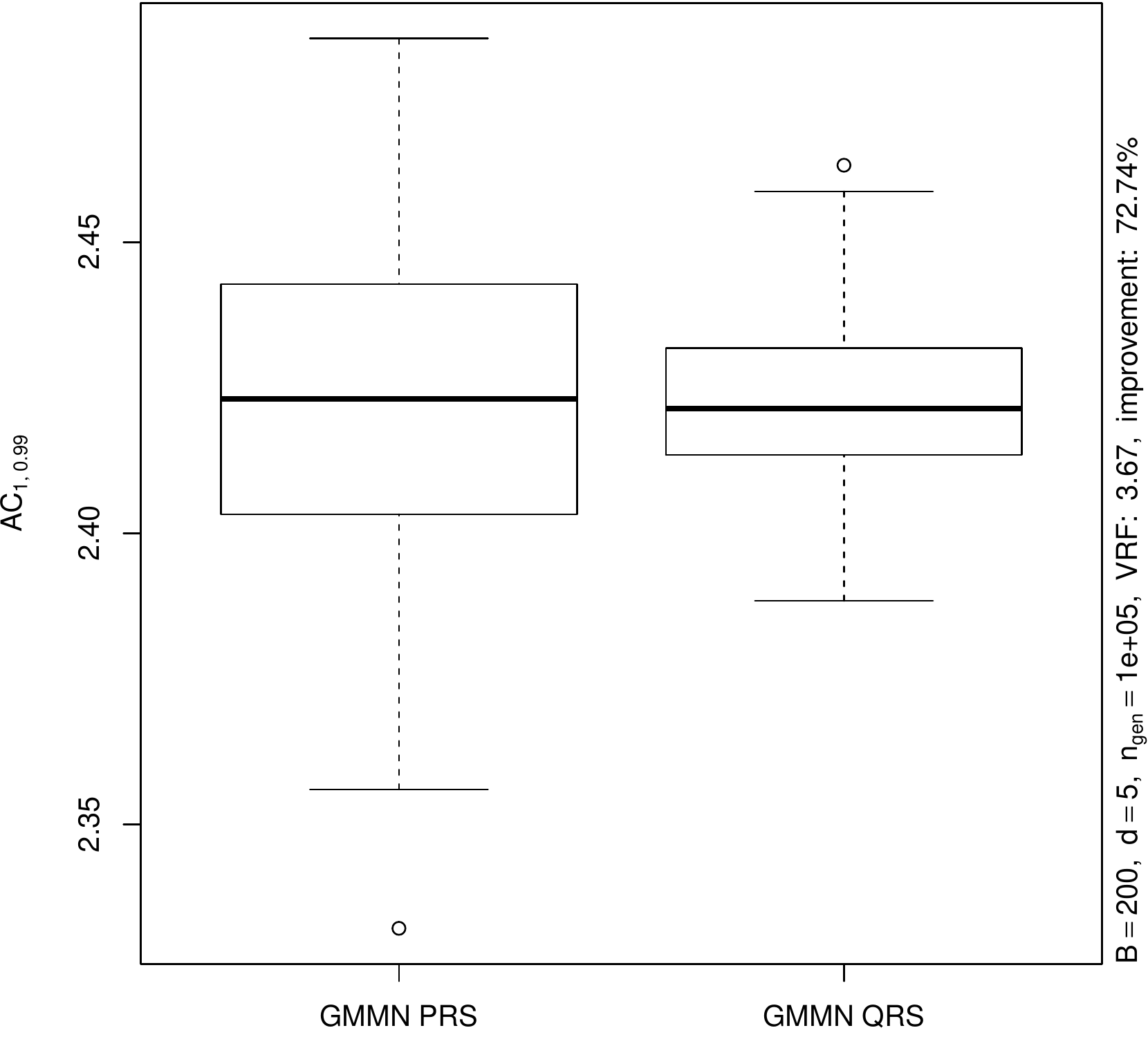}\hfill
  \includegraphics[width=0.32\textwidth]{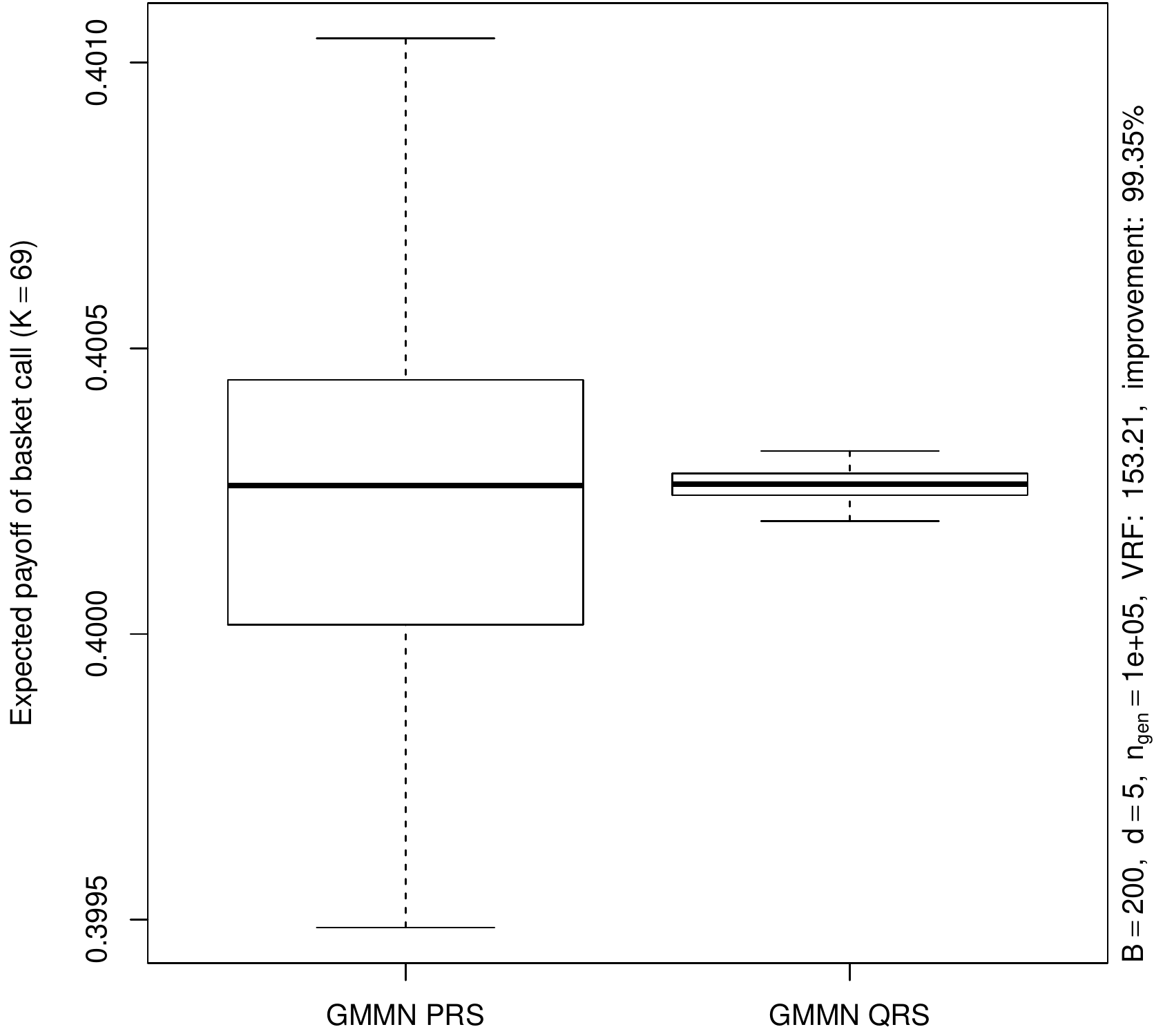}
  \includegraphics[width=0.32\textwidth]{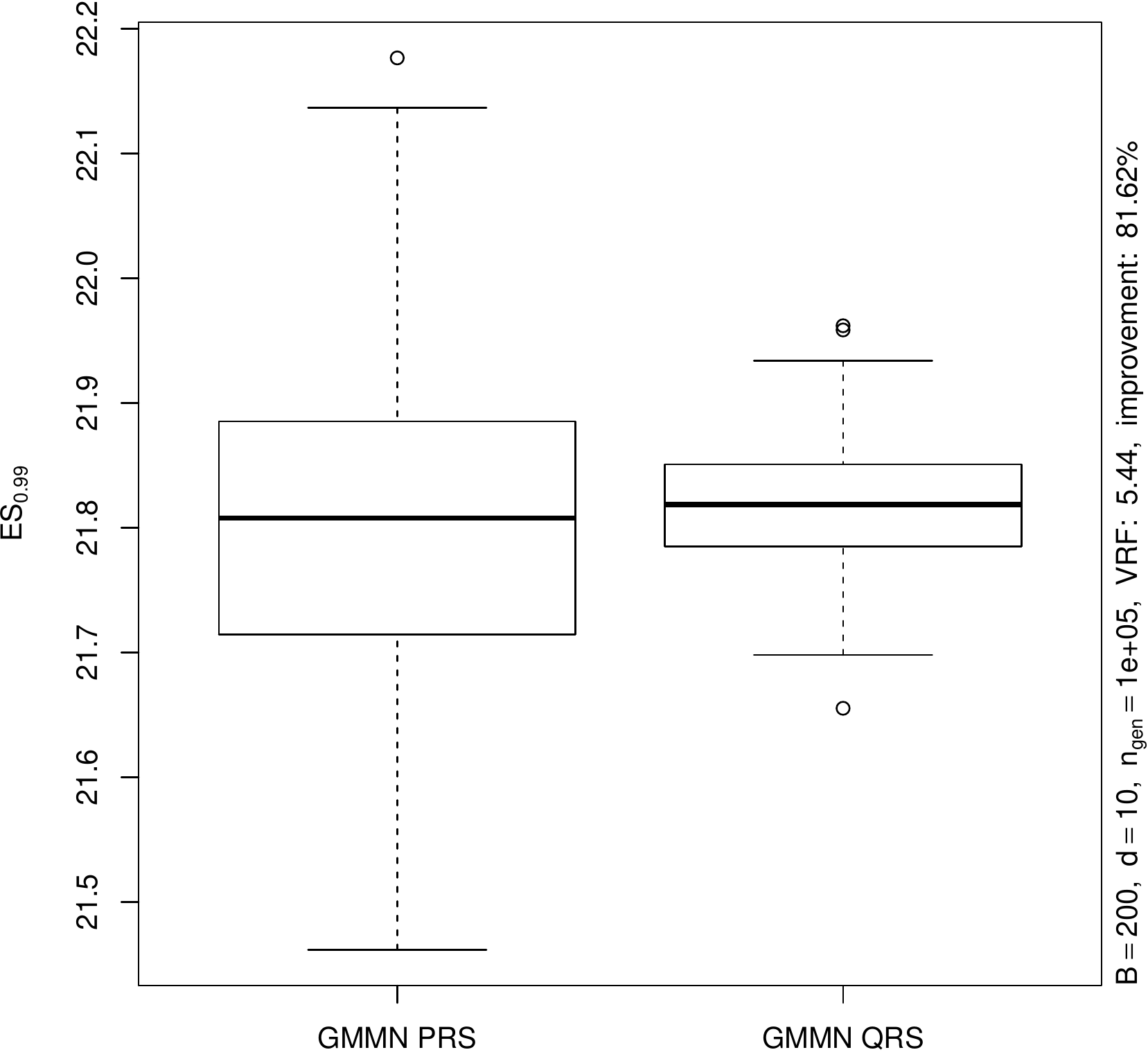}\hfill
  \includegraphics[width=0.32\textwidth]{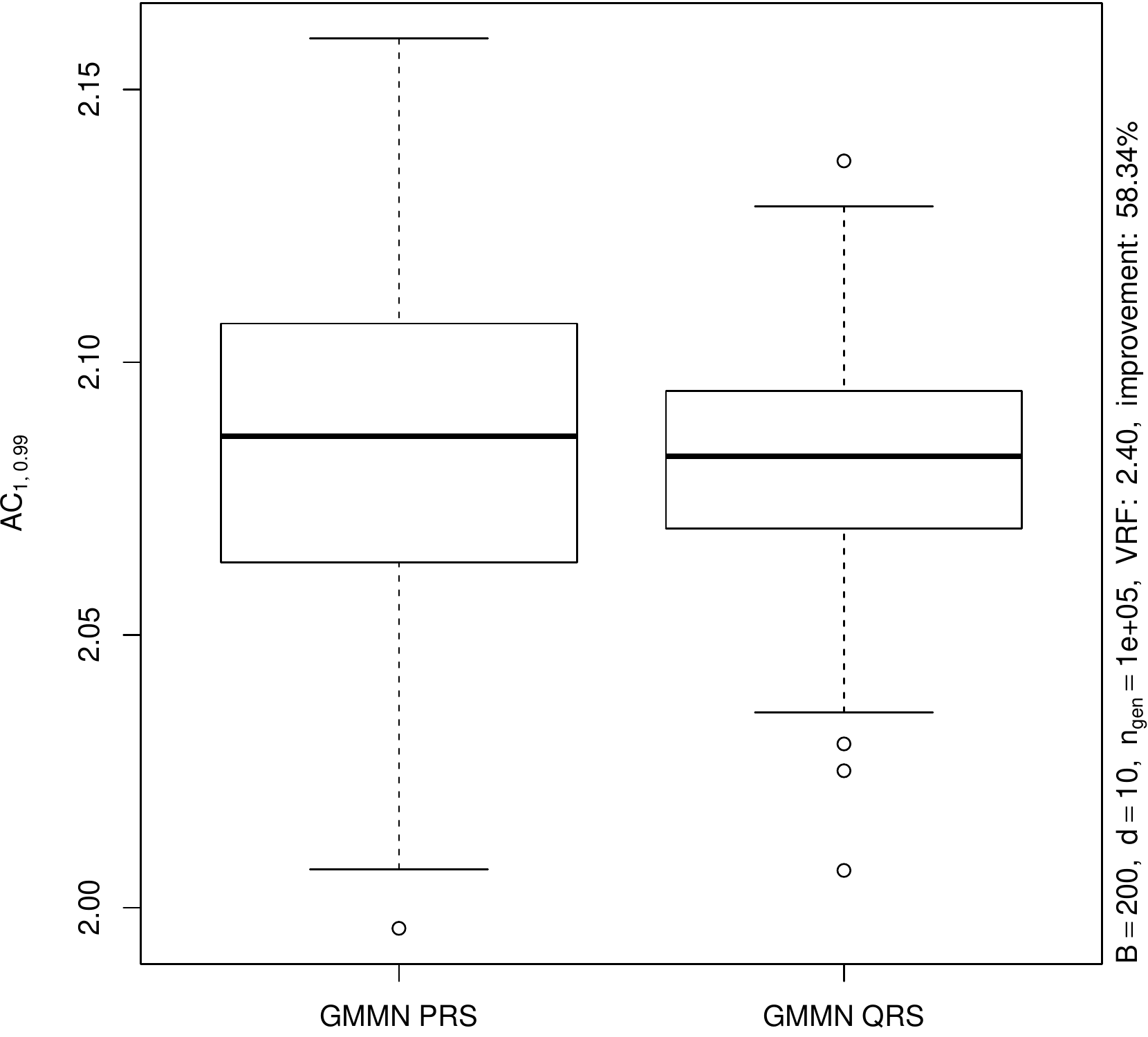}\hfill
  \includegraphics[width=0.32\textwidth]{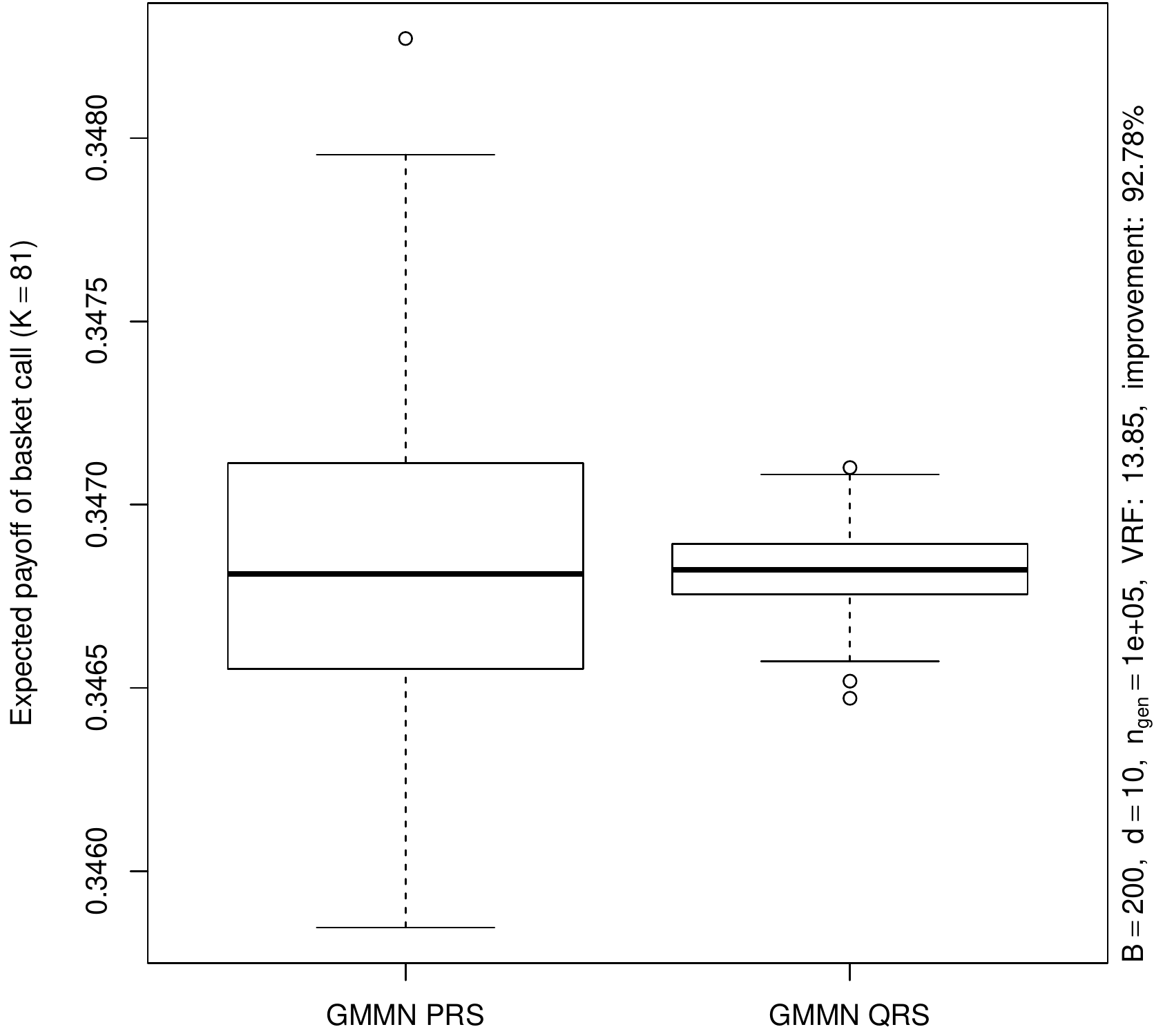}
  \caption{Box plots based on $B=200$ realizations of the GMMN MC estimator $\hat{\mu}^{\text{NN},\text{MC}}_{\ngen}$ (label ``GMMN
    PRS'') and the GMMN RQMC estimator $\hat{\mu}^{\text{NN}}_{\ngen}$ (label ``GMMN QRS'') of $\ES_{0.99}$ (left
    column), $\AC_{1,0.99}$ (middle column) and the expected payoff of a basket
    call (right column) for portfolio with dimensions $d=3$ (top row), $d=5$
    (middle row) and $d=10$ (bottom), using $\ngen=10^5$ samples for both
    estimators.}\label{fig:VRF:data}
\end{figure}

\section{Discussion}\label{sec:discussion}
This work has been inspired by the simple question of how to obtain quasi-random
samples for a large variety of multivariate distributions.  Until recently, this
was only possible for a few multivariate distributions with specific underlying
copulas. In general, for the vast majority of multivariate distributions,
obtaining quasi-random samples is a hard problem~\parencite{cambou2017}. Our
approach based on GMMNs provides a first universal method for doing so. It
depends on first learning a generator $f_{\hat{\bm{\theta}}}$ such that, given
$\bm{Z}$ (with independent components from some known distribution
such as the standard uniform or standard normal),
$f_{\hat{\bm{\theta}}}(\bm{Z})$ follows the targeted multivariate distribution.
Conditional on this first step being successful, we can then replace $\bm{Z}$
with $F_{\bm{Z}}^{-1}(\tilde{\bm{v}}_i)$, $i=1,\dots,\ngen$, where $\{\tilde{\bm{v}}_1,\dots,\tilde{\bm{v}}_{\ngen}\}$ is an RQMC point set, to
generate quasi-random samples from $\bm{X}$.

It is generally difficult to assess the low-discrepancy property of non-uniform
quasi-random samples. To evaluate the quality of our GMMN quasi-random samples,
we used visualization tools (Section~\ref{sec:GMMN:visual}), goodness-of-fit
statistics (Section~\ref{sec:GMMN:accuracy}), and investigated
variance reduction effects (Section~\ref{sec:conv:analysis}) when estimating
$\mu=\E(\Psi(\bm{X}))$ for a test function and for expected shortfall. As dependence
structures among the components of $\bm{X}$, we included various known copulas,
some of which allowed for quasi-random sampling which allowed us
to statistically assess the performance of our GMMN quasi-random samples. However,
we emphasize that the key feature of our method is that, given a sufficiently
large dataset with dependence structure not well described by any known parametric
copula model for which quasi-random sampling is available, we are now able to generate quasi-random samples from its
empirical distribution. We demonstrated this with a real dataset
in Section~\ref{sec:realdata}. Not only does a GMMN provide the best fitting model
in this application, allowing us to avoid the tedious and often computationally challenging
search that is typically required in classical copula modeling for an adequate dependence model,
we also obtain, at no additional cost, quasi-random samples from this GMMN --- a whole other challenge in classical copula modeling. This universality and
computability is an attractive feature of GMMNs for multivariate modeling.

However, this does not mean that the problem of quasi-random sampling for
multivariate distributions is completely solved. In high dimensions learning an
entire distribution is a hard problem, and so is learning the generator
$f_{\hat{\bm{\theta}}}$. At a superficial level, the literature on generative
NNs --- and the many headlines covering them --- may give the impression that
such NNs are now capable of generating samples from very high-dimensional
distributions.
This, of course, is not true; see, for example,
\cite{arjovsky2017, arora2018, tolstikhin2017}. In particular, while available
evidence is convincing that any \emph{specific} generated sample
$f_{\hat{\bm{\theta}}}(\bm{Z}_1)$, typically an image, can be very realistic in
the sense that it looks just like a typical training sample, this is not the
same as saying that the \emph{entire collection} of generated samples
$\{f_{\hat{\bm{\theta}}}(\bm{Z}_1), f_{\hat{\bm{\theta}}}(\bm{Z}_2), \dots\}$
will have the same distribution as the training sample. The latter is a much
harder problem to solve.
Unlike widely cited generative NNs such as variational autoencoders and generative adversarial networks,
GMMNs are capable of learning entire distributions, because they rely on the
$\MMD$-criterion as the cost function rather than, for example, the mean squared
error which does not measure the discrepancy between entire distributions.
Even so, this still does not mean GMMNs are practical for very high dimensions yet,
simply because the fundamental curse of dimensionality cannot be avoided
easily. At the moment, it is simply not realistic to hope that one can learn an
entire distribution in high dimensions from a training sample of only moderate
size.

Going forward there are two primary impediments to quasi-random sampling from
higher-dimensional copulas and distributions. Firstly, the problem of
distribution learning via generative NNs remains a challenging task. We may also
consider using other goodness-of-fit statistics for multivariate distributions
rather than the MMD as the cost function (provided that the statistic is
differentiable in order to train a generative NN). Secondly, we discovered from
our empirical investigation in Section~\ref{sec:conv:analysis} that the
convergence rates of GMMN RQMC estimators decrease with increasing
dimension. Preserving the low-discrepancy of RQMC point sets upon
transformations in high dimensions remains an open problem in this regard.

\appendix

\section{Analyzing GMMN QMC and GMMN RQMC estimators}
\label{sec:appendix}

\subsection{QMC point sets}
The idea behind quasi-random numbers is to replace pseudo-$U(0,1)^p$ random numbers with low-discrepancy points sets $P_{\ngen}$ to produce a more homogeneous coverage of $[0,1]^p$ in comparison to pseudo-random numbers. That is, with respect to a certain \emph{discrepancy measure}, the empirical
distribution of the $P_{\ngen}$ is closer to the uniform distribution
$\U(0,1)^p$ than a pseudo-random sample.

Established notions of the discrepancy of a point set $P_{\ngen}$ are as follows.
The \emph{discrepancy function} of $P_{\ngen}$ in an interval
$I=[\bm{0},\bm{b})=\prod_{j=1}^{p}[0,b_j)$, $b_j\in(0,1]$, $j=1,\dots,p$, is defined by
\begin{align*}
  D(I;P_{\ngen})= \frac{1}{\ngen} \sum_{i=1}^{\ngen} \I_{\{\bm{v}_i \in I\}} - \lambda(I),
\end{align*}
where $\lambda(I)$ is the $p$-dimensional Lebesgue measure of $I$. Thus the
discrepancy function is the difference between the number of points of $P_{\ngen}$ in
$I$ and the probability of a $p$-dimensional standard uniform random vector to
fall in $I$. For $\mathcal{A}=\{[\bm{0},\bm{b}):\bm{b}\in(0,1]^p\}$,
the \emph{star discrepancy} of $P_{\ngen}$ is defined by
\begin{align*}
  D^{*}(P_{\ngen}) = \sup_{I\in\mathcal{A}}|D(I;P_{\ngen})|.
\end{align*}
If $P_{\ngen}$ satisfies the condition
$D^{*}(P_{\ngen}) \in O(\ngen^{-1}(\log \ngen)^p)$, it is called a
\emph{low-discrepancy sequence} \parencite[][p.~143]{lemieux2009}.

 There are different approaches to construct (deterministic)
 low-discrepancy sequences; see \cite[Chapters~5--6]{lemieux2009}. The two main approaches involve either lattices (grids which behave well under projections) or digital nets/sequences. In our numerical investigations presented in Sections~\ref{sec:GMMN:copula}--\ref{sec:conv:analysis}, we worked with a type of digital net constructed using the Sobol' sequence; see \cite{sobol1967}.

\subsection{Analyzing the GMMN QMC estimator}\label{appendix:qmc:analysis}

In this section, we will  derive conditions
under which the (non-randomized) GMMN QMC estimator
\begin{align*}
  \frac{1}{\ngen} \sum_{i=1}^{\ngen} \Psi(q(\bm{v}_i))=\frac{1}{\ngen} \sum_{i=1}^{\ngen} h(\bm{v}_i),
\end{align*}
where $q=f_{\hat{\bm{\theta}}} \circ F^{-1}_{\bm{Z}}$ and $h=\Psi\circ q=\Psi \circ f_{\hat{\bm{\theta}}} \circ F^{-1}_{\bm{Z}}$, has a small error when approximating $\E(\Psi(\bm{Y}))$. In the following analysis, we need to further assume that $\supp(F_{\bm{X}})$ and $\supp(F_{\bm{Y}})$ are bounded.

The \emph{Koksma--Hlawka inequality} \parencite{niederreiter1992} for a function $g:[0,1]^p\rightarrow \IR$ says that
\begin{align*}
  \Biggl|\frac{1}{\ngen}\sum_{i=1}^{\ngen} g(\bm{v}_i) -\E(g(\bm{U}'))\Biggr|\le V(g) D^{*}(P_{\ngen}),
\end{align*}
where $\bm{U}'\sim\U(0,1)^p$ and the variation $V(g)$ is understood in the sense of Hardy and Krause; we refer to the right-hand side of the inequality as \emph{Koksma--Hlawka bound}. From this Koksma--Hlawka inequality, we can establish that if $g$ has finite bounded variation, that is $V(g)<\infty$, then the convergence rate for $\frac{1}{\ngen}\sum_{i=1}^{\ngen} g(\bm{v}_i)$ is determined by $D^{*}(P_{\ngen})=O(n_{\ngen}^{-1}(\log \ngen)^p)$.

We can use the Koksma--Hlawka inequality to analyze the convergence of the GMMN QMC estimator $\frac{1}{\ngen} \sum_{i=1}^{\ngen} \Psi(\bm{y}_i)$ of $\E(\Psi(\bm{Y}))$, where $\bm{y}_i=q(\bm{v}_i)$, $i=1,\dots,\ngen$ and $\bm{Y}\sim F_{\bm{Y}}$, by establishing the conditions under which $V(h)$ is bounded. To that end, consider the following proposition.

\begin{proposition}[Sufficient conditions for finiteness of the Koksma--Hlawka bound]\label{prop:bound:ass}
  Assume that $\supp(F_{\bm{Y}})$ is bounded and all appearing partial derivatives of $q$ and $\Psi$ exist and are
  continuous. Consider $q=f_{\hat{\bm{\theta}}}\circ F^{-1}_{\bm{Z}}$, the point set
  $P_{\ngen}=\{\bm{v}_1,\dots,\bm{v}_{\ngen}\}\subseteq [0,1)^p$ and let
  $\bm{y}_i=q(\bm{v}_i)$, $i=1,\dots,\ngen$, denote the GMMN quasi-random
  sample. Suppose that
  \begin{enumerate}
  \item\label{prop:bound:ass:1} $\Psi(\bm{y}) < \infty$ for all $\bm{y}\in \supp(F_{\bm{Y}})$ and
    \begin{align*}
      \frac{\partial^{|\bm{\beta}|_1}\Psi(\bm{y})}{\partial^{\beta_1}y_1\dots\partial^{\beta_d}y_d} < \infty,\quad\bm{y}\in \supp(F_{\bm{Y}}),
    \end{align*}
    for all $\bm{\beta}=(\beta_1,\dots,\beta_d)\subseteq\{0,\dots,d\}^d$ and $|\bm{\beta}|_1=\sum_{j=1}^d \beta_j\le d$;
  \item\label{prop:bound:ass:2} there exists an $M>0$ such that $|\D^k F^{-1}_{Z_{j}}| \le M$, for each
    $k,j=1,\dots,p$, where $\D^k$ denotes the $k$-fold derivative of its argument;
  \item\label{prop:bound:ass:3} there exists, for each layer $l=1,\dots,L+1$ of the NN $f_{\hat{\bm{\theta}}}$, an $N_l>0$ such that $|\D^k \phi_l| \le N_l$ for all $k=1,\dots,p$; and
  \item\label{prop:bound:ass:4} the parameter vector $\hat{\bm{\theta}}=(\widehat{W}_{1},\dots,\widehat{W}_{L+1},\hat{\bm{b}}_{1},\dots,\hat{\bm{b}}_{L+1})$ of the fitted NN is bounded.
  \end{enumerate}
  Then there exists a constant $c$ independent of $\ngen$, but possibly depending on $\Psi$,
  $\hat{\bm{\theta}}$, $M$ and $N_{1},\dots, N_{L+1}$, such that
  \begin{align*}
    \Biggl|\frac{1}{\ngen} \sum_{i=1}^{\ngen}\Psi(\bm{y}_i) - \E(\Psi(\bm{Y}))\Biggr|\le c D^{*}(P_{\ngen}).
  \end{align*}
\end{proposition}
\begin{proof}
To begin with note that for any $q$ such that
$q(\bm{U}') \sim F_{\bm{Y}}$, we know that $\bm{Y}$ is in distribution equal to
$q(\bm{U}')$ and thus $\E(\Psi(\bm{Y}))=\E\bigl(\Psi(q(\bm{U}'))\bigr)=\E(h(\bm{U}'))$. Based on this property, we can obtain the Koksma--Hlawka bound $V(h)D^{*}(P_{\ngen})$ for the change of variable $h$.
Following \cite[Section~5.6.1]{lemieux2009}, we can derive an expression for $V(h)$. To this end, let
\begin{align*}
  V^{(j)}(h;\bm{\alpha})=\int_{[0,1)^{j}} \Biggr| \frac{\partial^{j}h^{(\bm{\alpha})}(v_{\alpha_1},\dots,v_{\alpha_{j}})}{\partial v_{\alpha_j}\dots\partial v_{\alpha_1}}\Biggr|\,\rd v_{\alpha_1}\dots\,\rd v_{\alpha_j},
\end{align*}
where $h^{(\bm{\alpha})}(v_{\alpha_1},\dots,v_{\alpha_j})=h(\tilde{v}_1,\dots,\tilde{v}_p)$
for $\tilde{v}_k=v_k$ if $k\in\{\alpha_1,\dots,\alpha_j\}$ and $\tilde{v}_k=1$
otherwise. Then
\begin{align}
V(h)=\sum_{j=1}^{p} \sum_{\bm{\alpha}:|\bm{\alpha}|_1=j}V^{(j)}(h;\bm{\alpha}),\label{eq:bounded:var}
\end{align}
where the inner sum is taken over all $\bm{\alpha}=(\alpha_1,\dots,\alpha_j)$
with $\{\alpha_1,\dots,\alpha_j\}\subseteq\{1,\dots,p\}$ --- see also \cite[pp.~19--20]{niederreiter1992}, \cite[eq.~(4)]{hlawka1961} and \cite[eq.~(4')]{hlawkamuck1972}.
Following \cite{hlawkamuck1972} and \cite[Theorem~2.1]{constantinesavits1996},
we then have
\begin{align} \Biggr|\frac{\partial^{j}h^{(\bm{\alpha})}(v_{\alpha_1},\dots,v_{\alpha_j})}{\partial v_{\alpha_j}\dots\partial v_{\alpha_1}}\Biggr| = \!\!\!\!\sum_{1\leq|\bm{\beta}|_1\le j}\!\frac{\partial^{|\bm{\beta}|_1}\Psi}{\partial^{\beta_1}y_1\dots\partial^{\beta_d}y_d} \sum_{i=1}^{j} \!\sum_{(\bm{\kappa},\bm{k}) \in \pi_i({\bm{\kappa},\bm{k}})} \!\!\!\!\! c_{\bm{\kappa}} \prod_{m=1}^{i} \frac{\partial^{|\bm{\kappa}_m|_1}q_{k_m}^{(\bm{\alpha})}(v_{\alpha_1},\dots,v_{\alpha_{j}})}{\partial^{\kappa_{mj}}v_{\alpha_j}\dots\partial^{\kappa_{m1}}v_{\alpha_{1}}}, \label{eq:decomp:partial}
\end{align}
where $\bm{\beta}\in \IN_{0}^{d}$ and
where $\pi_i(\bm{\kappa},\bm{k})$ denotes the set of pairs
$(\bm{\kappa},\bm{k})$ such that $\bm{k}=(k_1,\dots,k_i)\in\{1,\dots,d\}^i$ and
$\bm{\kappa}=(\bm{\kappa}_1,\dots,\bm{\kappa}_i)$ with
$\bm{\kappa}_m=(\kappa_{m1},\dots,\kappa_{mj})\in\{0,1\}^j$, $m=1,\dots,i$, and
$\sum_{m=1}^{i} \kappa_{mi}=1$ for $i=1,\dots,j$; see
\cite{constantinesavits1996} %
for more details on $\pi_i(\bm{\kappa},\bm{k})$ and the constants
$c_{\bm{\kappa}}$. Furthermore, for index $j=1,\dots,d$, %
$q_{j}^{(\alpha)}(v_{\alpha_1},\dots,v_{\alpha_j})=q_j(\tilde{v}_1,\dots,\tilde{v}_p)$ %
and
$q_j(\tilde{v}_1,\dots,\tilde{v}_p)=\phi_{L+1}(\widehat{W}_{L+1j\cdot}\bm{a}_L+\hat{\bm{b}}_{L+1})$,
where $\bm{a}_l= \phi_l(\widehat{W}_l\bm{a}_{l-1}+\hat{\bm{b}}_l)$ for
$l=1,\dots,L$ with
$\bm{a}_0=F_{\bm{Z}}^{-1}(\tilde{\bm{v}})$ %
and where $\widehat{W}_{L+1j\cdot}$ denotes the $j$th row of $W_{L+1}$. %

Based on the decomposition in~\eqref{eq:decomp:partial}, a sufficient condition
to ensure that $V(h)<\infty$ is that all products of the form
\begin{align*}
 \frac{\partial^{|\bm{\beta}|_1}\Psi}{\partial^{\beta_1}y_1\dots\partial^{\beta_d}y_d}\prod_{m=1}^{i} \frac{\partial^{|\bm{\kappa}_m|_1}q_{k_m}^{(\bm{\alpha})}(v_{\alpha_1},\dots,v_{\alpha_{j}})}{\partial^{\kappa_{mj}}v_{\alpha_j}\dots\partial^{\kappa_{m1}}v_{\alpha_{1}}}, \quad i=1,\dots,j,
\end{align*}
are integrable.

To that end, Assumptions~\ref{prop:bound:ass:2}--\ref{prop:bound:ass:4} imply
  that all mixed partial derivatives of
  $q=f_{\hat{\bm{\theta}}}\circ F^{-1}_{\bm{Z}}$ are bounded. By the assumption
  of continuous partial derivatives of $q$, this implies that finite products of
  the form
  \begin{align*}
    \prod_{m=1}^{i} \frac{\partial^{|\bm{\kappa}_m|_1}q_{k_m}^{(\bm{\alpha})}(v_{\alpha_1},\dots,v_{\alpha_j})}{\partial^{\kappa_{mj}}v_{\alpha_j}\dots\partial^{\kappa_{m1}}v_{\alpha_1}}, \quad i=1,\dots,j,
  \end{align*}
  are integrable. By Assumption~\ref{prop:bound:ass:1},
  decomposition~\eqref{eq:decomp:partial} and H\"{o}lder's inequality, the
  quantity in~\eqref{eq:bounded:var} is bounded. This implies that $h$ has
  bounded variation, so that the Koksma--Hlawka bound is finite.
\end{proof}

The following remark provides insights into
Assumptions~\ref{prop:bound:ass:2}--\ref{prop:bound:ass:4} of
Proposition~\ref{prop:bound:ass}.
\begin{remark}\label{remark:bound:ass}
  $\U(a,b)^p$ for $a<b$, which is a popular choice for the input distribution,
  clearly satisfies Assumption~\ref{prop:bound:ass:2} in
  Proposition~\ref{prop:bound:ass}. %
  Assumption~\ref{prop:bound:ass:3} is satisfied for various commonly
  used activation functions, such as:
  \begin{enumerate}
  \item \emph{Sigmoid.} If $\phi_l(x)=1/(1+\mathrm{e}^{-x})$ for layer $l$, then $N_l=1$.
  \item \emph{ReLU.} If $\phi_l(x)=\max\{0,x\}$ for layer $l$, then $N_{l}=1$. In
  this case, only the first derivative is (partly) non-zero. Additionally, note
  that the ReLU activation function is not differentiable at $x=0$. However,
  even if $\phi_l=\max\{0,x\}$ for all $l=1,\dots,L+1$, the set of all pointwise
  discontinuities of the mixed partial derivatives of $q$ is a null set. Hence,
  the discontinuities do not jeopardize the proof of
  Proposition~\ref{prop:bound:ass}.
  \item \emph{Softplus.} If $\phi_l(x)=\log(1+\mathrm{e}^x)$ for layer $l$, then $N_l=1$. The Softplus activation function can be used as a smooth approximation of the ReLU activation function.
  \item \emph{Linear.} If $\phi_l(x)=x$ for layer $l$, then $N_l=1$. Only the
    first derivative is non-zero.
  \item \emph{Tanh.} If $\phi_l(x)=\tanh(x)$ for layer $l$,
    then $N_l=1$.
  \item \emph{Scaled exponential linear unit (SELU)}; see
    \cite{klambauer2017}. If, for layer $l$,
    \begin{align*}
      \phi_l(x)=\begin{cases*}%
        \lambda\alpha(\exp(-x)-1), & \text{if $x<0$}, \\
        \lambda x, & \text{if $x\ge 0$},
      \end{cases*}
    \end{align*}
    where $\lambda$ and $\alpha$ are prespecified constants, then
    $N_l=\max\{\lambda, \lambda\alpha, 1\}$. The same argument about discontinuities made with the ReLU activation function
    applies equally well to the case of the SELU activation function.
  \end{enumerate}
  Assumption~\ref{prop:bound:ass:4} of
  Proposition~\ref{prop:bound:ass} is satisfied in practice because NNs
  are always trained with regularization on the parameters, which means
  $\hat{\bm{\theta}}$ always lies in a compact set. Additionally note that in
  the general case where $q$ is characterized by a composition of NN layers and
  $F^{-1}_{\bm{Z}}$ with a different (but standard) activation function in each
  layer, all partial derivatives of $q$ exist and are continuous. Moreover, for
  the activation functions and input distributions listed above, all mixed
  partial derivatives of $q$ are bounded.
\end{remark}

\subsection{RQMC point sets} \label{appendix:rqmc}

In Monte Carlo applications, we need to randomize the low-discrepancy sequence $P_{\ngen}$ to obtain unbiased estimators and variance estimates.
To that end, we can randomize $P_{\ngen}$ via a $\bm{U}'\sim\U(0,1)^p$ to obtain
a randomized point set
$\tilde{P}_{\ngen}=\tilde{P}_{\ngen}(\bm{U}')=\{\tilde{\bm{v}}_1,\dots,\tilde{\bm{v}}_{\ngen}\}$,
where $\tilde{\bm{v}}_i=r(\bm{U}',\bm{v}_i)$, $i=1,\dots,\ngen$, for a certain
randomization function $r$. A simple randomization
to obtain an RQMC point set is to consider
$\tilde{\bm{v}}_i=(\bm{v}_i+\bm{U}')\mod 1$, $i=1,\dots,\ngen$, for
$\bm{U}'\sim\U(0,1)^p$, a so-called \emph{random shift}; see
\cite{cranleypatterson1976}.

In practice, more sophisticated alternatives to the random shift are often used.  One such slightly more sophisticated randomization scheme is the \emph{digital shift} method; see \cite[Chapter~6]{lemieux2009} and \cite{cambou2017}. In the same vein as the random shift, one adds a random uniform shift to points in $P_{\ngen}$, but with operations in $\IZ_b$, where $b$ is the base in which the digital net is defined, rather than simply adding two real numbers. We use $\tilde{P}^{\text{ds}}_{\ngen}$ to denote the RQMC point set obtained using the digital shift method.

Another randomization approach is to \emph{scramble} the digital net. This technique was originally proposed by \cite{owen1995}. In particular, the type of scrambling we work with is referred to as the \emph{nested uniform scrambling} (or \emph{full random scrambling}) method; see \cite{owen2003}. Since we primarily use this method throughout the paper, $\tilde{P}_{\ngen}$ will denote specifically the RQMC point set obtained using scrambling. The digital shift method is more computationally efficient in comparison to scrambling but because the distortion of the deterministic point set is fairly simple in the digital shift method, there are \emph{bad} functions one can construct such that the variance of the RQMC estimator is larger than that of the corresponding MC estimator; see \cite[Chapter~6]{lemieux2009}. Furthermore, when RQMC points are constructed with scrambling, we can justify (see Appendix~\ref{appendix:rqmc:analysis}) that an improved rate of $O(\ngen^{-3}(\log \ngen)^{p-1})$ is achievable for $\Var(\varhat{\mu}{\ngen}{NN})$; this translates to $O(\ngen^{-3/2}(\log \ngen)^{(p-1)/2})$ on the root mean squared error (RMSE) scale, which is more directly comparable to the convergence rate of $O(\ngen^{-1}(\log \ngen)^p)$ implied by the Koksma-Hlawka bound for the mean absolute error (MAE) of $\varhat{\mu}{\ngen}{NN}$ using QMC points (see Appendix~\ref{appendix:qmc:analysis}). Hence, even though the aforementioned bad functions do not often arise in practice, we primarily work with the scrambling randomization method to construct our RQMC point sets. Both the scrambling and the digital shift methods are available in the \R\ package \texttt{qrng} and can be accessed via \texttt{sobol(, randomize = "Owen")} and \texttt{sobol(, randomize = "digital.shift")} respectively.

The randomization schemes discussed above preserve the low-discrepancy property of $P_{\ngen}$ and the estimators of interest obtained using each type of RQMC point set are unbiased. Computing the estimator based on $B$ such randomized point sets and computing the sample variance of the resulting $B$ estimates then allows us to estimate the variance of the estimator of interest.

\subsection{Analyzing the GMMN RQMC estimator}\label{appendix:rqmc:analysis}

\subsubsection{GMMN RQMC estimators constructed with scrambled nets} \label{appendix:rqmc:scr:analysis}
For RQMC estimators $\frac{1}{\ngen}\sum_{i=1}^{\ngen} g(\tilde{\bm{v}})$ based on scrambled nets, \cite{owen1997b} initially
derived a convergence rate for the variance of the estimators under a certain
smoothness condition on $g$, where $g:[0,1]^p \rightarrow \IR$. \cite{owen2008}
then generalized his earlier result to allow a weaker smoothness condition for a larger class
of scrambled nets. Specifically, if
$\tilde{P}_{\ngen}=\{\tilde{\bm{v}}_1,\dots,\tilde{\bm{v}}_{\ngen}\}$ is a so-called
relaxed scrambled $(\lambda,q,m,p)$-net in base $b$ with bounded gain
coefficients --- for example, Sobol' sequences randomized using nested uniform
sampling belong to this class --- then we have the following result as a direct consequence of \cite{owen2008}.
\begin{theorem}[\cite{owen2008}]\label{thm:rqmc}
  If all the mixed partial derivatives (up to order $p$) of $g$ exist and are continuous, then
  \begin{align*}
    \Var\biggl(\frac{1}{\ngen}\sum_{i=1}^{\ngen} g(\tilde{\bm{v}}_i)\biggr)= O(\ngen^{-3}(\log \ngen)^{p-1}).
  \end{align*}
\end{theorem}
\begin{proof}
  See \cite[Theorem~3]{owen2008}.
\end{proof}
Now, for the GMMN RQMC estimator,
$\varhat{\mu}{\ngen}{NN}=\frac{1}{\ngen} \sum_{i=1}^{\ngen}
\Psi(q(\tilde{\bm{v}}_i))=\frac{1}{\ngen}\sum_{i=1}^{\ngen}
h(\tilde{\bm{v}}_i)$, the corollary below naturally follows from
Theorem~\ref{thm:rqmc} with some added analysis of the composite function $h$.
\begin{corollary} \label{corollary:rqmc:scr}
  If all the mixed partial derivatives (up to order $p$) of
  $h=\Psi\circ q= \Psi\circ f_{\hat{\bm{\theta}}}\circ F^{-1}_{\bm{Z}}$ exist
  and are continuous, then $\Var(\varhat{\mu}{\ngen}{NN})= O(\ngen^{-3}(\log \ngen)^{p-1})$.
\end{corollary}
  To analyze the mixed partial derivatives of $h$, it suffices to analyze each
component function separately.

For popular choices of input distributions (such as $\U(a,b)^p$ for $a<b$ or
$\N(0,1)^p$), the $k$-fold derivative $\D^k F^{-1}_{Z_{j}}$ exists and is
continuous (on $[a,b]$ or $\IR$ depending on the choice of input distribution) for each
$k,j=1,\dots,p$.

For each layer $l=1,\dots,L+1$ of the NN $f_{\hat{\bm{\theta}}}$, $\D^{k}\phi_l$
exists and is continuous for $k=1,\dots,p$ --- provided that we use (standard) activation
functions; see Remark~\ref{remark:bound:ass} for further details on suitable
activation functions. For NNs constructed using some popular activation
functions such as the ReLU and SELU, note that the set of all pointwise
discontinuities of the mixed partial derivatives of $f_{\hat{\bm{\theta}}}$ is a
set of Lebesgue measure zero and hence the proof of Theorem~\ref{thm:rqmc} holds. Alternatively, we can use the softplus activation function as a smoother
approximation of ReLU. Now in the most general case of NNs
$f_{\hat{\bm{\theta}}}$ being composed of layers with different (but standard)
activation functions, all mixed partial derivatives (up to order $p$) of
$f_{\hat{\bm{\theta}}}$ exist and are continuous almost everywhere.

Finally, it is certainly true that, for many functionals $\Psi$ that we care about in practice, such as those considered in Sections~\ref{sec:conv:analysis} and~\ref{sec:data:vr},
all of its mixed partial derivatives (up to order $p$) exist and are
continuous almost everywhere on $\IR^d$.

\subsubsection{GMMN RQMC estimators constructed with digitally shifted
  nets} \label{appendix:rqmc:shift:analysis} For GMMN RQMC estimators
$\varhat{\mu}{\ngen}{NN,ds}$ constructed using digitally shifted RQMC point sets
$\tilde{P}^{\text{ds}}_{\ngen}$, we can obtain an expression for
$\Var(\varhat{\mu}{\ngen}{NN,ds})$ under the condition that the composite
function $h$ is square integrable; see \cite[Proposition~6]{cambou2017}.

With added assumptions on the smoothness of  $h$, one can obtain improved convergence rates (compared to MC estimators) for $\Var(\varhat{\mu}{\ngen}{NN,ds})$. For example, under the assumptions of Proposition~\ref{prop:bound:ass}, $h$ has finite bounded variation in the sense of Hardy--Krause, which implies that $\Var(\varhat{\mu}{\ngen}{NN,ds})= O(\ngen^{-2}(\log \ngen)^{2p})$; see \cite{ecuyer2016}.

In practice, we observe that GMMN RQMC estimators constructed using both scrambled and digitally shifted nets achieve very similar convergence rates despite differences in the theoretical convergence rates. To that end, Figure~\ref{fig:aggregateES} shows plots of standard deviation estimates for estimating $\E(\Psi_2(\bm{X}))$ where we use the RQMC point sets $\tilde{P}^{\text{ds}}_{\ngen}$ for the same copula models as considered for Figure~\ref{fig:aggregateES:scramble} (which is based on GMMN RQMC estimators constructed using scramble nets) in Section~\ref{sec:conv:analysis}. The approximate convergence rates as implied by the regression coefficients $\alpha$ displayed in both figures are very similar across the various examples.

\begin{figure}[htbp]
  \centering
  \includegraphics[width=0.31\textwidth]{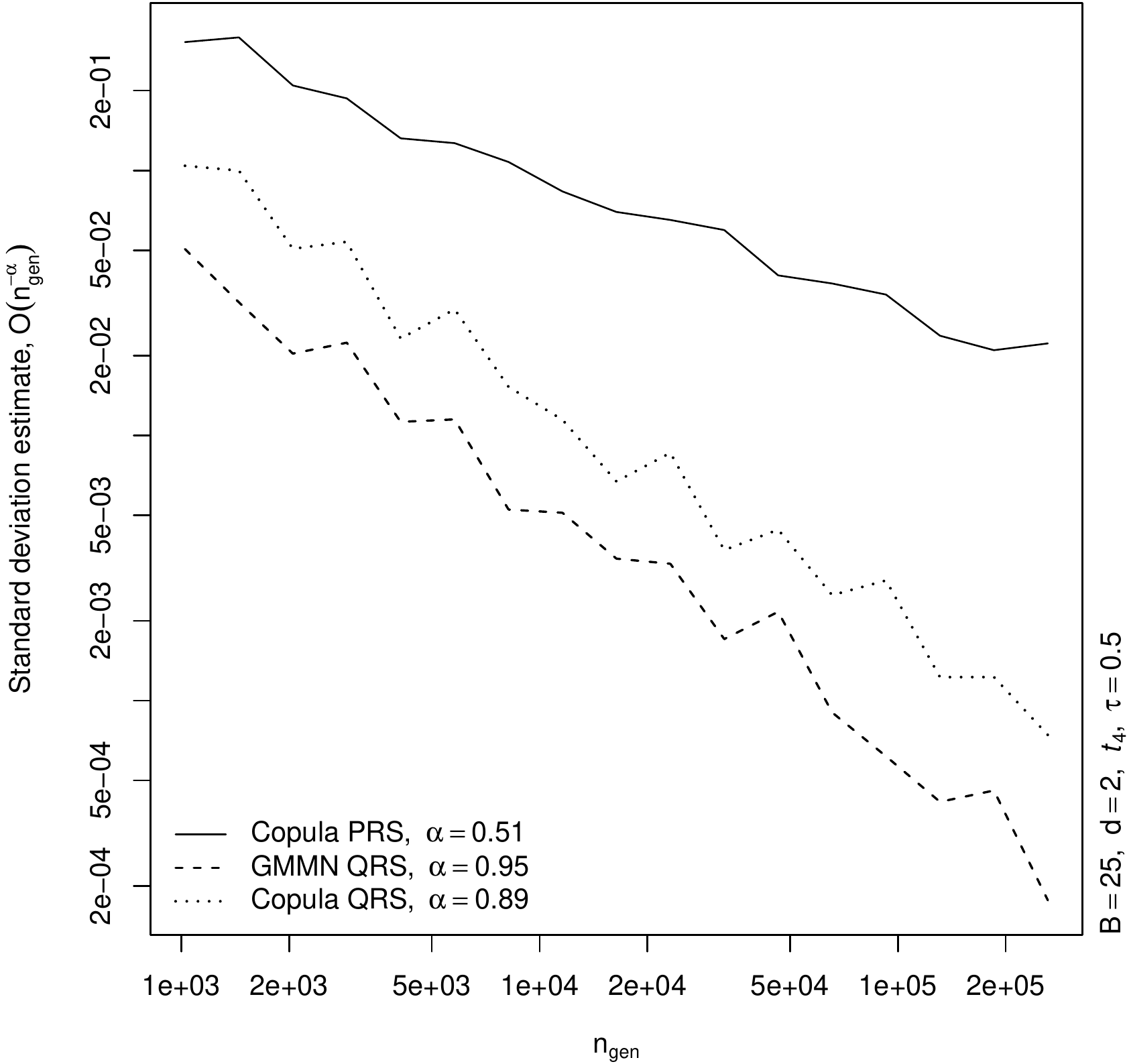}\hfill
  \includegraphics[width=0.31\textwidth]{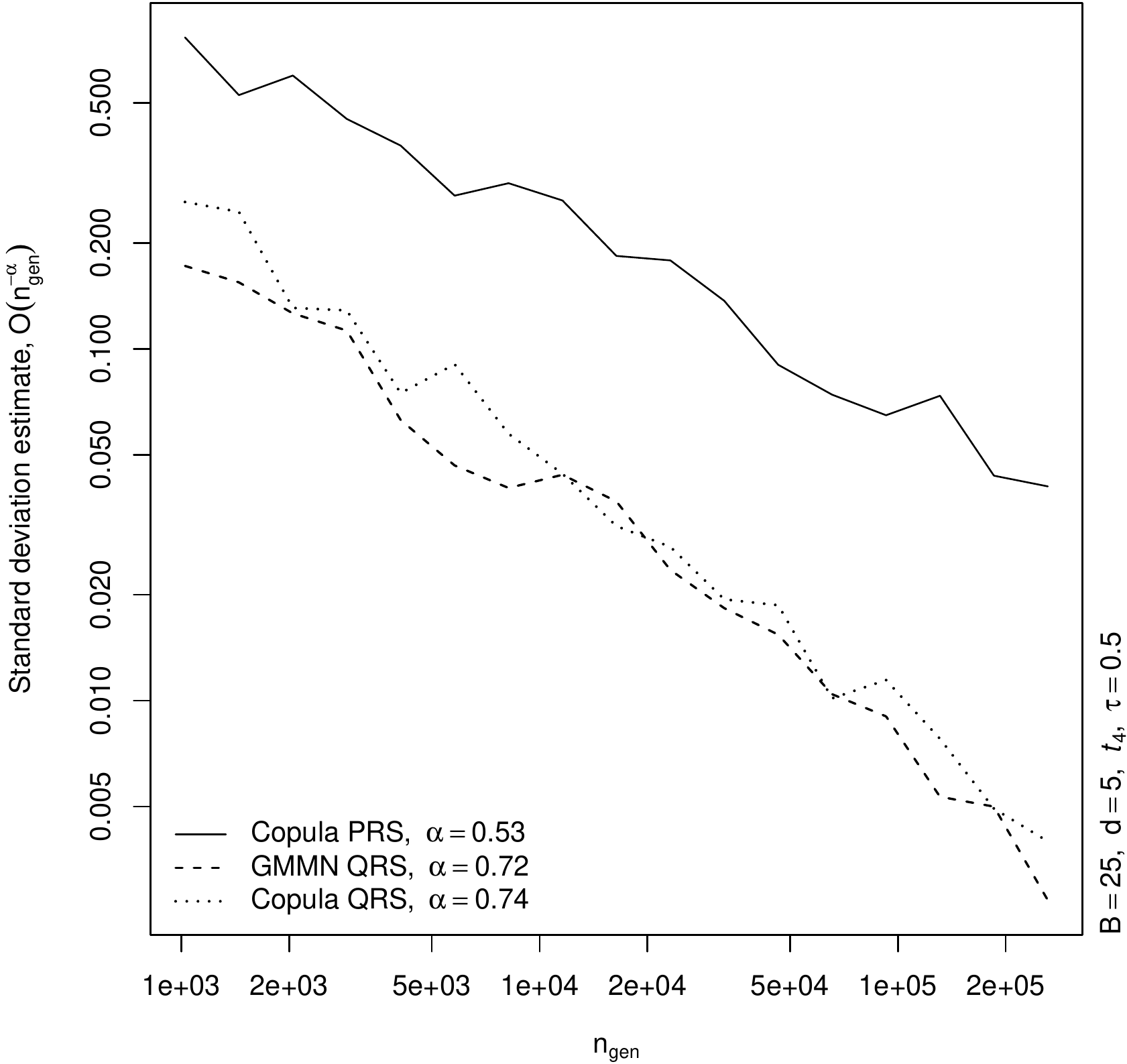}\hfill
  \includegraphics[width=0.31\textwidth]{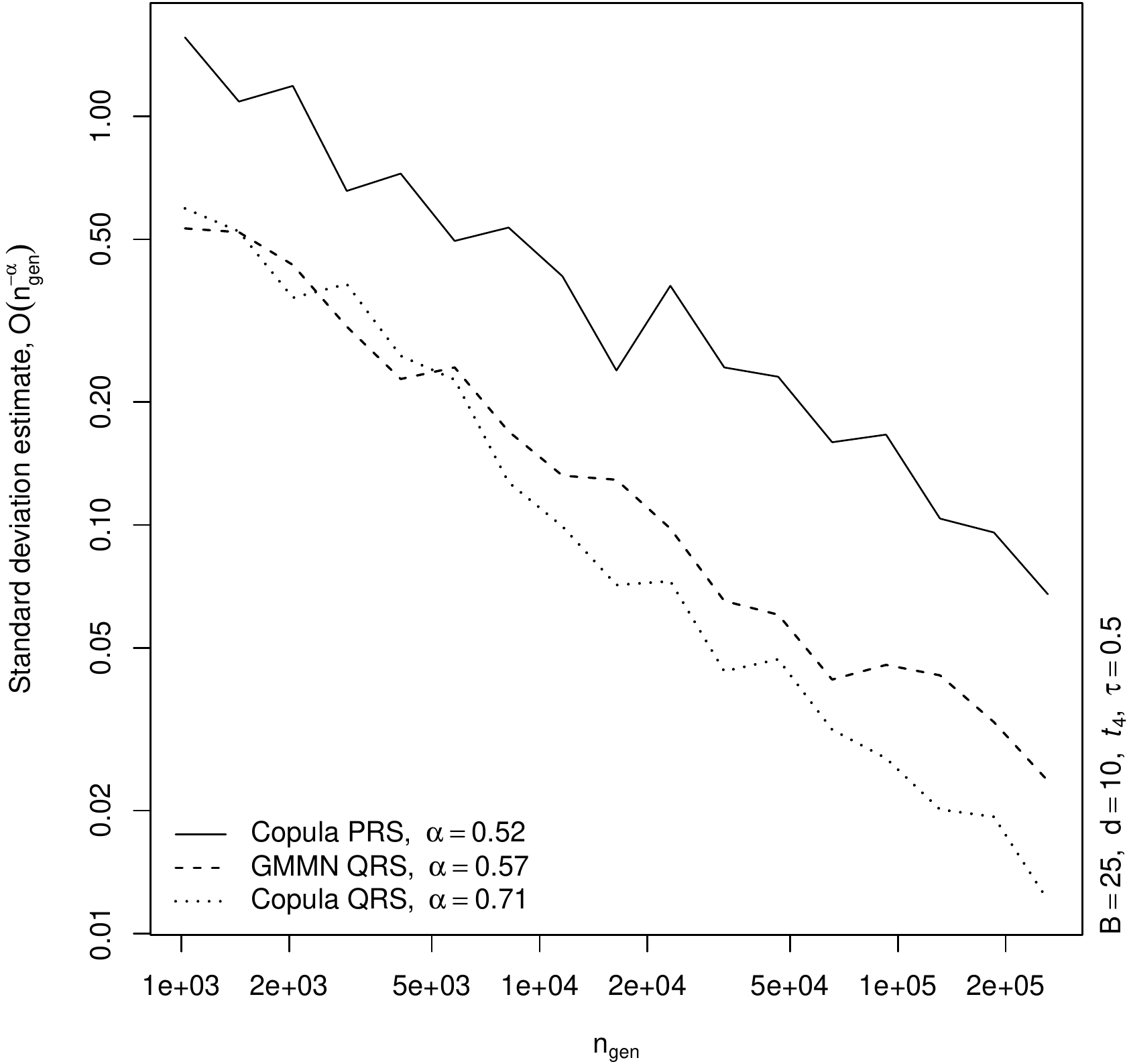}
  \includegraphics[width=0.31\textwidth]{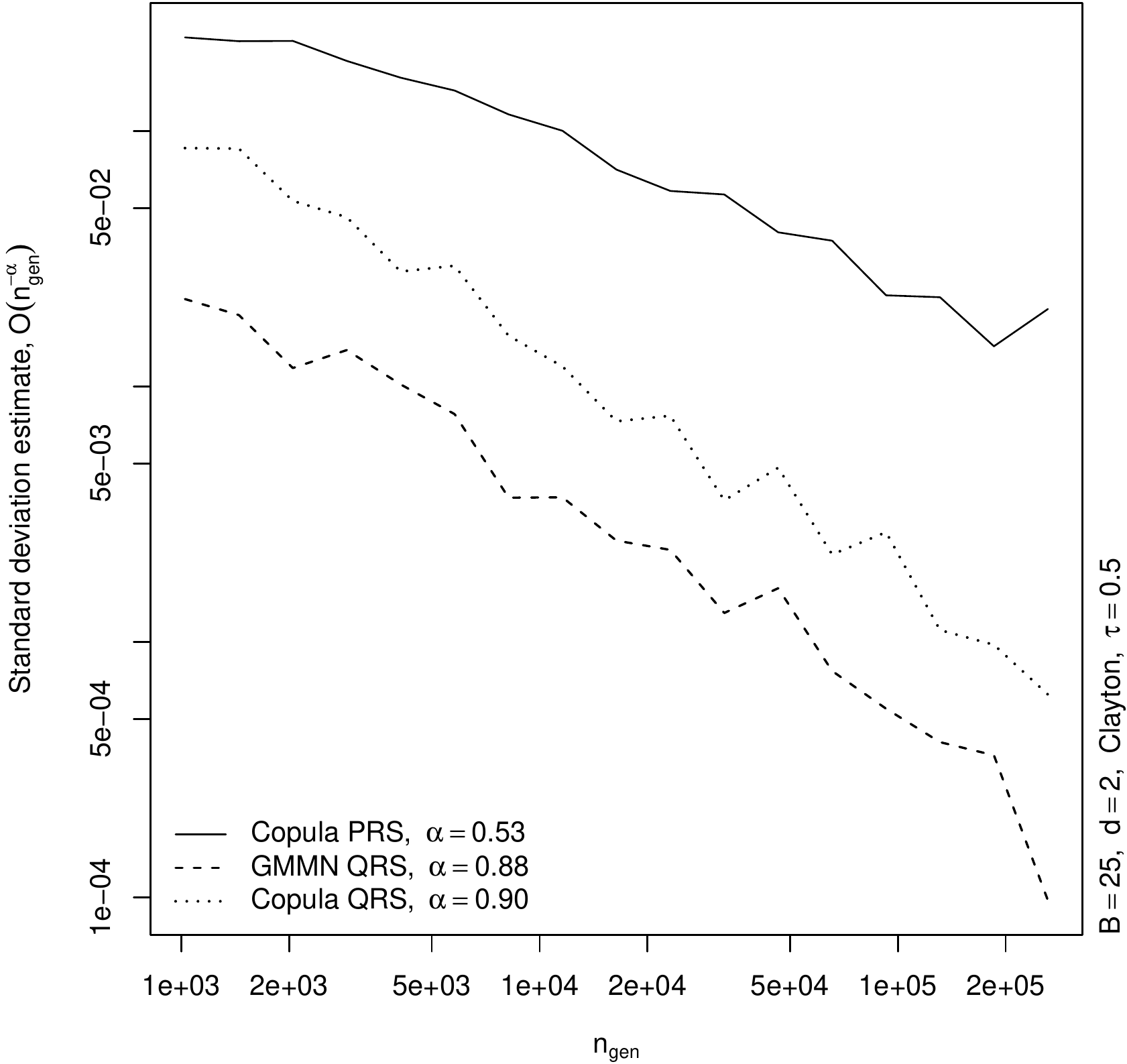}\hfill
  \includegraphics[width=0.31\textwidth]{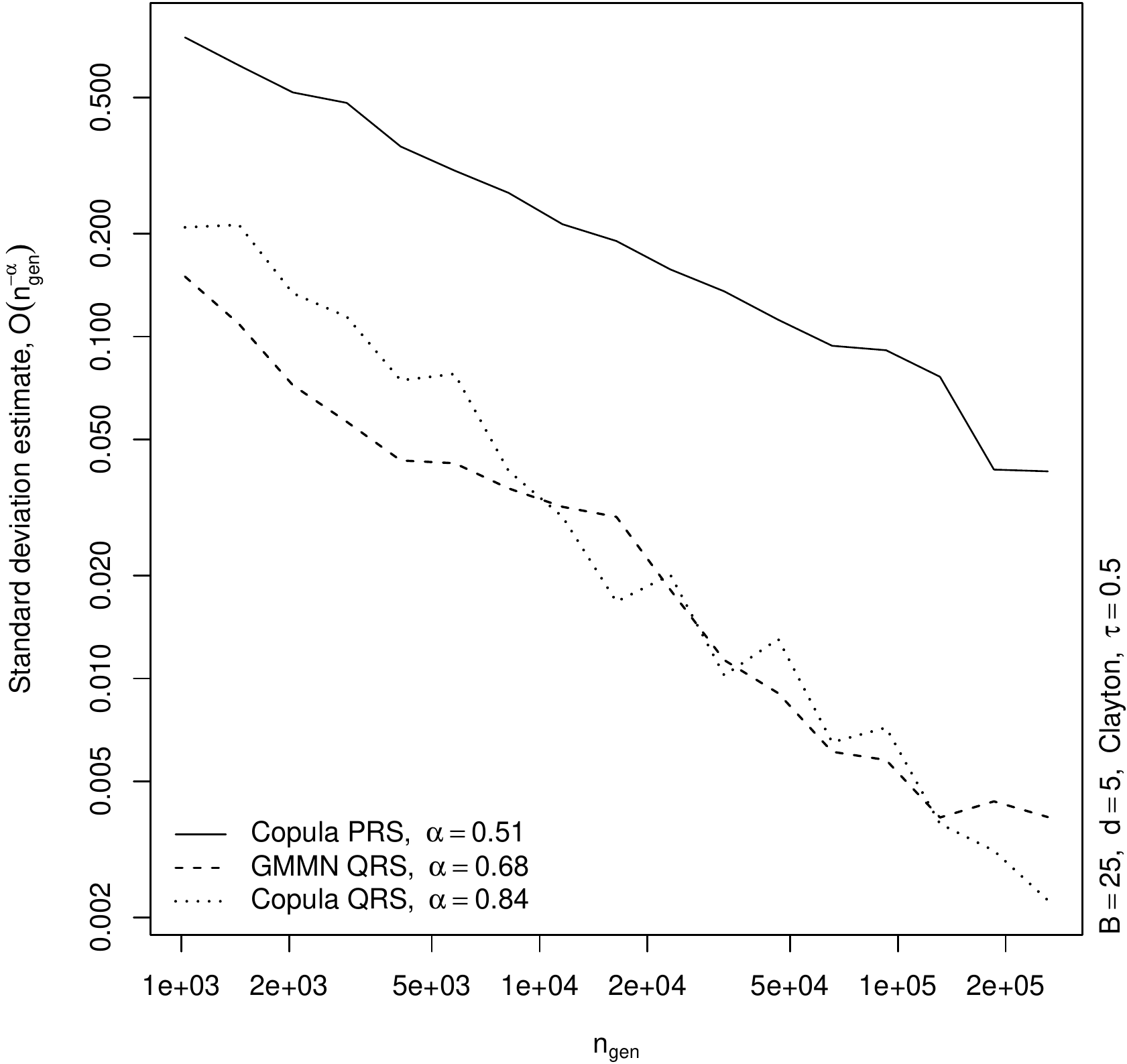}\hfill
  \includegraphics[width=0.31\textwidth]{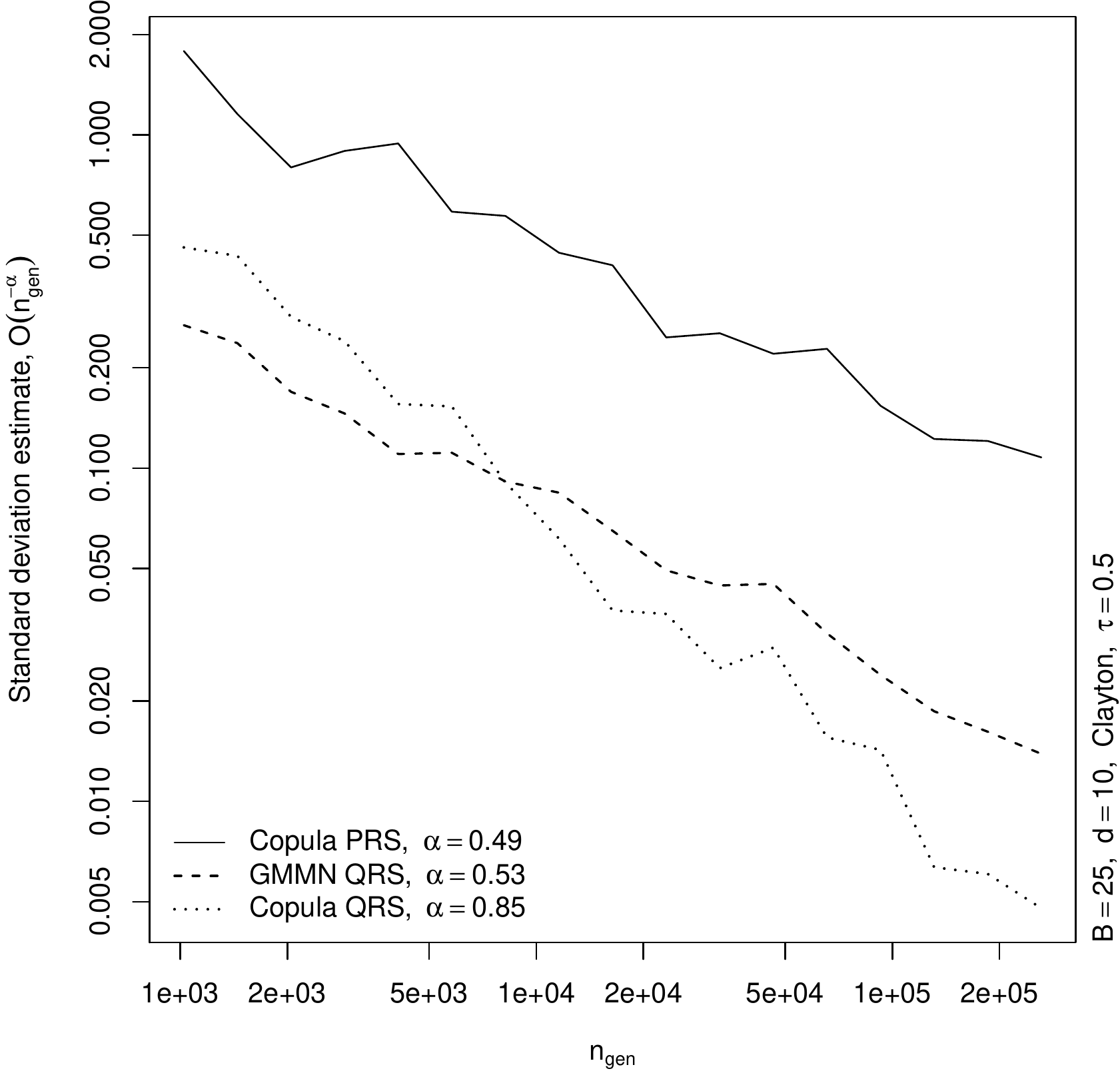}
  \includegraphics[width=0.31\textwidth]{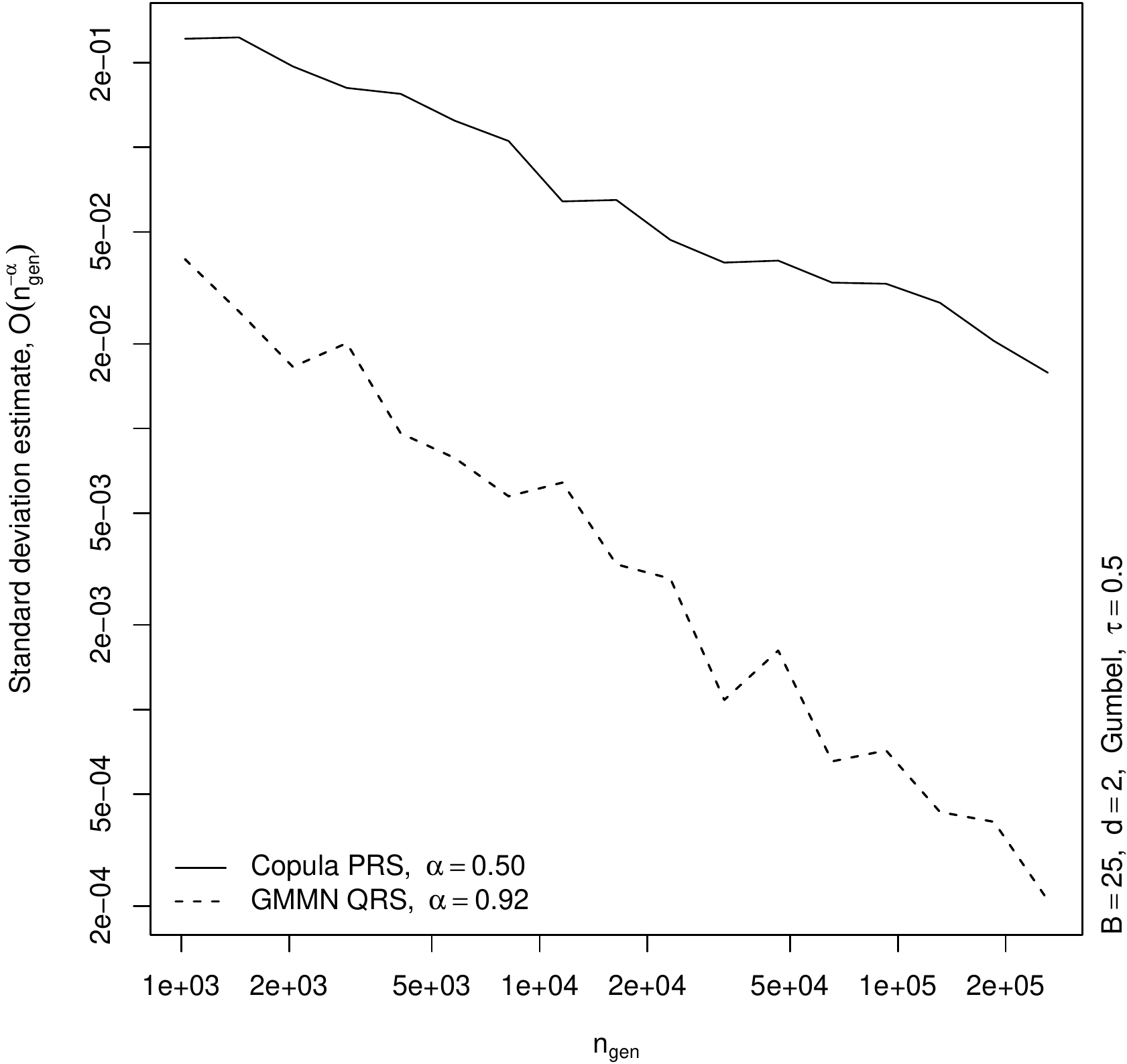}\hfill
  \includegraphics[width=0.31\textwidth]{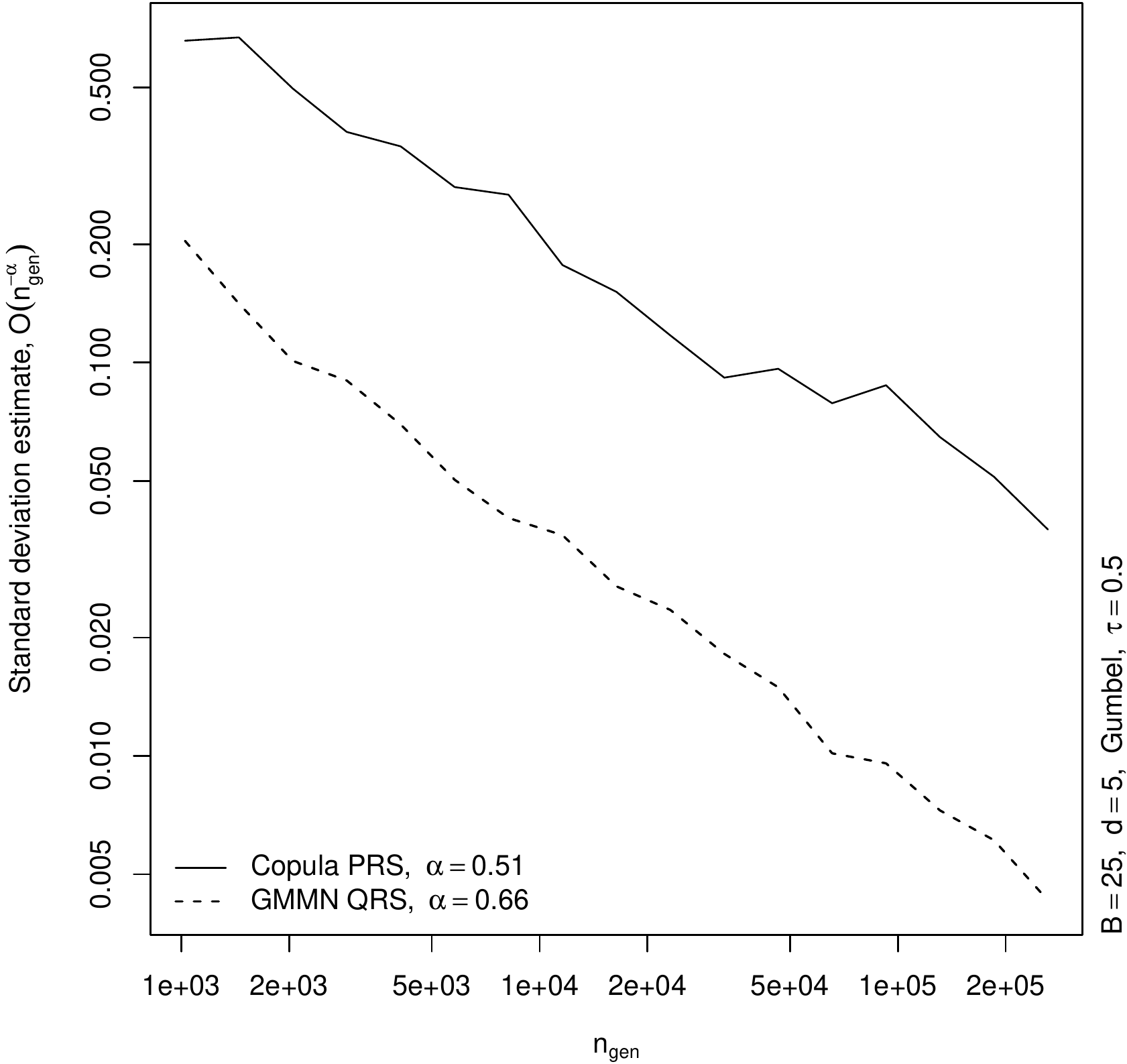}\hfill
  \includegraphics[width=0.31\textwidth]{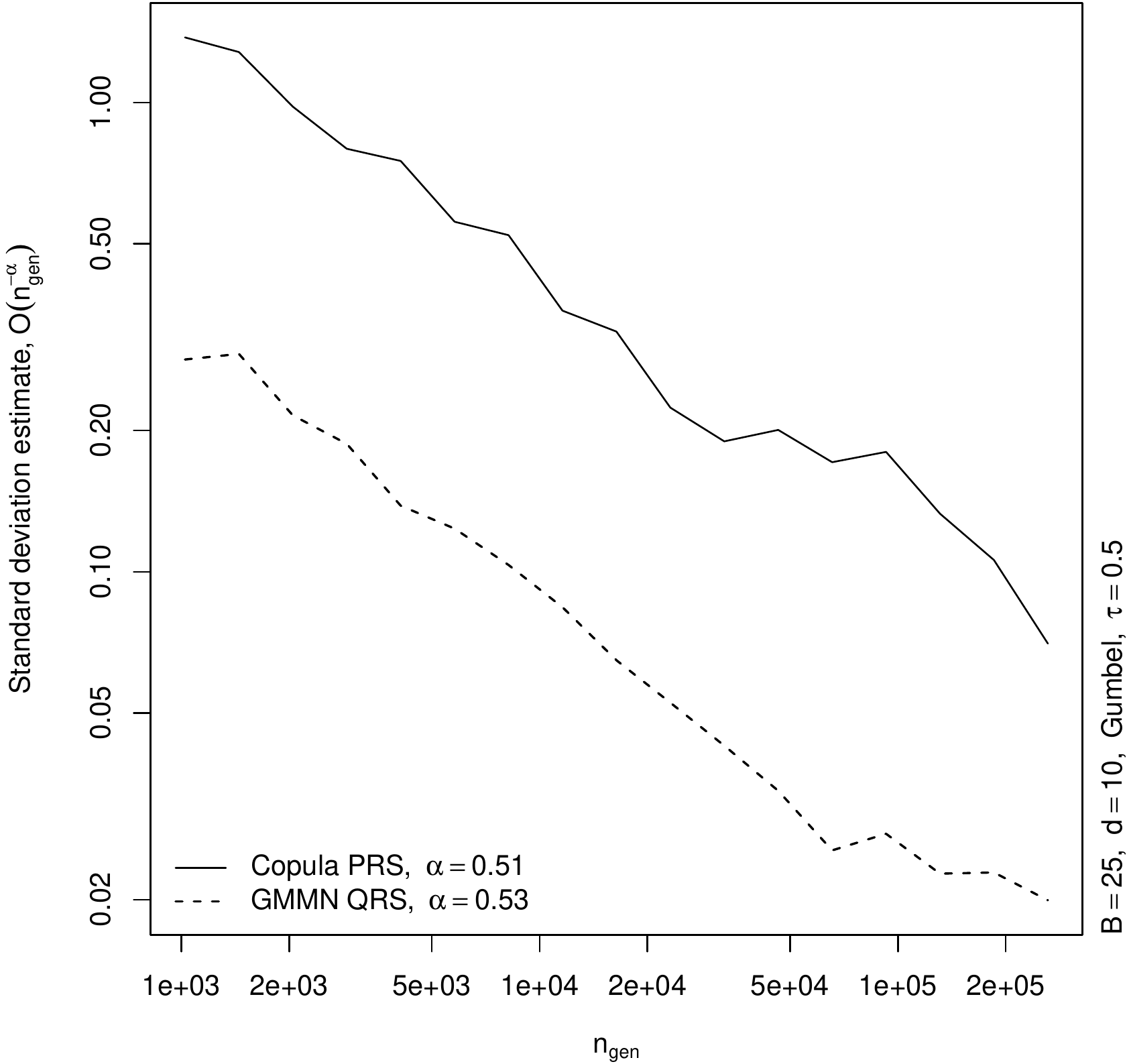}
  \includegraphics[width=0.31\textwidth]{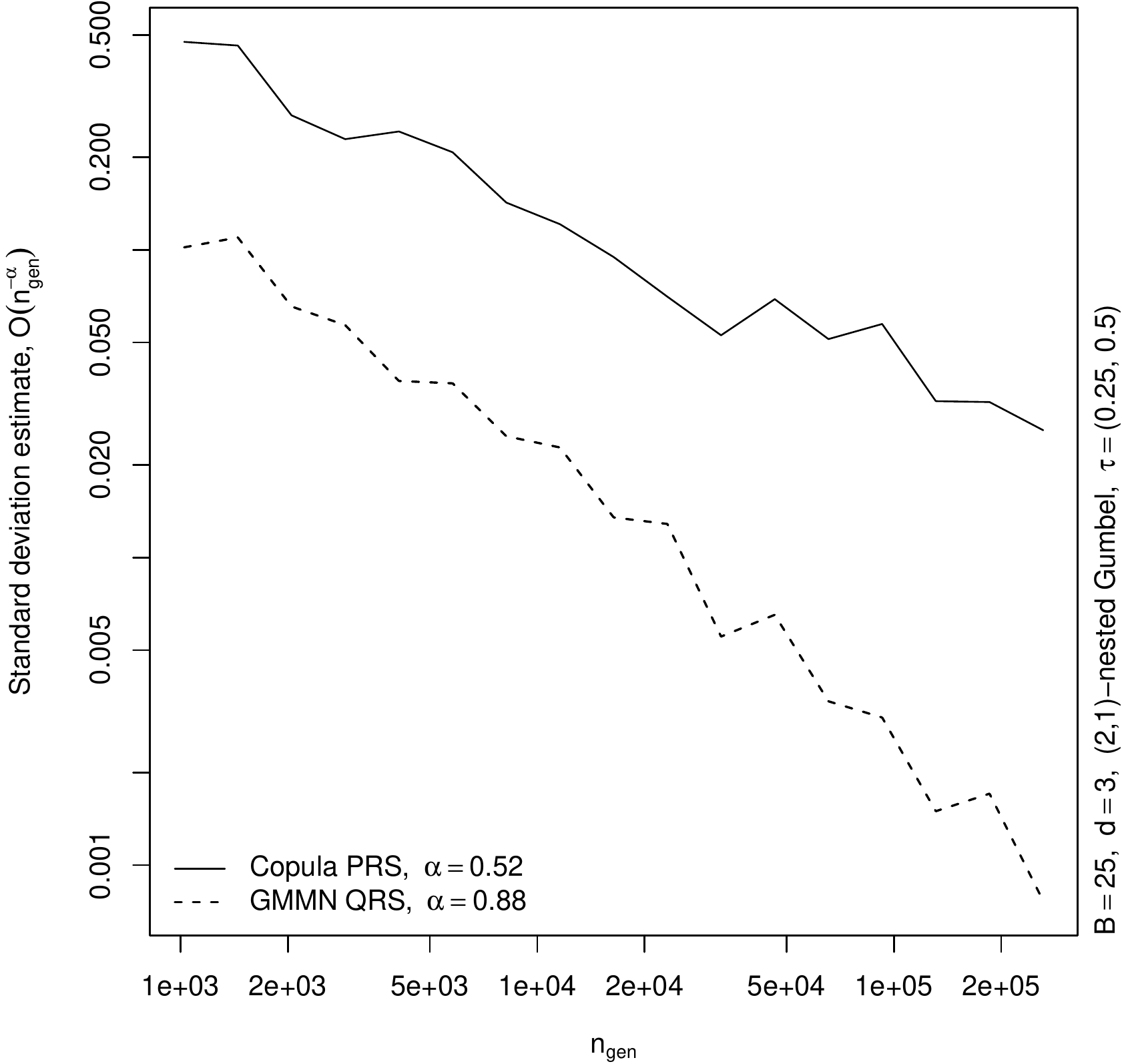}\hfill
  \includegraphics[width=0.31\textwidth]{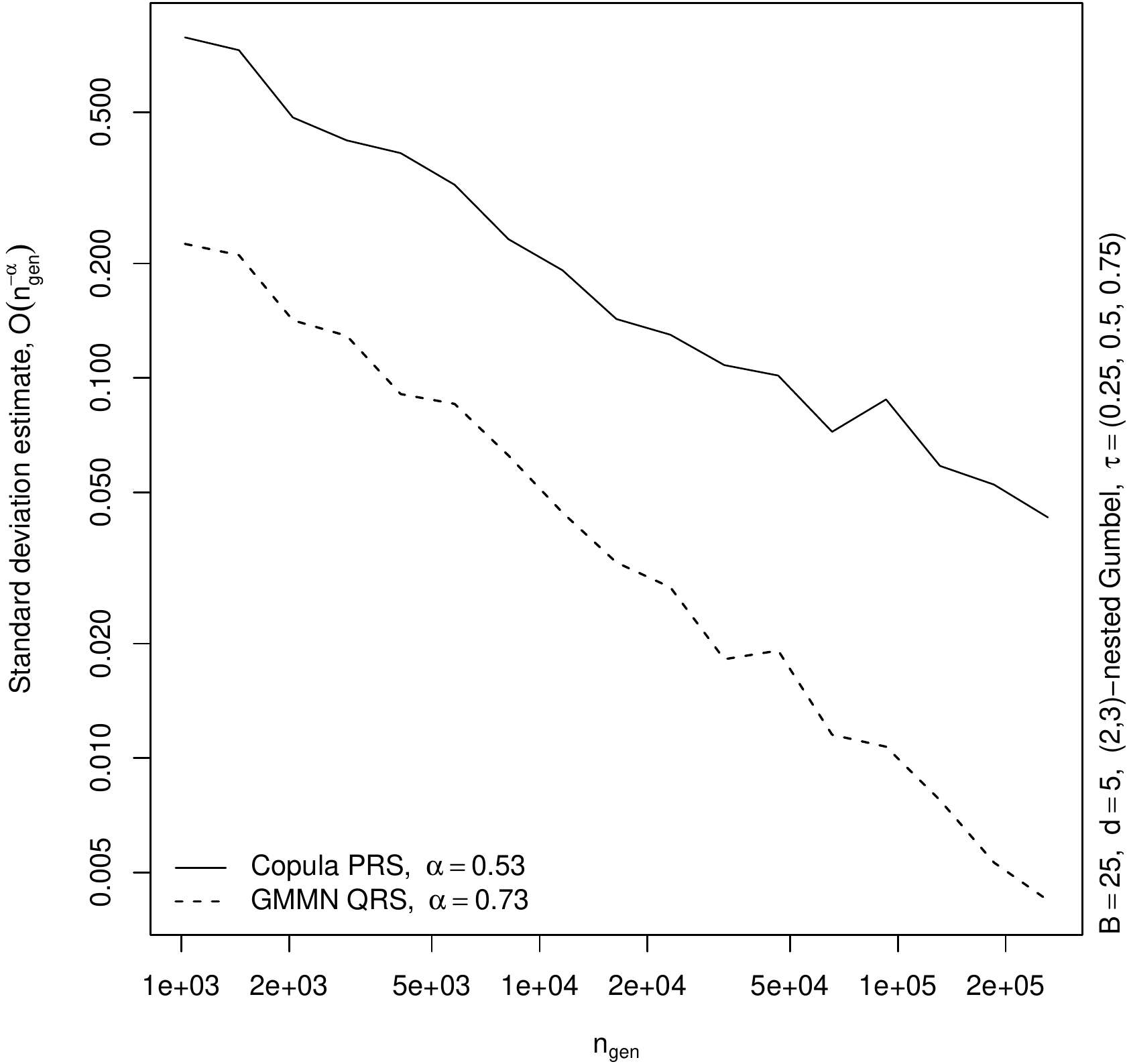}\hfill
  \includegraphics[width=0.31\textwidth]{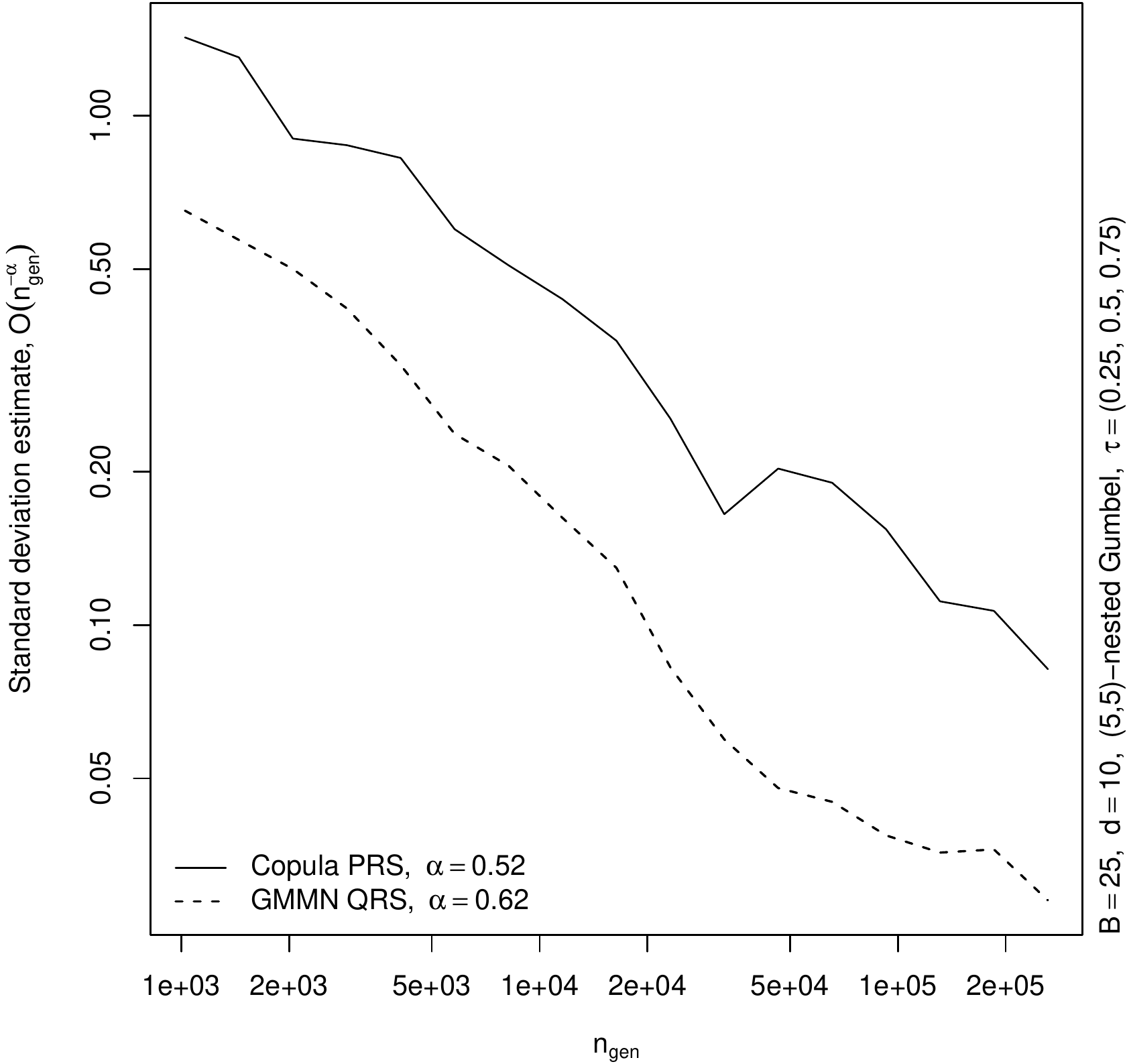}
  \caption{Standard deviation estimates based on $B=25$ replications for
    estimating $\E(\Psi_2(\bm{X}))$ via MC based on a pseudo-random sample (PRS), via the copula RQMC
    estimator (whenever available; rows 1--2 only) and via the
    GMMN RQMC estimator (based on digitally shifted nets). Note that in
    rows 1--3, $d\in\{2,5,10\}$, whereas in row 4,
    $d\in\{3,5,10\}$.}\label{fig:aggregateES}
\end{figure}

\section{Run time}\label{sec:timings}
Run time depends on factors such as the hardware used, the current workload, the
algorithm implemented, the programming language used, the implementation style,
compiler flags, whether garbage collection was used, etc. There is not even a
unique concept of time (system vs user vs elapsed time). Although none of our
code was optimized for run time, we still report on various timings here,
measured in elapsed time also known as wall-clock time.

\subsection{Training and sampling}
The results in this section are reproducible with the demo
\texttt{GMMN\_QMC\_timings} of the \R\ package \texttt{gnn}.

Table~\ref{tab:train:time} shows elapsed times in minutes for training a GMMN
on training data from $t_4$, Clayton (C), Gumbel (G) and nested Gumbel (NG) copulas in
dimensions 2, 3, 5 and 10 as described in Sections~\ref{sec:GMMN:visual} and
\ref{sec:GMMN:accuracy}. As is reasonable, the measured times are only affected
by the dimension, not by the type of dependence model.

\begin{table}[htbp]
  \centering
  \begin{tabular}{*{1}{c}*{3}{S[table-format=1.2]}}
    \toprule
    \multicolumn{1}{c}{$C$} & \multicolumn{1}{c}{$d = 2, 3$} & \multicolumn{1}{c}{$d = 5$} & \multicolumn{1}{c}{$d = 10$} \\
    \midrule
    $t_4$ & 5.52 & 7.01 & 9.46 \\
    C & 5.52 & 7.00 & 9.45 \\
    G & 5.52 & 7.01 & 9.46 \\
    NG & 6.01 & 7.01 & 9.44 \\
    \bottomrule
  \end{tabular}
  \caption{Elapsed times in minutes for training GMMNs of the same architecture
    as used in Sections~\ref{sec:GMMN:visual} and \ref{sec:GMMN:accuracy} with
    $\nepo=300$, $\ntrn=60\,000$ and $\nbat=5000$ on respective copula samples;
    training was done on one NVIDIA Tesla P100 GPU.}\label{tab:train:time}
\end{table}

Table~\ref{tab:run:time} contains elapsed times for generating $\ngen=10^5$
observations from the respective dependence model and sampling method on
two different machines, once on the NVIDIA Tesla P100 GPU used for training
and once locally on a 2018 2.7 GHz Quad-Core Intel Core i7 processor. The
results for the copula-based pseudo-random sampling method are averaged over 100
repetitions. The results for the copula-based quasi-random sampling method are
obtained as follows. If the conditional copulas involved in applying the inverse
Rosenblatt transform of the respective copula model are not available
analytically nor numerically, NA is reported; this applies to the nested Gumbel
copula.  And if they are only available numerically (by root finding), then a
reduced sample size of 1000 is used and the reported run times were obtained by
scaling up to $\ngen$ by multiplication with $100$; this applies to the Gumbel
copula. We also measured run times for $\ngen=10^6$ and $\ngen=10^7$ and they
scale proportionally as one would expect.

\begin{table}[htbp]
  \centering
  \begin{tabular}{c c
    S[table-format=1.4] S[table-format=4.4] *{2}{S[table-format=1.4]}
    S[table-format=1.4] S[table-format=4.4] *{2}{S[table-format=1.4]}}
    \toprule
    & & \multicolumn{4}{c}{2018 2.7 GHz Quad-Core Intel Core i7} & \multicolumn{4}{c}{NVIDIA Tesla P100 GPU}\\
    \cmidrule(lr{0.4em}){3-6}\cmidrule(lr{0.4em}){7-10}
    & & \multicolumn{2}{c}{Copula} & \multicolumn{2}{c}{GMMN} & \multicolumn{2}{c}{Copula} & \multicolumn{2}{c}{GMMN}\\
    \cmidrule(lr{0.4em}){3-4}\cmidrule(lr{0.4em}){5-6}\cmidrule(lr{0.4em}){7-8}\cmidrule(lr{0.4em}){9-10}
    $d$ & $C$ & \multicolumn{1}{c}{PRS} & \multicolumn{1}{c}{QRS} & \multicolumn{1}{c}{PRS} & \multicolumn{1}{c}{QRS} & \multicolumn{1}{c}{PRS} & \multicolumn{1}{c}{QRS} & \multicolumn{1}{c}{PRS} & \multicolumn{1}{c}{QRS}\\
    \midrule
    2     & $t_4$ & 0.0642 & 0.4420 & 1.2960 & 1.2720 & 0.1045 & 0.8210 & 3.6140 & 3.5820 \\
    2     & C     & 0.0144 & 0.0230 & 1.3110 & 1.3140 & 0.0308 & 0.0400 & 3.6290 & 3.5820 \\
    2     & G     & 0.0348 & 374.8000 & 1.3470 & 1.3310 & 0.0669 & 687.7000 & 3.6400 & 3.5750 \\
    (2,1) & NG    & 0.0633 & NA & 1.3360 & 1.3260 & 0.1369 & NA & 3.6560 & 3.6530 \\
    5     & $t_4$ & 0.1410 & 1.4830 & 1.4490 & 1.3830 & 0.2567 & 3.0150 & 3.7330 & 3.6960 \\
    5     & C     & 0.0425 & 0.0580 & 1.3890 & 1.5060 & 0.0936 & 0.1110 & 3.7670 & 3.6980 \\
    5     & G     & 0.0523 & 1529.5000 & 1.3930 & 1.3860 & 0.1161 & 2939.9000 & 3.7380 & 3.7010 \\
    (2,3) & NG    & 0.0989 & NA & 1.4020 & 1.4070 & 0.2167 & NA & 3.9450 & 3.7210 \\
    10    & $t_4$ & 0.2766 & 3.6080 & 1.5870 & 1.6720 & 0.4917 & 6.5290 & 3.9530 & 4.0990 \\
    10    & C     & 0.0734 & 0.1190 & 1.6430 & 1.6630 & 0.1806 & 0.2320 & 3.9680 & 4.1910 \\
    10    & G     & 0.0807 & 3579.6000 & 1.5500 & 1.5290 & 0.1984 & 7119.6000 & 4.0060 & 3.9370 \\
    (5,5) & NG    & 0.1324 & NA & 1.5470 & 1.5280 & 0.3087 & NA & 3.9740 & 3.9530 \\
    \bottomrule
  \end{tabular}
  \caption{Elapsed times in seconds for generating samples of size $\ngen=10^5$.}\label{tab:run:time}
\end{table}

We see from Table~\ref{tab:run:time} that quasi-random sampling from specific
copulas is available and can be fast, e.g., for $t_4$ and Clayton
copulas. However, we already see that quasi-random sampling gets more
time-consuming for larger $d$. For other copulas, such as Gumbel copulas, it can
be much more time consuming. Furthermore, as seen from the nested case and as is
currently the case for most copula models, a quasi-random sampling procedure is
not even available. By contrast, on the same machine, GMMNs show very close run
times, are barely affected by the dimension and are not affected by the
type of dependence model. For $d=10$, the GMMN quasi-random sampling
procedure even outperforms the $t_4$ quasi-random sampling procedure for which
the conditional copulas are analytically available; for $d=5$ the two procedures perform on par,
depending on the machine used.

This highlights the universality of using neural networks for dependence
modeling purposes. As an example, say a risk management application such as
estimating expected shortfall with variance reduction is based on a
$t_4$ copula and a regulator requires us to change the model to a Gumbel copula for
stress testing purposes. Suddenly run time increases substantially. Also, if the
regulator decides to incorporate hierarchies (as was more easily done for the
$t_4$ model due to its correlation matrix) by utilizing a nested Gumbel copula,
then there is suddenly no quasi-random sampling procedure known anymore. It is
one of the biggest drawbacks of parametric copula models in applications that the level of
difficulty of carrying out important statistical tasks such as sampling, fitting
and goodness-of-fit can largely depend on the class of copulas used. These
problems are eliminated with neural networks as dependence models.

\subsection{Fitting and training times for the real-data applications}
We now briefly present the run times for fitting the parametric copula models
and training the GMMNs used in Section~\ref{sec:realdata}. Recall that we
considered three dimensions $d\in\{3,5,10\}$ and that fitting, respectively
training, was only required once for each dimension, independently of the number
of applications considered.

Table~\ref{tab:run:time:fit:appl} contains the elapsed times in seconds. Recall
from Section~\ref{sec:realdata} that GMMNs provided the best fit, followed by
the unstructured $t$ copula.  The latter is in general a popular parametric
copula model in practice; see \cite{fischerkoeckschlueterweigert2009}. Comparing the last two columns of
Table~\ref{tab:run:time:fit:appl}, we see that fitting the $t$ copula is
comparably fast for $d=3$, %
however, already for $d=5$, run
time for training a GMMN is on par. For $d=10$,
training a GMMN is significantly faster.
\begin{table}[htbp]
  \centering
  \begin{tabular}{c *{3}{S[table-format=1.3]} *{2}{S[table-format=2.3]} S[table-format=3.3] S[table-format=2.3]}
    \toprule
    $d$ & \multicolumn{1}{c}{Gumbel} & \multicolumn{1}{c}{Clayton} & \multicolumn{1}{c}{Normal (ex)} & \multicolumn{1}{c}{Normal (un)} & \multicolumn{1}{c}{$t$ (ex)} & \multicolumn{1}{c}{$t$ (un)} & \multicolumn{1}{c}{GMMN} \\
    \midrule
    3  & 1.078 & 0.437 & 0.388 &  1.000 &  3.315 &   9.064 & 43.336\\
    5  & 1.291 & 0.455 & 0.435 &  6.071 &  5.932 &  41.131 & 41.235\\
    10 & 1.344 & 0.531 & 0.981 & 82.982 & 11.966 & 783.406 & 55.555\\
    \bottomrule
  \end{tabular}
  \caption{Elapsed times in seconds for fitting the respective parametric copula model and training the GMMN
    on one NVIDIA Tesla P100 GPU for the applications presented in Section~\ref{sec:realdata}.}\label{tab:run:time:fit:appl}
\end{table}

\subsection{TensorFlow vs R}
Finally, let us stress again what we initially said, namely, that run time
depends on many factors. In particular, one typically relies on TensorFlow
for the feed-forward step of input through the GMMN, %
which creates overhead especially for smaller data size $n$.  In
the demo \texttt{GMMN\_QMC\_timings}, we also provide a pure \R\ implementation
for this step for GMMNs considered in this work.

For each of $d\in\{2,5,10\}$, we randomly initialize $B=10$ GMMNs as in
Algorithm~\ref{algorithm:GMMN:train} and average the elapsed times of their
feed-forward steps when passing through data of size $n$ (chosen equidistant
in log-scale from $10$ to $10^6$) from the input distribution, once with
TensorFlow, and once with our own \R\ implementation. We then divide the
averaged run times of the \R\ implementation by the ones of the TensorFlow
implementation. Whenever the ratio is smaller (larger) than one, the \R\
implementation is faster (slower) than the TensorFlow implementation. We ran
this experiment once locally on the 2018 2.7 GHz Quad-Core Intel Core i7
processor and once on the NVIDIA Tesla P100 GPU. The results are depicted on
the left and on the right plot in Figure~\ref{fig:R:vs:TF},
respectively. Depending on the machine used, the \R\
implementation can be significantly faster, especially for small
$n$. %
\begin{figure}[htbp]
  \centering
  \includegraphics[width=0.48\textwidth]{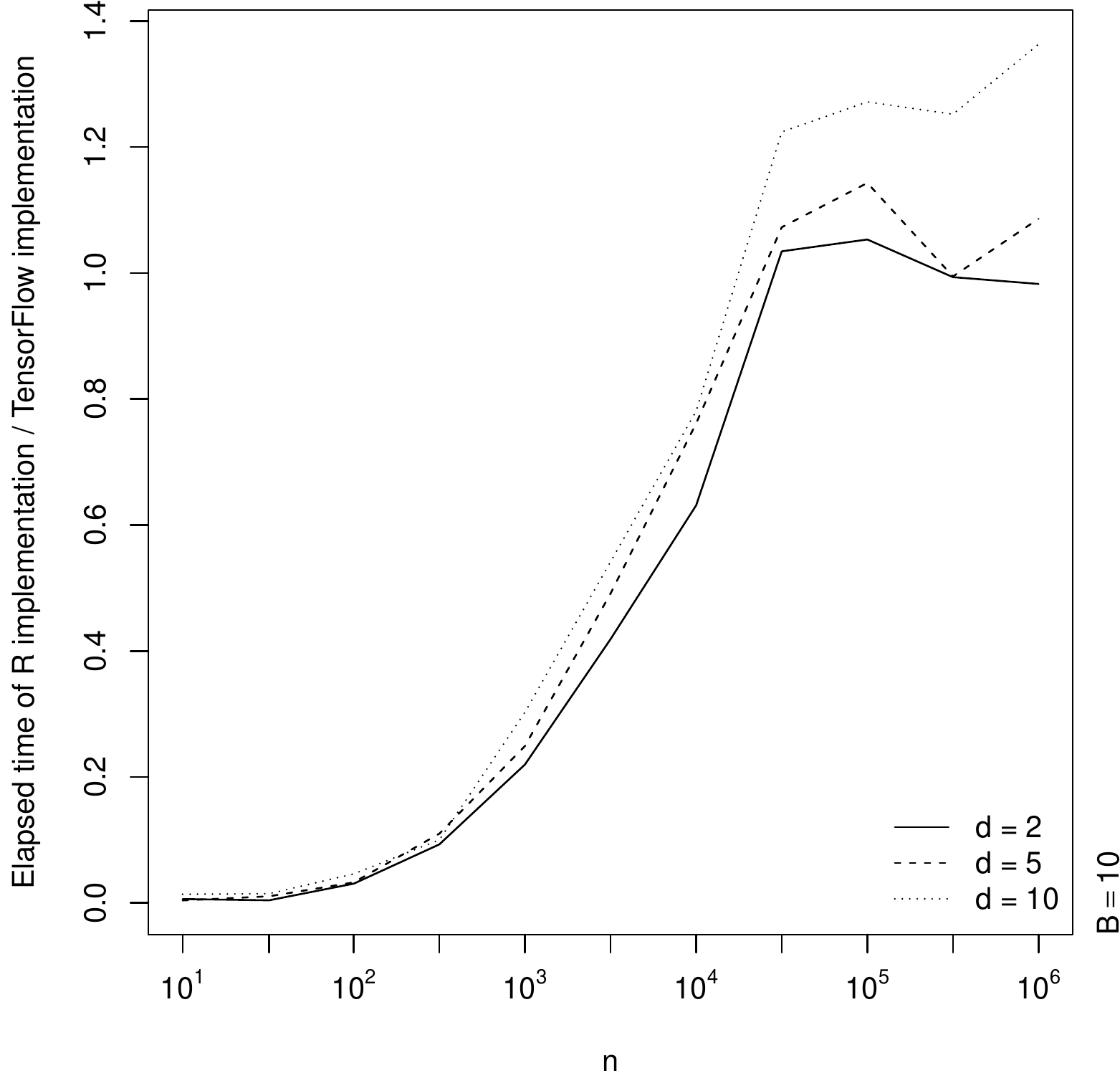}%
  \hfill%
  \includegraphics[width=0.48\textwidth]{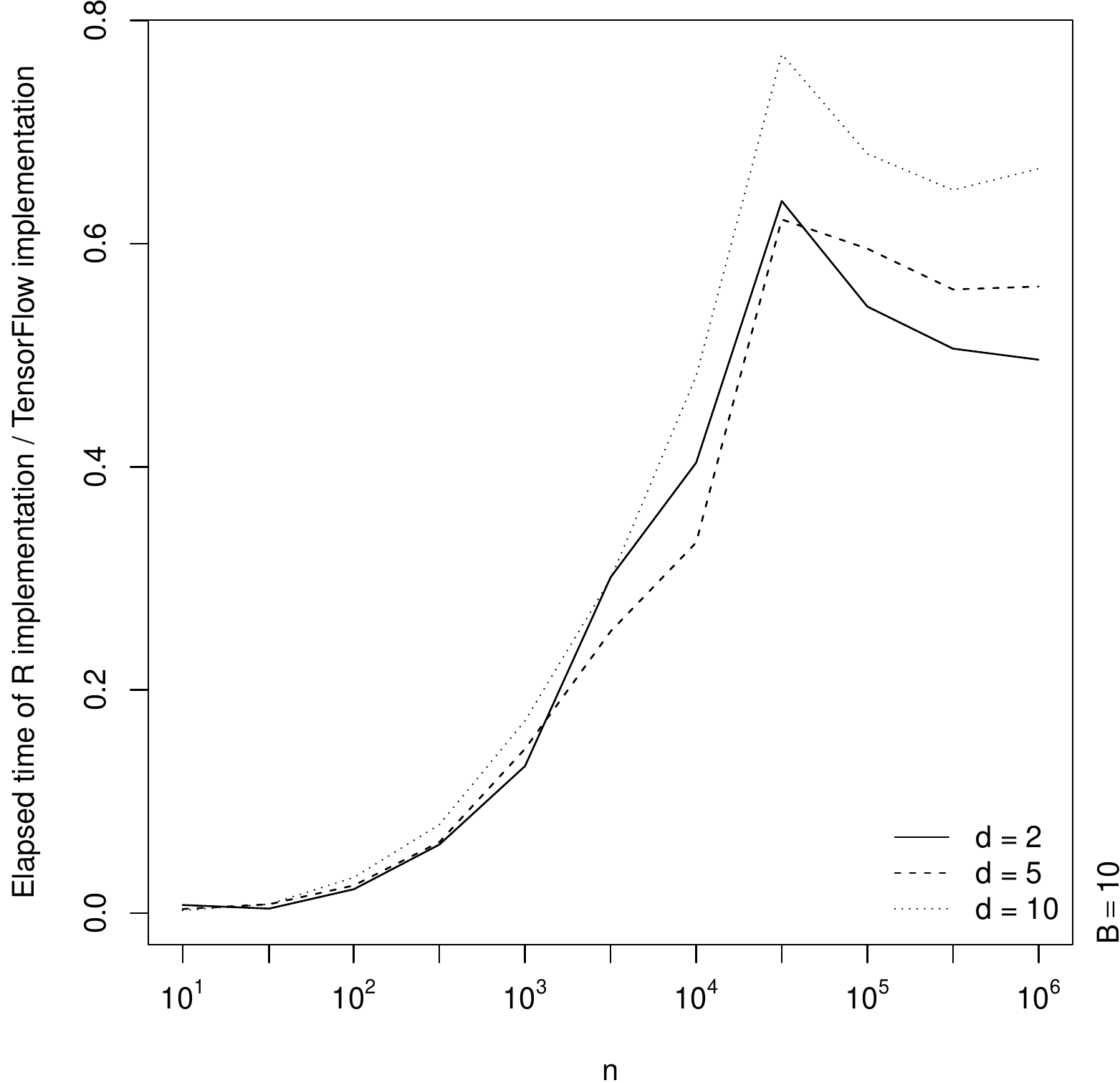}%
  \caption{Ratio of averaged elapsed times of an \R\ implementation over the TensorFlow
    implementation when evaluating randomly initialized GMMNs, once
    run on a 2018 2.7 GHz Quad-Core Intel Core i7 processor (left) and once
    on an NVIDIA Tesla P100 GPU (right).}\label{fig:R:vs:TF}
\end{figure}

\Urlmuskip=0mu plus 1mu\relax%
\printbibliography[heading=bibintoc]
\end{document}
